%% file: main.tex
\documentclass[10pt]{article} 
\usepackage[svgnames]{xcolor}
\usepackage[accepted]{tmlr}

\usepackage{hyperref}
\hypersetup{colorlinks=true, linkcolor=blue, citecolor=blue}

\usepackage{cancel}
\usepackage{amsthm}
\usepackage{amsmath, amsfonts, amssymb}
\usepackage[normalem]{ulem}
\usepackage{enumitem}
\usepackage{multirow}
\usepackage[capitalize]{cleveref}
\usepackage{booktabs}
\usepackage{multirow}
\usepackage{thmtools}
\usepackage{mathtools}
\usepackage{thm-restate}
\usepackage{graphicx}
\usepackage{subcaption}
\usepackage{minitoc}
\usepackage{wrapfig}

\title{Your Policy Regularizer is Secretly an Adversary}

\author{\name Rob Brekelmans \email brekelma@usc.edu \\
      \addr University of Southern California \\
      Information Sciences Institute \\
      \AND
      \name Tim Genewein \email timgen@deepmind.com \\
      \name Jordi Grau-Moya \email  \\
      \name Grégoire Delétang \email  \\
\name Markus Kunesch \email  \\
\name Shane Legg \email \\
\name Pedro Ortega \email 
pedro.ortega@gmail.com \\
      \addr DeepMind}

\def\openreview{\url{https://openreview.net/forum?id=XXXX}} 

\input{macros}

\begin{document}
\maketitle
\doparttoc 
\faketableofcontents




\begin{abstract}

Policy regularization methods such as maximum entropy regularization are widely used in reinforcement learning to improve the robustness of a learned policy. In this paper, \revhl{we 
unify and extend recent work showing that this robustness arises from hedging against worst-case perturbations of the reward function}, which are chosen from a limited set by an implicit adversary.
Using convex duality, we characterize the robust set of adversarial reward perturbations under \textsc{kl}- and $\alpha$-divergence regularization, which includes Shannon and Tsallis entropy regularization as special cases. Importantly,
generalization guarantees can be given within this robust set.   We provide detailed discussion of the worst-case reward perturbations, and present intuitive empirical examples to illustrate this robustness and its relationship with generalization. Finally, we discuss how our analysis complements 
previous results on adversarial reward robustness and path consistency optimality conditions.
\end{abstract}

\input{sections/new/0_intro}
\input{sections/new/1b_new_prelim}
\input{sections/new/2_adv_interp}
\input{sections/new/4_visualizations}
\input{sections/new/4_entropy_related}
\input{sections/new/5_conclusion}
\bibliographystyle{tmlr}
\bibliography{main}
\vspace*{-2cm}
\appendix
\onecolumn

\addcontentsline{toc}{section}{Appendix} 
\part{Appendix} 
\parttoc
\input{appendix/new_new_app/000_conj_optimality}
\input{appendix/new_new_app/01_conjugates}
\input{appendix/new_new_app/09_soft_value}

\input{appendix/new_new_app/05_feasible}
\input{appendix/new_new_app/07_value_form}

\input{appendix/new_new_app/08_entropy_reg}

\input{appendix/new_new_app/10_example}

\input{appendix/new_new_app/11_addl_results}





\end{document}

%% file: macros.tex
\newcommand{\omitindices}{}

\newcommand{\revhl}[1]{\textcolor{Black}{#1}}

\usepackage{pifont}
\newcommand{\header}{-.22cm}
\newcommand{\headerv}{\vspace*{-.175cm}}

\newcommand{\mudomain}{\mathcal{M}}

\newcommand{\taunotation}{\tau(\pi)}
\newcommand{\rprimedomain}{\mathbb{R}^{\mathcal{A} \times \mathcal{S} } }
\newcommand{\rprimedomainconj}{\mathbb{R}^{\mathcal{X}} }
\newcommand{\mudomainfn}{\mathbb{R}_+^{\mathcal{A} \times \mathcal{S}}}
\newcommand{\mudomainconj}{\mathbb{R}_+^{\mathcal{X}}}

\newcommand{\tick}{^{\prime}}
\newcommand{\obj}{\mathcal{RL}}
\newcommand{\primalvar}{\mu}
\newcommand{\dualvar}{\Delta r}
\newcommand{\pir}{\pi_{\pr}}
\newcommand{\mur}{\mu_{\pr}}
\newcommand{\modifyr}{r\tick}
\newcommand{\perturbr}{\Delta r}

\newcommand{\mr}{\modifyr}
\newcommand{\pr}{\perturbr}
\newcommand{\propt}{\prpi}
\newcommand{\mropt}{\mr_{\pi}}

\newcommand{\prnew}{\Delta \tilde{r}}
\newcommand{\lambdaplus}{\lambda(a,s)}
\newcommand{\lambdaopt}{\lambda_*(a,s)}

\newcommand{\mroptpi}{\mr_{\pi_*}}
\newcommand{\proptpi}{\pr_{\pi_*}}
\newcommand{\mrpiopt}{\mr_{\pi_*}}
\newcommand{\proptv}{\pr_{V_*}}
\newcommand{\prpi}{\pr_{\pi}}
\newcommand{\prpiopt}{\pr_{\pi_*}}
\newcommand{\prv}{\pr_{V}}
\newcommand{\prmu}{\pr_{\mu}}
\newcommand{\prrobust}{\pr} 
\newcommand{\prinit}{\pr} 
\newcommand{\prrob}{\prrobust} 

\newcommand{\mrrobust}{r_*\tick}
\newcommand{\piopt}{\pi_*}

\newcommand{\bullets}[1]{{#1}}

\newcommand{\highlight}{blue}

\newcommand{\simplex}{\Delta^{|\mathcal{A}|}}
\newcommand{\feasibleset}{\mathcal{R}_{\pi}}
\newcommand{\discussiondual}{Q}
\newcommand{\myconst}{\psi_{\pr}(s;\beta)}
\newcommand{\psipi}{\myconst}
\newcommand{\psipr}{\psi_{\pr}(s;\beta)}
\newcommand{\psipiopt}{\psi_{\pr_{\pi_*}}(s;\beta)}
\newcommand{\psipropt}{\psi_{\pr_{\pi_*}}(s;\beta)}
\newcommand{\psiq}{\psi_{Q}(s;\beta)}
\newcommand{\psipiq}{\psiq}
\newcommand{\psiqopt}{\psi_{{Q_*}}(s;\beta)}

\newcommand{\alphaconjn}{\frac{1}{\beta}\Omega^{*(\alpha)}_{\pi_0, \beta}}
\newcommand{\alphaconjmu}{\frac{1}{\beta}\Omega^{*(\alpha)}_{\pi_0, \beta}}

\newcommand{\alphaconj}{\frac{1}{\beta}\Omega^{*(\alpha)}_{\pi_0, \beta}}
\newcommand{\alphaconjm}{\frac{1}{\beta}\Omega^{*(\alpha)}_{\mu_0, \beta}}

\newcommand{\omegamu}{\frac{1}{\beta}\Omega^{(\alpha)}_{\pi_0}(\mu)}

\newcommand{\klconjpii}{\frac{1}{\beta}\Omega^{*}_{\pi_0, \beta}}

\newcommand{\alphanmu}{\omegamu}

\newcommand{\cmark}{\ding{51}}%
\newcommand{\xmark}{\ding{55}}%
\newcommand{\tcminus}[1]{\textcolor{FireBrick}{#1}}

\newcommand{\tcplus}[1]{\textcolor{DarkGreen}{#1}} 

\newcommand{\ptReg}{\dot{\Omega}}

\newcommand{\mysec}[1]{\hyperref[sec:#1]{Sec.~\ref*{sec:#1}}}
\newcommand{\myfig}[1]{\hyperref[fig:#1]{Fig.~\ref*{fig:#1}}}
\newcommand{\myapp}[1]{\hyperref[app:#1]{App.~\ref*{app:#1}}}
\newcommand{\myprop}[1]{\hyperref[prop:#1]{Prop.~\ref*{prop:#1}}}
\newcommand{\mythm}[1]{\hyperref[thm:#1]{Thm.~\ref*{thm:#1}}}
\newcommand{\mylem}[1]{\hyperref[lemma:#1]{Lemma~\ref*{lemma:#1}}}
\newcommand{\mylemma}[1]{\hyperref[lemma:#1]{Lemma~\ref*{lemma:#1}}}
\newcommand{\mycor}[1]{\hyperref[cor:#1]{Cor.~\ref*{cor:#1}}}
\newcommand{\mytab}[1]{\hyperref[table:#1]{Table~\ref*{table:#1}}}
\newcommand{\expof}[1]{ \exp \big\{ {#1} \big\} }

\newcommand{\brangle}{\big \rangle}
\newcommand{\blangle}{\big \langle}
\DeclareMathOperator*{\argmax}{arg\,max}

\newcommand{\transitionv}{{\mathbb{E}^{s\tick}_{a,s} \big[ V(s\tick) \big] }}
\newcommand{\transitionvstar}{{\mathbb{E}^{s\tick}_{a,s} \big[ V_*(s\tick) \big] }}
\newcommand{\transitionvinds}{{\mathbb{E}^{s\tick}_{a,s} \big[ V \big] }}
\newcommand{\transitionvstarinds}{{\mathbb{E}^{s\tick}_{a,s} \big[ V_* \big] }}

\usepackage[acronym,smallcaps,nowarn,section,nogroupskip,nonumberlist]{glossaries}
\glsdisablehyper
\newacronym{VI}{vi}{variational inference}
\newacronym{GVI}{gvi}{generalized variational inference}
\newacronym{MDP}{mdp}{Markov Decision Processes}
\newacronym{KL}{kl}{Kullback-Leibler}
\newcommand{\kl}{\textsc{kl} }
\newcommand{\KL}{D_{\textsc{KL}}}
\newacronym{LP}{lp}{linear programming}
\newacronym{RL}{rl}{reinforcement learning}
\newacronym{REPS}{reps}{Relative Entropy Policy Search}

\newcommand{\rulesep}{\unskip\ \vrule\ }

\newcommand{\nocontentsline}[3]{}
\newcommand{\tocless}[2]{\bgroup\let\addcontentsline=\nocontentsline#1{#2}\egroup}

%% file: sections/new/0_intro.tex
\headerv
\section{Introduction}\label{sec:intro}
\headerv

Regularization plays a crucial role in various settings across \gls{RL}, 
such as trust-region methods \citep{peters2010relative, schulman2015trust, schulman2017proximal, basserrano2021logistic}, offline learning \citep{levine2020offline, nachum2019dualdice, nachum2019algaedice, nachum2020reinforcement}, multi-task learning \citep{teh2017distral, igl2020multitask}, and soft 
$Q$-learning or actor-critic methods
\citep{fox2016taming, nachum2017bridging, haarnoja17a, haarnoja2018soft, grau2018soft}. 
Various justifications have been given for policy regularization, 
such as improved optimization \citep{ahmed2019understanding},  
connections with probabilistic inference \citep{levine2018reinforcement, 
kappen2012optimal, rawlik2013stochastic, wang2021variational}, 
 and robustness to perturbations in the environmental rewards or dynamics \citep{derman2021twice, eysenbach2021maximum, husain2021regularized}.



In this work, we use convex duality to 
analyze the reward robustness which naturally arises from policy regularization in \gls{RL}.
In particular, we interpret regularized reward maximization as a two-player game between the agent and an imagined adversary that modifies the reward function.  For a policy $\pi(a|s)$ regularized with a convex function $\Omega(\pi) = \mathbb{E}_{\pi}[\ptReg(\pi)]$ 
and regularization strength $1/\beta$,
we 
investigate statements of the form 
\begin{align}
\hspace*{-.3cm} \max \limits_{\pi(a|s)} \text{\small $(1-\gamma)$} 
\mathbb{E}_{\taunotation}
\left[ \sum \limits_{t=0}^\infty \gamma^t \bigg( r(a_t, s_t) -  {\frac{1}{\beta}} \ptReg \big(\pi(a_t | s_t)\big) \bigg) \right] = \max \limits_{\pi(a|s)} \min \limits_{r\tick(a,s) \in \feasibleset}  \text{\small $(1-\gamma)$} \mathbb{E}_{\taunotation} \left[ \sum \limits_{t=0}^\infty \gamma^t r\tick(a_t,s_t) \right],
\label{eq:reward_perturbation_stmt} 
\end{align}
\normalsize
where 
$r\tick(a,s)$ 
indicates a modified reward function chosen
from an appropriate robust set $\feasibleset$ (see \myfig{related}-\ref{fig:feasible_set_main}). 
\cref{eq:reward_perturbation_stmt} suggests that an agent may translate uncertainty in its estimate of the reward function into regularization of a learned policy, which is particularly relevant in applications such as inverse \gls{RL} \citep{ng2000algorithms, arora2021survey} or learning from human preferences \citep{christiano2017deep}.

\begin{figure}[t]
\centering
\vspace*{-1.1cm}
\includegraphics[width=.35\textwidth]{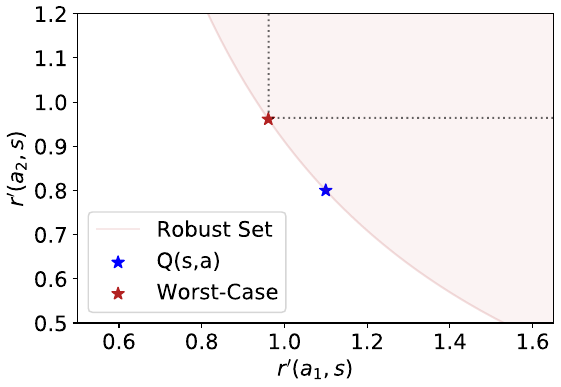}
\vspace*{-.05cm}
\caption{
\textbf{Robust set} $\feasibleset$ (red region) of perturbed reward functions to which a stochastic policy generalizes, in the sense of \cref{eq:generalization}.
Red star indicates the worst-case perturbed reward $\mroptpi = r - \proptpi$ (\myprop{optimal_perturbations}) chosen by the adversary.
The robust set also characterizes the set of reward perturbations $\pr(a,s)$ that are feasible for the adversary, which differs based on the choice of regularization function, regularization strength $\beta$, and reference distribution $\pi_0$ (see \mysec{visualizations_feasible} and \myfig{feasible_set_main}).  We show the robust set for the optimal single-step policy with value estimates $Q(a,s) = r(a,s)$ and \textsc{kl} divergence regularization to a uniform $\pi_0$, with $\beta = 1$.
Our robust set is larger and has a 
qualitatively 
different shape 
compared to the 
robust set of \citet{derman2021twice} (dotted lines, see \mysec{related}). 
}\label{fig:related}
\vspace*{-.3cm}
\end{figure}

This reward robustness further implies that
regularized policies achieve a form of `zero-shot' generalization to new environments where the reward is adversarially chosen.   
In particular, for any given $\pi(a|s)$ and a modified reward $\mr \in \feasibleset$ 
within the corresponding robust set,
we obtain the following performance guarantee 
\begin{equation}
 \hspace*{-.2cm} 
     \mathbb{E}
     _{\taunotation} 
     \left[ \sum \limits_{t=0}^\infty \gamma^t r\tick(a_t,s_t) \right] \geq 
     \mathbb{E}
    _{\taunotation} 
     \left[ \sum \limits_{t=0}^\infty \gamma^t \bigg( r(a_t, s_t) - {\frac{1}{\beta}} \ptReg 
     \big(\pi_t \big) \bigg) \right] .
    \label{eq:generalization}
\end{equation}
\normalsize
\cref{eq:generalization} states that the expected modified reward under $\pi(a|s)$, with $\mr \in \feasibleset$ as in \myfig{related}, will be greater than the value of the regularized objective with the original, unmodified reward.  
It is in this particular sense that we 
make claims about robustness and zero-shot generalization throughout the paper.   





Our analysis unifies 
recent work exploring similar interpretations \citep{ortega2014adversarial, husain2021regularized, eysenbach2021maximum, derman2021twice} 
as summarized in \mysec{discussion} and \mytab{novelty}.  
Our contributions include
\setlength\itemsep{0.025em}
\begin{itemize}
\item A thorough analysis of the robustness associated with \kl and $\alpha$-divergence policy regularization, which includes popular Shannon entropy regularization
as a special case.  Our derivations for the $\alpha$-divergence generalize the Tsallis entropy \gls{RL} framework of \citet{lee2019tsallis}.
\item We derive the worst-case reward perturbations $\prpi = r - \mr_{\pi}$ corresponding to any stochastic policy $\pi$ and a fixed regularization scheme (\myprop{optimal_perturbations}). 

\item For the optimal regularized policy in a given environment, we show that the corresponding worst-case reward perturbations 
match the advantage function for \textit{any} $\alpha$-divergence.   
We relate this finding to the path consistency optimality condition, which has been used to construct learning objectives in \citep{nachum2017bridging, chow2018path}, and a game-theoretic indifference condition, 
which occurs at a Nash equilibrium between the agent and adversary \citep{ortega2014adversarial}.   
\item  We visualize the set  $\feasibleset$ of adversarially perturbed rewards against which a regularized policy is robust in \myfig{related}-\ref{fig:feasible_set_main}, with details in \myprop{feasible}.
Our use of divergence instead of entropy regularization to analyze the robust set 
clarifies several unexpected conclusions from previous work.   In particular, 
similar plots in 
\citet{eysenbach2021maximum} 
suggest that
MaxEnt \textsc{rl} is not robust to the reward function of the training environment, and that increased regularization strength may hurt robustness.  Our analysis in \mysec{entropy_vs_div} and  \myapp{eysenbach} establishes the expected, \textcolor{black}{opposite} results.


\item We perform experiments for a sequential grid-world task in \mysec{visualizations} where, in contrast to previous work, we explicitly visualize the reward robustness and adversarial strategies resulting from our theory.
We use the path consistency or indifference conditions to certify optimality of the policy.
\end{itemize}

\begin{figure*}[t]
\vspace*{-1cm}
\centering
\resizebox{\textwidth}{!}{%
\begin{tabular}{lccccc}
    &  \text{ \citet{ortega2014adversarial}} & \text{\citet{eysenbach2021maximum}} &  \text{ \citet{husain2021regularized}}  & \citet{derman2021twice} & Ours \\ \toprule
    \bullets{Multi-Step Analysis} & \tcminus{\xmark} & \tcplus{\cmark} & \tcplus{\cmark} & \tcplus{\cmark} & \tcplus{\cmark} \\ 
    \bullets{Worst-Case $\pr(a,s)$} & \tcminus{policy form} & \tcminus{policy form} & \tcminus{value form} & \tcminus{policy} (via dual \textsc{lp} \cref{eq:dual_lp}) & \tcplus{policy $\&$ value forms}  \\ 
  \bullets{Robust Set} & \tcminus{\xmark} & 
    \tcplus{\cmark}  \tcminus{(see our \myapp{eysenbach})} & \tcminus{\xmark} & \tcplus{\cmark (flexible specification)} & \tcplus{\cmark}  \\ 
  \bullets{Divergence Used} & \tcminus{KL ($\alpha = 1$)} & \tcminus{\text{ Shannon entropy (\mysec{entropy_vs_div})}} & \tcplus{any convex $\Omega$} & \tcplus{derived from robust set}  & 
  \tcplus{any convex $\Omega$, $\alpha$-Div examples}  
  \\ 
    \bullets{$\mu(a,s)$ or $\pi(a|s)$ Reg.?} & 
     $\pi(a|s)$ & 
    $\pi(a|s)$ & \tcplus{Both}  & $\pi(a|s)$ & \tcplus{Both} \\
    \bullets{Indifference} & \tcplus{\cmark} & \tcminus{\xmark}  & \tcminus{\xmark}  & \tcminus{\xmark} & \tcplus{\cmark} \\ 
    \bullets{Path Consistency} & \tcminus{\xmark} & \tcminus{\xmark}  & \tcminus{\xmark}  &\tcminus{\xmark} & \tcplus{\cmark}   \\ 
    \bottomrule
\end{tabular}%
}
\vspace*{-.3cm}
\captionof{table}{Comparison to related work. }\label{table:novelty}
\headerv
\end{figure*}

%% file: sections/new/1b_new_prelim.tex
\headerv
\headerv
\section{Preliminaries}\label{sec:prelim}
\headerv

In this section, we review
\gls{LP} formulations of 
discounted \gls{MDP} and extensions to convex policy regularization.

\newcommand{\myset}{X}
\paragraph{Notation} \revhl{For a finite set $\mathcal{\myset}$, let $\mathbb{R}^{\mathcal{\myset}}$ denote the space of real-valued functions over $\mathcal{\myset}$, with $\mathbb{R}_+^{\mathcal{\myset}}$ indicating restriction to non-negative functions.  We let $\Delta^{|\mathcal{\myset}|}$ denote the probability simplex with dimension equal to the cardinality of $\mathcal{\myset}$.
For $ \mu, q \in \mathbb{R}^{\mathcal{\myset}}$, 
$\langle \mu , q \rangle = \sum_{x \in \mathcal{\myset}} \mu(x) q(x)$ indicates the inner product in Euclidean space. }


\headerv
\subsection{Convex Conjugate Function}\label{sec:conj_intro}
\headerv
We begin by reviewing the convex conjugate function, also known as the Legendre-Fenchel transform, which will play a crucial role throughout our paper.
For a convex function $\Omega(\primalvar)$ which, in our context, has domain $\primalvar \in \mudomainconj$,
the conjugate function $\Omega^{*}$ is defined via the optimization
\begin{align}
    \Omega^{*}( \dualvar ) = \sup \limits_{\primalvar \in 
    \mudomainconj
    } \, \big \langle \primalvar, \dualvar \big \rangle - \Omega(\primalvar), \label{eq:conjstar}
\end{align} 
where $\dualvar \in \rprimedomainconj$.
 The conjugate operation is an involution for proper, lower semi-continuous, convex $\Omega$ \citep{boyd2004convex},
 so that $(\Omega^{*})^{*} = \Omega$ and $\Omega^{*}$ is also convex.  
 We can thus represent $\Omega(\primalvar)$ via a conjugate optimization
 \begin{align}
     \Omega(\primalvar) = \sup_{\dualvar \in 
     \rprimedomainconj}
     \big \langle \primalvar, \dualvar \big \rangle - \Omega^{*}(\dualvar). \label{eq:conjomega}
 \end{align}
  Differentiating with respect to the optimization variable in \cref{eq:conjstar} or (\ref{eq:conjomega}) suggests the optimality conditions 
  \begin{align}
     \primalvar_{\dualvar} = \nabla \Omega^{*}( \dualvar )   \quad \qquad \dualvar_{\primalvar} = \nabla \Omega(\primalvar) \, .     \label{eq:optimal_policy_as_grad}
 \end{align}
 Note that the above conditions also imply relationships of the form $\primalvar_{\dualvar} = (\nabla \Omega)^{-1}(\dualvar)$.   This dual correspondence between 
 values of $\primalvar$ and $\dualvar$ will form the basis of our adversarial interpretation in \mysec{adversarial_perturb}.
 

\headerv
\subsection{Divergence Functions}\label{sec:divergence_intro}
\headerv
We are interested in the conjugate duality associated with 
policy regularization, which is often expressed using a
statistical divergence $\Omega(\mu)$ over a joint density $\mu(a,s) = \mu(s) \pi(a|s)$ (see \mysec{mdp_background}).  
In particular, we consider
the family of $\alpha$-divergences \citep{amari2016information, cichocki2010families}, which includes both the forward and reverse \textsc{kl} divergences as special cases.    
In the following, we consider extended divergences that accept unnormalized density functions as input \citep{zhu1995information} 
so that we may
 analyze function space dualities and evaluate Lagrangian relaxations without
projection
 onto the probability simplex.





\paragraph{KL Divergence}
The `forward' \textsc{kl} divergence to a reference policy $\pi_0(a|s)$ is commonly used for policy regularization in \gls{RL}.  
Extending the input domain to unnormalized measures, we write the divergence as
\begin{align}
        &\Omega_{\pi_0}(\mu) = \mathbb{E}_{\mu(s)}\Big[D_{KL}[\pi:\pi_0]\Big] = \sum \limits_{s \in \mathcal{S}} \mu(s) \sum \limits_{a \in \mathcal{A}}  \left(\pi(a|s)  \log \frac{\pi(a|s)}{\pi_0(a|s)} -  
      \pi(a|s) 
        + 
      \pi_0(a|s)\right). \label{eq:kl_def} 
\end{align}
\normalsize
Using a uniform reference  $\pi_0(a|s) = 1$ $\forall 
\, (a,s)$, we recover the Shannon entropy up to an additive constant.

\paragraph{$\alpha$-Divergence} The $\alpha$-divergence $\mathbb{E}_{\mu(s)}\big[D_{\alpha}[\pi_0:\pi]\big]$ over possibly unnormalized measures
is defined as 
\begin{align}
\hspace*{-.1cm} 
\Omega^{(\alpha)}_{\pi_0}(\mu)& =
\frac{1}{\alpha(1-\alpha)} 
\sum \limits_{s \in \mathcal{S}}
\mu(s)
    \bigg( (1-\alpha) 
    \sum \limits_{a \in \mathcal{A}}
    \pi_0(a|s) 
    + \alpha 
    \sum \limits_{a \in \mathcal{A}}
    \pi(a|s) 
    - \sum \limits_{a \in \mathcal{A}} 
    \pi_0(a|s)^{1-\alpha} \pi(a|s)^{\alpha} \bigg)  
    \label{eq:alpha_def} 
\end{align}
\normalsize
Taking the limiting behavior, we recover the `forward' \textsc{kl} divergence $D_{KL}[\pi:\pi_0]$ as $\alpha \rightarrow 1$ or the `reverse' \textsc{kl} divergence $D_{KL}[\pi_0:\pi]$ as $\alpha \rightarrow 0$.

To provide intuition for 
the $\alpha$-divergence, we define the deformed $\alpha$-logarithm as in \citet{lee2019tsallis}, which matches Tsallis's $q$-logarithm \citep{tsallis2009introduction} for $\alpha = 2-q$.  Its inverse is the $\alpha$-exponential, with
\begin{align}
    \log_{\alpha}(u) &=  \frac{1}{\alpha-1} \left( u^{\alpha-1}-1 \right) \, , 
    \qquad \qquad 
    \exp_{\alpha}(u) = [1 + (\alpha-1) \, u ]_+^{\frac{1}{\alpha-1}} \, . \label{eq:qexp}
\end{align}
\textcolor{black}{where $[\cdot]_+ = \max(\cdot,0)$ ensures fractional powers can be taken and suggests that $\exp_{\alpha}(u)=0$ for $u \leq 1/(1-\alpha)$.}
Using the $\alpha$-logarithm, we can rewrite the $\alpha$-divergence 
similarly to
the \kl divergence in 
\cref{eq:kl_def}
\resizebox{\columnwidth}{!}{\parbox{\columnwidth}{
\begin{align}
  \Omega^{(\alpha)}_{\pi_0}(\mu) &= \frac{1}{\alpha} \sum \limits_{s \in \mathcal{S}} \mu(s) \bigg( \sum \limits_{a \in \mathcal{A}} \pi(a|s) \log_{\alpha}\frac{\pi(a|s)}{\pi_0(a|s)} - \pi(a|s) + \pi_0(a|s) \bigg) \nonumber.
\end{align}
}}
\normalsize
For a uniform reference $\pi_0$, the $\alpha$-divergence 
differs from the Tsallis entropy by only the $1/\alpha$ factor and an additive constant (see \myapp{tsallis}).

\input{sections/conj_table}

\headerv
\subsection{Unregularized MDPs}\label{sec:mdp_background}
\headerv

A discounted \gls{MDP} is a tuple $\{ \mathcal{S}, \mathcal{A}, P, \nu_0, r, \gamma \}$ consisting of a state space $\mathcal{S}$, action space $\mathcal{A}$, transition dynamics $P(s\tick|s,a)$ for $s, s\tick \in \mathcal{S}$, $a \in \mathcal{A}$, initial state distribution $\nu_0(s) \in \Delta^{|\mathcal{S}|}$ in the probability simplex, and reward function $r(a,s):\mathcal{S} \times \mathcal{A} \mapsto \mathbb{R}$.   We also use a discount factor $\gamma\in (0,1)$ 
(\citet{puterman2014markov} Sec 6).

We consider an agent that seeks to maximize the expected discounted reward by acting according to a decision policy 
$\pi(a|s) \in \Delta^{|\mathcal{A}|}$  for each $ s \in \mathcal{S}$.  
The expected reward is calculated over trajectories ${\tau \sim \pi(\tau):= \nu_0(s_0) \prod \pi(a_t | s_t) P(s_{t+1} | s_t, a_t)}$, which begin from an initial $s_0 \sim \nu_0(s)$ and evolve according to the policy $\pi(a|s)$ and \gls{MDP} dynamics $P(s\tick|s,a)$  
\begin{align}
  \hspace*{-.2cm}
  \obj(r) := 
   \max \limits_{\pi(a|s)} \, (1-\gamma) \, \mathbb{E}_{\tau \sim \pi(\tau)} \bigg[ \sum \limits_{t=0}^{\infty} \gamma^t \, r(s_t, a_t)  \bigg] \, . \label{eq:policy_only} 
\end{align}
\normalsize
We assume that the policy is stationary and Markovian, and thus independent of both the timestep and trajectory history.

\paragraph{Linear Programming Formulation}
We will focus on 
a linear programming (\gls{LP}) form for the objective in \cref{eq:policy_only}, which is common in the literature on convex duality.
With optimization over the discounted 
state-action occupancy measure, 
${\mu(a,s) \coloneqq (1-\gamma)} \mathbb{E}_{\tau\sim \pi(\tau)} \left[ \sum_{t=0}^\infty \gamma^t \, \mathbb{I}(a_t = a, s_t = s)\right]$, 
we rewrite 
the objective as
\ifx\omitindices\undefined
\small
\begin{align}
   \hspace*{-.23cm}  
   \obj(r) \coloneqq
   \max \limits_{\mu(a,s)} \big \langle {\mu(a,s)}, {r(a,s)} \big \rangle \qquad \text{subject to } \quad \mu(a,s) &\geq 0 \qquad \forall (a,s) \in \mathcal{A} \times \mathcal{S},   \quad  \label{eq:primal_lp} \\
    \sum_{a} \mu(a,s)& = (1-\gamma) \nu_0(s) + \gamma \sum_{a\tick, s\tick} P(s|a\tick, s\tick) \mu(a\tick,s\tick) \qquad \forall s \in \mathcal{S}, \nonumber 
\end{align}
\else
\begin{align}
   \hspace*{-.23cm}  
   \obj(r) \coloneqq
   \max \limits_{\mu} \big \langle {\mu, r} \big \rangle \qquad \text{subject to } \quad \mu(a,s) &\geq 0 \qquad \forall (a,s) \in \mathcal{A} \times \mathcal{S},   \quad  \label{eq:primal_lp} \\
    \sum_{a} \mu(a,s)& = (1-\gamma) \nu_0(s) + \gamma \sum_{a\tick, s\tick} P(s|a\tick, s\tick) \mu(a\tick,s\tick) \qquad \forall s \in \mathcal{S} \nonumber.
\end{align}
\fi
\normalsize
We refer to the constraints in the second line of \cref{eq:primal_lp} as the \textit{Bellman flow constraints}, which force $\mu(a,s)$ to respect the \gls{MDP} dynamics.
We denote the set of feasible $\mu$ as $\mathcal{M} \subset \mudomainfn$.    For normalized $\nu_0(s)$ and $P(s|a\tick, s\tick)$,   we show in \myapp{optimal_policy} that $\mu(a,s) \in \mathcal{M}$ implies $\mu(a,s)$
is normalized. 


It can be shown that any feasible $\mu(a,s) \in \mathcal{M}$  induces a stationary $\pi(a|s) = \mu(a,s) / 
\mu(s)$, 
where  $\mu(s) \coloneqq \sum_{a\tick} \mu(a\tick,s)$ and $\pi(a|s) \in \Delta^{|\mathcal{A}|}$ is normalized by definition.
Conversely, any stationary policy $\pi(a|s)$ induces a unique state-action visitation distribution $\mu(a,s)$ (\citet{syed2008apprenticeship},
\citet{feinberg2012handbook} Sec. 6.3).
 Along with the definition of $\mu(a,s)$ above, 
this result demonstrates
 the equivalence of the optimizations in \cref{eq:policy_only} and \cref{eq:primal_lp}.  We will proceed with the \gls{LP} notation from \cref{eq:primal_lp} and assume $\mu(s)$ is induced by
$\pi(a|s)$ whenever the two appear together in an expression.

Importantly, the flow constraints in \cref{eq:primal_lp} lead to a dual optimization which reflects the familiar Bellman equations \citep{bellman1957markovian}. 
To see this, we introduce Lagrange multipliers 
\ifx\omitindices\undefined
$V(s)$ for each flow constraint and $\lambdaplus \geq 0$ for the nonnegativity constraints.
\else
$V \in \mathbb{R}^{\mathcal{S}}$ for each flow constraint and $\lambdaplus \in \mathbb{R}_+^{\mathcal{A} \times \mathcal{S}}$ for the nonnegativity constraints.
\fi
  Summing over $s \in \mathcal{S}$, and eliminating $\mu(a,s)$ by setting $d/d\mu(a,s) = 0$ yields the \textit{dual} \gls{LP}
\ifx\omitindices\undefined
\begin{align}
 \hspace*{-.2cm} \obj^{*}(r)&\coloneqq
  \min \limits_{V(s), \lambdaplus} \,
  (1-\gamma) \big \langle \nu_0(s), V(s) \big \rangle  
 \quad \text{  subject to } \,\,\, V(s) = r(a,s) + \gamma \transitionv + \lambdaplus 
 \label{eq:dual_lp} , 
\end{align}
\else
\begin{align}
 \hspace*{-.2cm} \obj^{*}(r)&\coloneqq
  \min \limits_{V, \lambda} \,
  (1-\gamma) \big \langle \nu_0, V \big \rangle  
 \quad \text{  subject to } \,\,\, V(s) = r(a,s) + \gamma \transitionv + \lambdaplus 
\quad \forall (a,s) \in \mathcal{A} \times \mathcal{S}  
 \label{eq:dual_lp} , 
\end{align}
\fi
\normalsize
where we have used $\transitionv$ as shorthand for $\mathbb{E}_{P(s\tick|a,s)}[ V(s\tick)]$ and reindexed the transition tuple from $(s\tick, a\tick, s)$ to $(s,a, s\tick)$ compared to \cref{eq:primal_lp}.   Note that the constraint applies for all $(a,s) \in \mathcal{A} \times \mathcal{S}$ and that $\lambdaplus \geq 0$.  By complementary slackness, 
we know that $\lambdaplus =0$ for $(a,s)$ such that $\mu(a,s) > 0$.




\headerv
\subsection{Regularized MDPs}\label{sec:mdp_regularized}
\headerv
We now consider regularizing the objective in \cref{eq:primal_lp} using a convex penalty function $\Omega(\mu)$ with coefficient $1/\beta$.   
We primarily focus on regularization using a conditional divergence $\Omega_{\pi_0}(\mu) \coloneqq \mathbb{E}_{\mu(s)\pi(a|s)}[\ptReg(\pi)]$ between the policy and a normalized reference distribution $\pi_0(a|s)$, as in \mysec{divergence_intro} and \citep{ortega2013thermodynamics,fox2016taming,haarnoja17a,haarnoja2018soft}.   
We also use the notation $\Omega_{\mu_0}(\mu) = \mathbb{E}_{\mu(a,s)}[\ptReg(\mu)]$ to indicate regularization of the full state-action occupancy measure to a normalized reference $\mu_0(a,s)$,  which appears, for example, in \gls{REPS} \citep{peters2010relative, belousov2019entropic}.   
The regularized objective $\obj_{\Omega, \beta}(r)$ is then defined as
\ifx\omitindices\undefined
\begin{equation}
\obj_{\Omega, \beta}(r) :=  \max \limits_{\mu(a,s) \in \mathcal{M} } 
   \blangle \mu(a,s), r(a,s) \brangle - \dfrac{1}{\beta}   \Omega_{\pi_0}(\mu) 
  \label{eq:primal_reg}
\end{equation}
\else
\begin{equation}
\obj_{\Omega, \beta}(r) :=  \max \limits_{\mu \in \mathcal{M} } 
   \blangle \mu, r \brangle - \dfrac{1}{\beta}   \Omega_{\pi_0}(\mu) 
  \label{eq:primal_reg}
\end{equation}
\fi
where $\Omega_{\pi_0}(\mu)$ contains an expectation under $\mu(a,s)$ as in \cref{eq:kl_def}-(\ref{eq:alpha_def}).
We can also derive a dual version of the regularized \gls{LP}, 
by first writing the Lagrangian relaxation of \cref{eq:primal_reg}
\ifx\omitindices\undefined
\begin{align}
  \max \limits_{\mu(a,s)} \min \limits_{V(s), \lambdaplus}  \, 
     (1-\gamma) \big\langle \nu_0(s), V(s) \big  \rangle 
     + &\langle \mu(a,s), r(a,s) + \gamma \transitionv -V(s) + \lambdaplus \rangle - \frac{1}{\beta} \Omega_{\pi_0}(\mu)  \label{eq:dual_lagrangian} 
\end{align}
\else
\begin{align}
  \max \limits_{\mu} \min \limits_{V, \lambda}  \, 
     (1-\gamma) \big\langle \nu_0, V \big  \rangle 
     + &\, \langle \mu, r + \gamma \transitionvinds -V + \lambda \rangle - \frac{1}{\beta} \Omega_{\pi_0}(\mu) . \label{eq:dual_lagrangian} 
\end{align}
\fi
\normalsize
Swapping the order of optimization 
under 
strong duality, we can recognize the 
maximization over ${\mu(a,s)}$ as a conjugate function $\frac{1}{\beta}\Omega^{*}_{\pi_0, \beta}$, as in \cref{eq:conjstar}, leading to a regularized dual optimization
\ifx\omitindices\undefined
 \begin{align}
 \hspace*{-.2cm} &  \obj_{\Omega, \beta}^{*}(r) = 
 \min \limits_{V(s), \lambdaplus}
 \,  (1-\gamma) \big\langle \nu_0(s), V(s) \big  \rangle  + \frac{1}{\beta}\Omega^{*}_{\pi_0, \beta} \bigg(r(a,s) +  \gamma \transitionv -V(s) + \lambdaplus \bigg)\label{eq:dual_reg} 
\end{align}
\else
 \begin{align}
 \hspace*{-.2cm} &  \obj_{\Omega, \beta}^{*}(r) = 
 \min \limits_{V, \lambda}
 \,  (1-\gamma) \big\langle \nu_0, V \big  \rangle  + \frac{1}{\beta}\Omega^{*}_{\pi_0, \beta} \Big(r  +  \gamma \transitionvinds -V + \lambda \Big)\label{eq:dual_reg} 
\end{align}
\fi
which involves optimization over dual variables $V(s)$ only and is unconstrained, in contrast to \cref{eq:dual_lp}. 
Dual objectives of this form 
appear 
in \citep{nachum2020reinforcement, belousov2019entropic, basserrano2021logistic, neu2017unified}.
\textcolor{black}{We emphasize the need to include the Lagrange multiplier $\lambdaplus$, with $\lambdaplus > 0$ when the optimal policy has $\piopt(a|s) = 0$, since an important motivation for $\alpha$-divergence regularization is to encourage sparsity in the policy (see \cref{eq:qexp}, \citet{lee2018sparse, lee2019tsallis, chow2018path}).}

\paragraph{Soft Value Aggregation}
In iterative algorithms such as (regularized) modified policy iteration \citep{puterman1978modified, scherrer15a}, it is useful to consider the \textit{regularized Bellman optimality operator} \citep{geist2019theory}.   For given estimates of the state-action value $Q(a,s) \coloneqq r(a,s) + \gamma \transitionv$, the operator $\mathcal{T}^{*}_{\Omega_{\pi_0, \beta}}$ updates $V(s)$ as
\ifx\omitindices\undefined
\begin{align} 
\hspace*{-.2cm} V(s)  \leftarrow \frac{1}{\beta} \Omega^{*}_{\pi_0,\beta}(Q) = \max \limits_{\pi \in \Delta^{|\mathcal{A}|}} \blangle \pi(a|s), Q(a,s) \brangle - \frac{1}{\beta} \Omega_{\pi_0}(\pi) \label{eq:soft_value}
\end{align}
\else
\begin{align} 
\hspace*{-.2cm} V(s)  \leftarrow \frac{1}{\beta} \Omega^{*}_{\pi_0,\beta}(Q) = \max \limits_{\pi \in \Delta^{|\mathcal{A}|}} \blangle \pi, Q \brangle - \frac{1}{\beta} \Omega_{\pi_0}(\pi) \label{eq:soft_value} .
\end{align}
\fi
\normalsize
Note that this conjugate optimization is performed in each state $s \in \mathcal{S}$ and explicitly constrains each $\pi(a|s)$ to be normalized.
Although we proceed with the notation of \cref{eq:primal_reg} and \cref{eq:dual_reg}, our later developments 
are compatible with the
`soft-value aggregation' perspective above.  See \myapp{soft_value} for detailed discussion.

%% file: sections/conj_table.tex
\begin{figure*}[t]
\vspace*{-1cm}
\resizebox{\textwidth}{!}{%
\centering
        \begin{tabular}{llll}
    \toprule
 \multicolumn{1}{c}{  Divergence } & 
  Conjugate & 
\multicolumn{1}{c}{Conjugate Expression} 
 & Optimizing Argument ($\pi_{\pr}$ or $\mu_{\pr}$)
 \\
        \midrule
    $\frac{1}{\beta} \KL[\pi:\pi_0]$ & $\frac{1}{\beta} \Omega^{*}_{\pi_0,\beta}(\pr)$ 
    & 
         $\frac{1}{\beta} \sum \limits_{a} 
         \pi_0(a|s) \expof{\beta \cdot \pr(a,s)} - 
         \frac{1}{\beta} $
         & $ \pi_0(a|s) \expof{\beta \cdot \pr(a,s)}$
         \\[1.5ex]
     $\frac{1}{\beta} \KL[\mu:\mu_0]$ & $\frac{1}{\beta} \Omega^{*}_{\mu_0,\beta}(\pr)$  
      &
     $\frac{1}{\beta} \sum \limits_{a,s}
         \mu_0(a,s) \expof{\beta \cdot \pr(a,s)} - 
        \frac{1}{\beta} $
          & $ \mu_0(a,s) \expof{\beta \cdot \pr(a,s)}$ 
         \\[1.5ex] \midrule
    $\frac{1}{\beta} D_{\alpha}[\pi_0:\pi]$ & $\frac{1}{\beta} \Omega^{*(\alpha)}_{\pi_0,\beta}(\pr)$  
     &
   {\footnotesize $\frac{1}{\beta} \frac{1}{\alpha}  \sum \limits_{a} \pi_0(a|s) \exp_{\alpha}\big\{ \beta \cdot \big( \pr(a,s) - \myconst \big) \big\}^{\alpha} - \frac{1}{\beta} \frac{1}{\alpha} +   \psi_{\pr}(s;\beta)  $}
    &
        $\pi_0(a|s) \exp_{\alpha} \big\{\beta \cdot \left(  \pr(a,s) - \psi_{\pr}(s;\beta)  \right) \big\}$
        \\[1.5ex]
      $\frac{1}{\beta} D_{\alpha}[\mu_0:\mu]$ & $\frac{1}{\beta} \Omega^{*(\alpha)}_{\mu_0,\beta}(\pr)
     $ & 
    $\frac{1}{\beta} \frac{1}{\alpha} \sum \limits_{a,s} \mu_0(a,s) \exp_{\alpha}\big\{ \beta \cdot \pr(a,s) \big\}^{\alpha} - \frac{1}{\beta} \frac{1}{\alpha}$ 
     & $\mu_0(a,s) \exp_{\alpha} \big\{ \beta \cdot \pr(a,s) \big\} $ 
        \\   \bottomrule
    \end{tabular}%
    }
     \captionof{table}{Conjugate Function expressions for \textsc{kl} and $\alpha$-divergence regularization of either the policy $\pi(a|s)$ or occupancy $\mu(a,s)$.  See \myapp{conj1}-\ref{app:conj4} for derivations.  The final column shows the optimizing argument in the definition of the conjugate function $\frac{1}{\beta}\Omega^{*}(\pr)$, for example 
     $\mu_{\pr} \coloneqq \argmax_{\mu} \langle \mu, \pr \rangle -\frac{1}{\beta}\Omega_{\mu_0}(\mu)$.
    Note that each conjugate expression for $\pi(a|s)$ regularization also contains an outer expectation over $\mu(s)$.}
    \label{table:conj_table}
\end{figure*}



%% file: sections/new/2_adv_interp.tex
\headerv
\section{Adversarial Interpretation}\label{sec:adversarial_perturb}
\headerv
\normalsize
In this section, we interpret regularization as implicitly providing robustness to adversarial perturbations of the reward function.  To derive our adversarial interpretation, recall from \cref{eq:conjomega} that conjugate duality yields an alternative representation of the regularizer 
\ifx\omitindices\undefined
\begin{equation}
\hspace*{-.2cm}
     \omegamu = \max \limits_{\prinit(a,s)} \langle \mu(a,s), \prinit(a,s) \rangle - \alphaconj(\prinit) . \label{eq:conjomega2} 
 \end{equation}
 \else
 \begin{equation}
\hspace*{-.2cm}
     \omegamu = \max \limits_{\prinit \in \rprimedomain} \, \langle \mu, \prinit \rangle - \alphaconj(\prinit) . \label{eq:conjomega2} 
 \end{equation}
 \fi
 \normalsize
 
Using this conjugate optimization to expand the regularization term in the primal objective of \cref{eq:primal_reg}, 
\ifx\omitindices\undefined
\begin{equation}
\hspace*{-.2cm} 
\obj_{\Omega, \beta}(r) = \max \limits_{\mu \in \mathcal{M}}  \min \limits_{\textcolor{\highlight}{\prinit(a,s)}} 
  \langle {\mu(a,s)}, {r(a,s) - \textcolor{\highlight}{\prinit(a,s)}} \rangle  + \alphaconj \big( \textcolor{\highlight}{\prinit} \big) . 
\label{eq:pedro_form} 
\end{equation}
\else
\begin{equation}
\hspace*{-.2cm} 
\obj_{\Omega, \beta}(r) = \max  \limits_{\mu \in \mathcal{M}}  \min \limits_{\textcolor{\highlight}{\prinit \in \rprimedomain}} 
  \, \langle {\mu}, {r - \textcolor{\highlight}{\prinit}} \rangle  + \alphaconj \big( \textcolor{\highlight}{\prinit} \big) . 
\label{eq:pedro_form} 
\end{equation}
\fi
\normalsize
We interpret 
\cref{eq:pedro_form}
as a two-player minimax game between an agent and an implicit adversary, where the agent chooses an occupancy measure $\mu(a,s) \in \mathcal{M}$ or its corresponding policy $\pi(a|s)$,
and 
the adversary chooses reward perturbations $\prinit(a,s)$ 
subject to the convex
conjugate $\alphaconj(\prinit)$ as a penalty function \citep{ortega2014adversarial}.

To understand the limitations this penalty imposes on the adversary, we transform the optimization over $\pr$ 
in \cref{eq:pedro_form} 
to a constrained optimization in \mysec{generalization}.  This allows us to characterize the feasible set of
reward perturbations
available to the adversary or, equivalently, the set of modified rewards $\mr(a,s) \in \feasibleset$ to which a particular stochastic policy is robust.
In \mysec{any_policy} and \ref{sec:optimal_policy}, we 
interpret the worst-case adversarial perturbations corresponding to an arbitrary stochastic policy and the optimal policy, respectively.



\headerv
\subsection{Robust Set of Modified Rewards}\label{sec:feasible}\label{sec:generalization}
\headerv

In order to link our adversarial interpretation to robustness and zero-shot generalization as in 
\cref{eq:reward_perturbation_stmt}-(\ref{eq:generalization}),
we characterize the feasible set of reward perturbations in the following proposition.
We state our proposition for policy regularization, and discuss differences for $\mu(a,s)$ regularization in \myapp{mu_feasible}.
\begin{restatable}[]{proposition}{feasible}\label{prop:feasible} 
Assume a normalized policy $\pi(a|s)$ for the agent is given, with $\sum_a \pi(a|s) = 1 \, \forall s \in \mathcal{S}$.   Under $\alpha$-divergence policy regularization to a normalized reference $\pi_0(a|s)$, 
\revhl{the optimization over $\pr(a,s)$ in \cref{eq:pedro_form} can be written in the following constrained form}
\ifx\omitindices\undefined
\begin{align}
\min \limits_{\pr(a,s) \in \mathcal{R}^{\Delta}_{\pi}} \blangle \mu(a,s), r(a,s) - \pr(a,s)\brangle \qquad \text{where} \quad
\mathcal{R}^{\Delta}_{\pi} \coloneqq \left\{ \prrob(a,s) \in  \textcolor{red}{\mathcal{F}} \, \bigg| \, 
\Omega^{*(\alpha)}_{\pi_0,\beta}(\prrob) \leq 0 \right\} ,
\label{eq:alpha_feasible} 
\end{align}
\else
\begin{align}
\min \limits_{\pr \in \mathcal{R}^{\Delta}_{\pi}} \blangle \mu, r - \pr \brangle \qquad \text{where} \quad
\mathcal{R}^{\Delta}_{\pi} \coloneqq \left\{ \prrob \in \rprimedomain  \, \bigg| \, 
\Omega^{*(\alpha)}_{\pi_0,\beta}(\prrob) \leq 0 \right\} ,
\label{eq:alpha_feasible} 
\end{align}
\fi
We refer to $\mathcal{R}^{\Delta}_{\pi} \subset \rprimedomain$ as the feasible set of reward perturbations available to the adversary.   This translates to a robust set $\feasibleset$ of modified rewards $\mr(a,s) = r(a,s) - \prrob(a,s)$ 
for the given policy.
These sets depend on the $\alpha$-divergence and regularization strength $\beta$ via the conjugate function.

For \textsc{kl} divergence regularization, the constraint is
\begin{align}
   \sum\limits_{a \in \mathcal{A}} \pi_0(a|s) \expof{\beta \cdot \prrob(a,s)} \leq 1 \, . \label{eq:kl_feasible}
\end{align}
\end{restatable}
\vspace*{-.2cm}
See \myapp{pf_feasible} for proof, and \mytab{conj_table} for the convex conjugate function $\alphaconj(\prrob)$ associated with various regularization schemes.
\revhl{The proof proceeds by evaluating the conjugate function at the minimizing argument $\prpi$ in \cref{eq:pedro_form} (see \mysec{any_policy}), with $\Omega^{*(\alpha)}_{\pi_0,\beta}(\propt)=0 \, \forall \alpha$ for normalized $\pi(a|s)$ and $\pi_0(a|s)$.  
The constraint 
then follows from the fact that $\Omega^{*(\alpha)}_{\pi_0,\beta}(\propt)$ is convex and increasing in $\pr$ \citep{husain2021regularized}.}
We visualize the robust set for a two-dimensional action space in \myfig{feasible_set_main}, with additional discussion in \mysec{visualizations_feasible}. 
As in \cref{eq:generalization},
we can provide `zero-shot' performance guarantees using this set of modified rewards.
\ifx\omitindices\undefined
For any perturbed reward in the robust set $\mr(a,s) \in \feasibleset$, 
we have $\langle \mu(a,s), \mr(a,s) \rangle \geq \langle \mu(a,s), r(a,s) \rangle - \omegamu$, 
\else
For any perturbed reward in the robust set $\mr \in \feasibleset$, 
we have $\langle \mu, \mr \rangle \geq \langle \mu, r \rangle - \omegamu$,
\fi
so that the policy achieves an expected modified reward which is at least as large as the regularized objective. 
However, notice that this 
form of robustness is sensitive to the exact value of the 
regularized objective function.
Although entropy regularization and divergence regularization with a uniform reference
induce the same optimal $\mu(a,s)$,
we highlight crucial differences in their reward robustness interpretations in \mysec{entropy_vs_div}.

\headerv
\subsection{Worst-Case Perturbations: Policy Form}
\label{sec:any_policy}
\headerv
From the feasible set in \myprop{feasible}, how should the adversary select its reward perturbations?
In the following proposition,
we use the optimality conditions in \cref{eq:optimal_policy_as_grad}
to solve for the \textit{worst-case} reward perturbations $\prpi(a,s)$ which minimize \cref{eq:pedro_form} for an fixed but arbitrary stochastic policy $\pi(a|s)$.   


\begin{restatable}[]{proposition}{optimalperturbations}
\label{prop:optimal_perturbations} 
For a given policy $\pi(a|s)$ or state-action occupancy $\mu(a,s)$, 
the worst-case adversarial reward perturbations $\prpi$ or $\prmu$ associated with a 
convex function $\Omega(\mu)$ and regularization strength $1/\beta$ 
are
\ifx\omitindices\undefined
\begin{align}
    \hspace*{-.23cm} \prpi(a,s) = \nabla_{\mu} \frac{1}{\beta} \Omega(\mu) \, .
    \label{eq:prop1}
\end{align}
\else
\begin{align}
    \hspace*{-.23cm} \prpi = \nabla_{\mu} \frac{1}{\beta} \Omega(\mu) \, .
    \label{eq:prop1}
\end{align}
\fi
\end{restatable}
See \myapp{prop2} for proof. 
We now provide example closed form expressions for the worst-case reward perturbations under common regularization schemes.
We emphasize that the same stochastic policy $\pi(a|s)$ or joint occupancy measure $\mu(a,s)$ can be associated with different adversarial perturbations depending on the choice of $\alpha$-divergence and strength $\beta$.\footnote{One exception is that a policy with $a\tick$ s.t. $\pi(a\tick|s)=0$ can only be represented using \textsc{kl} regularization if $\pi_0(a\tick|s) = 0$.}
\headerv
\paragraph{KL Divergence} For \kl divergence policy regularization, the worst-case reward perturbations 
are
\begin{align}
    \prpi(a,s) = \frac{1}{\beta} \log \frac{\pi(a|s)}{\pi_0(a|s)} \, , \label{eq:kl_perturbations}
\end{align}
\normalsize
which 
corresponds to the pointwise regularization $\prpi(a,s)=\ptReg_{\pi_0}(\pi(a|s)$ for each state-action pair, with $\Omega_{\pi_0}(\mu)= \mathbb{E}_{\mu(a,s)}[\ptReg_{\pi_0}(\pi(a|s)]$.   
{See \myapp{conj1}}.
We show an analogous result in \myapp{conj2}
for 
state-action occupancy regularization $\KL[\mu:\mu_0]$, where $\prmu(a,s) = \frac{1}{\beta} \log 
 \frac{\mu(a,s)}{\mu_0(a,s)}= \ptReg_{\mu_0}(\mu(a,s))$.
 

\paragraph{$\alpha$-Divergence} For \kl divergence regularization, the worst-case reward perturbations had a similar expression for conditional and joint regularization.   However, we observe notable differences 
for the $\alpha$-divergence in general.  For policy regularization to a reference $\pi_0$,
\begin{align}
    \hspace*{-.23cm} \prpi(a,s) = \frac{1}{\beta} \log_{\alpha} \frac{\pi(a|s)}{\pi_0(a|s)} + 
    \psipr ,
    \label{eq:optimal_perturbations}
\end{align}
\normalsize
where we define $\psipr$ as
\small
\begin{equation}
    \psipr \coloneqq \frac{1}{\beta}\frac{1}{\alpha} \left( \sum \limits_{a \in \mathcal{A}} \pi_0(a|s)- \sum \limits_{a \in \mathcal{A}} \pi_0(a|s)^{1-\alpha} \pi(a|s)^\alpha \right).
    \label{eq:pr_normalizer}
\end{equation}
\normalsize
As we discuss in \myapp{conj3}, $\psipr$ plays the role of a normalization constant for the optimizing argument $\pir(a|s)$ in the definition of $\alphaconjn(\pr)$ (see \cref{eq:conjstar}, \mytab{conj_table}).  This term arises from differentiating $\Omega^{(\alpha)}_{\pi_0}(\mu)$ with respect to $\mu(a,s)$ instead of from an explicit constraint.
Assuming the given $\pi(a|s)$ and reference $\pi_0(a|s)$ are normalized, note that $\psipr = \frac{1}{\beta}(1-\alpha) D_{\alpha}[\pi_0:\pi]$.   
With normalization, we also observe that $\psipr = 0$ for \kl divergence regularization ($\alpha=1$), which confirms \cref{eq:kl_perturbations} is a special case of \cref{eq:optimal_perturbations}.

For any given state-action occupancy measure $\mu(a,s)$ and joint $\alpha$-divergence regularization 
to a reference $\mu_0(a|s)$, the worst-case perturbations become
\begin{align}
    \prmu(a,s) = \frac{1}{\beta} \log_{\alpha}\frac{\mu(a,s)}{\mu_0(a,s)}, \label{eq:pr_mu}
\end{align}
\normalsize
with detailed derivations in \myapp{mu_results}.   
In contrast to \cref{eq:optimal_perturbations}, this expression lacks an explicit normalization constant, as this constraint is enforced by the Lagrange multipliers $V(s)$ and $\mu(a,s) \in \mathcal{M}$ (\myapp{optimal_policy}).

\input{sections/new/2a_optimal_value}

%% file: sections/new/2a_optimal_value.tex
\headerv
\subsection{Worst-Case Perturbations: Value Form}
\label{sec:value_form}


In the previous section, we analyzed the implicit adversary corresponding to \textit{any} stochastic policy $\pi(a|s)$ for a given  $\Omega, \pi_0,$ and $\beta$. We now take a dual perspective, where the adversary is given access to a set of dual variables $V(s)$ across states $s \in \mathcal{S}$ and selects reward perturbations $\prv(a,s)$.    We will eventually show in \mysec{optimal_policy} that these perturbations match the policy-form perturbations at optimality.

Our starting point is Theorem 3 of \citet{husain2021regularized}, which arises from taking the convex conjugate $(-\obj_{\Omega,\beta}(r))^*$ of the \textit{entire} regularized objective $\obj_{\Omega,\beta}(r)$, which is concave in $\mu(a,s)$.  See \myapp{husain}.
\begin{restatable}[{\citet{husain2021regularized}}]{theorem}{husain}
\label{thm:husain}
The optimal value of the regularized objective $\obj_{\Omega, \beta}(r)$ in \cref{eq:primal_reg}, or its dual $\obj^{*}_{\Omega, \beta}(r)$ in \cref{eq:dual_reg}, 
is equal to
\ifx\omitindices\undefined
\begin{align}
   \hspace*{-.3cm} \inf \limits_{V(s), \lambdaplus} \inf \limits_{\textcolor{\highlight}{\prv(a,s)}}   
   & (1-\gamma) \blangle \nu_0(s), V(s) \brangle 
    + \frac{1}{\beta} \Omega^{*(\alpha)}_{\pi_0,\beta}\left( 
    \textcolor{\highlight}{\prv}
    \right)  \quad
    \label{eq:husain_constrain} \\
    \text{subject to }\quad
    V(s) &= r(a,s) + \gamma \transitionv - \textcolor{\highlight}{\prv(a,s)} + \lambdaplus \quad \forall (a,s) \in \mathcal{A} \times \mathcal{S}.
    \nonumber
\end{align}
\else
\begin{align}
   \hspace*{-.3cm} \inf \limits_{V, \lambda} \inf \limits_{\textcolor{\highlight}{\prv}}   
   \, & \, (1-\gamma) \blangle \nu_0, V \brangle 
    + \frac{1}{\beta} \Omega^{*(\alpha)}_{\pi_0,\beta}\left( 
    \textcolor{\highlight}{\prv}
    \right)  
    \quad
    \label{eq:husain_constrain} \\
    \text{subject to }\,\,\,
    V(s) &= r(a,s) + \gamma \transitionv - \textcolor{\highlight}{\prv(a,s)} + \lambdaplus \quad \forall (a,s) \in \mathcal{A} \times \mathcal{S}.
    \nonumber
\end{align}
\fi

\normalsize
\end{restatable}
Rearranging the equality constraint to solve for $\prv(a,s)$ and substituting into the objective, this optimization recovers the regularized dual problem in \cref{eq:dual_reg}.   
We can also compare \cref{eq:husain_constrain} to the \textit{unregularized} dual problem in \cref{eq:dual_lp}, which does not include an adversarial cost and whose constraint 
$V(s) = r(a,s) + \gamma \transitionv + \lambdaplus$ implies an \textit{unmodified} reward, or $\prv(a,s) = 0$. 
Similarly to \mysec{any_policy}, the adversary incorporates the effect of policy regularization via the reward perturbations $\prv(a,s)$.
\subsection{Policy Form $=$ Value Form at Optimality}\label{sec:optimal_policy}
In the following proposition, we provide a link between the policy and value forms of the adversarial reward perturbations, showing that $\proptpi(a,s)=\proptv(a,s)$ for the optimal policy $\piopt(a|s)$ and value $V_*(s)$.
As in \citet{eysenbach2021maximum}, the uniqueness of the optimal policy 
implies that its robustness may be associated with an environmental reward $r(a,s)$ for a given
regularized \gls{MDP}.
 \begin{restatable}{proposition}{advantage} 
\label{prop:advantage} For the optimal policy $\pi_*(a|s)$ and value function $V_*(s)$ corresponding to $\alpha$-divergence policy regularization with strength $\beta$, 
the policy and value forms of the worst-case adversarial reward perturbations match, $\proptpi=\proptv$, and are related to the advantage function via
\normalsize
\begin{align}
\hspace*{-.2cm}
\proptpi(a,s)
&= Q_*(a,s) - V_*(s) + \lambdaopt ,\label{eq:advantage}
\end{align}
where we define $Q_*(a,s) \coloneqq r(a,s) + \gamma \transitionvstar$ and recall $\lambdaopt \pi_*(a|s) = 0$ by complementary slackness.
Note that $V_*(s)$ depends on the regularization scheme via the conjugate function $\alphaconj(\prv)$ in \cref{eq:husain_constrain}.   
\end{restatable}
\normalsize
 \begin{proof}
See \myapp{advantage_pf}.  We consider the optimal policy in an \gls{MDP} with $\alpha$-divergence policy regularization $\frac{1}{\beta}\Omega^{(\alpha)}_{\pi_0}(\mu)$, which is derived via similar derivations as \citet{lee2019tsallis} or by eliminating $\mu(a,s)$ in \cref{eq:dual_lagrangian}.
\small
\begin{align}
      &\piopt(a|s)=  
          \, \pi_0(a|s) \exp_{\alpha}\big\{ \beta \cdot \big( 
    Q_*(a,s)
      - V_*(s) + \lambdaplus
      - \psipropt  \big) \big\}. \label{eq:optimal_policy} 
\end{align}
\normalsize
We prove \myprop{advantage} by plugging 
this optimal policy into the worst-case reward perturbations from \cref{eq:optimal_perturbations}, $\proptpi(a,s) = \frac{1}{\beta}\log_{\alpha}\frac{\piopt(a|s)}{\pi_0(a|s)} + \psipropt$.
We can also use \cref{eq:advantage} to verify 
$\piopt(a|s)$ is normalized,
since 
$\psi_{\pr_{\piopt}}$
ensures normalization for
the policy corresponding to $\prpiopt$.
In \myapp{advantage_confirm}, we also show $\psiqopt = V_*(s) + \psipropt$, where $\psiqopt$ is a Lagrange multiplier enforcing normalization in \cref{eq:soft_value}.
\end{proof}

\paragraph{Path Consistency Condition}
\vspace*{\header}
The equivalence between $\proptpi(a,s)$ and $\proptv(a,s)$ at optimality matches the \textit{path consistency} conditions from \citep{nachum2017bridging, chow2018path} and suggests generalizations to general $\alpha$-divergence regularization.
Indeed, combining \cref{eq:optimal_perturbations} and (\ref{eq:advantage}) and rearranging,
\small
\begin{align}
 r(a,s) + \gamma \transitionvstar - \frac{1}{\beta} \log_{\alpha}\frac{\pi_*(a|s)}{\pi_0(a|s)} & - \psipiopt  =  V_*(s) - \lambdaopt \label{eq:path}
\end{align}
\normalsize
for all $s \in \mathcal{S}$ and $a \in \mathcal{A}$.
This is a natural result, since path consistency is obtained using the \textsc{kkt} optimality condition involving the gradient with respect to $\mu$ of the Lagrangian relaxation in \cref{eq:dual_lagrangian}.   Similarly, we have seen in \myprop{optimal_perturbations} that 
\ifx\omitindices\undefined
$\prpi(a,s) = \nabla_{\mu} \omegamu$.   
\else
$\prpi = \nabla_{\mu} \omegamu$.  
\fi
See \myapp{path_consistency}.

Path consistency conditions were previously derived for the Shannon entropy \citep{nachum2017bridging} and Tsallis entropy with $\alpha=2$ \citep{chow2018path}, but our expression in \cref{eq:path} provides a generalization to $\alpha$-divergences with arbitrary reference policies.  We provide more detailed discussion in \myapp{indifference_all}.

\headerv
\paragraph{Indifference Condition} 
As \citet{ortega2014adversarial} discuss for the single step case, the saddle point of the minmax optimization in \cref{eq:pedro_form} reflects an \textit{indifference} condition which is a well-known property of Nash equilibria in game theory \citep{osborne1994course}.  
Consider $Q(a,s) = r(a,s) + \gamma \transitionv$ to be the agent's estimated payoff for each action in a particular state.  For the optimal policy, value, and worst-case reward perturbations, \cref{eq:path} shows that the pointwise modified reward $Q_*(a,s) - \proptpi(a,s) = V_*(s)$ is equal to a constant.\footnote{This holds for actions with $\piopt(a|s)>0$ and $\small \lambdaplus = 0$.   
Note that we treat $Q(a,s)$ as the reward in the sequential case.}   
Against the optimal strategy of the adversary, the agent becomes indifferent between the actions in its mixed strategy.
The value or conjugate function $V_*(s) = \klconjpii(Q_*)$ (see \myapp{soft_value}) is known as the \textit{certainty equivalent} \citep{fishburn1988nonlinear, ortega2013thermodynamics}, 
which measures the total expected utility for an agent starting in state $s$, in a two-player game 
against an adversary defined by the regularizer $\Omega$ with strength $\beta$.
We empirically confirm the indifference condition in \myfig{perturb_opt} and 
\ref{fig:pyco_indifference}.

%% file: sections/new/4_visualizations.tex
\headerv
\section{Experiments}\label{sec:visualizations}
\headerv
In this section, we visualize the robust set and worst-case reward perturbations associated with policy regularization, using intuitive examples to highlight theoretical properties of our adversarial interpretation.  

\input{sections/feasible_fig_no_lambda}

\headerv
\subsection{Visualizing the Robust Set}\label{sec:visualizations_feasible}
\headerv

\input{sections/new/4_vis_feasible}
\input{sections/single_step_fig}

\subsection{Visualizing the Worst-Case Reward Perturbations}\label{sec:visualizations_worst_case}
\headerv
In this section, we consider \textsc{kl} divergence regularization to a uniform reference policy, which is equivalent to Shannon entropy regularization but more appropriate for analysis, as we discuss in \mysec{entropy_vs_div}. 
\headerv
\paragraph{Single Step Case}
In \cref{fig:perturb_opt}, we plot the \textit{negative} worst-case reward perturbations $-\pr_{\pi}(a,s)$ and modified reward for a single step decision-making case.   For the optimal policy in \cref{fig:perturb_opt}(a), the perturbations match the advantage function as in \cref{eq:advantage} and the perturbed reward for all actions matches the value function $V_*(s)$. 
While we have shown in \mysec{any_policy} that any stochastic policy may be given an adversarial interpretation, we see in \cref{fig:perturb_opt}(b) that the indifference condition does not hold for suboptimal policies.  

The nearly-deterministic policy in \myfig{perturb_opt}(a)(ii) also provides intuition for the unregularized case as {$\beta \rightarrow \infty$}. 
Although we saw in \mysec{value_form} that $\pr_{\piopt}(a,s) = 0 \,\, \forall a$ 
in the unregularized case,  \cref{eq:dual_lp} and (\ref{eq:advantage}) suggest that $\lambdaplus = V_*(s) - Q_*(a,s)$ plays a similar role to the (negative) reward perturbations in \myfig{perturb_opt}(a)(ii), with $\lambda(a_1,s)=0$ and $\lambda(a,s)>0$ for all other actions.

\headerv
\paragraph{Sequential Setting}
In \cref{fig:pyco}(a), we consider a grid world where the agent receives $+5$ for picking up the reward pill, $-1$ for stepping in water, and zero reward otherwise.  We train an agent 
using tabular $Q$-learning and a discount factor $\gamma = 0.99$. 
We visualize the worst-case reward perturbations $\prpi(a,s) = \frac{1}{\beta}\log \frac{\pi(a|s)}{\pi_0(a|s)}$ 
in each state-action pair for policies trained with various regularization strengths in \cref{fig:pyco}(b)-(d).  While it is well-known that there exists a unique optimal policy for a given regularized \gls{MDP}, 
our results 
additionally display the adversarial strategies and resulting Nash equilibria which can be associated with a regularization scheme specified by $\Omega$, $\pi_0$, $\alpha$, and $\beta$ in a given \gls{MDP}.

Each policy implicitly hedges against an adversary that perturbs the rewards according to the values and colormap shown.  For example, inspecting the state to the left of the goal state in panel \cref{fig:pyco}(b)-(c), we see that the adversary reduces the immediate reward for moving right (in red, $\proptpi > 0$). 
 Simultaneously, the adversary raises the reward for moving up or down towards the water (in blue).
This is in line with the constraints on the feasible set, which imply that the adversary must balance reward decreases with reward increases in each state. 
In \myapp{optimal_policy_confirm} \cref{fig:pyco_indifference}, we certify the optimality of each policy using the path consistency conditions, which also confirms that the adversarial perturbations have rendered the agent indifferent across actions in each state.

Although we observe that the agent with high regularization in \cref{fig:pyco}(b) is robust to a strong adversary, the value 
of the regularized objective is also lower in this case.   As expected, lower regularization strength reduces robustness to negative reward perturbations.
With low regularization in \cref{fig:pyco}(d), 
the behavior of the agent barely deviates from the deterministic policy in the face of the weaker adversary.





\input{sections/pyco_fig}

%% file: sections/feasible_fig_no_lambda.tex
\begin{figure*}[!t]
\vspace*{-1.6cm}
\centering
\begin{center}
\[\arraycolsep=.003\textwidth
\begin{array}{cccccc} 
 \begin{subfigure}{0.07\textwidth} \vspace*{-.24cm} \includegraphics[width=\textwidth]
 {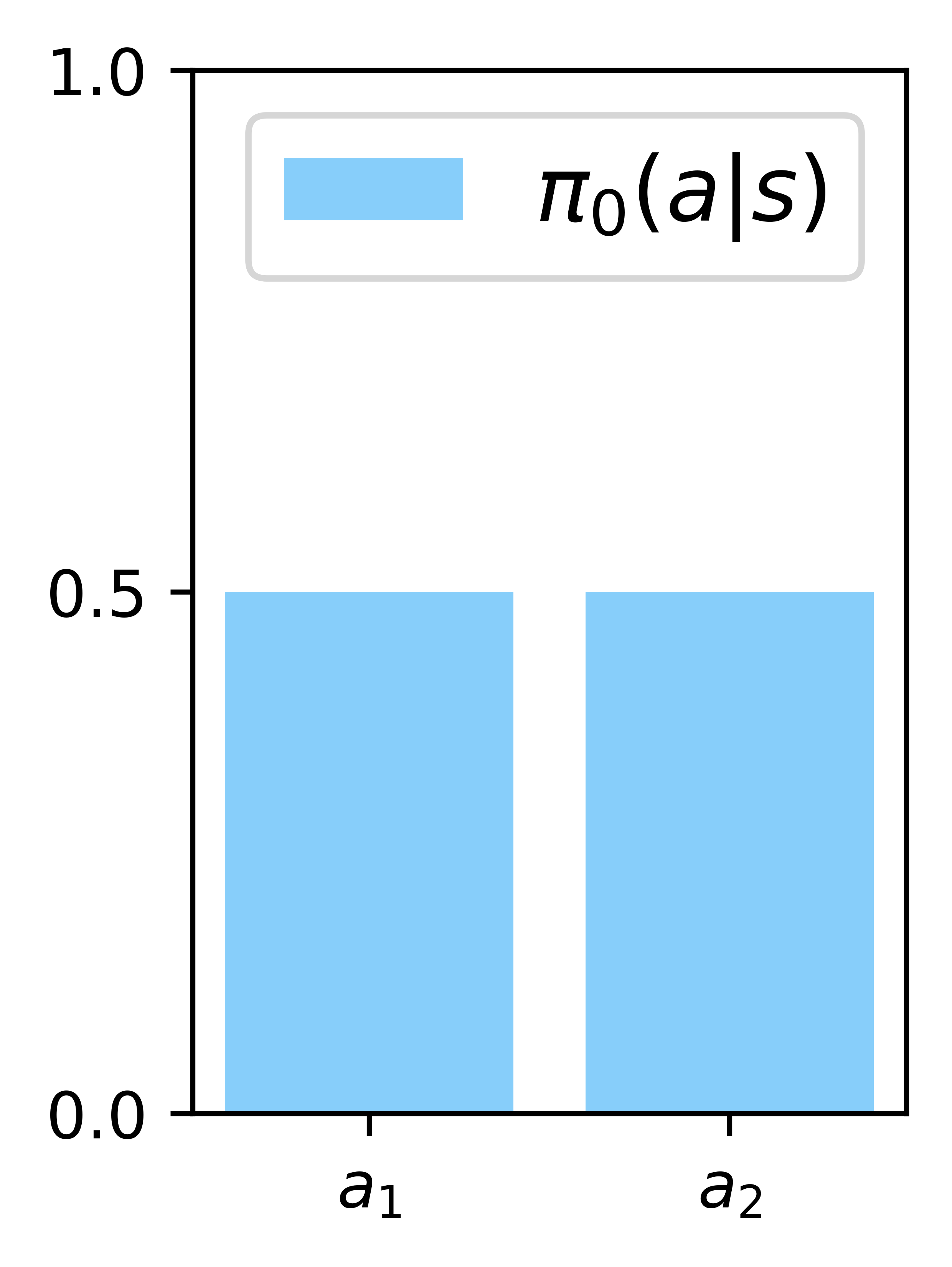} \end{subfigure}& 
\begin{subfigure}{0.175\textwidth}\includegraphics[width=\textwidth,  clip]
{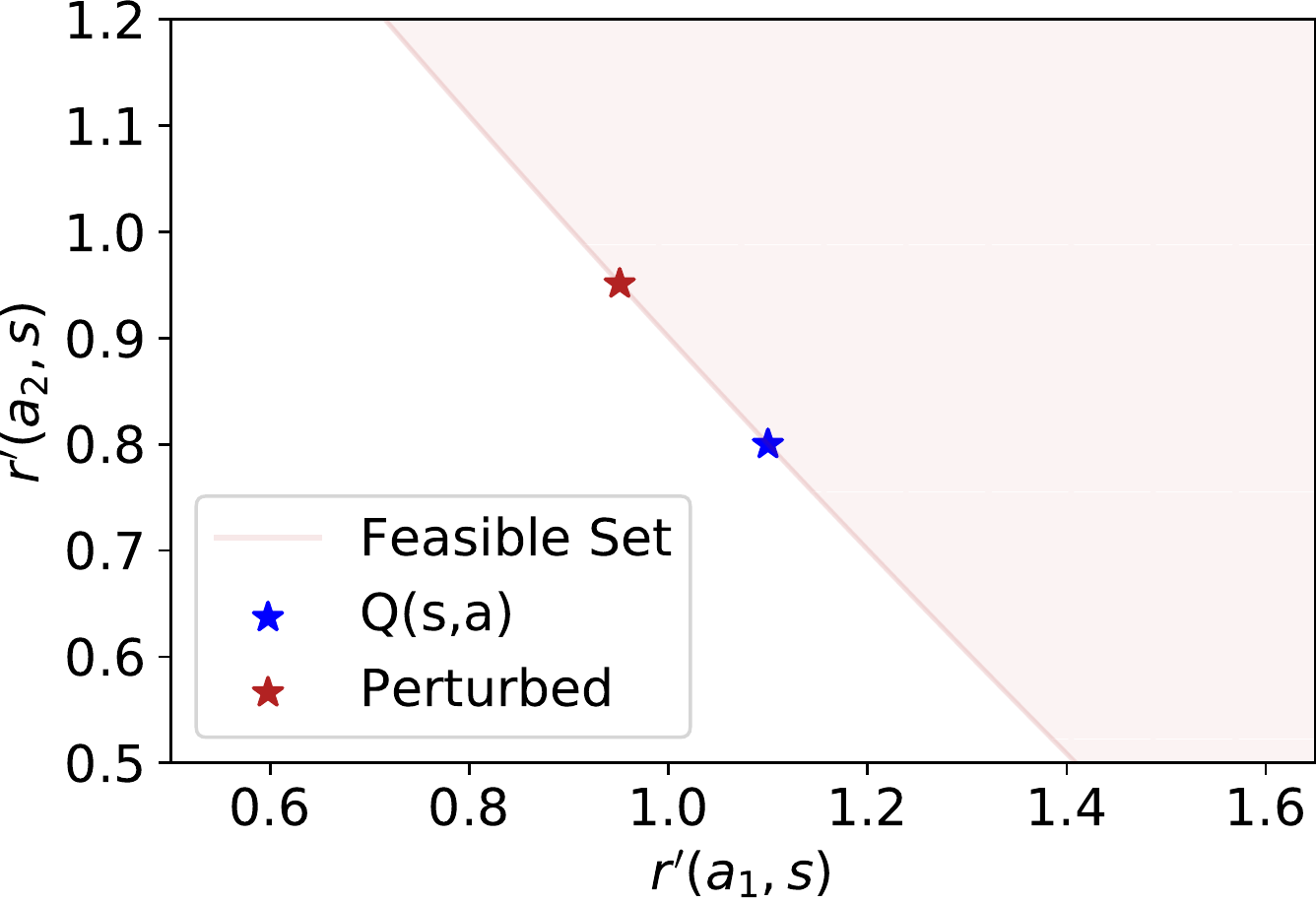}
\caption{ \centering \small $D_{KL}, \beta = 0.1$}\end{subfigure}& 
\begin{subfigure}{0.175\textwidth}\includegraphics[width=\textwidth,  clip]
{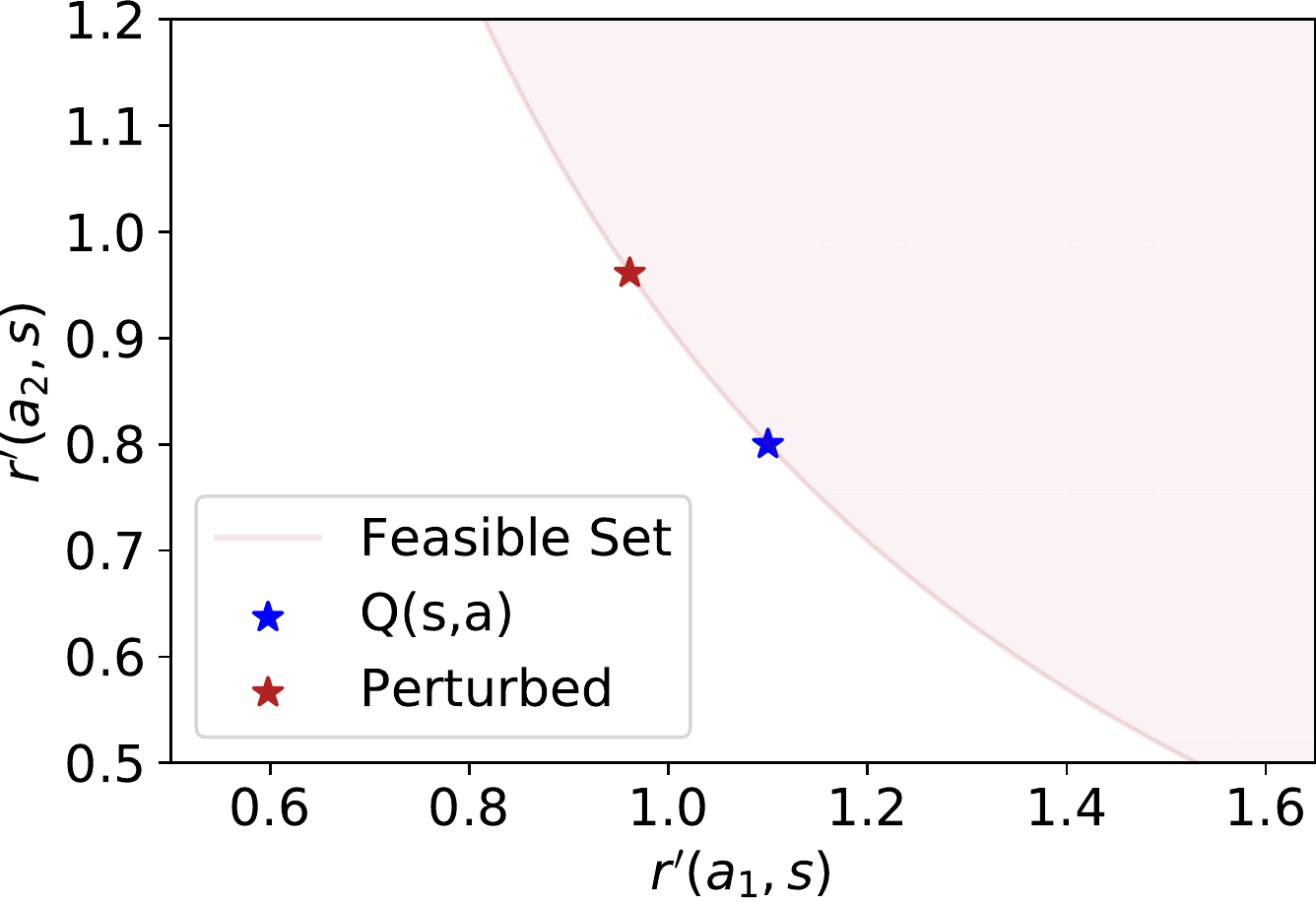}
\caption{ \centering  $D_{KL}, \beta = 1$}\end{subfigure}& 
\begin{subfigure}{0.175\textwidth}\includegraphics[width=\textwidth,  clip]
{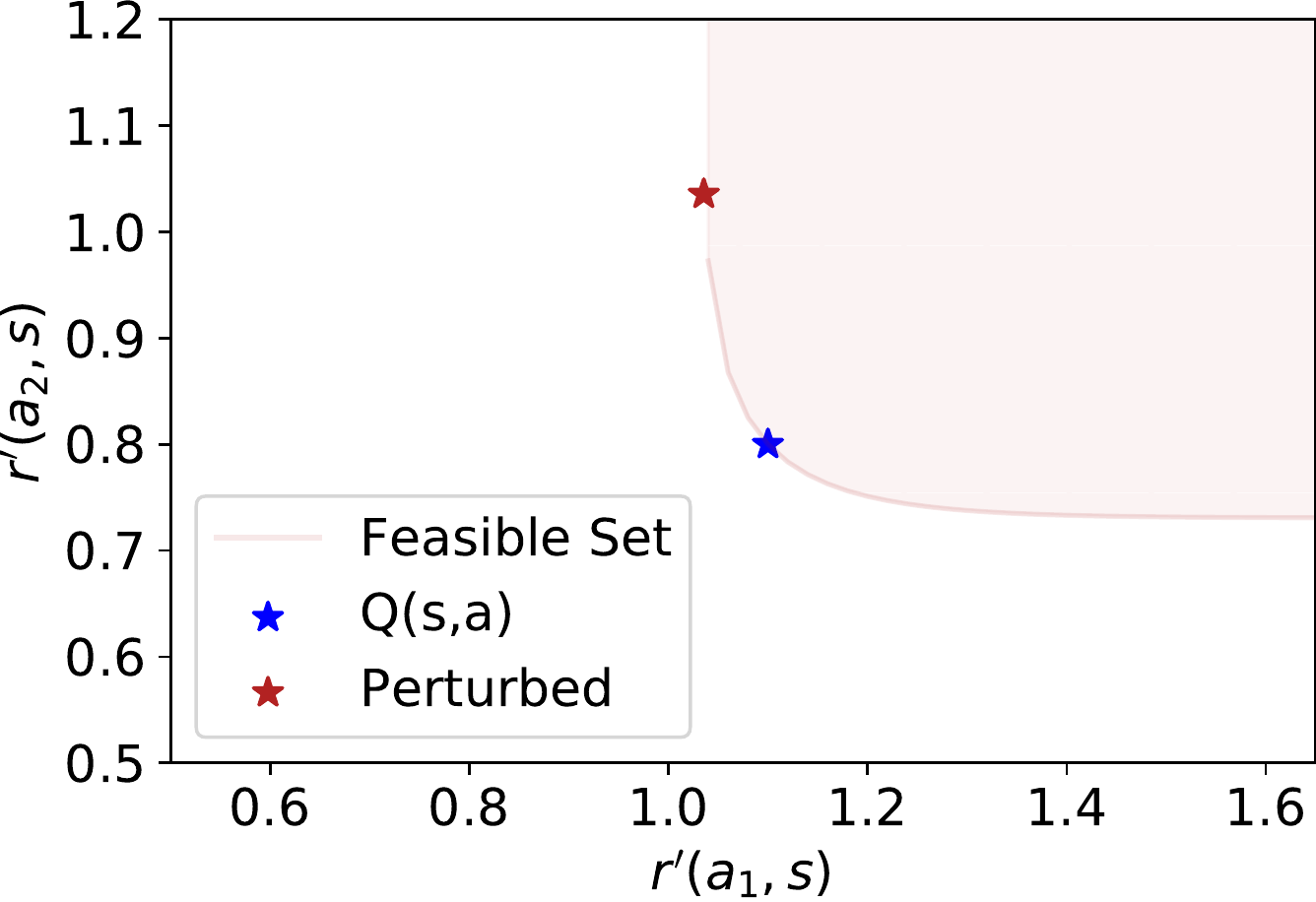}
\caption{ \centering \small $D_{KL}, \beta = 10$}\end{subfigure}& \rulesep
\begin{subfigure}{0.175\textwidth}\includegraphics[width=\textwidth, trim= 0 0 0 0, clip]
{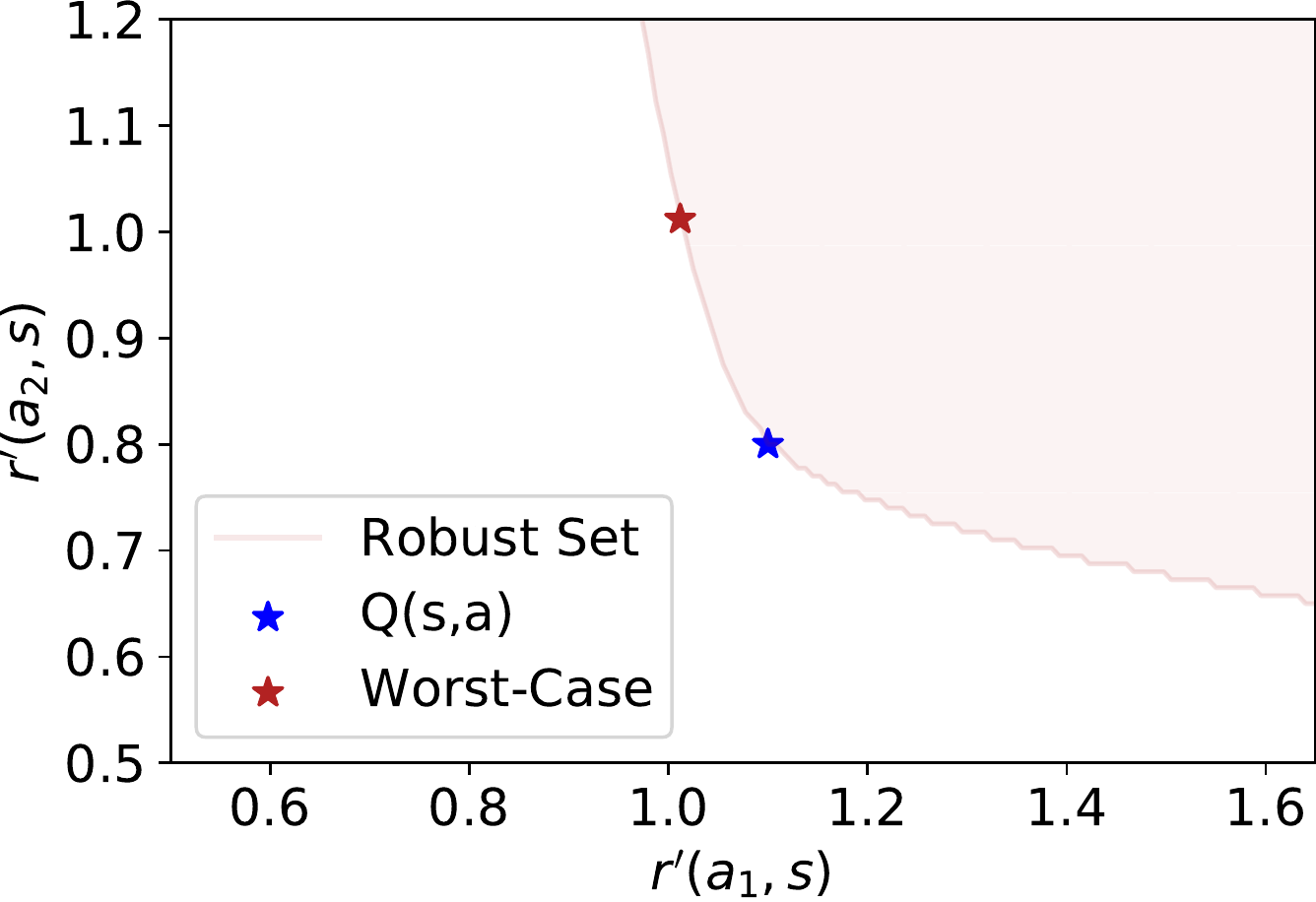}
\caption{\centering \small $\alpha = -1, \beta = 10$}\end{subfigure}&
\begin{subfigure}{0.175\textwidth}\includegraphics[width=\textwidth, trim= 0 0 0 0, clip]
{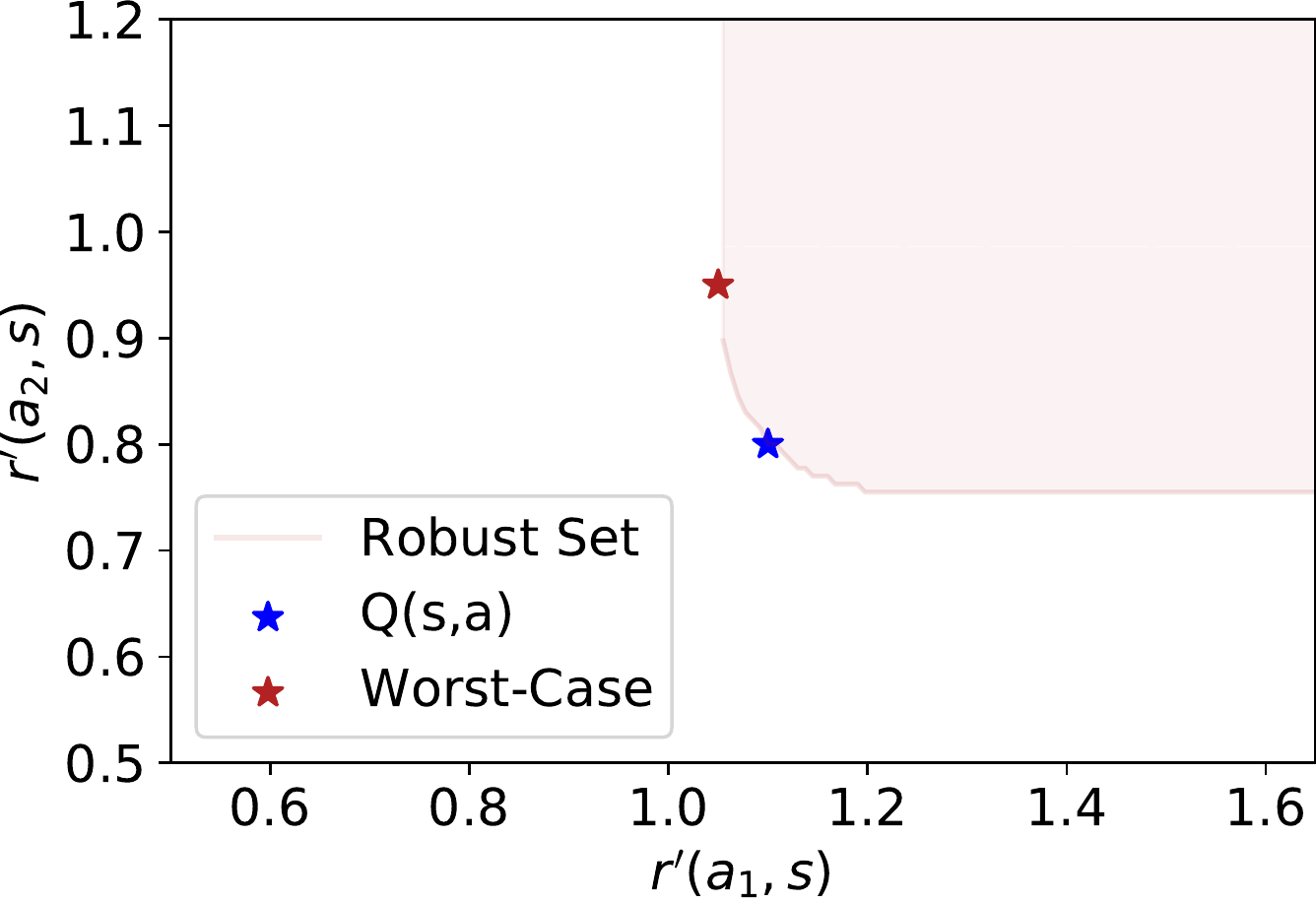}
\caption{ \centering \small $\alpha = 3, \beta = 10$}\end{subfigure}
\\ 
\begin{subfigure}{0.07\textwidth}\vspace*{-.24cm} \includegraphics[width=\textwidth]{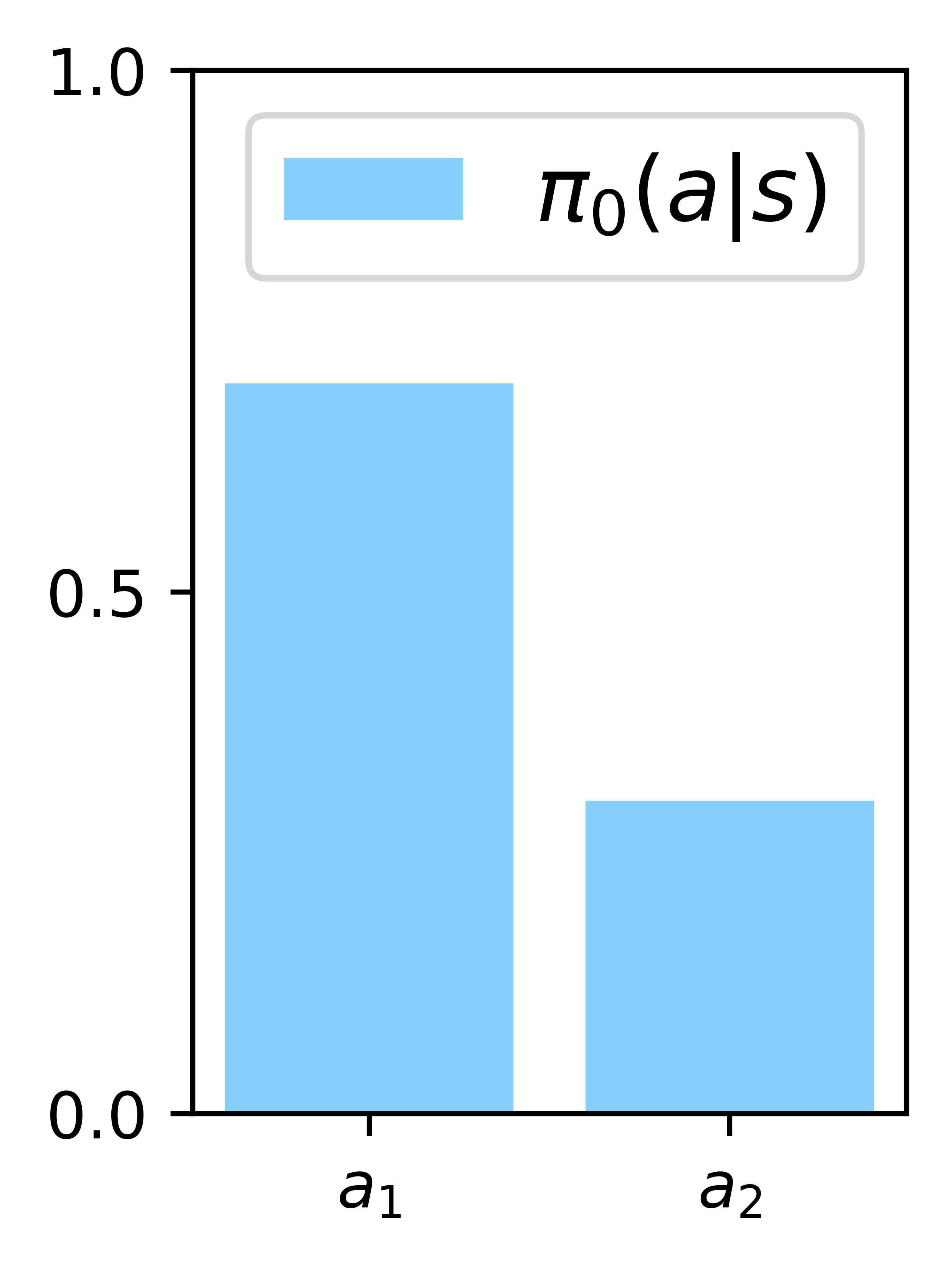}\end{subfigure}& 
\begin{subfigure}{0.175\textwidth}\includegraphics[width=\textwidth,  clip]
{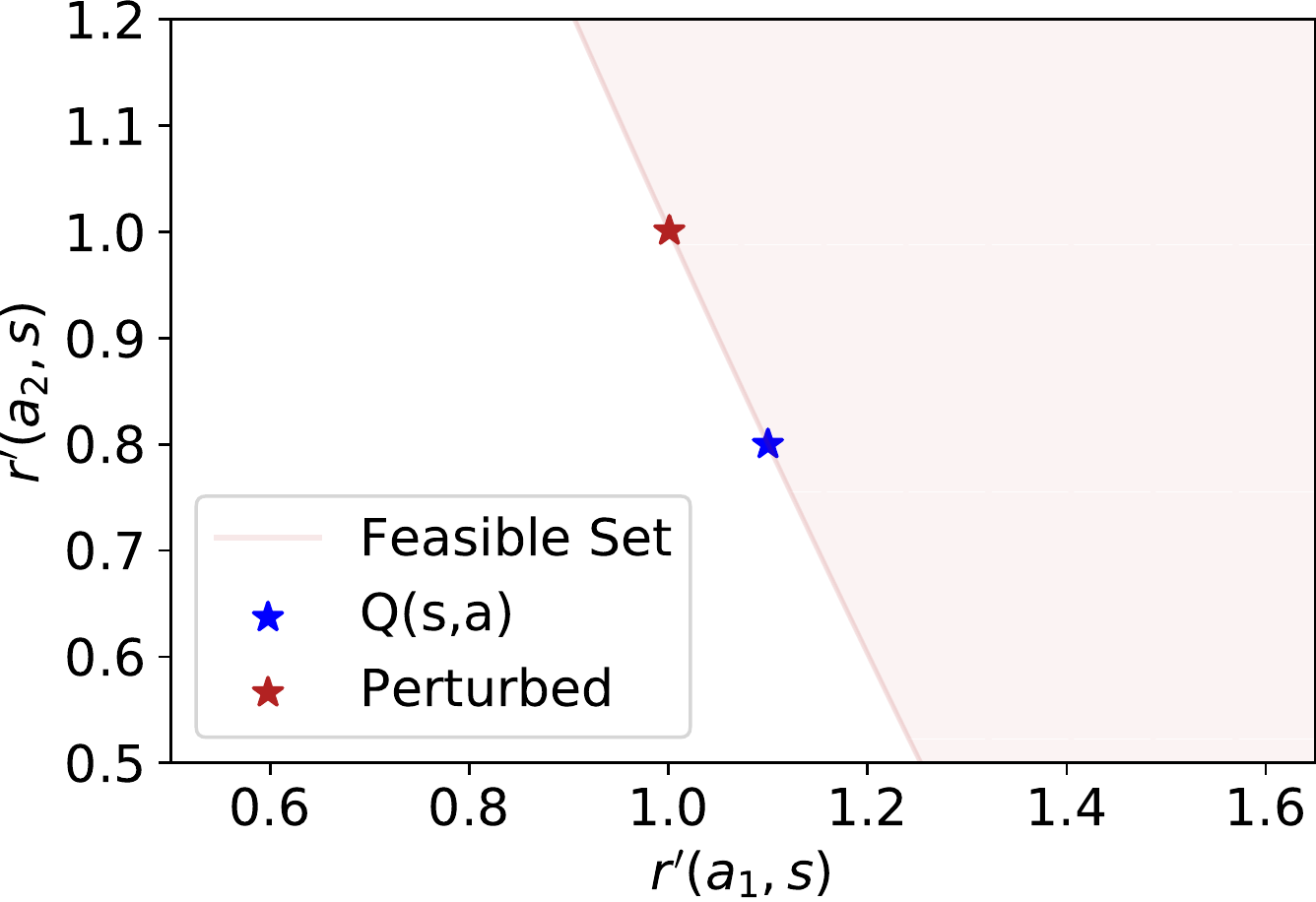}
\caption{ \centering \small $D_{KL}, \beta = 0.1$}\end{subfigure}& 
\begin{subfigure}{0.175\textwidth}\includegraphics[width=\textwidth,  clip]
{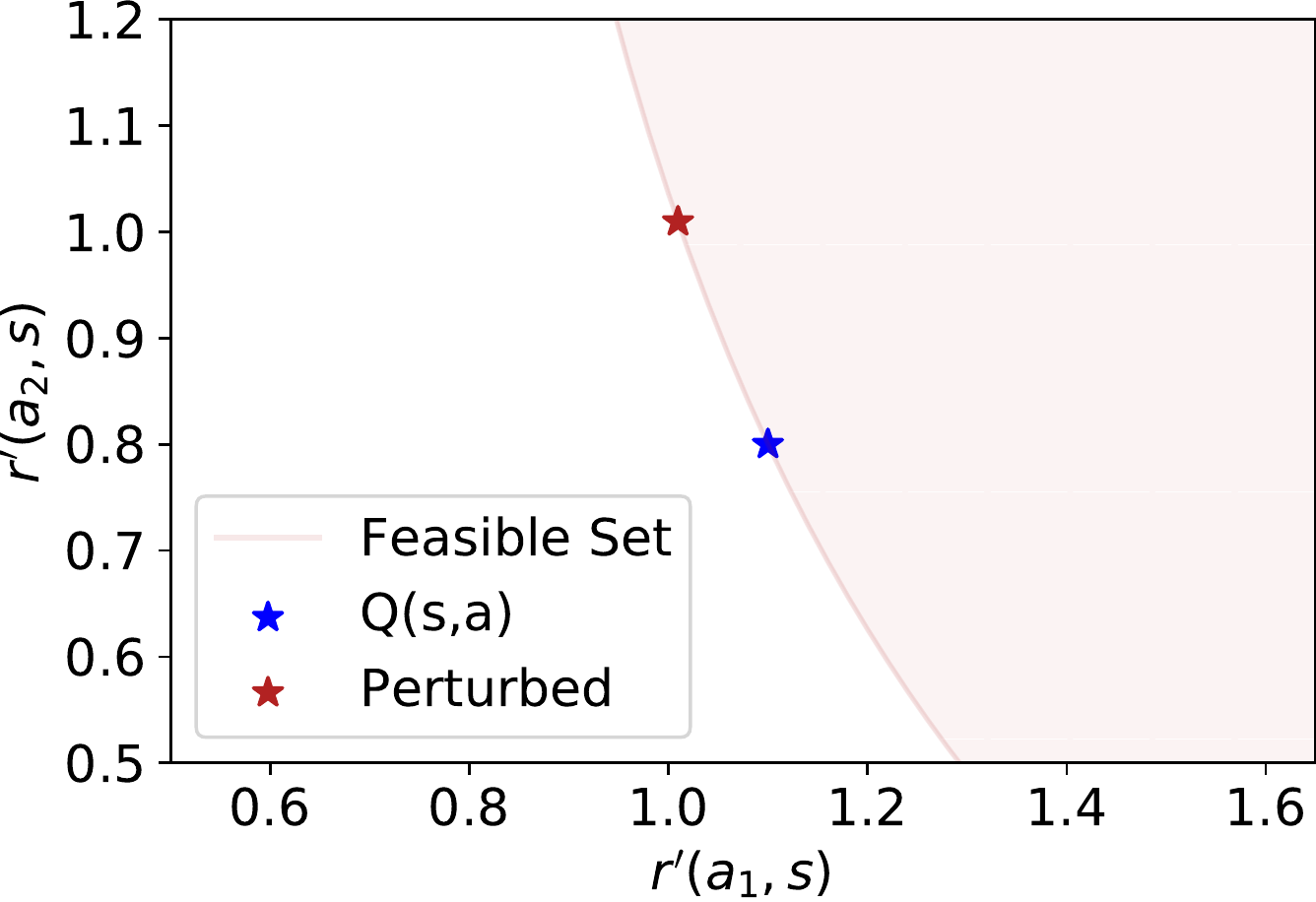}
\caption{ \centering \small $D_{KL}, \beta = 1$}\end{subfigure}& 
\begin{subfigure}{0.175\textwidth}\includegraphics[width=\textwidth,  clip]
{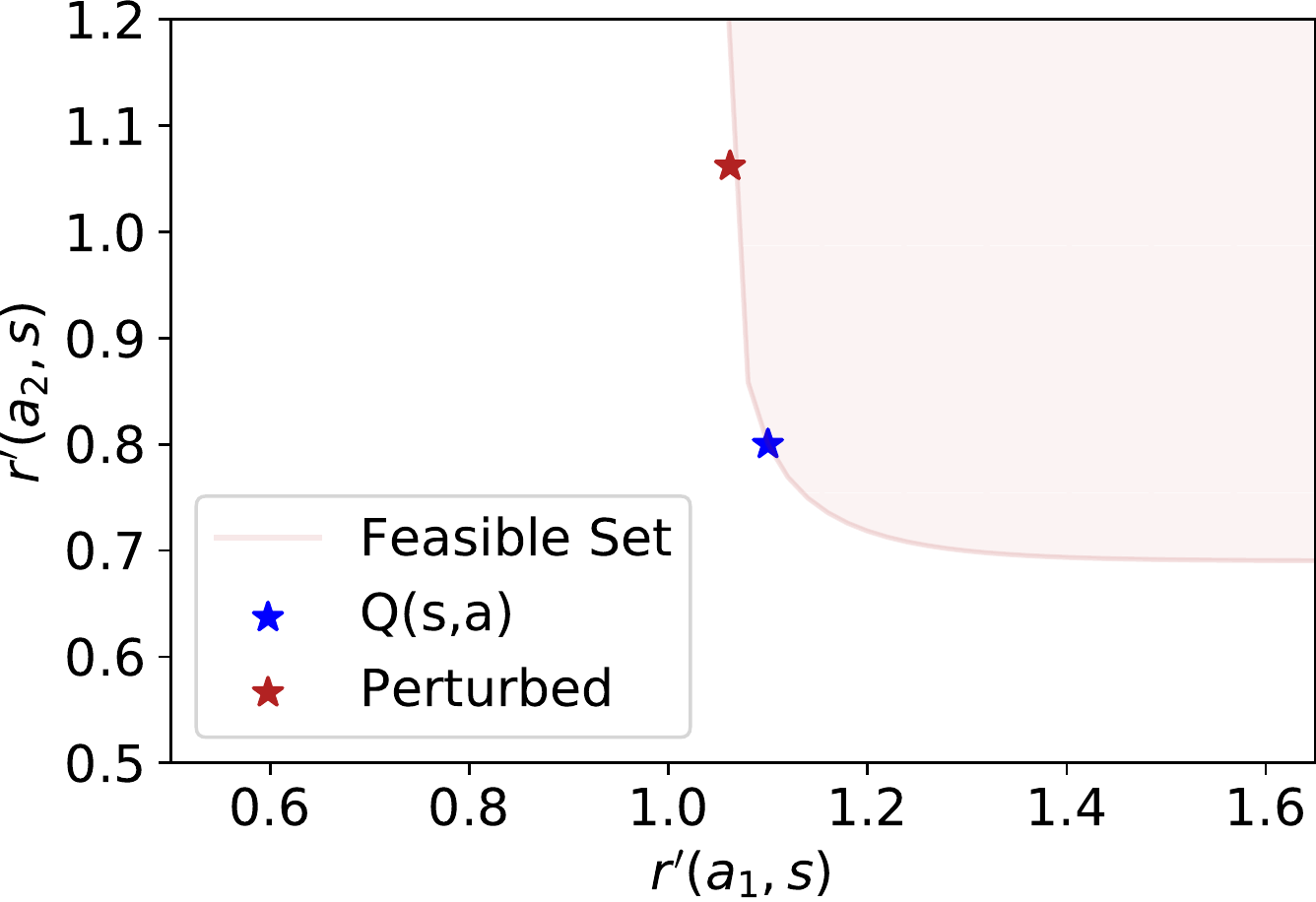}
\caption{\centering \small $D_{KL}, \beta = 10$}\end{subfigure}& \rulesep 
\begin{subfigure}{0.175\textwidth}\includegraphics[width=\textwidth,  clip]
{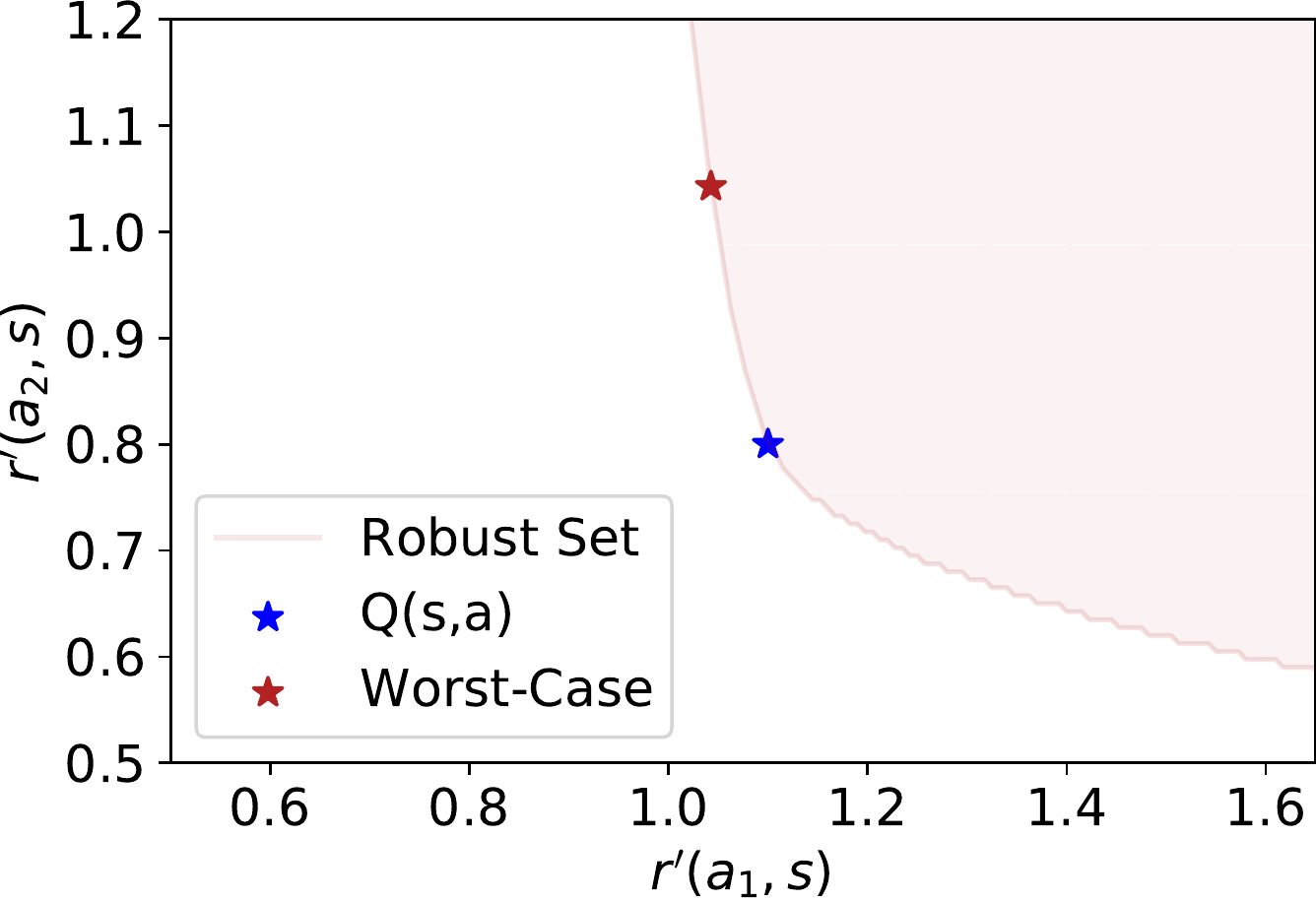}
\caption{\centering \small $\alpha = -1, \beta = 10$}\end{subfigure}&
\begin{subfigure}{0.175\textwidth}\includegraphics[width=\textwidth,  clip] 
{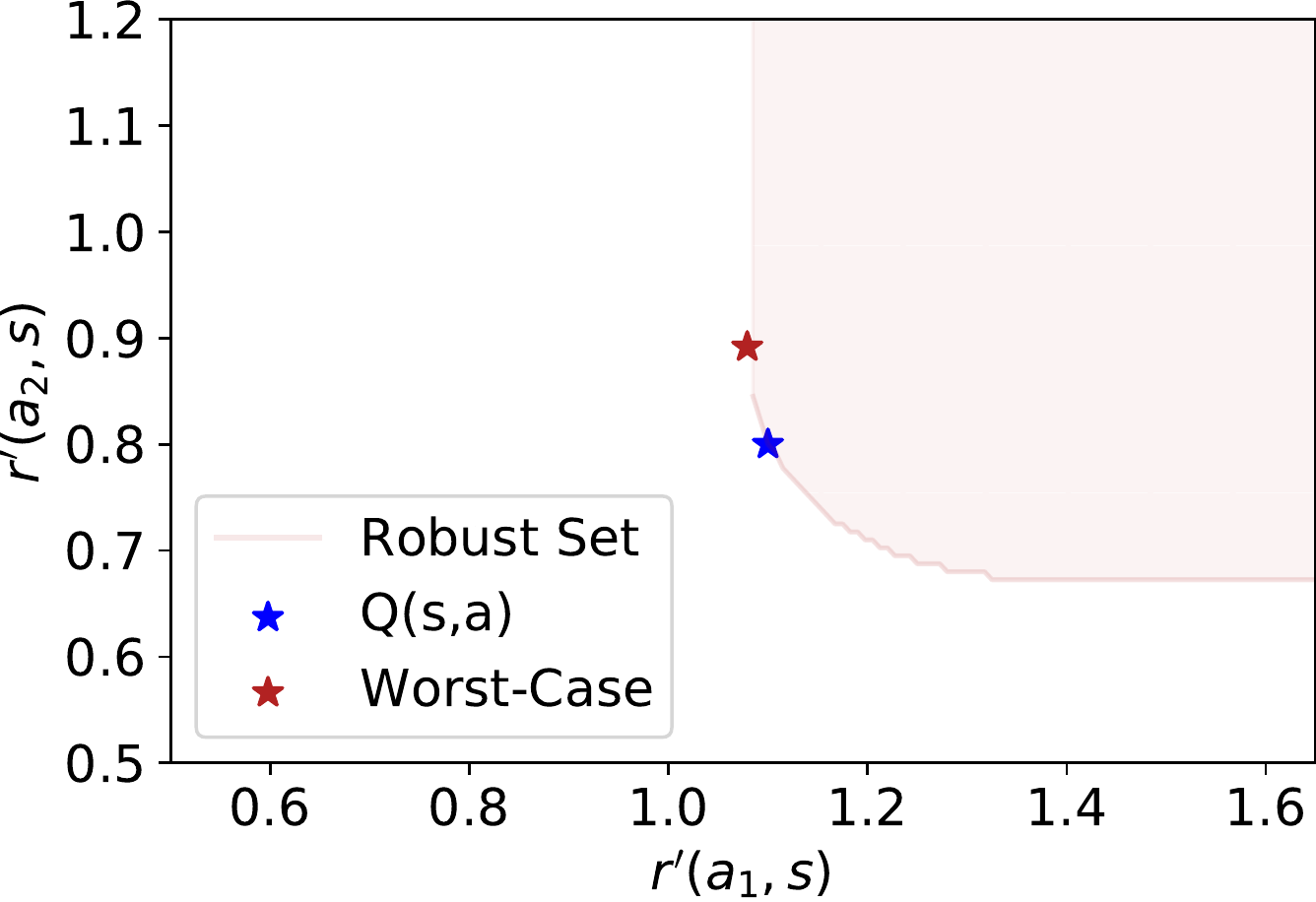}
\caption{ \centering \small $\alpha = 3, \beta = 10$}\end{subfigure}
\end{array}
\] 
\end{center}
\vspace*{-.38cm}
\caption{\textbf{Robust Set} (red region) of perturbed reward functions to which a stochastic policy generalizes, in the sense that the policy is guaranteed to achieve an expected modified reward greater than or equal to the value of the regularized objective (\cref{eq:generalization}).  
The robust set characterizes the perturbed rewards which are feasible for the adversary. Red stars indicate the worst-case perturbed reward $\mr_{\piopt} = r - \pr_{\piopt}$ (\myprop{optimal_perturbations}).
We show robust sets 
 for the optimal $\pi_*(a|s)$ with fixed $Q(a,s)=r(a,s)$ values (blue star), where the optimal policy differs based on the regularization parameters $\alpha, \beta, \pi_0$ (see \cref{eq:optimal_policy}).  The robust set is more restricted with decreasing regularization strength (increasing $\beta$), implying decreased generalization.
Importantly, the slope of the robust set boundary can be linked to the action probabilities under the policy (see \mysec{visualizations_feasible}). 
}\label{fig:feasible_set_main}
\vspace*{-.1cm}
\end{figure*}

%% file: sections/new/4_vis_feasible.tex
In \myfig{feasible_set_main}, we visualize the robust set of perturbed rewards for the optimal policy in a two-dimensional action space for the \textsc{kl} or $\alpha$-divergence, various $\beta$, and a uniform or non-uniform prior policy $\pi_0$.
Since the optimal policy can be easily calculated in the single-step case, we consider fixed $Q_*(a,s) = r(a,s) = \{1.1, 0.8\}$ and show the robustness of the optimal $\pi_*(a|s)$, which differs based on the choice of regularization scheme using \cref{eq:optimal_policy}.
We determine the feasible set of $\prrobust$ using the constraint in \myprop{feasible} (see \myapp{feasible_plot} for details), and plot the modified reward $\mr_{\piopt}(a,s) = Q_*(a,s) - \proptpi(a,s)$ for each action.  

Inspecting the constraint for the adversary 
in \cref{eq:kl_feasible}, note that both reward increases $\prrobust(a,s) < 0$ and reward decreases $\prrobust(a,s) > 0$ contribute non-negative terms at each action, which either up- or down-weight the reference policy $\pi_0(a|s)$.  The constraint on their summation forces the adversary to trade off between perturbations of different actions in a particular state.   
Further, since the constraints in \myprop{feasible} integrate over the action space, the rewards for \textit{all} actions in a particular state must be perturbed together.  
While it is clear that increasing the reward in both actions preserves the inequality in \cref{eq:generalization},  
\cref{fig:feasible_set_main} also includes regions where one reward decreases.

For high regularization strength ($\beta=0.1$), we observe that the boundary of the feasible set is nearly linear, with the slope $-\frac{\pi_0(a_1|s)}{\pi_0(a_2|s)}$ based on the ratio of 
action probabilities in a policy that matches the prior.  
The boundary steepens for lower regularization strength.  {We can use the indifference condition to provide further geometric insight.   First, drawing a line from the origin with slope $1$ will intersect the feasible set at the worst-case modified reward (red star) in each panel, with $\mrrobust(a_1,s) = \mrrobust(a_2,s)$.  At this point, the slope of the tangent line yields the ratio of action probabilities in the regularized policy, as we saw for the $\beta = 0.1$ case.  With decreasing regularization as $\beta \rightarrow \infty$, the slope approaches $0$ or $-\infty$ for a nearly deterministic policy and a rectangular feasible region.} 

Finally, we show the $\alpha$-divergence robust set with $\alpha \in \{-1, 3\}$ and $\beta = 10$ in \cref{fig:feasible_set_main} (d)-(e) and (i)-(j), with further visualizations in \myapp{additional_feasible}.  
Compared to the \textsc{kl} divergence, we find a wider robust set boundary for $\alpha =-1$.   For $\alpha = 3$ and $\beta=10$, the boundary is more strict and we observe much smaller reward perturbations as the optimal policy becomes deterministic ($\pi(a_1|s) = 1$) for both reference distributions.    
However, in contrast to the unregularized deterministic policy, the reward perturbations $\proptpi(a,s) \neq 0$ are nonzero.
We provide a worked example in \myapp{deterministic_alpha}, and note that 
indifference does not hold in this case, $\mr_{\piopt}(a_1,s) \neq \mr_{\piopt}(a_2,s)$, due to the Lagrange multiplier $\lambda_*(a_2,s)>0$. 

%% file: sections/single_step_fig.tex
\newcommand{\WidthS}{.2\textwidth}
\newcommand{\WidthM}{.22\textwidth}
\newcommand{\WidthMM}{.22\textwidth}
\newcommand{\WidthMMM}{.23\textwidth}
\newcommand{\HeightS}{22mm}
\newcommand{\HeightSS}{22mm}
\newcommand{\HeightM}{22mm}

\begin{figure*}[t]
\vspace*{-1.65cm}
\centering
\subcaptionbox{Optimal Policy \label{fig:perturb_opt_a}}{
\begin{minipage}{.485\textwidth}
\[\arraycolsep=.0025\textwidth
\begin{array}{cccc} 
  & \multicolumn{3}{c}{\text{(i) }\beta = 1.0} \\[.25ex]
\multirow{2}{*}[.25cm]{ \includegraphics[width=\WidthM,height=\HeightSS, trim=0 0 0 0, clip]{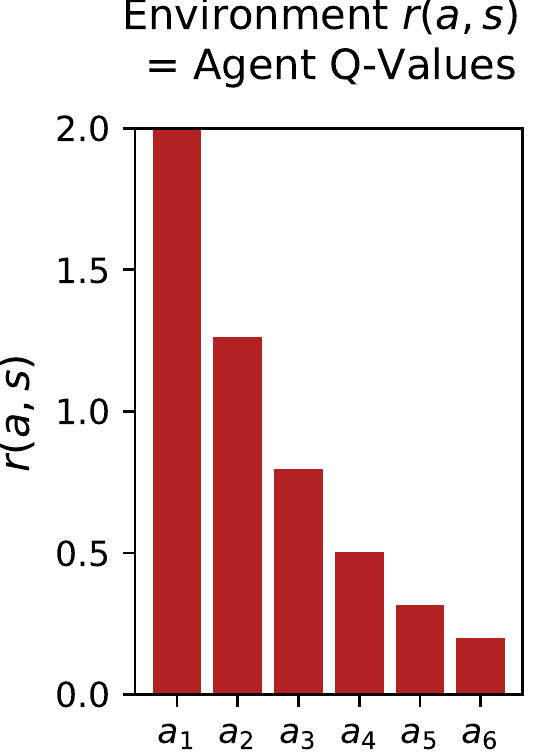} } & \rulesep
\includegraphics[width=\WidthS, height=\HeightS, trim=0 0 0 0, clip]{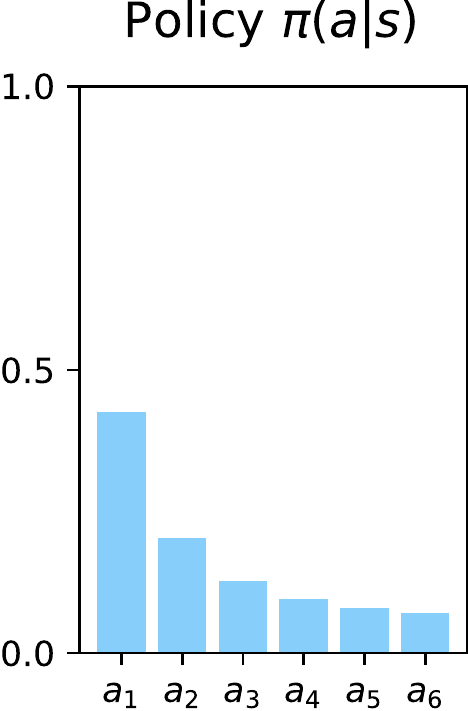} & 
\includegraphics[width=\WidthM, height=\HeightM, trim=0 0 0 0, clip]{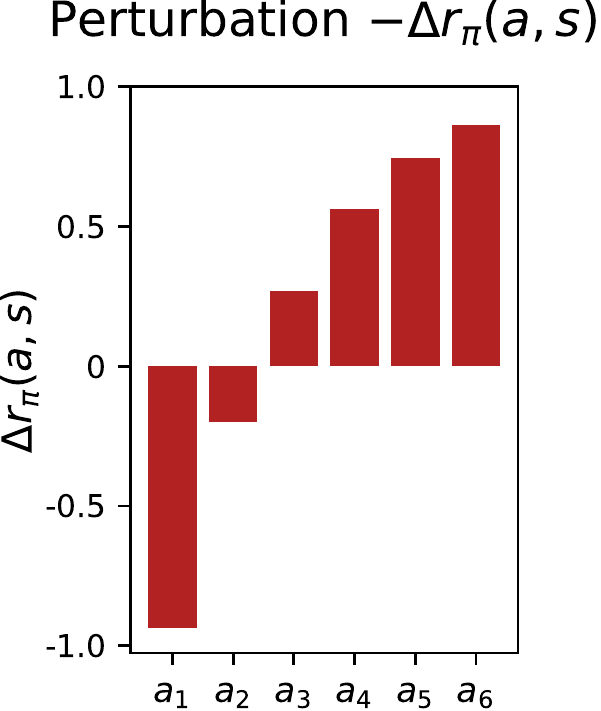} & 
\includegraphics[width=\WidthM, height=\HeightM, trim=0 0 0 0, clip]{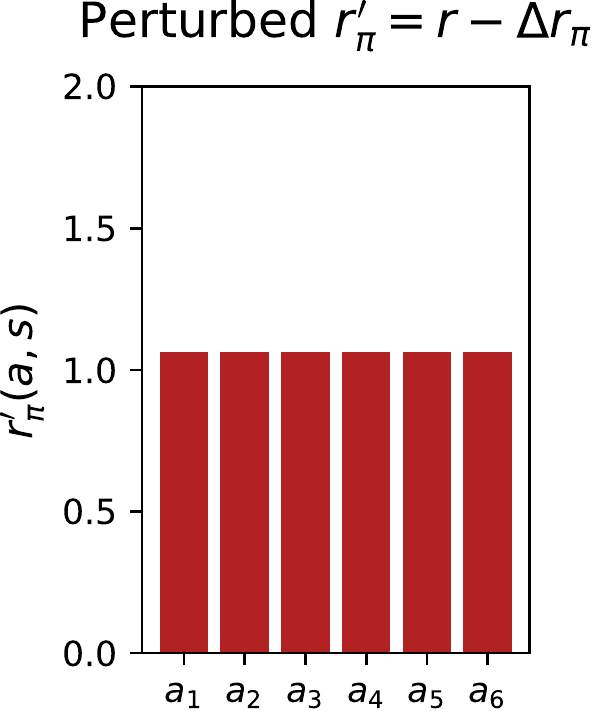} \\[.25ex]
 & \multicolumn{3}{c}{\text{(ii) } \beta = 10} \\[.25ex]
&\rulesep 
\includegraphics[width=\WidthS, height=\HeightS, trim=0 0 0 0, clip]{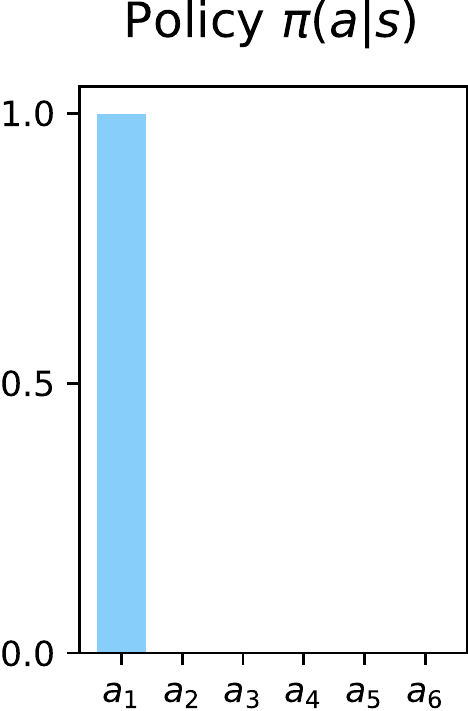} & 
\includegraphics[width=\WidthM, height=\HeightM, trim=0 0 0 0, clip]{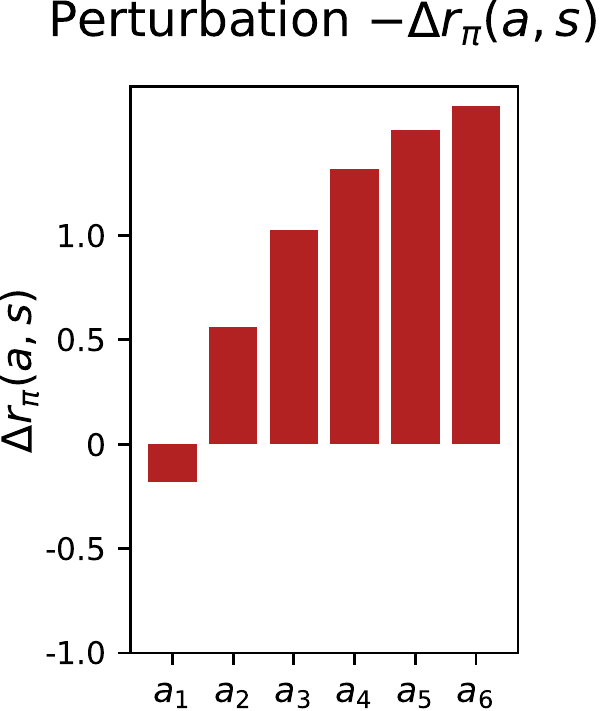} & 
\includegraphics[width=\WidthM, height=\HeightM, trim=0 0 0 0, clip]{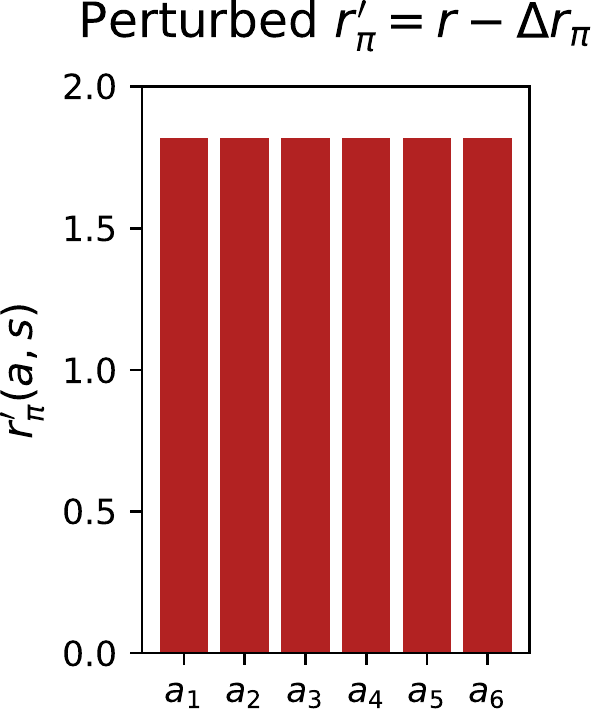} 
\end{array}
\]
\end{minipage}\hspace{.02\textwidth} 
}%
\subcaptionbox{Suboptimal Policy \label{fig:perturb_opt_a}}{
\begin{minipage}{.485\textwidth}
\[\arraycolsep=.005\textwidth
\begin{array}{cccc} 
 & \multicolumn{3}{c}{\text{(i) } \beta = 1.0} \\[.25ex]
 \multirow{2}{*}[.075cm]{
\includegraphics[width=\WidthM, height=\HeightSS, trim=0 0 0 0, clip]{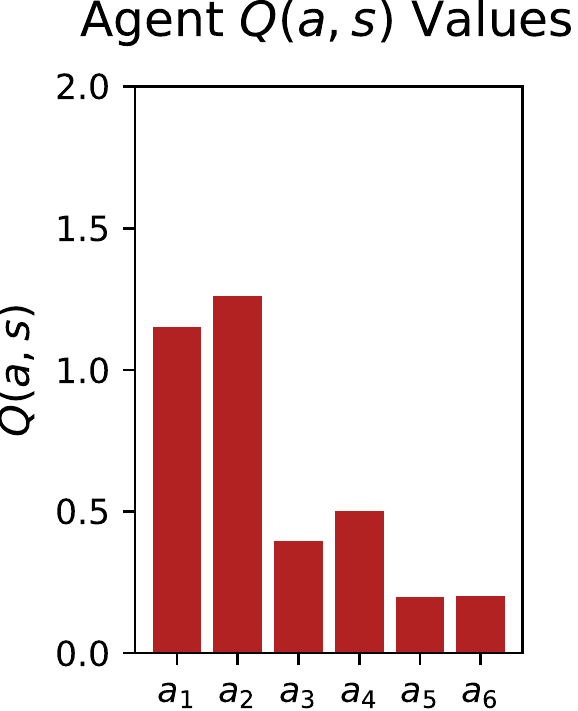}}  & \rulesep
\includegraphics[width=\WidthS, height=\HeightS, trim=0 0 0 0, clip]{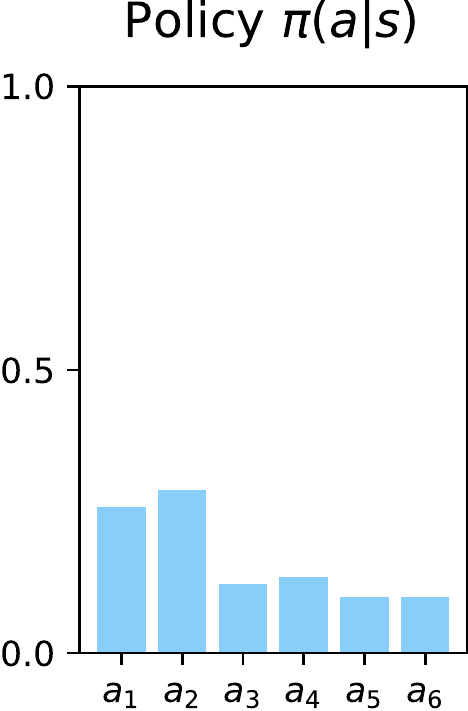} & 
\includegraphics[width=\WidthM, height=\HeightM, trim=0 0 0 0, clip]{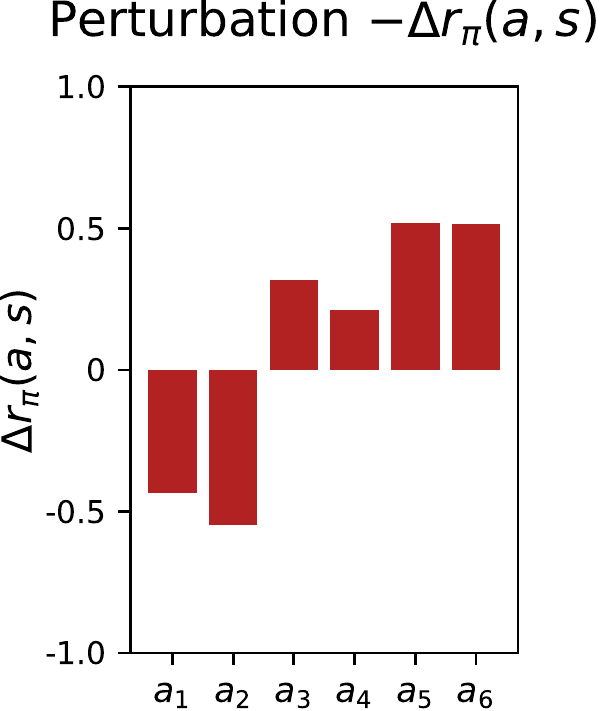} & 
\includegraphics[width=\WidthM, height=\HeightM, trim=0 0 0 0, clip]{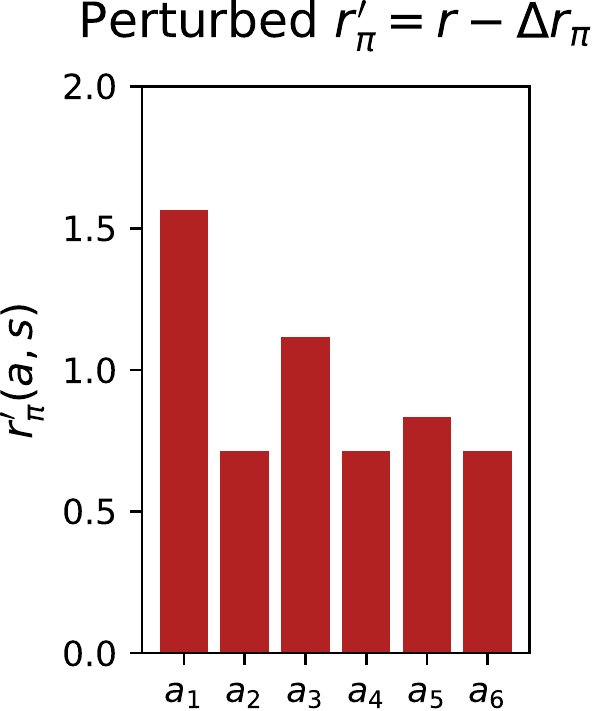} \\[.25ex]
 &  \multicolumn{3}{c}{\text{(ii) } \beta = 10} \\[.25ex]
 &\rulesep
\includegraphics[width=\WidthS, height=\HeightS, trim=0 0 0 0, clip]{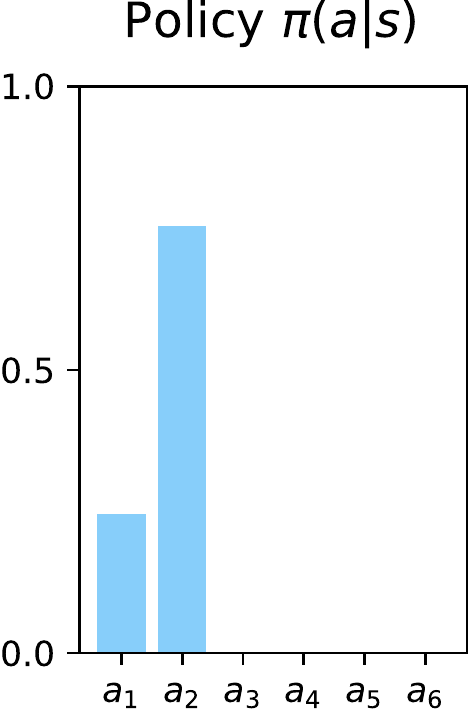} & 
\includegraphics[width=\WidthM, height=\HeightM, trim=0 0 0 0, clip]{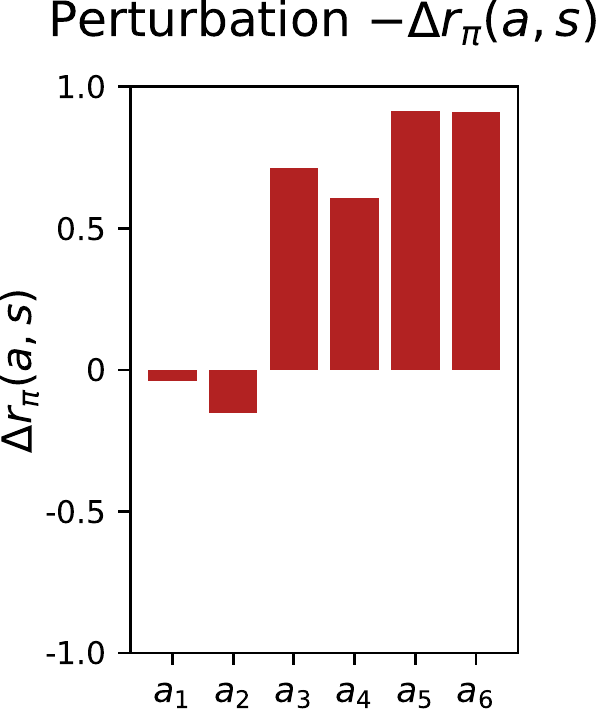} & 
\includegraphics[width=\WidthM, height=\HeightM, trim=0 0 0 0, clip]{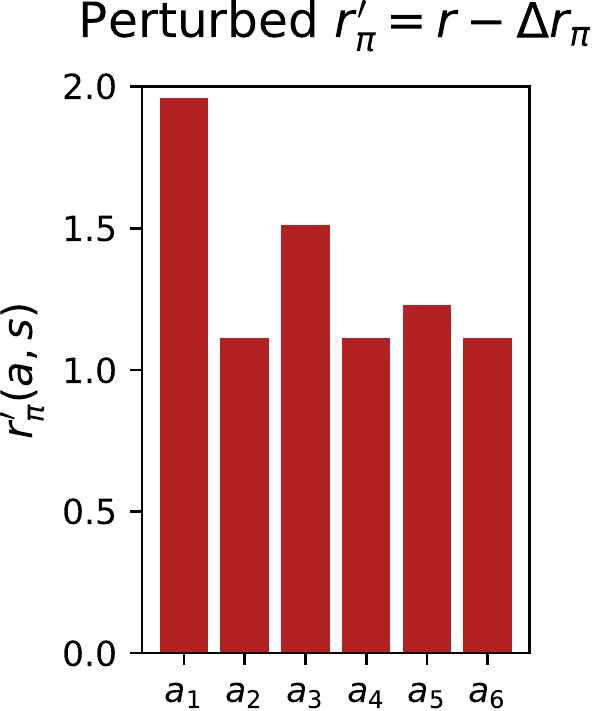} \\[.25ex]
\end{array}
\]
\end{minipage}
}
\vspace*{-0.15cm}
\caption{ \textbf{Single-Step Reward Perturbations} for \textsc{kl}~regularization to uniform reference policy $\pi_0(a|s)$.  $Q$-values in left columns are used for each $\beta$ in columns 2-4.  We report the worst-case $-\proptpi(a,s)$ (\cref{eq:optimal_perturbations}), so negative values correspond to reward decreases. 
\textbf{(a)} Optimal policy ($Q_*(a,s) = r(a,s)$) using the environment reward, where the perturbed $r^{\prime}(a,s) = c \, \, \forall a$ reflects the indifference condition. \textbf{(b)} Suboptimal policy where indifference does not hold.  In all cases, actions with high $Q(a,s)$ are robust to reward decreases. }\label{fig:perturb_opt}
\vspace*{-0.1cm}
\end{figure*}

%% file: sections/pyco_fig.tex
\begin{figure*}[t]
\vspace*{-1.2cm}
\begin{subfigure}{.2\textwidth}
\includegraphics[width=\textwidth]{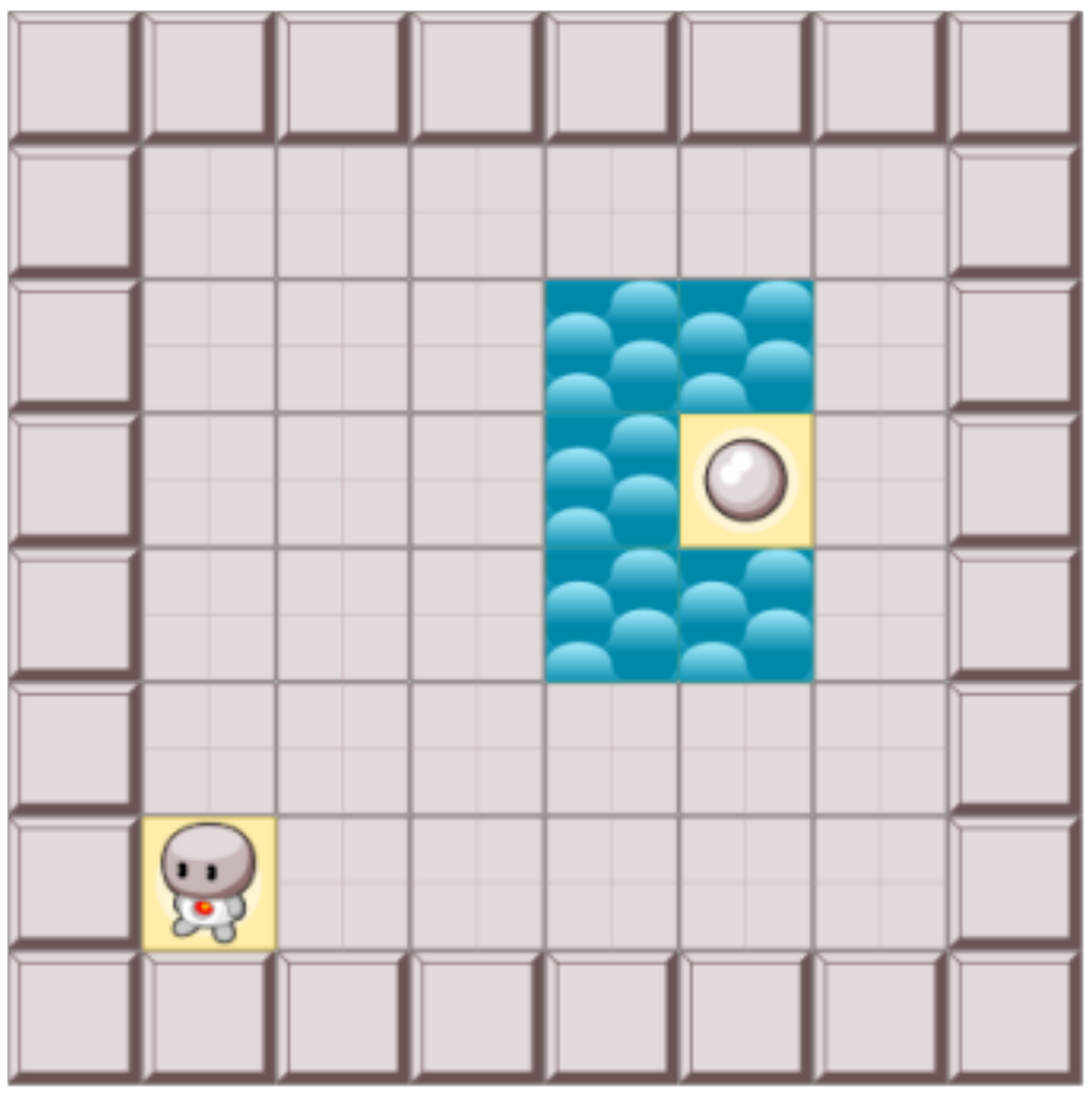} 
\caption{Environment \\
\small (uniform $\nu_0(s)$, \\
$r(a,s) = -1$ for water, \\
$r(a,s) = 5$ for goal) \label{fig:pycoa}}
\end{subfigure}
\begin{subfigure}{.26\textwidth}
\includegraphics[width=\textwidth]{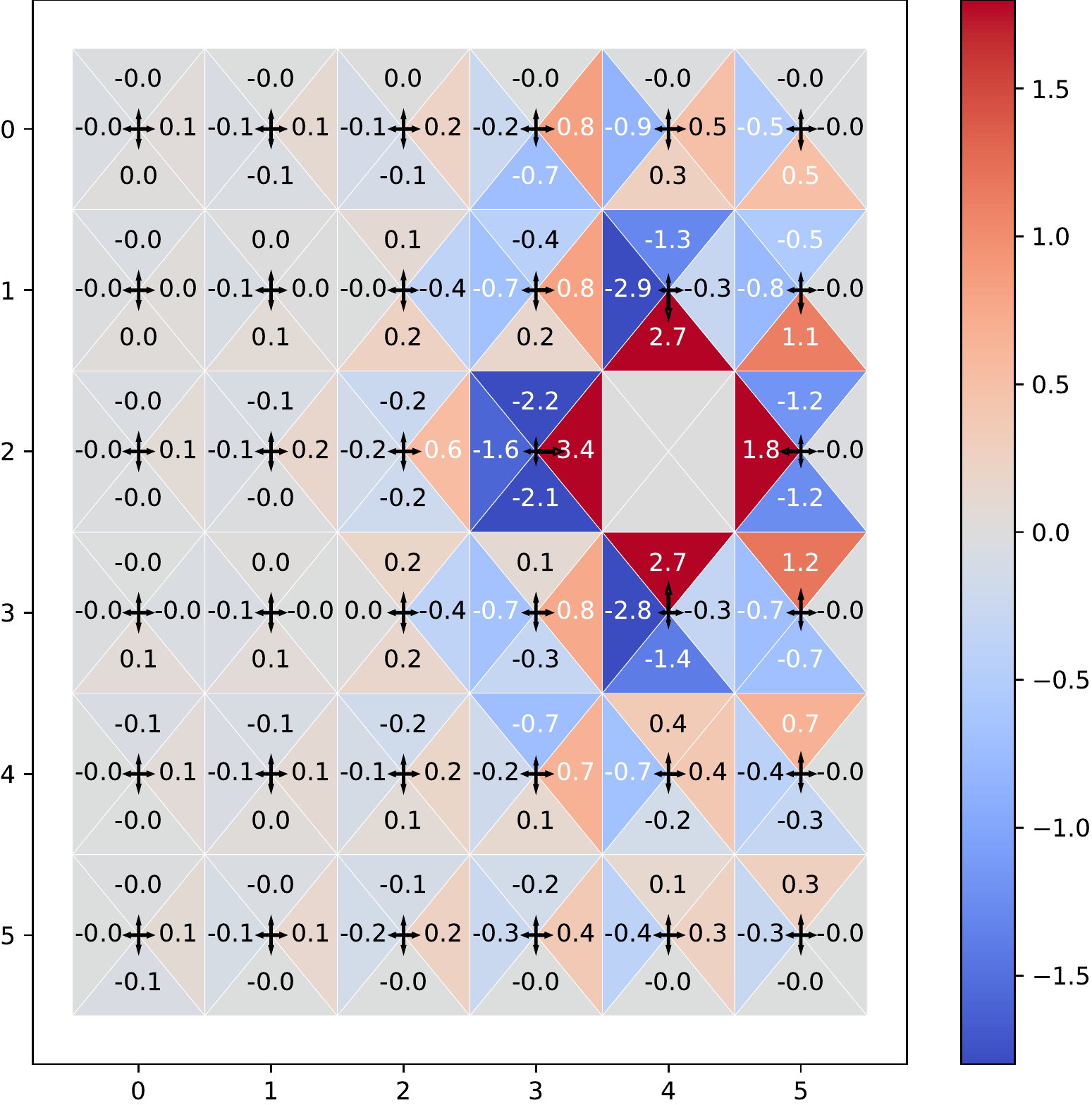}
\caption{$\beta = 0.2$ (High Reg.)}\label{fig:pyco_beta02m}
\end{subfigure}
\begin{subfigure}{.26\textwidth}
\includegraphics[width=\textwidth]{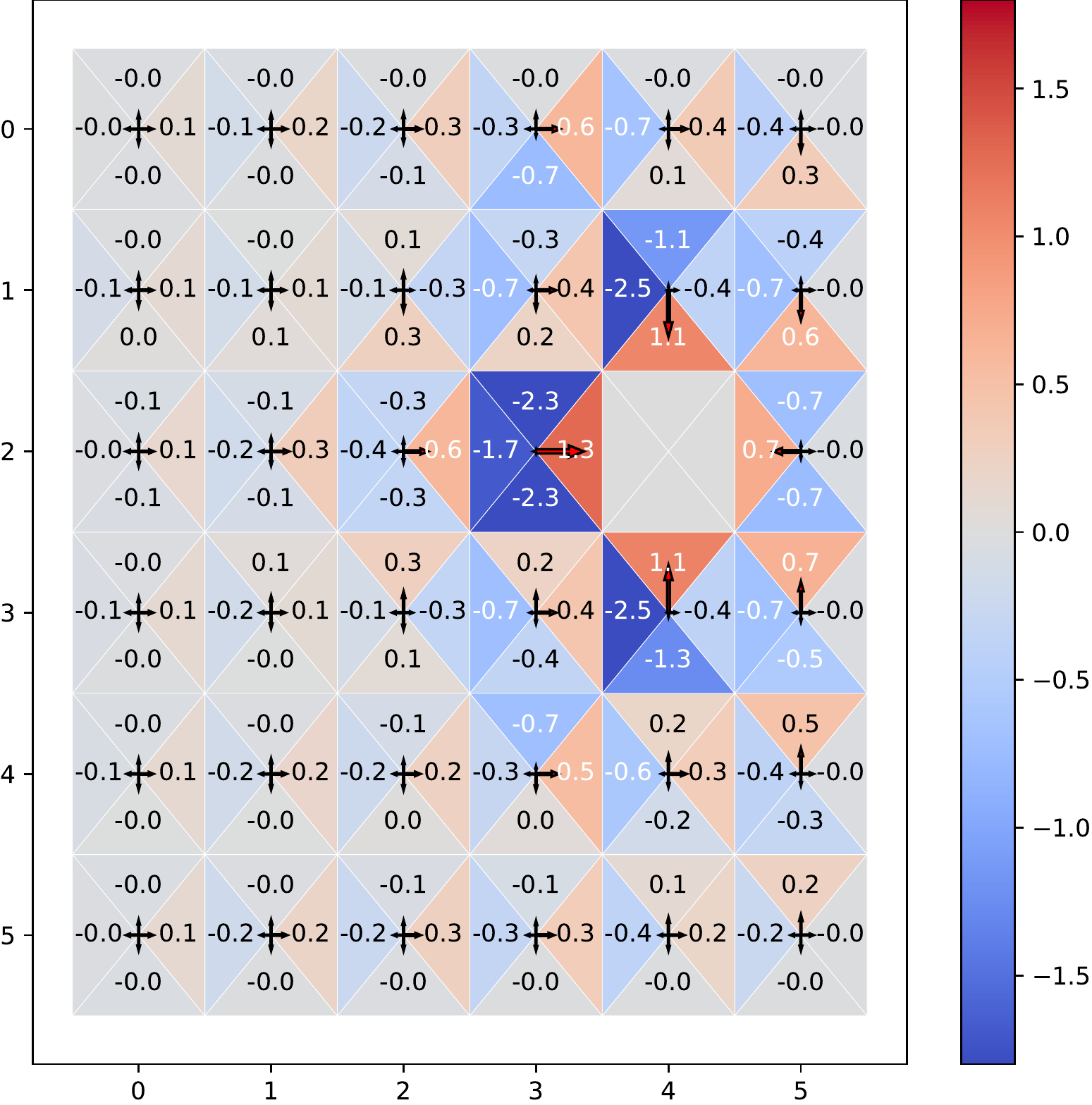}
\caption{$\beta = 1$}\label{fig:pyco_beta1m}
\end{subfigure}
\begin{subfigure}{.26\textwidth}
\includegraphics[width=\textwidth]{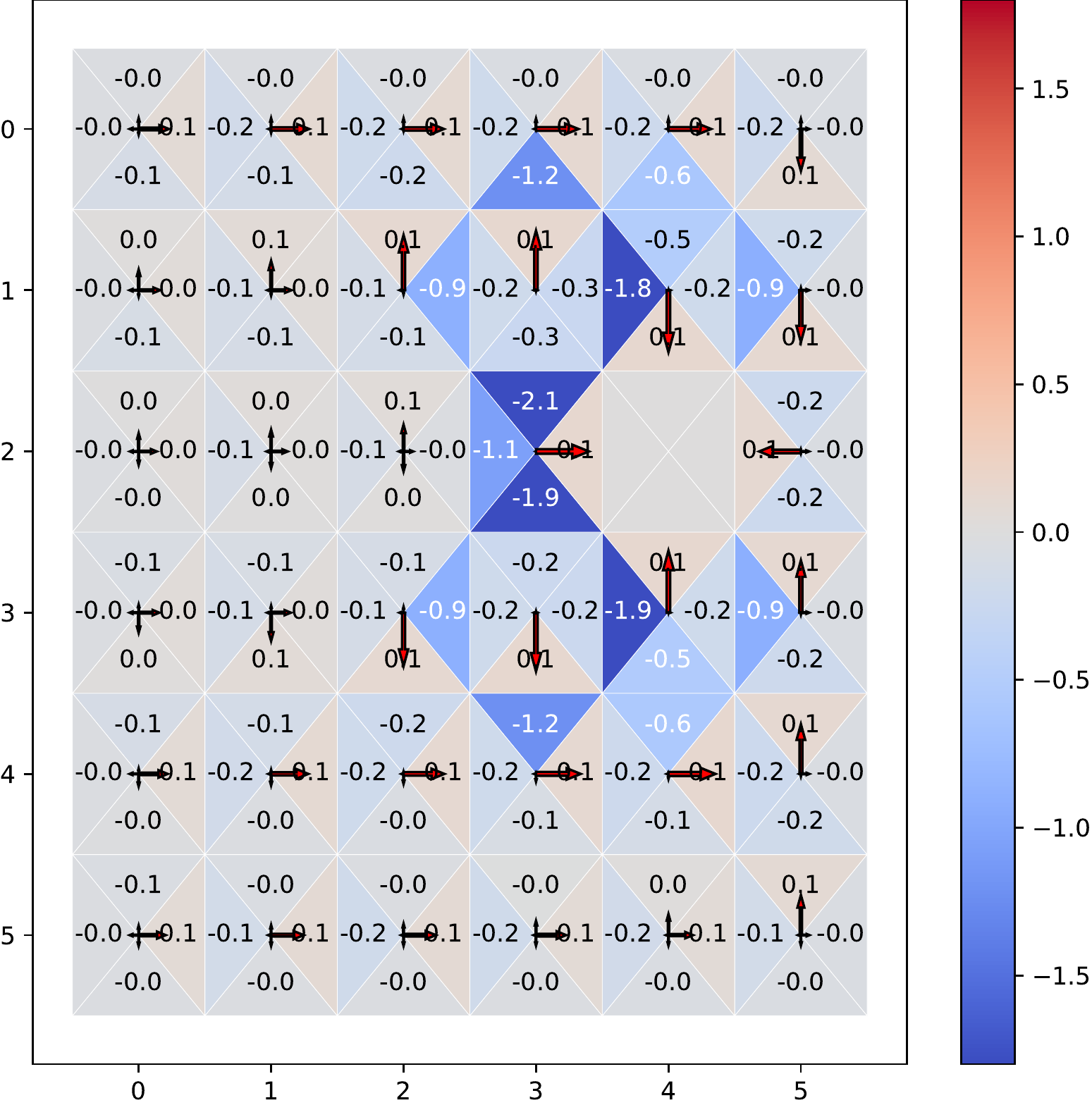}
\caption{$\beta = 10$ (Low Reg.)}\label{fig:pyco_beta10m}
\end{subfigure} 
\caption{\label{fig:pyco} \textbf{Grid-World Reward Perturbations.} \textbf{(a)} Sequential task. \textbf{(b)-(d)} Policies trained with Shannon entropy regularization of different strength. Action probabilities are indicated via relative arrow lengths; the goal-state is gray and without annotations. Colors indicate worst-case adversarial reward perturbations $\prpi(a,s) = \frac{1}{\beta}\log \frac{\pi(a|s)}{ \pi_0(a|s)}$ for each state and action (up, down, left, right) against which the policy is 
robust. 
Red (or positive $\prpi(a,s)$) implies that the policy is robust to reward decreases (up to the value shown) imposed by the adversary. These decreases are balanced by adversarial reward increases (blue) for other actions in the same state. 
We confirm the optimality of each policy using path consistency in \myapp{additional_results} \cref{fig:pyco_indifference}.}
\vspace*{-.1cm}
\end{figure*}


%% file: sections/new/4_entropy_related.tex
\headerv
\section{Discussion}\label{sec:entropy_vs_div} \label{sec:discussion}
\headerv
Our analysis in \mysec{adversarial_perturb}
unifies and extends several previous works analyzing the reward robustness of regularized policies \citep{ortega2014adversarial, eysenbach2021maximum, husain2021regularized}, as summarized in \mytab{novelty}.   \revhl{We highlight differences in the analysis of 
entropy-regularized policies 
in \mysec{entropy_vs_div}, and provide additional discussion of the closely-related work of \citet{derman2021twice} in \mysec{related}.
}

\headerv
 \subsection{Comparison with Entropy Regularization }\label{sec:entropy_vs_div}
 \headerv
As argued in \mysec{adversarial_perturb}, the worst-case reward perturbations preserve the value of the regularized objective function.  Thus, we should expect our robustness conclusions to depend on the exact form of the regularizer.



When regularizing with the Tsallis or Shannon ($\alpha = 1$) entropy, the worst-case reward perturbations become 
\small
\begin{align}
    \hspace*{-.2cm} \prpi(a,s) = \frac{1}{\beta} \log_{\alpha} \pi(a|s) + \frac{1}{\beta} \frac{1}{\alpha} \Big(1 - 
    \sum_{a \in \mathcal{A}} \pi(a|s)^{\alpha} \Big) \,. \label{eq:ptwise_perturb_ent}
\end{align}
\normalsize
See \myapp{entropy_nonnegative}, we also show that for $0 < \alpha \leq 1$,
these perturbations cannot decrease the reward, with $-\prpi(a,s) \geq 0$ and $\mr_{\pi}(a,s) \geq r(a,s)$.
In the rest of this section, we
argue that this property leads to several unsatisfying conclusions in previous work \citep{lee2019tsallis, eysenbach2021maximum}, which are resolved by using the \textsc{kl} and $\alpha$-divergence for analysis instead of the corresponding entropic quantities.\footnote{Entropy regularization corresponds to divergence regularization with the uniform reference distribution $\pi_0(a|s)$.}

First, 
this means that a Shannon entropy-regularized policy is only
`robust' to \textit{increases} in the reward function.   However, for useful generalization, we might hope that a policy still performs well when the reward function decreases in at least some states. 
 Including the reference distribution via divergence regularization resolves this issue,
 and we observe in \myfig{feasible_set_main} and \myfig{pyco} that the adversary 
 chooses reward decreases in some actions and increases in others.  For example, for the \textsc{kl} divergence, $\prpiopt(a,s) = \frac{1}{\beta}\log \frac{\piopt(a|s)}{\pi_0(a|s)} = Q_*(a,s) - V_*(s)$ implies robustness to reward decreases when $\piopt(a|s) > \pi_0(a|s)$ or $Q_*(a,s) > V_*(s)$.
 


Similarly, \citet{lee2019tsallis} note that for any $\alpha$,
\small
\begin{align*}
    \frac{1}{\beta}\Omega^{*(H_{\alpha})}_{\beta}(\discussiondual) = 
    \max \limits_{\pi \in \Delta^{|\mathcal{A}|}} \langle \pi, \discussiondual \rangle + \frac{1}{\beta} H_{\alpha}(\pi) \geq \discussiondual(a_{\max}, s) 
\end{align*}
\normalsize
\normalsize
where $a_{\max} = \argmax_a \discussiondual(a,s)$ and the Tsallis entropy $H_{\alpha}(\pi)$ equals the Shannon entropy for $\alpha = 1$. 
This soft value aggregation yields a result that is \textit{larger} than any particular $Q$-value. 
By contrast, for the $\alpha$-divergence, we show in \myapp{qmax} that for fixed $\beta$ and $\alpha > 0$,
 \small
\begin{equation}
\discussiondual(a_{\max},s) + \frac{1}{\beta}\frac{1}{\alpha} \log_{2-\alpha} \pi(a_{\max}| s) \leq \alphaconj(\discussiondual) \leq  \discussiondual(a_{\max},s) 
.
\end{equation}
\normalsize
This provides a more natural interpretation of the Bellman optimality operator $V(s) \leftarrow \alphaconj(Q)$ as a soft maximum operation.   As a function of $\beta$, we see in \myapp{alpha_results} and \ref{app:qmax} that the conjugate ranges between $\mathbb{E}_{\pi_0}[\discussiondual(a_{\max},s)] \leq \alphaconj(Q) \leq \discussiondual(a_{\max},s)$.

Finally, using entropy instead of divergence regularization also affects interpretations of the feasible set.  
\citet{eysenbach2021maximum} consider the same constraint as in \cref{eq:kl_feasible}, but without the reference $\pi_0(a|s)$
\small
\begin{align}
    \sum \limits_{a \in \mathcal{A}} \expof{\beta \cdot \pr(a,s)} \leq 1  \, \quad \forall s \in \mathcal{S} \, . \label{eq:ent_constraint}
\end{align}
\normalsize
This constraint suggests that the original reward function ($\pr = 0$) is not feasible for the adversary.  
More surprisingly, \citet{eysenbach2021maximum} App. A8 argues that \textit{increasing} regularization strength (with lower $\beta$) may lead to \textit{less} robust policies based on the constraint in \cref{eq:ent_constraint}.    In \myapp{eysenbach}, we discuss how including $\pi_0(a|s)$ in the constraint via divergence regularization (\myprop{feasible}) avoids this conclusion.   As expected, 
\cref{fig:feasible_set_main} shows that increasing regularization strength leads to more robust policies.

\headerv
\subsection{Related Algorithms}\label{sec:related}
\headerv
Several recent works provide algorithmic insights which build upon convex duality 
and complement or extend our analysis.  
\revhl{\citet{derman2021twice} 
derive practical iterative algorithms based on a general equivalence between robustness and regularization, which can be used to enforce
robustness to \textit{both} reward perturbations (through policy regularization) and changes in environment dynamics (through value regularization).
 For policy regularization, \citet{derman2021twice} translate
  the specification of a desired robust set 
  into
  a regularizer using the convex conjugate of the set indicator function.   In particular, \citet{derman2021twice} associate \textsc{kl} divergence or (scaled) Tsallis entropy policy regularization with the robust set $\mathcal{R}^{\Delta}_{\pi} \coloneqq \{ 
\prrob \, | \,
\prrob(a,s) \in [\frac{1}{\beta} \frac{1}{\alpha} \log_{\alpha} \frac{\pi(a|s)}{\pi_0(a|s)}, \infty) \, \, \forall \,  (a,s) \in  \mathcal{A} \times \mathcal{S}\}$.
Our analysis proceeds in the opposite direction, from regularization to robustness, 
using the conjugate of the divergence.    While the worst-case perturbations 
result in the same modified objective, 
our approach
yields a larger robust set with qualitatively different shape (see \myfig{related}).}

\citet{zahavy2021reward} analyze a general `meta-algorithm' which alternates between updates of the occupancy measure $\mu(a,s)$ and modified reward $\mr(a,s)$ in online fashion.
This approach highlights the fact that the modified reward $\mr_{\pi}$ or worst-case perturbations $\prpi$ change as the policy or occupancy measure is optimized.  
The results of \citet{zahavy2021reward} and \citet{husain2021regularized} hold for general convex \gls{MDP}s, which encompass common exploration and imitation learning objectives beyond the policy regularization setting
we consider.

As discussed in \mysec{optimal_policy}, path consistency conditions have been used to derive practical learning objectives in \citep{nachum2017bridging, chow2018path}.   These algorithms might be extended to general $\alpha$-divergence regularization via \cref{eq:path}, which involves an arbitrary reference policy $\pi_0(a|s)$ that can be learned adaptively as in \citep{teh2017distral, grau2018soft}.

Finally, previous work has used dual optimizations similar to \cref{eq:dual_reg} to derive alternative Bellman error losses \citep{dai2018sbeed, belousov2019entropic, nachum2020reinforcement, basserrano2021logistic}, highlighting how convex duality can be used to bridge between policy regularization and Bellman error aggregation \citep{belousov2019entropic, husain2021regularized}.

%% file: sections/new/5_conclusion.tex
\headerv
\section{Conclusion}\label{sec:conclusion}
\headerv

In this work, we analyzed the robustness of convex-regularized \gls{RL} policies to worst-case perturbations of the reward function,
which implies generalization to adversarially chosen reward functions from within a particular robust set. 
We have characterized this robust set of reward functions for \textsc{kl} and $\alpha$-divergence regularization, provided a unified discussion of existing works on reward robustness, and clarified apparent differences in robustness arising from entropy versus divergence regularization.
{Our advantage function interpretation of the worst-case reward perturbations provides a complementary perspective on how $Q$-values appear as dual variables in convex programming forms of regularized \gls{MDP}s.  Compared to a deterministic, unregularized policy, 
a stochastic, regularized policy 
places probability mass on a wider set of actions and 
requires state-action value adjustments via the advantage function or adversarial reward perturbations. }
Conversely, a regularized agent, acting based on given $Q$-value estimates, implicitly hedges against the anticipated perturbations of an appropriate adversary.

%% file: appendix/new_new_app/000_conj_optimality.tex
\section{Implications of Conjugate Duality Optimality Conditions}\label{app:implications}
In this section, we show several closely-related results which are derived from the conjugate optimality conditions.   We provide additional commentary in later Appendix sections which more closely follow the sequence of the main text.

First, recall from \cref{sec:conj_intro} the definition of the conjugate optimizations for functions over $\mathcal{X} \coloneqq \mathcal{A} \times \mathcal{S}$.  We restrict $\mu \in \mudomainfn$ to be a nonnegative function over $\mathcal{X}$, so that
\begin{equation}
    \begin{aligned}
        \frac{1}{\beta} \Omega^{*}(\dualvar) = \sup_{\primalvar \in \mudomainfn} \big \langle \primalvar, \dualvar \big \rangle - \frac{1}{\beta} \Omega(\primalvar), 
    \end{aligned} \hspace*{.05\textwidth}
     \begin{aligned}
    \frac{1}{\beta} \Omega(\primalvar) = \sup_{\dualvar \in \rprimedomain} \big \langle \primalvar, \dualvar \big \rangle - \frac{1}{\beta} \Omega^{*}(\dualvar), 
 \end{aligned}
 \label{eq:conj_app}
\end{equation}
and the implied optimality conditions are
\ifx\omitindices\undefined
\begin{equation}
    \begin{aligned}
    \dualvar(a,s) = \frac{1}{\beta} \nabla_{\primalvar} \Omega(\primalvar) = \bigg( \nabla_{\dualvar} \frac{1}{\beta} \Omega^{*} \bigg)^{-1}(\primalvar) 
    \end{aligned} \hspace*{.05\textwidth}
     \begin{aligned}
        \primalvar(a,s) = \frac{1}{\beta} \nabla_{\dualvar} \Omega^*(\dualvar) = \bigg( \nabla_{\primalvar} \frac{1}{\beta} \Omega \bigg)^{-1}(\dualvar) \, .
 \end{aligned}
\label{eq:conj_optimality_app}
\end{equation}
\else
\begin{equation}
    \begin{aligned}
    \dualvar = \frac{1}{\beta} \nabla_{\primalvar} \Omega(\primalvar) = \bigg( \nabla_{\dualvar} \frac{1}{\beta} \Omega^{*} \bigg)^{-1}(\primalvar) 
    \end{aligned} \hspace*{.05\textwidth}
     \begin{aligned}
        \primalvar = \frac{1}{\beta} \nabla_{\dualvar} \Omega^*(\dualvar) = \bigg( \nabla_{\primalvar} \frac{1}{\beta} \Omega \bigg)^{-1}(\dualvar) \, .
 \end{aligned}
\label{eq:conj_optimality_app}
\end{equation}
\fi

\subsection{Proof of \myprop{optimal_perturbations} : Policy Form Worst-Case Reward Perturbations }\label{app:prop2}
\optimalperturbations*
\begin{proof}
The reward perturbations are defined via conjugate optimization for $\Omega(\primalvar)$ in \cref{eq:conj_app}, 
where $\pr \in \rprimedomain$.
The proposition follows directly from the optimality condition in \cref{eq:conj_optimality_app}, and we focus on the \ifx\omitindices\undefined
$\dualvar(a,s) = \frac{1}{\beta} \nabla_{\primalvar} \Omega(\primalvar)$ 
\else
$\dualvar = \frac{1}{\beta} \nabla_{\primalvar} \Omega(\primalvar)$ 
\fi
condition for convenience.
\end{proof}
In \myapp{conjugates}, we derive the explicit forms for the worst-case reward perturbations for \textsc{kl} and $\alpha$-divergence regularization from \mysec{any_policy} of the main text.  See \myapp{conjugates} \cref{table:table_refs} for references to particular derivations.

Note that we do not consider further 
constraints on $\mu$ in the conjugate optimization.    Instead, we view the Bellman flow constraints $\mu(a,s) \in \mathcal{M}$ (and normalization constraint $\mu(a,s) \in \Delta^{|\mathcal{A}| \times |\mathcal{S}|}$) as arising from the overall (regularized) \gls{MDP} optimization in \cref{eq:primal_lp} or (\ref{eq:primal_reg}), as we discuss in the next subsection.

\subsection{Optimal Policy in a Regularized MDP}\label{app:optimal_policy}

In \cref{lemma:flow} below, we show that the Bellman flow constraints \cref{eq:primal_lp}, which are enforced by the optimal Lagrange multipliers $V_*(s)$, ensure that the optimal $\mu_*(a,s)$ is normalized.  This suggests that an explicit normalization constraint is not required.    In \myprop{optimal_policy}, we then proceed to derive the optimal policy in a regularized \textsc{mdp} using the conjugate optimality conditions in \cref{eq:conj_optimality_app}.

\begin{restatable}[Flow Constraints Ensure Normalization]{lemma}{flow}
\label{lemma:flow} 
Assume that the initial state distribution $\nu_0(s)$ and transition dynamics $P(s\tick|a,s)$  are normalized, with $\sum_s \nu_0(s) = 1$ and $\sum_{s\tick} P(s\tick|a,s) = 1$.   If a state-occupancy measure satisfies the Bellman flow constraints $\mu(a,s) \in \mathcal{M}$, then it is a normalized distribution $\mu(a,s) \in \Delta^{|\mathcal{A}| \times |\mathcal{S}|}$.
\end{restatable}
\begin{proof}
Starting from the Bellman flow constraints $\sum_{a} \mu(a,s) = (1-\gamma) \nu_0(s) + \gamma \sum_{a\tick,s\tick} P(s|a\tick,s\tick) \mu(a\tick, s\tick)$,
we consider taking the summation over states $s \in \mathcal{S}$,
\scriptsize
\begin{align}
       \sum \limits_{a, s} \mu(a,s) &= (1-\gamma) \sum \limits_{s} \nu_0(s) + \gamma \sum \limits_{s} \sum \limits_{a\tick,s\tick} P(s|a\tick,s\tick) \mu(a\tick, s\tick) \overset{(1)}{=} (1-\gamma) + \gamma \sum \limits_{s} P(s|a\tick,s\tick) \sum \limits_{a\tick,s\tick}  \mu(a\tick, s\tick) \overset{(2)}{=} (1-\gamma) + \gamma \sum \limits_{a\tick,s\tick}  \mu(a\tick, s\tick) \nonumber
\end{align}
\normalsize
where $(1)$ uses the normalization assumption on $\nu_0(s)$ and the distributive law, and $(2)$ uses the normalization assumption on $P(s|a\tick,s\tick)$.  Finally, we rearrange the first and last equality to obtain
\begin{align}
  (1-\gamma)  \sum \limits_{a, s} \mu(a,s)  = (1-\gamma) \qquad \implies \qquad \sum \limits_{a, s} \mu(a,s) = 1
\end{align}
\normalsize
which shows that $\mu(a,s)$ is normalized as a joint distribution over $a \in \mathcal{A},s \in \mathcal{S}$, as desired.
\end{proof}

\begin{restatable}[Optimal Policy in Regularized MDP]{proposition}{optimalpolicy}
\label{prop:optimal_policy} 
Given the optimal value function $V_*(s)$ and Lagrange multipliers $\lambdaopt$, the optimal policy in the regularized \textsc{mdp} is given by
\small
\ifx\omitindices\undefined
\begin{align*}
    \mu_*(a,s) &= \frac{1}{\beta} \nabla \Omega^*\left( r(a,s) + \transitionvstar - V_*(s) + \lambdaopt  \right) = \Big( \nabla_{\primalvar} \frac{1}{\beta} \Omega \Big)^{-1}\Big( r(a,s) + \transitionvstar - V_*(s) + \lambdaopt\Big) .
\end{align*}
\else
\begin{align*}
    \mu_* &= \frac{1}{\beta} \nabla \Omega^*\left( r + \transitionvstarinds - V_* + \lambda_*  \right) = \Big( \nabla_{\primalvar} \frac{1}{\beta} \Omega \Big)^{-1}\Big( r + \transitionvstarinds - V_* + \lambda_* \Big) .
\end{align*}
\fi
\normalsize
This matches the conjugate conditions in \cref{eq:conj_optimality_app} using the arguments $\pr(a,s) \leftarrow r(a,s) + \transitionvstar - V_*(s) + \lambdaopt$.  
\end{restatable}
\begin{proof}
In \mysec{mdp_regularized}, we moved from the regularized primal optimization (\cref{eq:primal_reg}) to the dual optimization (\cref{eq:dual_reg}) via the regularized Lagrangian
\small
\ifx\omitindices\undefined
\begin{align*}
  \min \limits_{V(s), \lambdaplus} \max \limits_{\mu(a,s)}  \, 
     (1-\gamma) \big\langle \nu_0(s), V(s) \big  \rangle + \langle \mu(a,s), r(a,s) + \gamma \transitionv -V(s) + \lambdaplus \rangle - \frac{1}{\beta} \Omega_{\pi_0}(\mu) 
\end{align*}
\else
\normalsize
\begin{align*}
  \min \limits_{V, \lambda} \max \limits_{\mu}  \, 
     (1-\gamma) \big\langle \nu_0, V \big  \rangle + \langle \mu, r + \gamma \transitionvinds -V + \lambda \rangle - \frac{1}{\beta} \Omega_{\pi_0}(\mu) 
\end{align*}
\fi
\normalsize
Note that the Lagrange multipliers $\lambdaplus$ enforce $\mu(a,s) \geq 0$ while $V(s)$ enforces the flow constraints and thus, by \cref{lemma:flow}, normalization of $\mu(a,s)$.   We recognized the final two terms as a conjugate optimization
\ifx\omitindices\undefined
\scriptsize
\begin{align}
   \hspace*{-.25cm} \frac{1}{\beta} \Omega^*_{\pi_0,\beta}\Big( r(a,s) + \gamma \transitionv -V(s) + \lambdaplus \Big) = \max \limits_{\mu(a,s)} \blangle \mu(a,s), r(a,s) + \gamma \transitionv -V(s) + \lambdaplus \brangle - \frac{1}{\beta} \Omega_{\pi_0}(\mu) \label{eq:conj_opt_mdp}
\end{align}
\normalsize
\else
\begin{align}
   \hspace*{-.25cm} \frac{1}{\beta} \Omega^*_{\pi_0,\beta}\Big( r + \gamma \transitionvinds -V + \lambda \Big) = \max \limits_{\mu} \blangle \mu, r + \gamma \transitionvinds -V + \lambda \brangle - \frac{1}{\beta} \Omega_{\pi_0}(\mu) \label{eq:conj_opt_mdp}
\end{align}
\fi
to yield a dual optimization over $V(s)$ and $\lambdaplus$ only in \cref{eq:dual_reg}.   After solving the dual optimization for the optimal $V_*(s)$, $\lambda_*(a,s)$, we can recover the optimal policy in the \textsc{mdp} using the optimizing argument of \cref{eq:conj_opt_mdp}.   Differentiating \cref{eq:conj_opt_mdp} and solving for $\mu$ yields 
\ifx\omitindices\undefined
$\nabla_{\mu} \frac{1}{\beta} \Omega(\mu) = r(a,s) + \gamma \transitionv -V(s) + \lambdaplus$ 
\else
$\nabla_{\mu} \frac{1}{\beta} \Omega(\mu) = r + \gamma \transitionvinds -V + \lambda$ 
\fi
which we invert to obtain \myprop{optimal_policy}.   The other equality follows from the conjugate optimality conditions in \cref{eq:conj_optimality_app}. 
\end{proof}

For $\alpha$-divergence regularization, the optimal policy or state-action occupancy is given by the `optimizing argument' column of \cref{table:conj_table}, up to reparameterization of $\pr(a,s) \leftarrow r(a,s) + \transitionvstar - V_*(s) + \lambdaopt$ as the dual variable.
In this case, note that the argument to the conjugate function accounts for the flow and nonnegativity constraints via $V_*(s)$ and $\lambdaopt$.
In particular, we have
\begin{align}
 &\text{Policy Reg., \myapp{conj3} \cref{eq:opt_pi_alpha}} \nonumber \\
   &\phantom{\text{policy reg}} \mu_*(a,s) = \mu(s) \pi_0(a|s) \exp_{\alpha} \left\{ \beta \cdot \left( r(a,s) + \gamma \transitionvstar - V_*(s) + \lambdaplus - \psipiopt \right) \right\}  \,\,\, \label{eq:opt_policy_app} \\[1.25ex]
 &\text{Occupa}\text{ncy  Reg., \myapp{conj4} \cref{eq:opt_mu_alpha}}& \nonumber    \\ 
    &\phantom{\text{policy reg}} \mu_*(a,s) = \mu_0(a,s) \exp_{\alpha} \left\{\beta \cdot \left( r(a,s) + \gamma \transitionvstar - V_*(s) + \lambdaplus \right) \right\}  
   \label{eq:opt_policy_mu_app}
\end{align}
\normalsize
where $\psipiopt = \frac{1}{\beta} \frac{1}{\alpha}( 1 - \sum_a \pi_0(a|s)^{1-\alpha} \pi_*(a|s)^{\alpha})$ appears from differentiating $\nabla \frac{1}{\beta} \Omega_{\pi_0}(\mu)$ as in \cref{eq:pr_normalizer} or \myapp{conj3}.   
This means that the optimal policy is only available in self-consistent fashion, with the normalization constant inside the $\exp_{\alpha}$,
which can complicate practical applications \citep{lee2019tsallis, chow2018path}.


\subsection{Proof of \myprop{advantage}:  Policy Form Worst-Case Perturbations match Value Form at Optimality} \label{app:advantage_pf}
The substitution $\pr(a,s) \leftarrow r(a,s) + \transitionvstar - V_*(s) + \lambdaopt$ above already anticipates the result in \myprop{advantage}, which links the reward perturbations for the optimal policy $\pr_{\pi_{*}}$ or state-action occupancy $\pr_{\mu_{*}}$ to the advantage function.   See the proof of \mythm{husain} in \myapp{husain} for additional context in relation to the value-form reward perturbations $\prv(a,s)$.


\advantage*
\begin{proof}
The result follows by combining \myprop{optimal_perturbations}, which states that 
\ifx\omitindices\undefined
$\prpi(a,s)=\nabla_{\mu} \frac{1}{\beta} \Omega(\mu)$, 
\else
$\prpi =\nabla_{\mu} \frac{1}{\beta} \Omega(\mu)$,
\fi
and \myprop{optimal_policy}, which implies $\nabla_{\mu} \frac{1}{\beta} \Omega(\mu_*) = r(a,s) + \gamma \transitionvstar -V_*(s) + \lambdaopt$ as a condition for optimality of $\{\mu_*(a,s),V_*(s),\lambdaopt\}$.
Thus, for the optimal policy $\piopt(a|s)$ and Lagrange multipliers  $\{V_*(s), \lambdaopt \}$, we have $\pr_{\pi_*}(a,s) =r(a,s) + \gamma \transitionvstar -V_*(s) + \lambdaopt$ and similarly for $\pr_{\mu_*}(a,s)$.

We can confirm this using the expression for the optimal policy in \cref{eq:opt_policy_app} and the worst-case reward perturbations in \mysec{any_policy}.  For example, recalling $\mu_*(a,s) = \mu(s) \pi_*(a|s)$, we can write the $\alpha$-divergence policy regularization case as $    \pr_{\pi_*}(a,s) = \frac{1}{\beta} \log_{\alpha} \frac{\pi_*(a|s)}{\pi_0(a|s)} + \psipiopt =  r(a,s) + \gamma \transitionvstar -V_*(s) + \lambdaplus  \, \pm \cancel{\psipiopt} .$
\end{proof}

\subsection{Path Consistency and KKT Conditions}\label{app:path_consistency}
Finally, note that the \textsc{kkt} optimality conditions \citep{boyd2004convex} include the condition that we have used in the proof of \myprop{advantage}.   At optimality, we have
\begin{align}
r(a,s) +  \gamma \transitionvstar - V_*(s) + \lambdaopt - \nabla \frac{1}{\beta}\Omega(\mu^*) = 0 .  \label{eq:kkt_path0}
\end{align}
This \textsc{kkt} condition is used to derive path consistency objectives in \citet{nachum2017bridging, chow2018path}.

For general $\alpha$-divergence policy regularization, we substitute $\nabla \frac{1}{\beta}\Omega_{\pi_0}^{(\alpha)}(\mu_*) = \pr_{\piopt}(a,s) = \frac{1}{\beta} \log_{\alpha} \frac{\pi_*(a|s)}{\pi_0(a|s)} + \psipiopt$ using \cref{eq:optimal_perturbations} (see \myapp{conj3} for detailed derivations).  This leads to the condition
\begin{align}
r(a,s) +  \gamma \transitionvstar - V_*(s) + \lambdaopt - \frac{1}{\beta} \log_{\alpha} \frac{\pi_*(a|s)}{\pi_0(a|s)} - \psipiopt = 0 , \label{eq:kkt_path0}
\end{align}
which matches \cref{eq:path}.
We compare our $\alpha$-divergence path consistency conditions to previous work in \myapp{indifference_all}.

\subsection{Modified Rewards and Duality Gap for Suboptimal Policies}\label{app:mr_suboptimal}
We can also use the conjugate duality of state-action occupancy measures and reward functions ($r(a,s)$ or $\mr(a,s)$) to express the optimality gap for a suboptimal $\mu(a,s)$.   Consider the regularized primal objective as a (constrained) conjugate optimization,
\ifx\omitindices\undefined
\begin{align}
    \obj_{\Omega, \beta}(r) :=  \frac{1}{\beta} \Omega^{*}(r) &= \max \limits_{\mu(a,s) \in \mathcal{M} } 
   \blangle \mu(a,s), r(a,s) \brangle - \dfrac{1}{\beta}   \Omega(\mu) \\
   &\geq \blangle \mu_{\mr}(a,s), r(a,s) \brangle - \dfrac{1}{\beta} \Omega(\mu_{\mr}) \label{eq:ineq} 
\end{align}
\else
\begin{align}
    \obj_{\Omega, \beta}(r) :=  \frac{1}{\beta} \Omega^{*}(r) &= \max \limits_{\mu \in \mathcal{M} } 
   \blangle \mu, r \brangle - \dfrac{1}{\beta}   \Omega(\mu) \\
   &\geq \blangle \mu_{\mr}, r \brangle - \dfrac{1}{\beta} \Omega(\mu_{\mr}) \label{eq:ineq} 
\end{align}
\fi
where the inequality follows from the fact that any feasible 
\ifx\omitindices\undefined
$\mu_{\mr}(a,s) \in \mathcal{M}$ 
provides a lower bound on the objective.   We use the notation $\mu_{\mr}(a,s)$ to anticipate the fact that, assuming appropriate domain considerations, we would like to associate this occupancy measure with a modified reward function $\mr(a,s)$ using the conjugate optimality conditions in \cref{eq:conj_optimality_app} with $\mr$ as the dual variable.   
\else
$\mu_{\mr} \in \mathcal{M}$
provides a lower bound on the objective.   We use the notation $\mu_{\mr}$ to anticipate the fact that, assuming appropriate domain considerations, we would like to associate this occupancy measure with a modified reward function $\mr$ using the conjugate optimality conditions in \cref{eq:conj_optimality_app} (with $\mr$ as the dual variable). 
\fi
In particular, for a given $\Omega$, we use the fact that 
\ifx\omitindices\undefined
$\mu_{\mr}(a,s) = \frac{1}{\beta}\nabla \Omega^{*}(\mr)$ 
\else
$\mu_{\mr} = \frac{1}{\beta}\nabla \Omega^{*}(\mr)$ 
\fi
to recognize the conjugate duality gap as a Bregman divergence.  Rearranging \cref{eq:ineq},
\ifx\omitindices\undefined
\small
\begin{align}
    \frac{1}{\beta} \Omega^{*}(r) - \blangle \mu_{\mr}(a,s), r(a,s) \brangle + \dfrac{1}{\beta} \Omega(\mu_{\mr}) &\geq 0 \\
    \frac{1}{\beta} \Omega^{*}(r) - \blangle \mu_{\mr}(a,s), r(a,s) \brangle + \blangle \mu_{\mr}(a,s), \mr(a,s) \brangle - \dfrac{1}{\beta} \Omega^*(\mr) &\geq 0 \\
    \frac{1}{\beta} \Omega^{*}(r) - \dfrac{1}{\beta} \Omega^*(\mr)  - \blangle r(a,s)- \mr(a,s),  \underbrace{ \frac{1}{\beta}\nabla \Omega^{*}(\mr)}_{\mu_{\mr(a,s)}} \brangle &\geq 0 \\
    D_{\Omega^{*}}[r : \mr ] &\geq 0
\end{align}
\else
\begin{align}
    \frac{1}{\beta} \Omega^{*}(r) - \blangle \mu_{\mr}, r \brangle + \dfrac{1}{\beta} \Omega(\mu_{\mr}) &\geq 0 \\
    \frac{1}{\beta} \Omega^{*}(r) - \blangle \mu_{\mr}, r \brangle + \blangle \mu_{\mr}, \mr \brangle - \dfrac{1}{\beta} \Omega^*(\mr) &\geq 0 \\
    \frac{1}{\beta} \Omega^{*}(r) - \dfrac{1}{\beta} \Omega^*(\mr)  - \blangle r- \mr,  \underbrace{ \frac{1}{\beta}\nabla \Omega^{*}(\mr)}_{\mu_{\mr}} \brangle &\geq 0 \\
    D_{\Omega^{*}}[r : \mr ] &\geq 0
\end{align}
\fi
\normalsize
where the last line follows from the definition of the Bregman divergence \citep{amari2016information}.
For example, using the \textsc{kl} divergence $\Omega(\mu) = D_{KL}[\mu : \mu_0]$, one can confirm that the Bregman divergence generated by $\Omega^{*}$ is also a \textsc{kl} divergence, $D_{KL}[\mu_{\mr} : \mu_{r^{*}}]$ \citep{belousovblog, banerjee2005clustering}. 

%% file: appendix/new_new_app/01_conjugates.tex
\section{Convex Conjugate Derivations}\label{app:conjugates}
In this section, we derive the convex conjugate associated with \textsc{kl} and $\alpha$-divergence regularization of the policy $\pi(a|s)$ or state-action occupancy $\mu(a,s)$. We summarize these results in \mytab{conj_table}, with 
equation references in \mytab{table_refs}.
In both cases,  we treat the regularizer $\frac{1}{\beta}\Omega(\mu)$ as a function of $\mu(a,s)$ and optimize over all states jointly,
\begin{align}
  \frac{1}{\beta} \Omega^{*}(\cdot) = \sup_{\primalvar \in \mudomainfn} \big \langle \primalvar, \cdot \big \rangle - \frac{1}{\beta} \Omega(\primalvar).
\end{align}
\normalsize
\ifx\omitindices\undefined
 These conjugate derivations can be used to reason about the optimal policy via $\frac{1}{\beta}\Omega^{*}\big( r(a,s) + \gamma \transitionv - V(s) - \lambdaplus \big)$, as argued in \myapp{optimal_policy}, or the worst case reward perturbations using $\frac{1}{\beta}\Omega^{*}( \pr )$.   We use $\pr$ as the argument or dual variable throughout this section.
\else
 These conjugate derivations can be used to reason about the optimal policy via $\frac{1}{\beta}\Omega^{*}\big( r + \gamma \transitionvinds - V - \lambda \big)$, as argued in \myapp{optimal_policy}, or the worst case reward perturbations using $\frac{1}{\beta}\Omega^{*}( \pr )$.   We use $\pr$ as the argument or dual variable throughout this section.
\fi

In \myapp{soft_value}, we derive alternative conjugate functions which optimize over the policy in each state, where $\pi(a|s) \in \Delta^{|\mathcal{A}|}$ is constrained to be a normalized probability distribution.   This conjugate arises in considering soft value aggregation or regularized iterative algorithms as in \mysec{mdp_regularized}.   
See \mytab{table_refs_q} for equation references.

 \begin{figure*}[h]
 \begin{minipage}{.52\textwidth}
 \centering
 \begin{tabular}{lccc}
    \toprule
    Divergence $\Omega$ & \text{ $\frac{1}{\beta}\Omega^*(\pr)$} 
    &  $\pr_{\mu}(a,s)$ & $\mu_{\pr}(a,s)$ \\ \midrule
    $\frac{1}{\beta} D_{KL}[\pi:\pi_0]$  & 
    \cref{eq:kl_pi_conj_closed_form} & \cref{eq:mu_for_kl_pi} & \cref{eq:opt_arg_kl_pi} \\
    $\frac{1}{\beta} D_{KL}[\mu:\mu_0]$ &
    \cref{eq:conj_mu_kl} & \cref{eq:worst_case_kl_mu} & \cref{eq:opt_pol_kl_mu}\\
    $\frac{1}{\beta} D_{\alpha}[\pi_0:\pi]$ & 
    \cref{eq:alpha_pi_conj_closed_form} 
    & \cref{eq:alpha_worst_case_reward_perturb} &  \cref{eq:opt_pi_alpha}  \\ 
    $\frac{1}{\beta} D_{\alpha}[\mu_0:\mu]$ &  
    \cref{eq:conj_mu_result}  & \cref{eq:alpha_pr_mu} &  \cref{eq:opt_mu_alpha} \\ \bottomrule
\end{tabular}
\captionof{table}{\small Equations for $\pr$ or `\textsc{mdp} Optimality' Conjugate \\
\centering ($\mu$ Optimization, No Normalization Constraint) } \label{table:table_refs}
\end{minipage}%
\begin{minipage}{.47\textwidth}
 \centering
 \vspace*{.97cm}
 \begin{tabular}{lccc}
    \toprule
    Divergence $\Omega$ 
    & \text{ $\frac{1}{\beta}\Omega^*(Q)$}
    &  $Q_{\pi}(a,s)$ & $\pi_{Q}(a,s)$ \\ \midrule
    $\frac{1}{\beta} D_{KL}[\pi:\pi_0]$  
    &  \cref{eq:kl_bellman} & \cref{eq:kl_q_arg_qopt} & \cref{eq:kl_pi_arg_qopt} \\
    $\frac{1}{\beta} D_{\alpha}[\pi_0:\pi]$ 
    &  \cref{eq:alpha_bellman} & \cref{eq:alpha_pi_arg_qopt} 
    &   \cref{eq:alpha_q_arg_qopt}  \\ \bottomrule
\end{tabular}
\captionof{table}{ \small Equations for `Soft Value' $V_*(s)$ Conjugates \\
\centering ($\pi$ Optimization, Normalization Constraint) } \label{table:table_refs_q}
\end{minipage}
\vspace*{-.3cm}
\end{figure*}
\subsection{KL Divergence Policy Regularization: $\frac{1}{\beta} \Omega^{*}_{\pi_0,\beta}(\pr)$} \label{app:conj1}
The conjugate function for \kl divergence from the policy $\pi(a|s)$ to a reference $\pi_0(a|s)$ has the following closed form
\begin{align}
    \frac{1}{\beta} \Omega^{*}_{\pi_0,\beta}(\pr) &=  \frac{1}{\beta} \sum \limits_{s} \mu(s)  \left(  \sum \limits_{s} \pi_0(a|s) \expof{\beta \cdot \pr(a,s)}  - 1 \right)
    \, . \label{eq:kl_pi_conj_closed_form}
\end{align}
\begin{proof}
We start from the optimization in \cref{eq:conjstar} or (\ref{eq:conj_app}), using conditional \textsc{kl} divergence regularization $\Omega_{\pi_0}(\mu) = \mathbb{E}_{\mu(s)}[D_{KL}[\pi:\pi_0]]$ as in \cref{eq:kl_def}.
\ifx\omitindices\undefined
\begin{align}
    \frac{1}{\beta} \Omega_{\pi_0,\beta}^{*}(\pr) &= \max \limits_{\mu(a,s)} \blangle \mu(a,s), \pr(a,s) \brangle - \frac{1}{\beta} \blangle \mu(a,s) \log \frac{\mu(a,s)}{\mu(s) \pi_0(a|s)} \brangle + \frac{1}{\beta} \sum \limits_{a,s} \mu(a,s) - \frac{1}{\beta} \sum \limits_{a,s} \mu(s) \pi_0(a|s)  \label{eq:full_conj_kl} \\
    \implies &\pr(a,s) = \nabla_{\mu} \left( \frac{1}{\beta} \blangle \mu(a,s) \log \frac{\mu(a,s)}{\mu(s) \pi_0(a|s)} \brangle + \frac{1}{\beta} \sum \limits_{a,s} \mu(a,s) - \frac{1}{\beta} \sum \limits_{a,s} \mu(s) \pi_0(a|s)  \right) \label{eq:optimality_condition_kl}
\end{align}
\else
\begin{align}
    \frac{1}{\beta} \Omega_{\pi_0,\beta}^{*}(\pr) &= \max \limits_{\mu} \blangle \mu, \pr \brangle - \frac{1}{\beta} \sum \limits_{a,s} \mu(a,s) \log \frac{\mu(a,s)}{\mu(s) \pi_0(a|s)} + \frac{1}{\beta} \sum \limits_{a,s} \mu(a,s) - \frac{1}{\beta} \sum \limits_{a,s} \mu(s) \pi_0(a|s)  \label{eq:full_conj_kl} \\
    \implies &\pr = \nabla_{\mu} \left( \frac{1}{\beta} \sum \limits_{a,s}  \mu(a,s) \log \frac{\mu(a,s)}{\mu(s) \pi_0(a|s)} + \frac{1}{\beta} \sum \limits_{a,s} \mu(a,s) - \frac{1}{\beta} \sum \limits_{a,s} \mu(s) \pi_0(a|s)  \right) \label{eq:optimality_condition_kl}
\end{align}
\fi
\paragraph{Worst-Case Reward Perturbations $\pr_{\pi}(a|s)$}  We can recognize \cref{eq:optimality_condition_kl} as an instance of \myprop{optimal_perturbations}.   
Noting that the marginal $\mu(s)$ depends on $\mu(a,s)$, we make use of the identity
$\sum_{s\tick} \frac{\partial}{\partial \mu(a,s)}\mu(s\tick) = \sum_{s\tick, a\tick} \frac{\partial}{\partial \mu(a,s)} \mu(a\tick, s\tick) = \sum_{s\tick, a\tick} \delta(a\tick,s\tick = a,s) = 1$ as in \citep{neu2017unified, lee2019tsallis}.   Differentiating, we obtain
\small
\begin{align}
    \pr(a,s) &=
    \frac{1}{\beta} \log \frac{\mu(a,s)}{\mu(s) \pi_0(a|s)} - \frac{1}{\beta} \underbrace{ \sum \limits_{a,s} \frac{\partial \mu(a,s)}{\partial \mu(a^\prime, s^\prime)} }_{1}
    + \frac{1}{\beta}  \sum \limits_{a,s} \frac{\mu(a,s)}{\mu(s)} \underbrace{ 
    \frac{\partial \sum_{a^{\prime\prime}} \mu(a^{\prime\prime},s)}{\partial \mu(a^\prime, s^\prime)}}
    _{\delta(s=s^{\prime})} 
    + \frac{1}{\beta} -  \frac{1}{\beta}  \sum \limits_{a,s}  
    \underbrace{
    \frac{\partial \sum_{a^{\prime\prime}} \mu(a^{\prime\prime},s)}{\partial \mu(a^\prime, s^\prime)}
    }_{\delta(s=s^{\prime})} \pi_0(a|s) \nonumber \\
    &= \frac{1}{\beta} \log \frac{\mu(a,s)}{\mu(s) \pi_0(a|s)} +  \frac{1}{\beta}  \sum \limits_{a} \frac{\mu(a,s)}{\mu(s)}  
    - \frac{1}{\beta} \sum \limits_{a} \pi_0(a|s) \nonumber \\
     &= 
     \frac{1}{\beta} \log \frac{\mu(a,s)}{\mu(s) \pi_0(a|s)} \, . \label{eq:mu_for_kl_pi}
\end{align}
\normalsize
In the last line, we 
assume $\sum_a \pi_0(a|s) = 1$ and note that $\sum_a \frac{\mu(a,s)}{\mu(s)} = \frac{\sum_a \mu(a,s)}{\mu(s)} = \frac{\mu(s)}{\mu(s)} =1 $. 

\paragraph{Optimizing Argument $\pir(a|s)$}  
We derive the conjugate function by  
solving for the optimizing argument $\mu(a,s)$ in terms of $\pr(a,s)$ and substituting back into \cref{eq:full_conj_kl}.  
Defining $\pi_{\pr}(a|s) = \frac{\mur(a,s)}{\mur(s)}$ as the policy induced by the optimizing argument $\mur(a,s)$ in \cref{eq:mu_for_kl_pi}, we can solve for $\pi_{\pr}$ to obtain
\begin{align}
    \pi_{\pr}(a|s) = \pi_0(a|s) \expof{\beta \cdot \pr(a,s)} \label{eq:opt_arg_kl_pi}
\end{align}
\paragraph{Conjugate Function $\frac{1}{\beta} \Omega^{*}_{\pi_0,\beta}(\pr)$}  We plug this back into the conjugate optimization \cref{eq:full_conj_kl}, with $ \mu_{\pr}(a,s) = \mu(s) \pi_{\pr}(a,s)$. 
Assuming $\pi_0(a|s)$ is normalized, we also have $\sum_{a,s} \mu(s) \pi_0(a|s)=1$ and 
\scriptsize
\begin{align}
   \hspace*{-.1cm} \frac{1}{\beta} \Omega_{\pi_0,\beta}^{*}(\pr) &=  \sum \limits_{a,s} \mu(s) \pi_{\pr}(a|s) \cdot \pr(a,s)  - \frac{1}{\beta} \Big( 
\sum \limits_{a,s} \mu(s)\pi_{\pr}(a|s) \cdot \cancel{\log} \cancel{\frac{\pi_0}{\pi_0}} \cancel{\exp}\{ \beta  \pr(a,s) \} - \mu(s)\pi_0(a|s) \expof{\beta \pr(a,s)} + \mu(s)\pi_0(a|s)
\Big) \nonumber 
\\
&= \frac{1}{\beta} \sum \limits_{s} \mu(s) \left( \sum \limits_{s} \pi_0(a|s) \expof{\beta \cdot \pr(a,s)}  - 1 \right)
\label{eq:final_kl_conj_pi}
\end{align}
\normalsize
as desired.  Note that the form of the conjugate function also depends on the regularization strength $\beta$.

Finally, we verify that our other conjugate optimality condition 
\ifx\omitindices\undefined
$\prpi(a,s) = \big( \nabla_{\dualvar} \frac{1}{\beta} \Omega_{\pi_0,\beta}^{*}\big)^{-1}(\mur)$, or $\mur(a,s) = \nabla_{\pr} \frac{1}{\beta} \Omega_{\pi_0,\beta}^{*}(\prpi)$ holds for this conjugate function.  Indeed, differentiating with respect to $\pr(a,s)$ above, we see that $\nabla_{\pr} \frac{1}{\beta} \Omega_{\pi_0,\beta}^{*}(\pr) = \mu(s) \pi_0(a|s) \exp\{ \beta \cdot \pr(a,s) \}$ matches $\mur(a,s)= \mu(s)\pir(a|s)$ via \cref{eq:opt_arg_kl_pi}.
\else
$\prpi = \big( \nabla_{\dualvar} \frac{1}{\beta} \Omega_{\pi_0,\beta}^{*}\big)^{-1}(\mur)$, or $\mur = \nabla_{\pr} \frac{1}{\beta} \Omega_{\pi_0,\beta}^{*}(\prpi)$ holds for this conjugate function.  Indeed, differentiating with respect to $\pr(a,s)$ above, we see that $\frac{\partial}{\partial \pr(a,s)} \frac{1}{\beta} \Omega_{\pi_0,\beta}^{*}(\pr) = \mu(s) \pi_0(a|s) \exp\{ \beta \cdot \pr(a,s) \}$ matches $\mur(a,s)= \mu(s)\pir(a|s)$ via \cref{eq:opt_arg_kl_pi}.
\fi
\normalsize
\end{proof}
Although our regularization $\Omega_{\pi_0}(\mu) = \mathbb{E}_{\mu(s)}[D_{KL}[\pi:\pi_0]]$ applies at each $\pi_0(a|s)$, we saw that performing the conjugate optimization over $\mu(a,s)$ led to an expression for a policy $\pir(a|s) = \mur(a,s)/\mu(s)$ that is {normalized} \textit{by construction} $\sum_a \pir(a|s) =  \frac{\sum_a \mur(a,s)}{\mu(s)}  =1$.  
Conversely, for a given normalized $\pi(a|s)$, the above conjugate conditions yield $\prpi(a,s)$ such that \cref{eq:opt_arg_kl_pi} is also normalized.

\headerv
\subsection{KL Divergence Occupancy Regularization: $\frac{1}{\beta} \Omega^{*}_{\mu_0,\beta}(\pr)$} \label{app:conj2}
\headerv
Nearly identical derivations as \myapp{conj1} apply when regularizing the divergence $\Omega_{\mu_0}(\mu) = D_{KL}[\mu:\mu_0]$ between the joint state-action occupancy $\mu(a,s)$ and a reference $\mu_0(a,s)$.  This leads to the following results
\begin{align}
   & \text{\textbf{Worst-Case Perturbations:}}   && &  \pr_{\mu}(a,s) &=  \frac{1}{\beta} \log \frac{\mu(a,s)}{\mu_0(a,s)} \label{eq:worst_case_kl_mu}  && \\[1.5ex]
   & \text{\textbf{Optimizing Argument:}}   && & \mu_{\pr}(a,s) &= \mu_0(a,s) \exp \left\{ \beta \cdot \pr(a,s) \right\} . \label{eq:opt_pol_kl_mu} && \\[1.5ex]
      & \text{\textbf{Conjugate Function:}}    && & \frac{1}{\beta} \Omega^{*}_{\mu_0,\beta}(\pr) &=  \frac{1}{\beta} \sum \limits_{a,s} \mu_0(a,s)  \expof{\beta \cdot \pr(a,s)}  - \frac{1}{\beta} \sum \limits_{a,s} \mu_0(a,s)  \label{eq:conj_mu_kl} && 
  \end{align}
Such regularization schemes appear in \gls{REPS} \citep{peters2010relative}, while \citet{basserrano2021logistic} consider both policy and occupancy regularization. 




\headerv
\subsection{$\alpha$-Divergence Policy Regularization: $\frac{1}{\beta} \Omega^{*(\alpha)}_{\pi_0,\beta}(\pr)$ } \label{app:conj3}
\headerv
The conjugate for $\alpha$-divergence regularization of the policy $\pi(a|s)$ to a reference $\pi_0(a|s)$ takes the form
\small
\begin{align}
    \frac{1}{\beta} \Omega^{*(\alpha)}_{\pi_0,\beta}(\pr) &= \frac{1}{\beta}\frac{1}{\alpha} \sum \limits_{a,s}  \mu(s) \bigg( \pi_0(a|s) \bigg[ 1 + \beta (\alpha-1) \big( \pr(a,s) - \myconst \big) \bigg]^{\frac{\alpha}{\alpha-1}}  -1 \bigg)
    + \sum \limits_{s} \mu(s) \, \myconst \label{eq:alpha_pi_conj_closed_form}   .
\end{align}
\normalsize
where $\psi_{\pr}(s;\beta)$ is a normalization constant for the optimizing argument $\pi_{\pr}(a|s)$ corresponding to $\pr(a,s)$.

We provide explicit derivations of the conjugate function instead of leveraging $f$-divergence duality \citep{belousov2019entropic, nachum2020reinforcement} in order to account for the effect of optimization over the joint distribution $\mu(a,s)$.   We will see in \myapp{alpha_conj_norm} see that the conjugate in \cref{eq:alpha_pi_conj_closed_form} takes a similar to form as the conjugate with restriction to normalized $\pi(a|s) \in \Delta^{|\mathcal{A}|}$, where this constraint is not captured using $f$-divergence function space duality.

\begin{proof}  We begin by writing the $\alpha$-divergence $\Omega^{(\alpha)}_{\pi_0}(\mu) = \mathbb{E}_{\mu(s)}[D_{\alpha}[\pi_0:\pi]]$ as a function of the occupancy measure $\mu$, with $\pi(a|s) = \frac{\mu(a,s)}{\mu(s)}$.   As in \myprop{optimal_perturbations}, the conjugate optimization implies an optimality condition for $\pr(a,s)$.
\ifx\omitindices\undefined
\scriptsize
\begin{align}
    \hspace*{-.1cm} \frac{1}{\beta} \Omega^{*(\alpha)}_{\pi_0,\beta}(\pr) &= \max \limits_{\mu(a,s)} \blangle \mu(a,s), \pr(a,s) \brangle - \frac{1}{\beta}\frac{1}{\alpha (1-\alpha)} \left( (1-\alpha) \sum \limits_{a,s} \mu(s) \pi_0(a|s) + \alpha \sum \limits_{a,s} \mu(a,s) - \sum \limits_{a,s} \mu(s) \pi_0(a|s)^{1-\alpha} \left(\frac{\mu(a,s)}{\mu(s)}\right)^{\alpha} \right) \label{eq:conj_opt_pi_alpha} \\
    \implies &\pr(a,s) = \nabla_{\mu}  \frac{1}{\beta}\frac{1}{\alpha (1-\alpha)} \left( (1-\alpha) \sum \limits_{a,s} \mu(s) \pi_0(a|s) + \alpha \sum \limits_{a,s} \mu(a,s) - \sum \limits_{a,s} \mu(s) \pi_0(a|s)^{1-\alpha} \left(\frac{\mu(a,s)}{\mu(s)}\right)^{\alpha}  \right) \label{eq:optimality_cond_alpha}
\end{align}
\normalsize
\else
\footnotesize
\begin{align}
    \hspace*{-.1cm} \frac{1}{\beta} \Omega^{*(\alpha)}_{\pi_0,\beta}(\pr) &= \max \limits_{\mu} \blangle \mu, \pr \brangle - \frac{1}{\beta}\frac{1}{\alpha (1-\alpha)} \left( (1-\alpha) \sum \limits_{a,s} \mu(s) \pi_0(a|s) + \alpha \sum \limits_{a,s} \mu(a,s) - \sum \limits_{a,s} \mu(s) \pi_0(a|s)^{1-\alpha} \left(\frac{\mu(a,s)}{\mu(s)}\right)^{\alpha} \right) \label{eq:conj_opt_pi_alpha} \\
    \implies &\pr = \nabla_{\mu}  \frac{1}{\beta}\frac{1}{\alpha (1-\alpha)} \left( (1-\alpha) \sum \limits_{a,s} \mu(s) \pi_0(a|s) + \alpha \sum \limits_{a,s} \mu(a,s) - \sum \limits_{a,s} \mu(s) \pi_0(a|s)^{1-\alpha} \left(\frac{\mu(a,s)}{\mu(s)}\right)^{\alpha}  \right) \label{eq:optimality_cond_alpha}
\end{align}
\normalsize
\fi

\paragraph{Worst-Case Reward Perturbations $\pr_{\pi}(a|s)$}  
We now differentiate with respect to $\mu(a,s)$, using similar derivations as in \citet{lee2019tsallis}.   While we have already written \cref{eq:optimality_cond_alpha} using $\mu(a,s) = \mu(s) \pi(a|s)$, we again emphasize that $\mu(s)$ depends on $\mu(a,s)$.
Differentiating, we obtain
\footnotesize
\begin{align}
\pr(a,s) &= \nabla_{\mu} \alphanmu   \label{eq:pr_equals_grad} \\
&= \frac{1}{\beta}\frac{1}{\alpha(1-\alpha)}  \sum \limits_{a\tick,s\tick}  \frac{\partial}{\partial\mu(a,s)} \bigg( (1-\alpha) \mu(s\tick)\pi_0(a\tick|s\tick) + \alpha \mu(a\tick, s\tick) - \mu(s\tick)^{1-\alpha} \pi_0(a\tick|s\tick)^{1-\alpha} \mu(a\tick,s\tick)^{\alpha} \bigg)  \nonumber \\
&=  \frac{1}{\beta}\frac{1}{\alpha(1-\alpha)} \bigg( 
(1-\alpha) \sum \limits_{s\tick} \underbrace{\frac{\partial \sum_{a\tick} \mu(a\tick, s\tick)}{\partial\mu(a,s)}}_{\delta(s = s\tick)} \sum \limits_{a\tick} \pi_0(a\tick|s)  + \alpha 
\sum \limits_{a\tick,s\tick} \underbrace{\frac{\partial \mu(a\tick,s\tick)}{\partial \mu(a,s)}}_{\delta(a\tick,s\tick = a,s)}  - 
 \nonumber \\ 
&\phantom{==} -\sum \limits_{a\tick,s\tick}  \alpha \mu(a\tick,s\tick)^{\alpha-1} \underbrace{\frac{\partial \mu(a\tick,s\tick)}{\partial \mu(a,s)}}_{\delta(a\tick,s\tick = a,s)} \mu(s\tick)^{1-\alpha} \pi_0(a\tick|s\tick)^{1-\alpha}- (1-\alpha) \sum \limits_{a\tick,s\tick} \mu(s\tick)^{-\alpha} \underbrace{\frac{\partial \sum_{a\tick} \mu(a\tick, s\tick)}{\partial \mu(a,s)}}_{\delta(s = s\tick)} \pi_0(a\tick|s\tick)^{1-\alpha} \mu(a\tick,s\tick)^{\alpha} \bigg) \nonumber \\
&= \frac{1}{\beta}\frac{1}{\alpha(1-\alpha)} \bigg( (1-\alpha) {\sum_{a} \pi_0(a|s)} + \alpha - \alpha \left(\frac{\mu(a,s)}{\mu(s) \pi_0(a|s)} \right)^{\alpha-1} - (1-\alpha) \sum \limits_{a\tick}   \pi_0(a|s)^{1-\alpha}\left(\frac{\mu(a,s)}{\mu(s)} \right)^{\alpha}  \bigg)   \nonumber \\
&\overset{(1)}{=} \frac{1}{\beta}\frac{1}{\alpha} 
+ \frac{1}{\beta}\frac{1}{1-\alpha}
   - \frac{1}{\beta}\frac{1}{1-\alpha} \left(\frac{\pi(a|s)}{\pi_0(a|s)} \right)^{\alpha-1} - \frac{1}{\beta}\frac{1}{\alpha} \sum \limits_{a} \pi_0(a|s)^{1-\alpha} \pi(a|s)^{\alpha}  \nonumber \\ 
   &= \frac{1}{\beta} \, 
 \frac{1}{\alpha-1}  \left( \Big(\frac{\pi(a|s)}{\pi_0(a|s)} \Big)^{\alpha-1} - 1 \right) 
   + \frac{1}{\beta}  \frac{1}{\alpha} \left( 
    1 - \sum \limits_{a} \pi_0(a|s)^{1-\alpha} \pi(a|s)^{\alpha}
   \right) \label{eq:worst_case_reward_perturb}
\end{align} 
\normalsize
where we have rewritten $(1)$ in terms of the policy $\pi(a|s) = \frac{\mu(a,s)}{\mu(s)}$ and assumed $\pi_0(a|s)$ is normalized.   

Letting $\pir(a|s)$ indicate the policy which is in dual correspondence with $\pr(a,s)$, we would eventually like to invert the equality in \cref{eq:worst_case_reward_perturb} to solve for $\pi(a|s)$ in each $(a,s)$.   However,  the final term depends on a sum over all actions.  
To handle this, we define 
    \begin{align}
        \psi_{\pr}(s;\beta) = \frac{1}{\beta}  \frac{1}{\alpha} \big( 
    \sum \limits_{a} \pi_0(a|s) - \sum \limits_{a} \pi_0(a|s)^{1-\alpha} \pir(a|s)^{\alpha} \big) . \label{eq:psi_pr}
    \end{align}
 Since $\pir(a|s) = \frac{\mur(a,s)}{\mu(s)}$ is normalized by construction, the constant $\psipr$ with respect to actions
has appeared naturally when optimizing with respect to $\mu(a,s)$.  In \myapp{alpha_conj_norm}-\ref{app:advantage_confirm},
we will relate this quantity to the Lagrange multiplier used to enforce normalization when optimizing over $\pi(a|s) \in \Delta^{|\mathcal{A}|}$. 

Finally, we use \cref{eq:worst_case_reward_perturb} to write $\prpi(a,s)$ as
\begin{align}
  \prpi(a,s)  &= \frac{1}{\beta}\log_{\alpha}\frac{\pi(a|s)}{\pi_0(a|s)} + \myconst \, . \label{eq:alpha_worst_case_reward_perturb}
\end{align}

\paragraph{Optimizing Argument $\pir(a|s)$}  Solving for the policy in \cref{eq:alpha_worst_case_reward_perturb} and denoting this as $\pir(a|s)$, 
\small
\begin{align}
 \hspace*{-.2cm}  \pir(a|s) &= \pi_0(a|s) \exp_{\alpha}\big \{ \beta \cdot \big( \pr(a,s) - \myconst \big) \big\}  = \pi_0(a|s) \big[1+\beta(\alpha-1) \left( \pr(a,s) - \myconst \right) \big]_+^{\frac{1}{\alpha-1}}. 
 \label{eq:opt_pi_alpha} 
\end{align}
\normalsize
Note that $\prpi(a|s)$ is defined in self-consistent fashion due to the dependence of $\myconst$ on $\pir(a|s)$ in \cref{eq:psi_pr}.  Further, $\myconst$ does not appear as a divisive normalization constant for general $\alpha$, which is inconvenient for practical applications \citep{lee2019tsallis, chow2018path}.

\paragraph{Conjugate Function $\frac{1}{\beta} \Omega^{*(\alpha)}_{\pi_0,\beta}(\pr)$} 
Finally, we plug this into the conjugate optimization \cref{eq:conj_opt_pi_alpha}.   Although we eventually need to obtain a function of $\pr(a,s)$ only, we write $\pir(a|s)$ in initial steps to simplify notation.
\scriptsize
\ifx\omitindices\undefined
\begin{align}
    \alphaconjn(\pr) &= \blangle \mu(s) \pir(a|s), \pr(a,s) \brangle - \frac{1}{\beta} \frac{1}{\alpha(1-\alpha)}\bigg( (1-\alpha) \sum \limits_{a,s} \mu(s) \pi_0(a|s) + \alpha \sum \limits_{a,s} \mu(s) \pir(a|s)  - \sum \limits_{a,s} \mu(s)\pir(a|s)  {\left(\frac{\pir(a|s) }{\pi_0(a|s)}\right)^{\alpha-1}} \bigg) \nonumber \\
    &= \blangle \mu(s) \pir(a|s), \pr(a,s) \brangle - \frac{1}{\beta} \frac{1}{\alpha} \sum \limits_{a,s} \mu(s) \pi_0(a|s) - \frac{1}{\beta} \frac{1}{1-\alpha} \sum \limits_{a,s} \mu(s) \pir(a|s) \label{eq:app_intermediate_alpha_derivation} \\
    &\phantom{= \blangle \mu(s) \pir(a|s), \pr(a,s) \brangle } + \frac{1}{\beta} \frac{\textcolor{black}{1}}{\alpha \textcolor{black}{(1-\alpha)}} \sum \limits_{a,s} \mu(s)\pir(a|s)  \bigg( 1+ \underbrace{\cancel{\beta(\alpha-1)}}_{-1} \big( \pr(a,s) - \myconst \big)  \bigg) \nonumber \\[1.5ex]
    &\overset{(1)}{=} \bigg(\underbrace{1- \frac{1}{\alpha}}_{\frac{\alpha-1}{\alpha}} \bigg) \blangle \mu(s) \pir(a|s), \pr(a,s) \brangle - \frac{1}{\beta} \frac{1}{\alpha}\sum \limits_{a,s}  \mu(s) \pi_0(a|s) - \bigg(  \underbrace{ \frac{1}{\beta}\frac{1}{1-\alpha} - \frac{1}{\beta}\frac{1}{\alpha(1-\alpha)}}_{-\frac{1}{\beta}\frac{1}{\alpha}} \bigg) \sum \limits_{a,s}  \mu(s) \pir(a|s) + \frac{1}{\alpha} \myconst \nonumber \\
    &\overset{(2)}{=} \frac{1}{\beta}\frac{1}{\alpha} \sum \limits_{a,s}  \mu(s) \pir(a|s) \, + \textcolor{blue}{\frac{\beta}{\beta}} \frac{\alpha-1}{\alpha} \sum \limits_{a,s} \mu(s) \pir(a|s) \cdot \pr(a,s)    \, \textcolor{blue}{\pm \frac{\beta}{\beta} \frac{\alpha-1}{\alpha} \myconst} + \frac{1}{\alpha} \myconst  - \frac{1}{\beta} \frac{1}{\alpha}\sum \limits_{a,s}  \mu(s) \pi_0(a|s)  \nonumber
\end{align}
\else
\begin{align}
    \alphaconjn(\pr) &= \sum \limits_{a,s} \mu(s) \pir(a|s) \cdot  \pr(a,s)  - \frac{1}{\beta} \frac{1}{\alpha(1-\alpha)}\bigg( (1-\alpha) \sum \limits_{a,s} \mu(s) \pi_0(a|s) + \alpha \sum \limits_{a,s} \mu(s) \pir(a|s)  - \sum \limits_{a,s} \mu(s)\pir(a|s)  {\left(\frac{\pir(a|s) }{\pi_0(a|s)}\right)^{\alpha-1}} \bigg) \nonumber \\
    &= \sum \limits_{a,s} \mu(s) \pir(a|s) \cdot \pr(a,s) - \frac{1}{\beta} \frac{1}{\alpha} \sum \limits_{a,s} \mu(s) \pi_0(a|s) - \frac{1}{\beta} \frac{1}{1-\alpha} \sum \limits_{a,s} \mu(s) \pir(a|s) \label{eq:app_intermediate_alpha_derivation} \\
    &\phantom{= \blangle \mu(s) \pir(a|s), \pr(a,s) \brangle } + \frac{1}{\beta} \frac{\textcolor{black}{1}}{\alpha \textcolor{black}{(1-\alpha)}} \sum \limits_{a,s} \mu(s)\pir(a|s)  \bigg( 1+ \underbrace{\cancel{\beta(\alpha-1)}}_{-1} \big( \pr(a,s) - \myconst \big)  \bigg) \nonumber \\[1.5ex]
    &\overset{(1)}{=} \bigg(\underbrace{1- \frac{1}{\alpha}}_{\frac{\alpha-1}{\alpha}} \bigg) \sum \limits_{a,s} \mu(s) \pir(a|s) \cdot \pr(a,s) - \frac{1}{\beta} \frac{1}{\alpha}\sum \limits_{a,s}  \mu(s) \pi_0(a|s) - \bigg(  \underbrace{ \frac{1}{\beta}\frac{1}{1-\alpha} - \frac{1}{\beta}\frac{1}{\alpha(1-\alpha)}}_{-\frac{1}{\beta}\frac{1}{\alpha}} \bigg) \sum \limits_{a,s}  \mu(s) \pir(a|s) + \frac{1}{\alpha} \myconst \nonumber \\
    &\overset{(2)}{=} \frac{1}{\beta}\frac{1}{\alpha} \sum \limits_{a,s}  \mu(s) \pir(a|s) \, + \textcolor{blue}{\frac{\beta}{\beta}} \frac{\alpha-1}{\alpha} \sum \limits_{a,s} \mu(s) \pir(a|s) \cdot \pr(a,s)    \, \textcolor{blue}{\pm \frac{\beta}{\beta} \frac{\alpha-1}{\alpha} \myconst} + \frac{1}{\alpha} \myconst  - \frac{1}{\beta} \frac{1}{\alpha}\sum \limits_{a,s}  \mu(s) \pi_0(a|s)  \nonumber
\end{align}
\fi
\normalsize
where in $(1)$ we note that $\frac{1}{\beta}\frac{1}{\alpha(1-\alpha)}\cdot \beta (\alpha-1) = -\frac{1}{\alpha}$.   In $(2)$, we add and subtract the term in blue, which will allow to factorize an additional term of $[1+\beta(\alpha-1)(\pr - \myconst)]$ and obtain a function of $\pr(a,s)$ only   
\scriptsize
\begin{align}
    \alphaconjn(\pr) &= \frac{1}{\beta}\frac{1}{\alpha} \sum \limits_{a,s}  \mu(s) \pir(a|s) \bigg( 1 + \beta (\alpha-1) \big( \pr(a,s) - \myconst \big) \bigg) + \bigg( \underbrace{\frac{\alpha-1}{\alpha}  + \frac{1}{\alpha}}_{1} \bigg)  \myconst - \frac{1}{\beta}\frac{1}{\alpha}\sum \limits_{a,s}  \mu(s) \pi_0(a|s) \nonumber \\
    &\overset{(1)}{=} \frac{1}{\beta}\frac{1}{\alpha} \sum \limits_{a,s}  \mu(s) \bigg( \pi_0(a|s) \bigg[ 1 + \beta (\alpha-1) \big( \pr(a,s) - \myconst \big) \bigg]^{\frac{\alpha}{\alpha-1}}    - 1 \bigg)   + \sum \limits_{s} \mu(s) \, \myconst \rangle \label{eq:alpha_final_conj_pi}
\end{align}
\normalsize
where in $(1)$ we have used \cref{eq:opt_pi_alpha} and $\frac{1}{\alpha-1}+1 = \frac{\alpha}{\alpha-1}$, along with $\sum_a \pi_0(a|s) = 1$.
\end{proof}
\paragraph{Confirming Conjugate Optimality Conditions}
Finally, we confirm that differentiating \cref{eq:alpha_final_conj_pi} with respect to $\pr(a,s)$ yields the conjugate condition 
\ifx\omitindices\undefined
$\pir(a|s) = \nabla \alphaconjn(\pr)$.  
\else
$\pir = \nabla \alphaconjn(\pr)$.  
\fi
Noting that $\frac{\alpha}{\alpha-1} - 1 = \frac{1}{\alpha-1}$,
\scriptsize
\begin{equation}
\resizebox{\textwidth}{!}{$
\nabla \alphaconjn(\pr) = \frac{\beta(\alpha-1)}{\beta \, \alpha} \frac{\alpha}{\alpha-1} \sum \limits_{s}  \mu(s) \sum \limits_{a} \pi_0(a|s) \big[ 1 + \beta (\alpha-1) \big( \pr(a,s) - \myconst \big) \big]^{\frac{1}{\alpha-1}} \left(  \frac{\partial \pr(a,s)}{\partial\pr(a\tick,s\tick)} - \frac{\partial \psipr}{{\partial\pr(a\tick,s\tick)}} \right) + \sum \limits_{s}  \mu(s) \frac{\partial\psipr}{{\partial \pr(a\tick,s\tick)}} $
} \nonumber
\end{equation}
\normalsize
which simplifies to $\pir(a|s) = \frac{\partial}{\partial \pr(a,s)} \alphaconjn(\pr) = \pi_0(a|s) \big[ 1 + \beta (\alpha-1) \big( \pr(a,s) - \myconst \big) \big]^{\frac{1}{\alpha-1}}$ and matches \cref{eq:opt_pi_alpha}.
\normalsize
\subsection{$\alpha$-Divergence Occupancy Regularization: $\frac{1}{\beta} \Omega^{*(\alpha)}_{\mu_0,\beta}(\pr)$ } \label{app:conj4}\label{app:mu_results}
\normalsize
The conjugate function $\frac{1}{\beta} \Omega^{*(\alpha)}_{\mu_0,\beta}(\pr)$ for $\alpha$-divergence regularization of the state-action occupancy  $\mu(a,s)$ to a reference $\mu_0(a,s)$ can be written in the following form 
\begin{align}
    \alphaconjm(\pr) &= \frac{1}{\beta}\frac{1}{\alpha}  \sum \limits_{a,s} \mu_0(a,s)\big[1+\beta(\alpha-1)\pr(a,s)\big]^{\frac{\alpha}{\alpha-1}} - \frac{1}{\beta}\frac{1}{\alpha} \label{eq:conj_mu_result} 
\end{align}
Note that this conjugate function can also be derived directly from the duality of general $f$-divergences, and matches the form of conjugate considered in \citep{belousov2019entropic, nachum2020reinforcement}.   
\begin{proof}
\textbf{Worst-Case Reward Perturbations $\pr_{\mu}(a|s)$  } 
\begin{align}
    \pr(a,s) &= \frac{1}{\beta} \frac{1}{\alpha(1-\alpha)} \nabla_{\mu} \left( (1-\alpha) \sum \limits_{a,s} \mu_0(a,s) + \alpha \sum \limits_{a,s} \mu(a,s) - \sum \limits_{a,s} \mu_0(a,s)^{1-\alpha} \mu(a,s)^{\alpha} \right) \label{eq:alpha_pr_mu} \\
    &= \frac{1}{\beta} \frac{1}{1-\alpha} - \frac{1}{\beta} \frac{1}{1-\alpha} \mu_0(a,s)^{1-\alpha} \mu(a,s)^{\alpha-1}  \nonumber
\end{align}
\textbf{Optimizing Argument $\mu_{\pr}(a,s)$. }  Solving for $\mu_{\pr}(a,s)$,
 \begin{align}
    \mu_{\pr}(a,s) &=  \mu_0(a,s) \exp_{\alpha} \{ \beta \cdot \pr(a,s) \} =\mu_0(a,s) [1 + \beta(1-\alpha)\pr(a,s)]_+^{ \frac{1}{\alpha-1}} \label{eq:opt_mu_alpha} 
\end{align}
\textbf{Conjugate Function $\frac{1}{\beta} \Omega^{*(\alpha)}_{\mu_0,\beta}(\pr)$. } 
Plugging this back into the conjugate optimization, we finally obtain
\small
\ifx\omitindices\undefined
\begin{align}
    \alphaconjm &= \blangle \mu_{\pr}(a,s), \pr(a,s) \brangle - \frac{1}{\beta} \frac{1}{\alpha(1-\alpha)}\bigg( (1-\alpha) \sum \limits_{a,s} \mu_0(a,s) + \alpha \sum \limits_{a,s} \mu_{\pr}(a,s)  - \sum \limits_{a,s} \mu_{\pr}(a,s)  \underbrace{\left(\frac{\mu_{\pr}(a,s) }{\mu_0(a,s)}\right)^{\alpha-1}}_{1 + \beta(\alpha-1) \pr(a,s)} \bigg) \nonumber \\
    &= \left(1- \frac{1}{\alpha} \right) \blangle \mu_{\pr}(a,s), \pr(a,s) \brangle - \frac{1}{\beta} \frac{1}{\alpha}\sum \limits_{a,s} \mu_0(a,s) +  \left( \frac{1}{\beta}\frac{1}{\alpha(1-\alpha)} - \frac{1}{\beta}\frac{1}{1-\alpha} \right) \sum \limits_{a,s} \mu_{\pr}(a,s) \\
    &= \text{ \scriptsize $\frac{\textcolor{blue}{\alpha -1}}{\alpha}  \blangle \mu_{0}(a,s)
    \big[1+\beta(\alpha-1)\pr(a,s)\big]^{\frac{1}{\alpha-1}}, \textcolor{blue}{\pr(a,s)} \brangle + \frac{1}{\beta} \frac{1}{\alpha} \sum \limits_{a,s} \mu_{0}(a,s)\big[1+\beta(\alpha-1)\pr(a,s)\big]^{\frac{1}{\alpha-1}} - \frac{1}{\beta} \frac{1}{\alpha} \sum \limits_{a,s} \mu_{0}(a,s)$} \nonumber \\
    &=  \sum \limits_{a,s} \mu_{0}(a,s)\big[1+\beta(\alpha-1)\pr(a,s)\big]^{\frac{1}{\alpha-1}}\cdot \frac{1}{\beta}\frac{1}{\alpha} \big( 1 + \beta \textcolor{blue}{(\alpha-1) \pr(a,s)} \big) - \frac{1}{\beta}\frac{1}{\alpha} \sum \limits_{a,s}\mu_0(a,s) \\
    &=  \frac{1}{\beta}\frac{1}{\alpha} \sum \limits_{a,s} \mu_0(a,s)\big[1+\beta(\alpha-1)\pr(a,s)\big]^{\frac{\alpha}{\alpha-1}} - \frac{1}{\beta}\frac{1}{\alpha} \sum \limits_{a,s}\mu_0(a,s)
\end{align}
\else
\begin{align}
    \alphaconjm &= \sum \limits_{a,s} \mu_{\pr}(a,s) \cdot \pr(a,s) - \frac{1}{\beta} \frac{1}{\alpha(1-\alpha)}\bigg( (1-\alpha) \sum \limits_{a,s} \mu_0(a,s) + \alpha \sum \limits_{a,s} \mu_{\pr}(a,s)  - \sum \limits_{a,s} \mu_{\pr}(a,s)  \underbrace{\left(\frac{\mu_{\pr}(a,s) }{\mu_0(a,s)}\right)^{\alpha-1}}_{= \textcolor{blue}{1 + \beta(\alpha-1) \pr(a,s)} (\cref{eq:opt_mu_alpha})} \bigg) \nonumber \\
    &= \left(1- \frac{1}{\alpha} \right) \sum \limits_{a,s} \mu_{\pr}(a,s) \cdot \pr(a,s) - \frac{1}{\beta} \frac{1}{\alpha}\sum \limits_{a,s} \mu_0(a,s) +  \left( \frac{1}{\beta}\frac{1}{\alpha(1-\alpha)} - \frac{1}{\beta}\frac{1}{1-\alpha} \right) \sum \limits_{a,s} \mu_{\pr}(a,s) \\
    &= \text{ \scriptsize $\frac{\textcolor{blue}{\alpha -1}}{\alpha}  \sum \limits_{a,s} \mu_{0}(a,s)
    \big[1+\beta(\alpha-1)\pr(a,s)\big]^{\frac{1}{\alpha-1}} \cdot \textcolor{blue}{\pr(a,s)} + \frac{1}{\beta} \frac{1}{\alpha} \sum \limits_{a,s} \mu_{0}(a,s)\big[1+\beta(\alpha-1)\pr(a,s)\big]^{\frac{1}{\alpha-1}} - \frac{1}{\beta} \frac{1}{\alpha} \sum \limits_{a,s} \mu_{0}(a,s)$} \nonumber \\
    &=  \sum \limits_{a,s} \mu_{0}(a,s)\big[1+\beta(\alpha-1)\pr(a,s)\big]^{\frac{1}{\alpha-1}}\cdot \frac{1}{\beta}\frac{1}{\alpha} \big( 1 + \beta \textcolor{blue}{(\alpha-1) \pr(a,s)} \big) - \frac{1}{\beta}\frac{1}{\alpha} \sum \limits_{a,s}\mu_0(a,s) \\
    &=  \frac{1}{\beta}\frac{1}{\alpha} \sum \limits_{a,s} \mu_0(a,s)\big[1+\beta(\alpha-1)\pr(a,s)\big]^{\frac{\alpha}{\alpha-1}} - \frac{1}{\beta}\frac{1}{\alpha} \sum \limits_{a,s}\mu_0(a,s)
\end{align}
\fi
\normalsize
where, to obtain the exponent in the last line, note that $\frac{1}{\alpha-1}+1=\frac{\alpha}{\alpha-1}$.
\end{proof}

%% file: appendix/new_new_app/09_soft_value.tex
\section{Soft Value Aggregation}\label{app:soft_value}\label{app:sv_background}
\normalsize

Soft value aggregation \citep{fox2016taming, haarnoja17a} and the regularized Bellman optimality operator \citep{neu2017unified, geist2019theory} also rely on the convex conjugate function, but with a slightly different setting than our derivations for the optimal regularized policy or reward perturbations in \myapp{conjugates}.  In particular, in each state we optimize over the policy $\pi(a|s) \in \Delta^{|\mathcal{A}|}$ using an explicit normalization constraint (\cref{eq:soft_value_app}).

We derive the regularized Bellman optimality operator from the primal objective in \cref{eq:primal_reg}.   Factorizing $\mu(a,s) = \mu(s) \pi(a|s)$, we can imagine optimizing over $\mu(s)$ and $\pi(a|s) \in \simplex$ separately,
\begin{align}
   \max \limits_{\mu(s)\rightarrow \mathcal{M}} \max \limits_{\pi(a|s) \in \simplex} \min \limits_{V(s)}  \, \, 
     (1-\gamma) \big\langle \nu_0(s), V(s) \big  \rangle \, 
     + &\, \blangle \mu(a,s), r(a,s) +  \gamma \transitionv -V(s)\brangle - \frac{1}{\beta} \Omega_{\pi_0}(\mu) . \label{eq:dual_lagrangian_app}
\end{align}
Eliminating $\mu(s)$ (by setting $d/d\mu(s) = 0$) leads to a constraint on the form of $V(s)$, since both may be viewed as enforcing the Bellman flow constraints.  
\begin{align}
    V(s)  &= \blangle \pi(a|s), \, r(s,a) + \gamma \transitionv \brangle - \frac{1}{\beta} \Omega_{\pi_0}(\pi) \label{eq:bellman_opt} . 
\end{align}
\normalsize
We define $Q(s,a) \coloneqq r(s,a) + \gamma \transitionv$ and write 
\ifx\omitindices\undefined
$V(s) = \langle \pi(a|s), \, Q(a,s) \rangle - \frac{1}{\beta} \Omega_{\pi_0}(\pi)$ moving forward.
\else
$V(s) = \langle \pi, \, Q \rangle - \frac{1}{\beta} \Omega_{\pi_0}(\pi)$ moving forward.
\fi

As an operator for iteratively updating $V(s)$, \cref{eq:bellman_opt} corresponds to the regularized Bellman operator $\mathcal{T}_{\Omega_{\pi_0}, \beta}$ and may be used to perform policy evaluation for a given $\pi(a|s)$ \citep{geist2019theory}.    
The regularized Bellman \textit{optimality} operator $\mathcal{T}^*_{\Omega_{\pi_0}, \beta}$, which can be used for value iteration or modified policy iteration \citep{geist2019theory}, arises from including the maximization over $\pi(a|s) \in \simplex$ from \cref{eq:dual_lagrangian_app}
\ifx\omitindices\undefined
\begin{align}
   V(s)  \leftarrow \frac{1}{\beta} \Omega^{*}_{\pi_0,\beta}(Q) = \max \limits_{\pi \in \Delta^{|\mathcal{A}|}} \blangle \pi(a|s), Q(a,s) \brangle - \frac{1}{\beta} \Omega_{\pi_0}(\pi) \, . \label{eq:soft_value_app} 
\end{align}  
\else
\begin{align}
   V(s)  \leftarrow \frac{1}{\beta} \Omega^{*}_{\pi_0,\beta}(Q) = \max \limits_{\pi \in \Delta^{|\mathcal{A}|}} \blangle \pi , Q \brangle - \frac{1}{\beta} \Omega_{\pi_0}(\pi) \, . \label{eq:soft_value_app} 
\end{align}  
\fi

\paragraph{Comparison of Conjugate Optimizations}
\cref{eq:soft_value_app} has the form of a conjugate optimization $\frac{1}{\beta} \Omega^{*}_{\pi_0,\beta}(Q)$ \citep{geist2019theory}.  However, in contrast to the setting of \myapp{optimal_policy} and \myapp{conjugates}, we optimize over the policy in each state, rather than the state-action occupancy $\mu(a,s)$.  Further, we must include normalization and nonnegativity constraints $\pi(a|s) \in \Delta^{|\mathcal{A}|}$, which can be enforced using Lagrange multipliers $\psi_Q(s;\beta)$ and $\lambdaplus$.   
We derive expressions for this conjugate function for the \textsc{kl} divergence in \myapp{kl_conj_norm} and $\alpha$-divergence in \myapp{alpha_conj_norm}, and plot the value $V_*(s)$ as a function of $\alpha$ and $\beta$ in \myapp{alpha_results}.

Compared with the optimization for the optimal policy in \cref{eq:conj_opt_mdp}, note that the argument of the conjugate function does not include the value function $V(s)$ in this case.   We will highlight relationship between the normalization constants $\psiqopt$, $\psipiopt$, and $V_*(s)$ in \myapp{advantage_confirm}, where $\psiqopt = V_*(s) + \psipiopt$ as in \citet{lee2019tsallis} App. D.




\subsection{KL Divergence Soft Value Aggregation: $\frac{1}{\beta} \Omega^{*}_{\pi_0,\beta}(Q)$} \label{app:kl_conj_norm}
We proceed to derive a closed form for the conjugate function of the \textsc{kl} divergence $\Omega_{\pi_0}(\pi)$ as a function of $\pi(a|s) \in \Delta^{|\mathcal{A}|}$, which we write using the $Q$-values as input 
 \scriptsize
\begin{align}
 \frac{1}{\beta} \Omega^{*}_{\pi_0,\beta}(Q) =  \max \limits_{\pi(a|s)}
 &\langle \pi, Q \rangle - \frac{1}{\beta} \left( \sum \limits_{a} \pi(a|s) \log \frac{\pi(a|s) }{\pi_0(a|s)}  + \sum \limits_{a} \pi(a|s) -  \sum \limits_{a} \pi_0(a|s) \right) - \psi_Q(s;\beta) \left(\sum \limits_{a \in \mathcal{A}} \pi(a|s) - 1 \right) + \sum \limits_{a \in \mathcal{A}} \lambdaplus  \nonumber 
\end{align}
\normalsize
\footnotesize
\begin{align}
 \implies Q(a,s)   &= \frac{1}{\beta}\log \frac{\pi(a|s) }{\pi_0(a|s)} + \psi_Q(s;\beta)  - \lambdaplus
 \hphantom{ - \frac{1}{\beta} \sum \limits_{a} \pi_0(a|s)  - \psi_Q(s;\beta) \left(\sum \limits_{a \in \mathcal{A}} \pi(a|s) - 1 \right) } \label{eq:kl_q_arg_qopt} 
\end{align}
\normalsize
\paragraph{Optimizing Argument} Solving for $\pi$ yields the optimizing argument
\small
\begin{align}
     \pi_Q(a|s)  &= \pi_0(a|s) \expof{ \beta \cdot \big( Q(a,s) - \psi_Q(s;\beta)) \big)}  \label{eq:kl_pi_arg_qopt}
\end{align}
\normalsize
where we can ignore the Lagrange multiplier for the nonnegativity constraint $\lambdaplus$ since $\exp\{ \cdot \}\geq0$ ensures $\pi_Q(a|s) \geq 0.$
We can pull the normalization constant out of the exponent to solve for
\begin{align}
    \psi_Q(s;\beta) = \frac{1}{\beta} \log \sum \limits_a \pi_0(a|s) \expof{\beta \cdot Q(a,s)} \, . \label{eq:psi_q_kl}
\end{align}
Plugging \cref{eq:kl_pi_arg_qopt} into the conjugate optimization,
\scriptsize
\begin{align}
\frac{1}{\beta} \Omega^{*}_{\pi_0,\beta}(Q) = \blangle \pi_Q, Q \brangle - \frac{1}{\beta} \bigg( \sum \limits_{a} \pi_Q(a|s) \cdot \cancel{\log} \cancel{\frac{\pi_0(a|s)}{\pi_0(a|s)}} \cancel{\exp}\big\{ \beta \cdot \big( Q(a,s) - \psi_Q(s;\beta) \big)\big\}  +  \underbrace{\sum \limits_{a} \pi_Q(a|s)}_{1} - 1 \bigg) - \psi_Q(s;\beta) \bigg(\underbrace{\sum \limits_{a \in \mathcal{A}} \pi_Q(a|s)}_{1} - 1 \bigg) \nonumber \, .
\end{align}
\normalsize
\paragraph{Conjugate Function} We finally recover the familiar log-mean-exp form for the \textsc{kl}-regularized value function 
\begin{align}
    V(s) \leftarrow \frac{1}{\beta} \Omega^{*}_{\pi_0,\beta}(Q) =  \psi_Q(s;\beta) = \frac{1}{\beta} \log \sum \limits_a \pi_0(a|s) \expof{\beta \cdot Q(a,s)}. \label{eq:kl_bellman}
\end{align}
Notice that the conjugate or value function $V(s) \leftarrow \frac{1}{\beta} \Omega^{*}_{\pi_0,\beta}(Q)$ is exactly equal to the normalization constant of the policy $\psi_Q(s;\beta)$.   We will show in \myapp{alpha_conj_norm} that this property \textit{does not hold} for general $\alpha$-divergences, with example visualizations in \myapp{advantage_confirm} \myfig{conj_by_alpha_app}.

\subsection{$\alpha$-Divergence Soft Value Aggregation: $\frac{1}{\beta} \Omega^{*}_{\pi_0,\beta}(Q)$} \label{app:alpha_conj_norm}
We now consider soft value aggregation using the $\alpha$-divergence, where in contrast to \myapp{conj3}, we perform the conjugate optimization over $\pi(a|s) \in \Delta^{|\mathcal{A}|}$ in each state, with Lagrange multipliers $\psiq$ and $\lambdaplus$ to enforce normalization and nonnegativity.
\small
\begin{align}
 \frac{1}{\beta} \Omega^{*}_{\pi_0,\beta}(Q) =  \max \limits_{\pi(a|s)} &\blangle \pi,Q \brangle - \frac{1}{\beta} \frac{1}{\alpha} \sum \limits_{a} \pi_0(a|s) - \frac{1}{\beta} \frac{1}{1-\alpha} \sum \limits_{a} \pi(a|s) + \frac{1}{\beta} \frac{1}{\alpha(1-\alpha)} \sum \limits_{a} \pi_0(a|s)^{1-\alpha} \pi(a|s)^{\alpha} \label{eq:soft_val_opt_alpha} \\
 & \hphantom{\blangle \pi,Q \brangle - \frac{1}{\beta}} - \psipiq \left(\sum \limits_{a} \pi(a|s) - 1 \right) + \sum \limits_{a} \lambdaplus \nonumber \\
 \implies Q(a,s)  &= \frac{1}{\beta}\log_{\alpha} \frac{\pi(a|s) }{\pi_0(a|s)} + \psiq   - \lambdaplus  \hphantom{ - \frac{1}{\beta} \sum \limits_{a} \pi_0(a|s)  - \psipiq \left(\sum \limits_{a \in \mathcal{A}} \pi(a|s) - 1 \right) } \label{eq:alpha_q_arg_qopt}
\end{align}
\normalsize
\paragraph{Optimizing Argument} Solving for $\pi$ yields the optimizing argument for the soft value aggregation conjugate,
\begin{align}
 \pi_Q(a|s)  &= \pi_0(a|s) \exp_{\alpha} \big\{ \beta \cdot \big( Q(a,s) + \lambdaplus - \psiq ) \big) \big\} .  \label{eq:alpha_pi_arg_qopt}
\end{align}
Unlike the case of the standard exponential, we cannot easily derive a closed-form solution for $\psi_Q(s;\beta)$.

Note that the expressions in \cref{eq:alpha_q_arg_qopt} and \cref{eq:alpha_pi_arg_qopt} are similar to the form of the worst-case reward perturbations $\prpi(a|s)$ in \cref{eq:alpha_worst_case_reward_perturb} and optimizing policy $\pir(a|s)$ in \cref{eq:opt_pi_alpha}, except for the fact that $\psiq$ arises as a Lagrange multiplier and does not have the same form as $\psipr = \frac{1}{\beta}(1-\alpha)D_{\alpha}[\pi_0:\pi]$ as in \cref{eq:pr_normalizer} and \cref{eq:psi_pr}.   We will find that $\psiq$ and $\psipr$ differ by a term of $V_*(s)$ in \myapp{advantage_confirm} (\cref{eq:value_and_normalizers}).

\paragraph{Conjugate Function} Plugging \cref{eq:alpha_pi_arg_qopt} into the conjugate optimization,  we use similar derivations as in \cref{eq:app_intermediate_alpha_derivation}-\cref{eq:alpha_final_conj_pi} to write the conjugate function, or regularized Bellman optimality operator as
\begin{align}
     V(s) \leftarrow \frac{1}{\beta} \Omega^{*(\alpha)}_{\pi_0,\beta}(Q) &=  \frac{1}{\beta}\frac{1}{\alpha}\sum_a \pi_0(a|s) \Big[ 1 + \beta (\alpha-1) \big( Q(a,s) + \lambdaplus - \psipiq \big) \Big]_+^{\frac{\alpha}{\alpha-1}}  - \frac{1}{\beta}\frac{1}{\alpha} +  \psipiq \label{eq:alpha_bellman} \\
     &= \frac{1}{\beta}\frac{1}{\alpha}\sum_a \pi_0(a|s) \exp_{\alpha} \Big\{ \beta \cdot  \big( Q(a,s) + \lambdaplus - \psipiq \big) \Big\}^{\alpha}  - \frac{1}{\beta}\frac{1}{\alpha} +  \psipiq \nonumber
\end{align}
\normalsize
\textbf{Comparison with KL Divergence Regularization} 
Note that for general $\alpha$, the conjugate or value function $V(s) = \frac{1}{\beta} \Omega^{*}_{\pi_0,\beta}(Q)$ in \cref{eq:alpha_bellman} is \textit{not} equal to the normalization constant of the policy $\psi_Q(s;\beta)$.    We discuss this further in the next section.

We also note that the form of the conjugate function is similar using two different approaches:  optimizing over $\pi$ with an explicit normalization constraint, as in \cref{eq:alpha_bellman}, or optimizing over $\mu$ with regularization of $\pi$ but no explicit normalization constraint, as in \myapp{conj3} or \mytab{conj_table}.    This is in contrast to the \textsc{kl} divergence, where the normalization constraint led to a log-mean-exp conjugate in \cref{eq:kl_q_arg_qopt} which is different from \myapp{conjugates} \cref{eq:kl_pi_conj_closed_form}.

\subsection{Relationship between Normalization Constants $\psi_{\pr_{\piopt}}$, $\psi_{Q_*}$, and Value Function $V_*(s)$}\label{app:advantage_confirm}

In this section, we analyze the relationship between the conjugate optimizations that we have considered above, either optimizing over $\mu(a,s)$ as in deriving the optimal policy, or optimizing over $\pi(a|s) \in \Delta^{|\mathcal{A}|}$ as in the regularized Bellman optimality operator or soft-value aggregation.   Using $Q(a,s) = r(a,s) + \gamma \transitionv$, 
\ifx\omitindices\undefined
\small
\begin{flalign}
 & \,\,  \textit{Optimal Policy}  \textit{ (or Worst-Case Reward Perturbations)}  \text{  (\myapp{conj3})} \label{eq:opt_pol_opt}   &\\
   & \quad  \frac{1}{\beta} \Omega_{\pi_0,\beta}^{*(\alpha)}\Big(r(a,s) + \gamma \transitionv - V(s) + \lambdaplus \Big)  = \max \limits_{\mu(a,s) \in \mathcal{F}} \blangle \mu(a,s), r(a,s) + \gamma \transitionv - V(s) + \lambdaplus  \brangle -  \frac{1}{\beta} \Omega^{(\alpha)}_{\pi_0}(\mu) \nonumber & \\
 & \,\, \textit{Soft Value Aggregation}  \text{ (\myapp{alpha_conj_norm})},  \label{eq:sv_opt}  & \\
 & \quad \,\, V(s) \leftarrow \frac{1}{\beta} \Omega_{\pi_0,\beta}^{*(\alpha)}\Big(r(a,s) + \gamma \transitionv   \Big) = \max \limits_{\pi(a|s) \in \Delta^{|\mathcal{A}|}} \blangle \mu(a,s), r(a,s) + \gamma \transitionv   \brangle - \frac{1}{\beta} \Omega^{(\alpha)}_{\pi_0}(\pi) \nonumber  & 
\end{flalign}
\normalsize
\else
\small
\begin{flalign}
 & \,\, \textit{Optimal Policy}  \textit{ (or Worst-Case Reward Perturbations)}  \text{  (\myapp{conj3})} \label{eq:opt_pol_opt}   &\\
   & \qquad \qquad  \frac{1}{\beta} \Omega_{\pi_0,\beta}^{*(\alpha)}\Big(r + \gamma \transitionvinds - V + \lambda \Big)  = \max \limits_{\mu(a,s) \in \mathcal{F}} \blangle \mu, r + \gamma \transitionvinds - V + \lambda  \brangle -  \frac{1}{\beta} \Omega^{(\alpha)}_{\pi_0}(\mu) \nonumber & \\
 & \,\, \textit{Soft Value Aggregation}  \text{ (\myapp{alpha_conj_norm})},  \label{eq:sv_opt}  & \\
 & \qquad \qquad \,\, V(s) \leftarrow \frac{1}{\beta} \Omega_{\pi_0,\beta}^{*(\alpha)}\Big(r + \gamma \transitionvinds   \Big) = \max \limits_{\pi \in \Delta^{|\mathcal{A}|}} \blangle \mu, r + \gamma \transitionvinds   \brangle - \frac{1}{\beta} \Omega^{(\alpha)}_{\pi_0}(\pi) \nonumber  & 
\end{flalign}
\normalsize
\fi
Note that the arguments differ by a term of $V(s)$.    We ignore the apparent difference in the $\lambdaplus$ term, which can be considered as an argument of the conjugate in \cref{eq:sv_opt} since a linear term of 
\ifx\omitindices\undefined
$\langle \mu(a,s), \lambdaplus \rangle$ 
\else
$\langle \mu, \lambda \rangle$ 
\fi
appears when enforcing $\pi \in \Delta^{|\mathcal{A}|}$. 
Evaluating the optimizing arguments,
\small
\begin{flalign}
    &  \,\, \textit{Optimal Policy } \textit{ (or Worst-Case Reward Perturbations)}  \text{  (\myapp{conj3}, \cref{eq:opt_pi_alpha}, \mytab{conj_table})} \nonumber   & \\
    & \qquad \qquad \pi(a|s) = \pi_0(a|s) \exp_{\alpha} \Big\{ \beta \cdot  \big( Q(a,s) + \lambdaplus - V(s) - \psipr \big) \Big\} & \label{eq:opt_pol_pol} \\
 &  \,\,  \textit{Soft Value Aggregation} \textit{  (\myapp{alpha_conj_norm}, \cref{eq:alpha_bellman})},  \nonumber  & \\
& \qquad  \qquad \pi(a|s) = \pi_0(a|s) \exp_{\alpha} \Big\{ \beta \cdot  \big( Q(a,s) + \lambdaplus  - \psiq \big) \nonumber & 
\end{flalign}
\normalsize
For the optimal $V_*(s)$ and $Q_*(a,s) = r(a,s) + \gamma \transitionvstar - V_*(s)$, the two policies match.
This can be confirmed using similar reasoning as in \citet{lee2019tsallis} App. D-E or \citep{geist2019theory} to show that iterating the regularized Bellman optimality operator leads to the optimal policy and value.   

\input{appendix/app_results/psi_comparison}%

\paragraph{Relationship between $V_*(s)$ and $\psiqopt$}
This implies the condition which is the main result of this section. 
\begin{align}
     \psiqopt = V_*(s) + \psipropt \, . \label{eq:value_and_normalizers}
\end{align}
In \myfig{conj_by_alpha_app} we empirically confirm this identity and inspect how each quantity varies with $\beta$ and $\alpha$ (\myapp{alpha_results})\footnote{
Note that $\psipropt = \frac{1}{\beta}\frac{1}{\alpha}\left( \sum_a \pi_0(a|s) - \sum_a \pi_0(a|s)^{1-\alpha} \piopt(a|s)^{\alpha} \right)$ appears from differentiating $\frac{1}{\beta}\Omega_{\pi_0}^{(\alpha})(\mu)$ with respect to $\mu$ (\myapp{conj3} \cref{eq:psi_pr}).  We also write this as $\psipropt= \frac{1}{\beta}(1-\alpha) D_{\alpha}[\pi_0 : \piopt]$ for normalized $\pi_0$ and $\piopt$.} 

\cref{eq:value_and_normalizers} highlights distinct roles for the value function $V_*(s)$ and the Lagrange multiplier $\psiqopt$ enforcing normalization of $\pi(a|s)$ in soft-value aggregation ( \cref{eq:soft_val_opt_alpha} or \cref{eq:sv_opt}).   It is well known that these coincide in the case of \textsc{kl} divergence regularization, with $V_*(s) =\psiqopt$ as in \myapp{kl_conj_norm}.   We can also confirm that $\psipropt = \frac{1}{\beta}\frac{1}{\alpha}\left( \sum_a \pi_0(a|s) - \sum_a \pi_0(a|s)^{1-\alpha} \piopt(a|s)^{\alpha} \right) = 0$ vanishes for \textsc{kl} regularization ($\alpha = 1$) and normalized $\pi_0$ and the normalized optimal policy $\piopt$.   

However, in the case of $\alpha$-divergence regularization, optimization over the joint $\mu(a,s)$ in \cref{eq:opt_pol_opt} introduces the additional term $\psipropt$, which is not equal to 0 in general.

\paragraph{Relationship between Conjugate Functions}
We might also like to compare the value of the conjugate functions in \cref{eq:opt_pol_opt} and \cref{eq:sv_opt}, in particular to understand how including $V_*(s)$ as an argument and optimizing over $\pi$ versus $\mu$ affect the optima.
We write the expressions for the conjugate function in each case, highlighting the terms from \cref{eq:value_and_normalizers} in blue.  
\ifx\omitindices\undefined
\small
\begin{flalign}
&  \textit{\small Optimal Policy } \textit{ \small (or Worst-Case Reward Perturbations)}  \text{  (\myapp{conj3}, \cref{eq:alpha_pi_conj_closed_form},  \mytab{conj_table})} \nonumber  &   \\
 & \qquad  \frac{1}{\beta} \Omega_{\pi_0,\beta}^{*(\alpha)}\Big(r(a,s) + \gamma \transitionvstar - V_*(s) + \lambdaopt \Big) \label{eq:compare_optpol_conj} \\
 & \qquad \qquad = \frac{1}{\beta}\frac{1}{\alpha}\sum_a \pi_0(a|s) \exp_{\alpha} \Big\{ \beta \cdot  \big( Q_*(a,s) + \lambdaopt - V_*(s) - \psipr \big) \Big\}^{\alpha}  - \frac{1}{\beta}\frac{1}{\alpha} +  \textcolor{blue}{\psipropt} & \nonumber \\
 &   \textit{\small Soft Value Aggregation} \text{  \small (\myapp{alpha_conj_norm}, \cref{eq:alpha_bellman})},  \label{eq:compare_sv_conj} &  \\
 & \qquad \textcolor{blue}{V_*(s)} = \frac{1}{\beta}\frac{1}{\alpha}\sum_a \pi_0(a|s) \exp_{\alpha} \Big\{ \beta \cdot  \big( Q_*(a,s) + \lambdaopt - \psipiq \big) \Big\}^{\alpha}  - \frac{1}{\beta}\frac{1}{\alpha} +  \textcolor{blue}{\psipiq}  & \nonumber
\end{flalign}
\normalsize
\else
\small
\begin{flalign}
&  \textit{\small Optimal Policy } \textit{ \small (or Worst-Case Reward Perturbations)}  \text{  (\myapp{conj3}, \cref{eq:alpha_pi_conj_closed_form},  \mytab{conj_table})} \nonumber  &   \\
 & \qquad  \frac{1}{\beta} \Omega_{\pi_0,\beta}^{*(\alpha)}\Big(r + \gamma \transitionvstarinds - V_* + \lambda_* \Big) \label{eq:compare_optpol_conj} \\
 & \qquad \qquad = \frac{1}{\beta}\frac{1}{\alpha}\sum_a \pi_0(a|s) \exp_{\alpha} \Big\{ \beta \cdot  \big( Q_*(a,s) + \lambdaopt - V_*(s) - \psipr \big) \Big\}^{\alpha}  - \frac{1}{\beta}\frac{1}{\alpha} +  \textcolor{blue}{\psipropt} & \nonumber \\
 &   \textit{\small Soft Value Aggregation} \text{  \small (\myapp{alpha_conj_norm}, \cref{eq:alpha_bellman})},  \nonumber &  \\
 & \,\,\quad \textcolor{blue}{V_*(s)}  \leftarrow \frac{1}{\beta} \Omega_{\pi_0,\beta}^*\Big(r + \gamma \transitionvstarinds \Big) \label{eq:compare_sv_conj} \\
 & \qquad \qquad = \frac{1}{\beta}\frac{1}{\alpha}\sum_a \pi_0(a|s) \exp_{\alpha} \Big\{ \beta \cdot  \big( Q_*(a,s) + \lambdaopt - \psipiq \big) \Big\}^{\alpha}  - \frac{1}{\beta}\frac{1}{\alpha} +  \textcolor{blue}{\psipiq}  & \nonumber
\end{flalign}
\normalsize
\fi
Note that we have rewritten 
\ifx\omitindices\undefined
$V_*(s) \leftarrow \frac{1}{\beta} \Omega_{\pi_0,\beta}^*(r(a,s) + \gamma \transitionvstar )$ 
\else
$V_*(s) \leftarrow \frac{1}{\beta} \Omega_{\pi_0,\beta}^*(r + \gamma \transitionvstarinds )$ 
\fi
directly as $V_*(s)$.   To further simplify, note that the optimal policy matches as in \cref{eq:opt_pol_pol}, with 
\begin{align*}
\piopt(a|s) &= \pi_0(a|s) \exp_{\alpha}\{ \beta \cdot (Q_*(a,s) + \lambdaopt - \psiqopt )\} \\
&= \pi_0(a|s) \exp_{\alpha}\{ \beta \cdot (Q_*(a,s) + \lambdaopt - V_*(s) - \psipiopt )\} . \label{eq:81}
\end{align*}
\normalsize
Since $\pi_*(a|s) = \pi_0(a|s) \exp_{\alpha}\{ \cdot \}$, we can write terms of the form $\pi_0(a|s) \exp_{\alpha}\{ \cdot \}^{\alpha}$ in \cref{eq:compare_optpol_conj}-(\ref{eq:compare_sv_conj}) as $\pi_0(a|s)^{1-\alpha} \piopt(a|s)^{\alpha}$, where the exponents of $\pi_0(a|s)$ add to 1.
Finally, we use this expression to simplify the value function expression in  \cref{eq:compare_sv_conj}, eventually recovering the equality in \cref{eq:value_and_normalizers}
\ifx\omitindices\undefined
\begin{align}
  \frac{1}{\beta} \Omega_{\pi_0,\beta}^{*(\alpha)}(r(a,s) + \gamma \transitionvstar )=  \textcolor{blue}{V_*(s)} &=  \frac{1}{\beta}\frac{1}{\alpha}\sum_a \pi_0(a|s)^{1-\alpha} \piopt(a|s)^{\alpha}  - \frac{1}{\beta}\frac{1}{\alpha} +  {\psipiq} \\
    &=  \textcolor{blue}{\psipiq} - \textcolor{blue}{\psipiopt} \label{eq:subtractive_relationship}
\end{align}
\else
\begin{align}
  \frac{1}{\beta} \Omega_{\pi_0,\beta}^{*(\alpha)}\Big(r  + \gamma \transitionvstarinds \Big)=  \textcolor{blue}{V_*(s)} &=  \frac{1}{\beta}\frac{1}{\alpha}\sum_a \pi_0(a|s)^{1-\alpha} \piopt(a|s)^{\alpha}  - \frac{1}{\beta}\frac{1}{\alpha} +  {\psipiq} \\
    &=  \textcolor{blue}{\psipiq} - \textcolor{blue}{\psipiopt} \label{eq:subtractive_relationship}
\end{align}
\fi
In the second line, we use the fact that $\psipiopt = \frac{1}{\beta} \frac{1}{\alpha} \sum_{a} \pi_0(a|s)- \frac{1}{\beta} \frac{1}{\alpha}  \sum_{a} \pi_0(a|s)^{1-\alpha} \piopt(a|s)^{\alpha}$ from \cref{eq:psi_pr}.
We can use the same identity to show that the conjugate in \cref{eq:compare_optpol_conj} evaluates to zero,
\ifx\omitindices\undefined
\small
\begin{align}
    \frac{1}{\beta} \Omega_{\pi_0,\beta}^{*(\alpha)}\Big(r(a,s) + \gamma \transitionvstar - V_*(s) + \lambdaopt \Big) = \underbrace{ \frac{1}{\beta}\frac{1}{\alpha}\sum_a \pi_0(a|s)^{1-\alpha} \piopt(a|s)^{\alpha} - \frac{1}{\beta}\frac{1}{\alpha}}_{\psipiopt} + \psipropt =
    0 \label{eq:conj_zero}
\end{align}
\normalsize
\else
\begin{align}
    \frac{1}{\beta} \Omega_{\pi_0,\beta}^{*(\alpha)}\Big(r + \gamma \transitionvstarinds - V_* + \lambda_* \Big) = \underbrace{ \frac{1}{\beta}\frac{1}{\alpha}\sum_a \pi_0(a|s)^{1-\alpha} \piopt(a|s)^{\alpha} - \frac{1}{\beta}\frac{1}{\alpha}}_{\psipiopt} + \psipropt =
    0 \label{eq:conj_zero}
\end{align}
\fi
In \mylemma{conjugatezero}, we provide a more detailed proof and show that this identity also holds for suboptimal policies and their worst-case reward perturbations, $\frac{1}{\beta} \Omega_{\pi_0,\beta}^{*(\alpha)}(\prpi) = 0$, where \cref{eq:conj_zero} is a special case for $\prpiopt(a,s) = r(a,s) + \gamma \transitionvstar - V_*(s) + \lambdaopt$.

Finally, note that the condition in \cref{eq:conj_zero} implies that for the optimal $V_*(s)$, the regularized dual objective $\obj_{\Omega, \beta}^{*}(r) = 
 \min_{V, \lambda}
 \,  (1-\gamma) \langle \nu_0, V \rangle  + \frac{1}{\beta}\Omega^{*}_{\pi_0, \beta} \Big(r  +  \gamma \mathbb{E}_{a,s}^{s^{\prime}} V -V + \lambda \Big)$ in \cref{eq:dual_reg} reduces to the value function averaged over initial states, $\obj^{*}_{\Omega, \beta}(r) = (1-\gamma)\langle \nu_0, V_* \rangle$.   This is intuitive since $V_*(s)$ measures the regularized objective attained from running the optimal policy for infinite time in the discounted \gls{MDP}.

\subsection{Plotting Value Function as a Function of Regularization Parameters $\alpha, \beta$}\label{app:alpha_results}

\paragraph{Confirming Relationship between Normalization $\psipropt$, $\psiqopt$ and Value Function $V_*(s)$}
In \myfig{conj_by_alpha_app}, we plot both $V_*(s)$ and $\psi_{Q_*}(s;\beta)$ for various values of $\alpha$ (x-axis) and $\beta$ (in each panel).   
We also plot $\psipiopt = \frac{1}{\beta}(1-\alpha) D_{\alpha}[\pi_0:\piopt]$ in the third row, and confirm the identity in \cref{eq:value_and_normalizers} in the fourth row.

 As we also observe in \myfig{value_agg_main}, the soft value function or certainty equivalent $V_*(s) = \alphaconjn(Q)$ is not monotonic in $\alpha$ for this particular set of single-step rewards (the same $r(a,s)$ as in \myfig{perturb_opt} or \myfig{value_agg_main}).   Note the small scale of the $y$-axis in the top row of \myfig{conj_by_alpha_app}.

While it can be shown that $\psipr$ is convex as a function of $\beta$, we see that $\psipr$ is not necessarily convex in $\alpha$ and appears to be monotonically decreasing in $\alpha$.   Finally, we find that the identity in \cref{eq:value_and_normalizers}-(\ref{eq:subtractive_relationship}) holds empirically, with only small numerical optimization issues.

\paragraph{Value $V_*(s)$ as a Function of $\beta$, $\alpha$}
In \cref{fig:value_agg_main}, we visualize the optimal value function $V_*(s) = \frac{1}{\beta}\Omega^{*}_{\pi_0,\beta}(Q)$, for \textsc{kl} or $\alpha$-divergence regularization and different choices of regularization strength $1/\beta$.   
The choice of divergence particularly affects the aggregated value at low regularization strength, although we do not observe a clear pattern with respect to $\alpha$.\footnote{See \citet{belousov2019entropic, lee2018sparse, lee2019tsallis}, or \cref{app:alpha_results} for additional discussion of the effect of $\alpha$-divergence regularization on learned policies.}
In all cases, the value function ranges between $\max_a Q(a,s)$ for an unregularized deterministic policy as $\beta \rightarrow \infty$, and the expectation under the reference policy $\mathbb{E}_{\pi_0}[Q(a,s)]$ for strong regularization as $\beta \rightarrow 0$.   We also discuss this property in \mysec{entropy_vs_div}.

%% file: appendix/app_results/psi_comparison.tex
\begin{figure*}
\begin{minipage}{\textwidth}
\vspace*{-.2cm}
\begin{subfigure}{0.24\textwidth}\includegraphics[width=\textwidth]{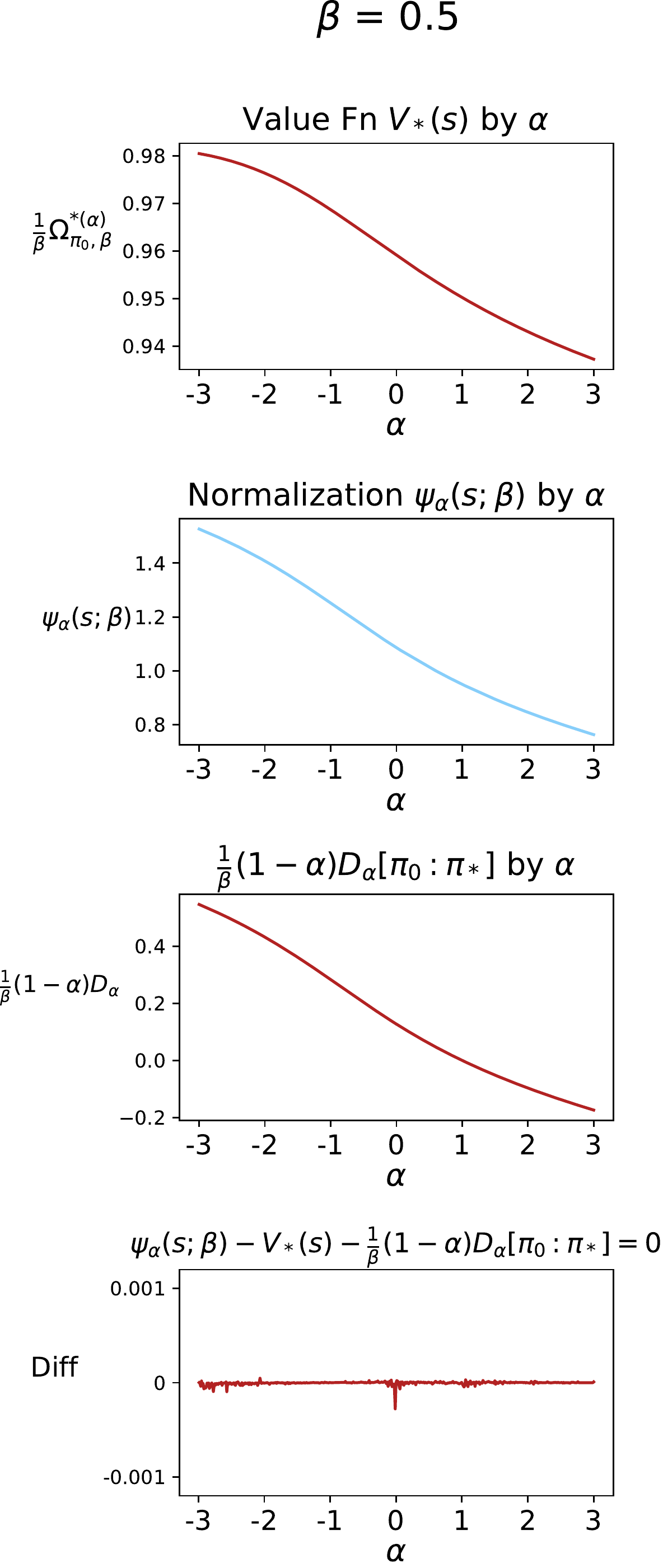}
\caption{$\beta=0.5$}\end{subfigure}
\begin{subfigure}{0.24\textwidth}\includegraphics[width=\textwidth]{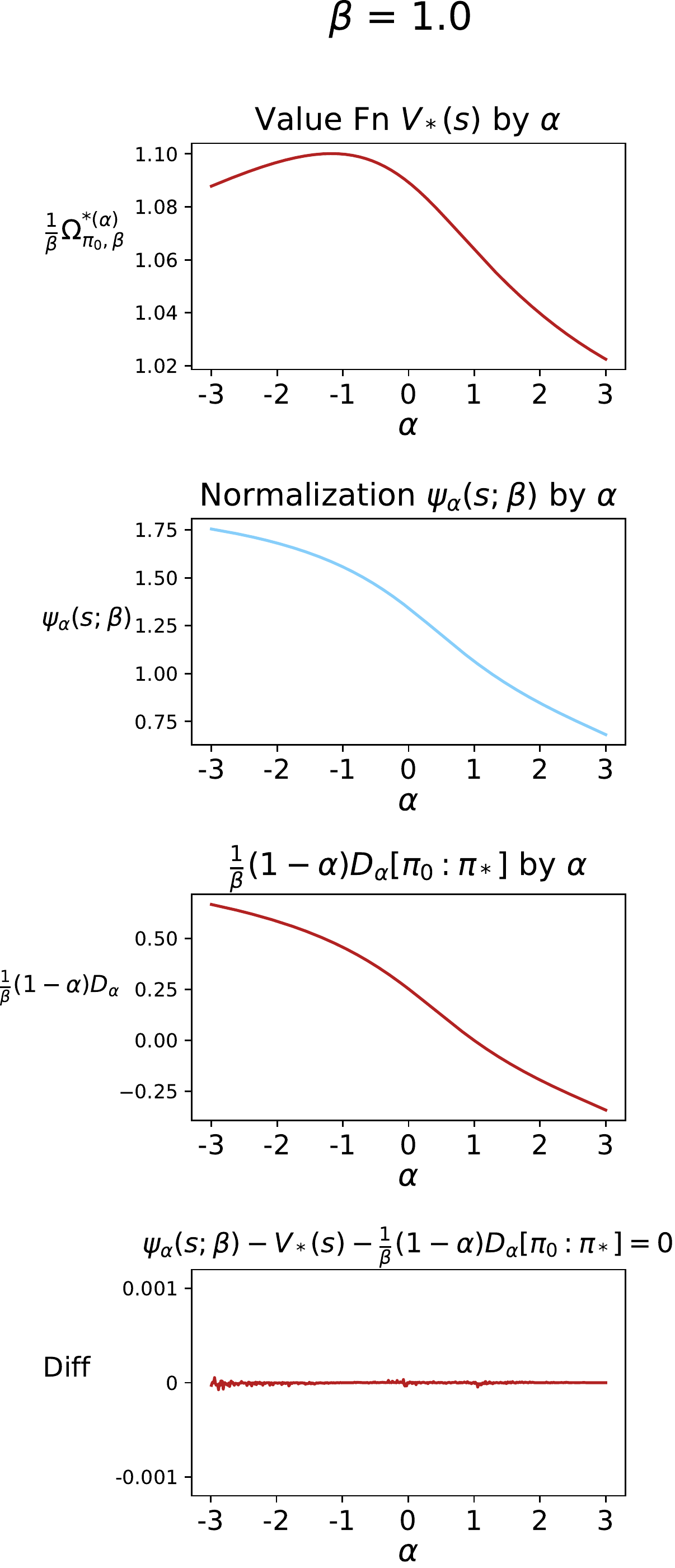}
\caption{$\beta=1$}\end{subfigure}
\begin{subfigure}{0.24\textwidth}\includegraphics[width=\textwidth]{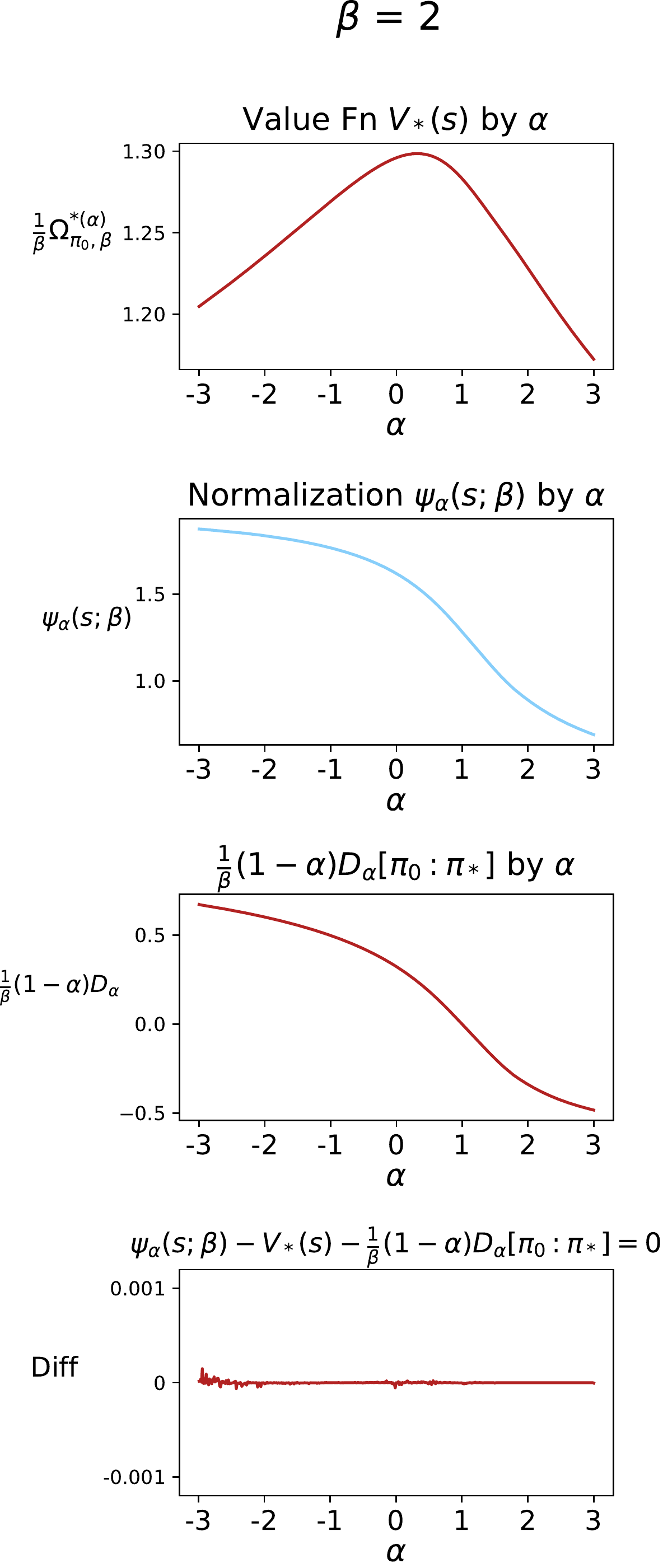}
\caption{$\beta=2$}\end{subfigure}
\begin{subfigure}{0.24\textwidth}\includegraphics[width=\textwidth]{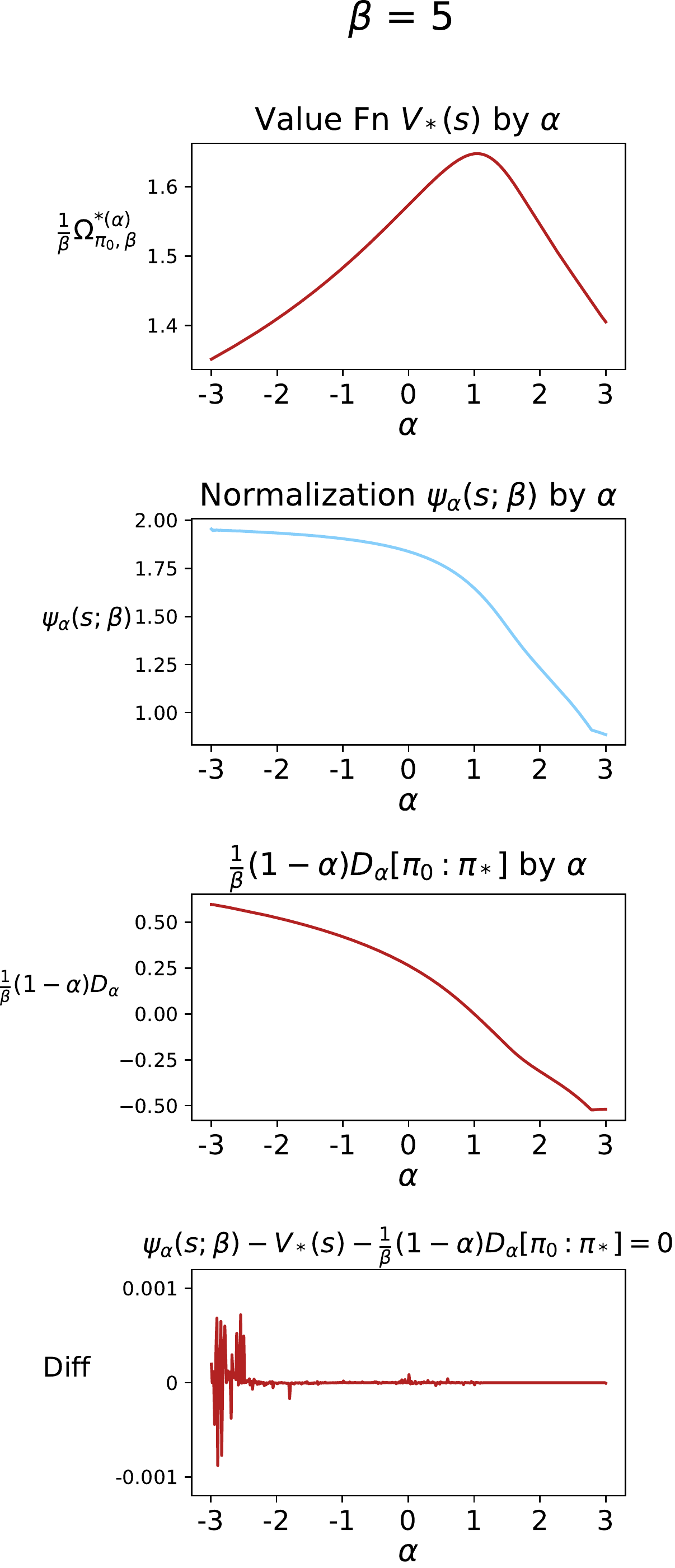}
\caption{$\beta=5$}\end{subfigure}
\caption{ Value Function $V_*(s) = \alphaconjn(Q_*)$ (first row) and Normalization Constant $\psiqopt$ (second row) as a function of $\alpha$ for various regularization strengths $\beta$.   We use the same rewards as in  \myfig{perturb_opt} and \myfig{value_agg_main} and a uniform reference.
We plot $\psipr = \frac{1}{\beta}(1-\alpha)D_{\alpha}[\pi_0:\piopt]$ in the third row, and confirm the identity $V_*(s) = \psiqopt - \psipr$ from \cref{eq:value_and_normalizers} and (\ref{eq:subtractive_relationship}) in the last row.   We find that this equality holds for all $\alpha$ up to small optimization errors on the order of $10^{-3}$.
}\label{fig:conj_by_alpha_app}
\end{minipage}
\begin{minipage}{\textwidth}
\vspace*{1cm}
\centering
\begin{subfigure}{0.4\columnwidth}\includegraphics[width=.676\textwidth]{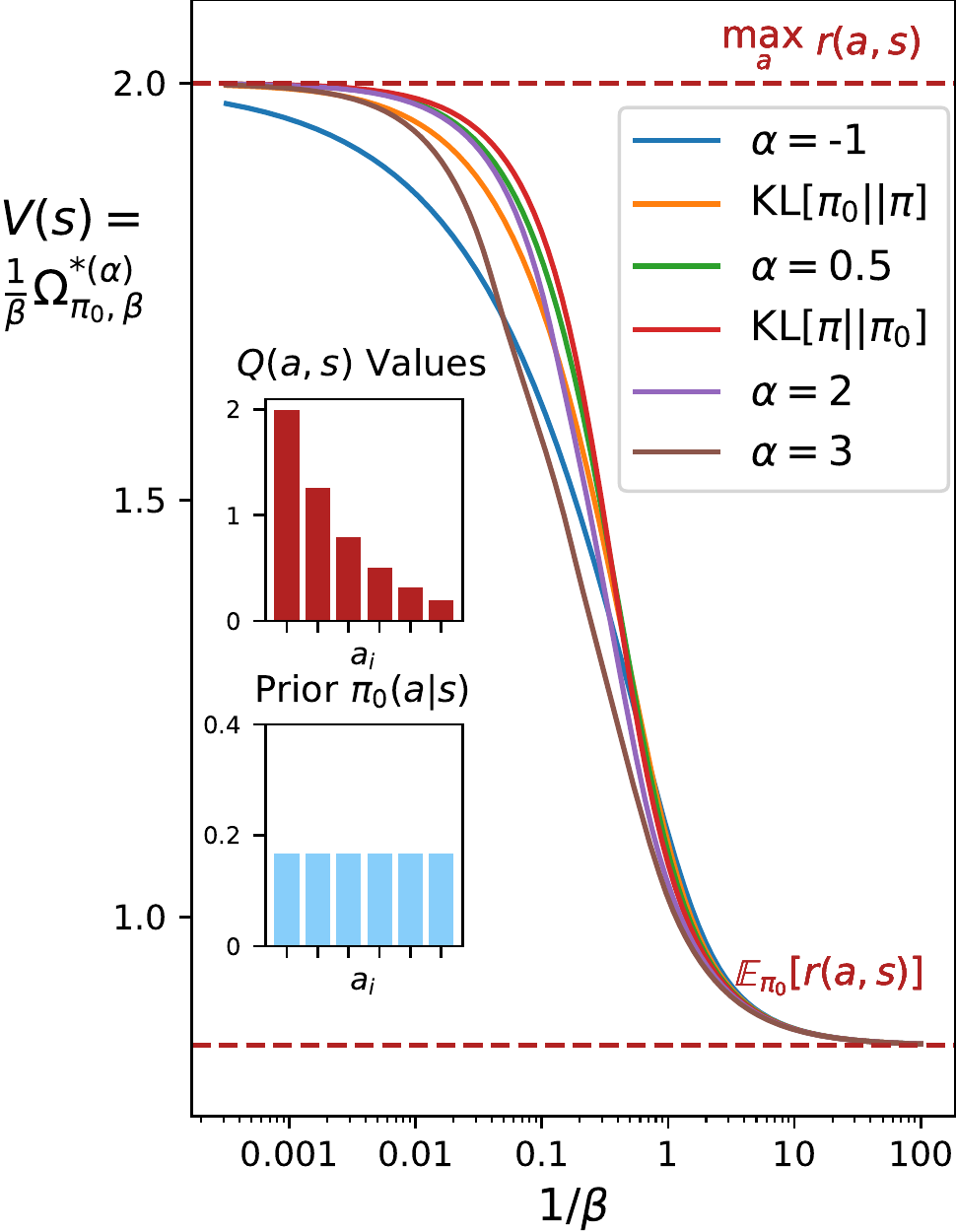}
\caption{Uniform $\pi_0(a|s)$}
\end{subfigure}\hspace*{.01\columnwidth}
\begin{subfigure}{0.4\columnwidth} 
\vspace*{-.1cm}\includegraphics[width=.676\textwidth]{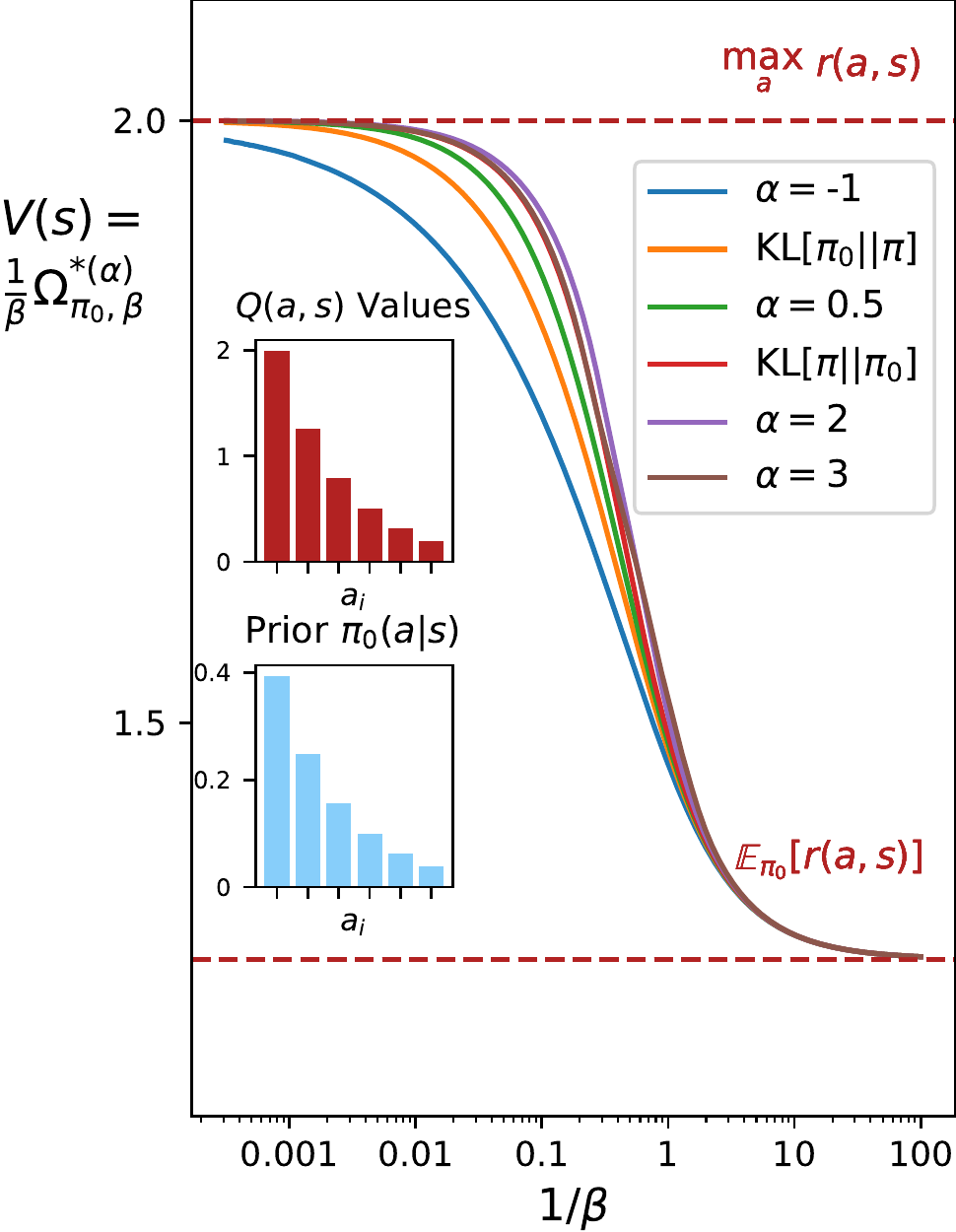}
\caption{$\pi_0(a|s) \propto r(a,s) $}
\end{subfigure}
\captionof{figure}{Value function $V(s) = \alphaconjm(Q)$ as a function of $\beta$ (x-axis) and $\alpha$ (colored lines), using $Q(a,s)$ and $\pi_0(a|s)$ from the left inset.   See \cref{eq:kl_bellman} and \cref{eq:alpha_bellman} for closed forms.}\label{fig:value_agg_main} 
\end{minipage}
\end{figure*}

%% file: appendix/new_new_app/05_feasible.tex
\section{Robust Set of Perturbed Rewards}\label{app:feasible_pf}\label{app:feasible}
In this section, we characterize the robust set of perturbed rewards to which a given policy $\pi(a|s)$ or $\mu(a,s)$ is robust, which also provides performance guarantees as in \cref{eq:generalization} and also describes the set of strategies available to the adversary.
For proving \myprop{feasible}, we focus our discussion on policy regularization with \textsc{kl} or $\alpha$-divergence regularization and compare with state-occupancy regularization in \myapp{mu_feasible}.

\subsection{Proof of \myprop{feasible}:  Robust Set of Perturbed Rewards for Policy Regularization} \label{app:pf_feasible}
We begin by stating two lemmas, which we will use to characterize the robust set of perturbed rewards.  All proofs are organized under paragraph headers below the statement of \myprop{feasible}.
\begin{restatable}{lemma}{conjugatezero} \label{lemma:conjugatezero}
For the worst-case reward perturbation $\prpi(a,s)$ associated with a given, normalized policy $\pi(a|s)$ 
and $\alpha$- or \textsc{kl}-divergence regularization, the conjugate function evaluates to zero,
\begin{align}
  \alphaconjmu(\prpi) = 0 \, .
\end{align}
\end{restatable}

\begin{restatable}{lemma}{increasing} \label{lemma:increasing}
The conjugate function $\alphaconjmu(\pr)$ is increasing in $\pr$.   In other words, if $\prnew(a,s) \geq \pr(a,s)$ for all $(a,s) \in \mathcal{A} \times \mathcal{S}$, then $\alphaconjmu(\prnew) \geq \alphaconjmu(\pr).$
\end{restatable}

\feasible*
\begin{proof}
Recall the adversarial optimization in \cref{eq:pedro_form} for a fixed $\mu(a,s) = \mu(s)\pi(a|s)$ 
\ifx\omitindices\undefined
\begin{align}
    \hspace*{-.2cm}   \min \limits_{\textcolor{black}{\prinit(a,s)}} 
  \langle {\mu(a,s)}, {r(a,s) - \textcolor{\highlight}{\prinit(a,s)}} \rangle  + \alphaconj \big( \textcolor{black}{\prinit} \big), \label{eq:pedro_form_pf}
\end{align}
\else
\begin{align}
    \hspace*{-.2cm}   \min \limits_{\textcolor{black}{\prinit}} 
  \langle {\mu}, \, {r - \textcolor{\highlight}{\prinit}} \rangle  + \alphaconj \big( \textcolor{black}{\prinit} \big), \label{eq:pedro_form_pf}
\end{align}
\fi
which we would like to transform into a constrained optimization.   From \mylemma{conjugatezero}, we know that $\alphaconj(\prpi) = 0$ for the optimizing argument $\prpi$ in \cref{eq:pedro_form_pf},  but it is not clear whether this should appear as an equality or inequality constraint.   We now show that the constraint $\alphaconj(\pr) \geq 0$ changes the value of the objective, whereas the constraint $\alphaconj(\pr) \leq 0$ does not change the value of the optimization.  

\textbf{$\mathbf{\geq}$ Inequality}
First, consider the optimization 
\ifx\omitindices\undefined
$\min_{\textcolor{black}{\prinit(a,s)}} \langle {\mu(a,s)}, {r(a,s) - \textcolor{black}{\prinit(a,s)}} \rangle$ subject to $\alphaconj(\pr) \geq 0$.   
\else
$\min_{\textcolor{black}{\prinit}} \langle {\mu}, {r - \textcolor{black}{\prinit}} \rangle$ subject to $\alphaconj(\pr) \geq 0$.  
\fi
From the optimizing argument $\prpi$, consider an increase in the reward perturbations $\prnew(a,s) \geq \prpi(a,s) \, \, \forall (a,s)$ where $\exists (a,s)$  s.t. $\mu(a,s) > 0$ and $\prnew(a,s) > \prpi(a,s)$.    By \mylemma{increasing}, we have $\alphaconj(\pr) \geq \alphaconj(\prpi) = 0$.   However, the objective now satisfies $\langle \mu, r - \prnew \rangle < \langle \mu, r - \prpi \rangle$ for fixed $\mu(a,s)$, which is a contradiction since $\prpi$ provides a global minimum of the convex objective in \cref{eq:pedro_form_pf}.  

\textbf{$\mathbf{\leq}$ Inequality}
We would like to show that this constraint does not introduce a different global minimum of \cref{eq:pedro_form_pf}.
Assume there exists $\prnew(a,s)$ with $\alphaconj(\prnew) < 0$ and $\langle \mu_{\pi} , r - \prnew \rangle < \langle \mu_{\pi} , r - \prpi \rangle$ for the occupancy measure $\mu_{\pi}$ associated with the given policy $\pi$.   By convexity of $\alphaconj(\pr)$, we know that a first-order Taylor approximation around $\prpi$ everywhere underestimates the function, 
\ifx\omitindices\undefined
$\alphaconj(\prnew) \geq \alphaconj(\prpi) + \langle \prnew(a,s)  - \prpi(a,s),  \nabla \alphaconj(\prpi) \rangle$.
\else
$\alphaconj(\prnew) \geq \alphaconj(\prpi) + \langle \prnew  - \prpi,  \nabla \alphaconj(\prpi) \rangle$.
\fi
Noting that $\mu_{\pi} = \nabla \alphaconj(\prpi)$ by the conjugate optimality conditions (\cref{eq:optimal_policy_as_grad}, \myapp{implications}), we have $\alphaconj(\prnew)- \alphaconj(\prpi) \geq  \langle \mu_{\pi} , \prnew  \rangle -\langle \mu_{\pi} , \prpi \rangle$.    This now introduces a contradiction, since we have assumed both that $\alphaconj(\prnew)- \alphaconj(\prpi) < 0$, and that $\prnew(a,s)$ provides a global minimum, where $\langle \mu_{\pi} , r - \prnew \rangle < \langle \mu_{\pi} , r - \prpi \rangle$ implies $\langle \mu_{\pi} , \prnew  \rangle -\langle \mu_{\pi} , \prpi \rangle > 0.$    Thus, including the inequality constraint $\alphaconj(\prnew) \leq 0$ cannot introduce different minima.

This constraint is consistent with the constrained optimization and generalization guarantee in \cref{eq:reward_perturbation_stmt}-(\ref{eq:generalization}), where it is clear that increasing the modified reward away from the boundary of the robust set  (i.e. decreasing $\pr(a,s)$ and $\alphaconj(\pr)$) is feasible for the adversary and preserves our performance guarantee.   
See \citet{eysenbach2021maximum} A2 and A6 for alternative reasoning.
\end{proof}

\textbf{Proof of \mylemma{conjugatezero}}
For $\alpha$-divergence policy regularization and a given $\pi(a|s)$, we substitute the worst-case reward perturbations $\prpi(a,s) = \frac{1}{\beta}\log_{\alpha}\frac{\pi(a|s)}{\pi_0(a|s)} + \myconst$ (\cref{eq:optimal_perturbations} or \cref{eq:alpha_worst_case_reward_perturb}) in the conjugate function $\alphaconjmu(\prpi)$ (\cref{eq:alpha_pi_conj_closed_form} or \mytab{conj_table}).   Assuming $\sum_a \pi(a|s) = \sum_a \pi_0(a|s) = 1$, we have  
\small
\begin{align*}
    \hspace*{-.1cm} \alphaconjmu(\prpi) &= 
    \frac{1}{\beta} \frac{1}{\alpha}  \sum \limits_{a} \pi_0(a|s) \exp_{\alpha}\big\{ \beta \cdot \big(\textcolor{blue}{\prpi(a,s)} - \psipi \big) \big\}^{\alpha} - \frac{1}{\beta} \frac{1}{\alpha} +  \psipi \\
    &=  \frac{1}{\beta} \frac{1}{\alpha}  \sum \limits_{a} \pi_0(a|s) \Big[ 1+ \beta (\alpha-1) \bigg( \textcolor{blue}{ \frac{1}{\beta} \frac{1}{\alpha-1} \big( \frac{\pi(a|s)}{\pi_0(a|s)}^{\alpha-1}  -1 \big) + \psipi } - \psipi \Big) \Big]_+^{\frac{\alpha}{\alpha-1}} - \frac{1}{\beta} \frac{1}{\alpha} +   \psipi \\ 
    &=  \frac{1}{\beta} \frac{1}{\alpha}  \sum \limits_{a} \pi_0(a|s)^{1-\alpha} \pi(a|s)^{\alpha} - \frac{1}{\beta} \frac{1}{\alpha} +   \psipi  \\
    &= 0.
\end{align*}
\normalsize
In the last line, we recall that $\psipi = \frac{1}{\beta} \frac{1}{\alpha} \sum_{a} \pi_0(a|s)- \frac{1}{\beta} \frac{1}{\alpha}  \sum_{a} \pi_0(a|s)^{1-\alpha} \pi(a|s)^{\alpha}$ from \cref{eq:pr_normalizer} or (\ref{eq:psi_pr}). 

For \textsc{kl} regularization, we plug $\prpi(a,s) = \frac{1}{\beta}\log\frac{\pi(a|s)}{\pi_0(a|s)}$ (\cref{eq:kl_perturbations},(\ref{eq:mu_for_kl_pi})) into the conjugate in \cref{eq:kl_pi_conj_closed_form} 
or \mytab{conj_table},
\footnotesize
\begin{align}
    \frac{1}{\beta}\Omega^{*}_{\pi_0,\beta}(\prpi) &= \frac{1}{\beta} \sum \limits_{a} \pi_0(a|s) \expof{\beta \, \textcolor{blue}{\prpi(a,s)}} - \frac{1}{\beta} = \frac{1}{\beta} \sum \limits_{a} \pi_0(a|s) \expof{\beta  \textcolor{blue}{\frac{1}{\beta}\log\frac{\pi(a|s)}{\pi_0(a|s)}}} - \frac{1}{\beta} =  \frac{1}{\beta} \sum \limits_{a}\pi(a|s)-\frac{1}{\beta} = 0 . \nonumber
\end{align}
\normalsize
\paragraph{Proof of \mylemma{increasing}}  \textcolor{black}{See \citet{husain2021regularized} Lemma 3.}

\subsection{Robust Set for $\alpha$-Divergence under $\mu(a,s)$ Regularization} \label{app:mu_feasible}
For state-action occupancy regularization and \textsc{kl} divergence, \mylemma{conjugatezero} holds with $\frac{1}{\beta} \Omega^*_{\mu_0,\beta}(\pr_{\mu}) = 0$ for normalized $\mu(a,s)$ and $\pr_{\mu}(a,s) = \frac{1}{\beta}\log \frac{\mu(a,s)}{\mu_0(a,s)}$.
However, the reasoning in \myapp{feasible_pf} no longer holds for $\alpha$-divergence regularization to a reference $\mu_0(a,s)$. Substituting the worst-case reward perturbations (\cref{eq:pr_mu} or (\ref{eq:alpha_pr_mu})) into the conjugate function (\cref{eq:conj_mu_result} or \mytab{conj_table})
\begin{align}
    \frac{1}{\beta}\Omega^{*(\alpha)}_{\mu_0,\beta}(\prmu) &= 
    \frac{1}{\beta} \frac{1}{\alpha}  \sum \limits_{a} \mu_0(a,s) \exp_{\alpha}\big\{ \beta \cdot \textcolor{blue}{\prmu(a,s)}  \big\}^{\alpha} - \frac{1}{\beta} \frac{1}{\alpha} \label{eq:conj_mu_123} \\
    &=  \frac{1}{\beta} \frac{1}{\alpha}  \sum \limits_{a} \mu_0(a,s) \Big[ 1+ \beta (\alpha-1) \bigg( \textcolor{blue}{ \frac{1}{\beta} \frac{1}{\alpha-1} \big( \frac{\mu(a,s)}{\mu_0(a,s)}^{\alpha-1}  -1 \big)} \Big) \Big]_+^{\frac{\alpha}{\alpha-1}} - \frac{1}{\beta} \frac{1}{\alpha} \nonumber \\
    &= \frac{1}{\beta} \frac{1}{\alpha}  \sum \limits_{a} \mu_0(a,s)^{1-\alpha} \mu(a,s)^{\alpha} - \frac{1}{\beta}\frac{1}{\alpha} \nonumber
\end{align}
whose value is not equal to $0$ in general and instead is a function of the given $\mu(a,s)$.   
This may result in the original environmental reward not being part of the robust set, since substituting $\pr(a,s)= 0$ into \cref{eq:conj_mu_123} results in $ \frac{1}{\beta}\Omega^{*(\alpha)}_{\mu_0,\beta}(\pr) = 0$.   

\headerv
\subsection{Plotting the $\alpha$-Divergence Feasible Set}\label{app:feasible_plot}
\headerv
To plot the boundary of the feasible set in the single step case, for the \textsc{kl} divergence regularization in two dimensions, we can simply solve for the $\pr(a_2,s)$ which satisfies the constraint $\sum_a \pi_0(a|s) \exp\{ \beta \, \pr(a|s) \} = 1$ for a given $\pr(a_1,s)$
\begin{align}
    \pr(a_2,s) = \frac{1}{\beta} \log \frac{1}{\pi_0(a_2|s)}(1- \pi_0(a_1|s) \expof{\beta \cdot \pr(a_1,s)}) \, .
\end{align}
The interior of the feasible set contains $\pr(a_1,s)$ and $\pr(a_2,s)$ that are greater than or equal to these values.

However,  we cannot analytically solve for the feasible set boundary for general $\alpha$-divergences, since the conjugate function $\alphaconjmu(\pr)$ 
depends on the normalization constant of $\pir(a,s)$.
Instead, we perform exhaustive search over a range of $\pr(a_1, s)$ and $\pr(a_2,s)$ values.  For each pair of candidate reward perturbations, we use \textsc{cvx-py} \citep{diamond2016cvxpy} to solve the conjugate optimization and evaluate $\alphaconjmu(\pr)$.
We terminate our exhaustive search and record the boundary of the feasible set when we find that $\alphaconjn(\pr) = 0$ within appropriate precision.     

%% file: appendix/new_new_app/07_value_form.tex
\headerv
\section{Value Form Reward Perturbations}\label{app:value_form}
\headerv

\input{appendix/new_new_app/07a_husain}

See \myapp{advantage_pf} for the proof of \myprop{advantage}, which equates the form of $\prv(a,s)$ and $\prpi(a,s)$ at optimality in the regularized \gls{MDP}.   

\subsection{Path Consistency (Comparison with \citet{nachum2017bridging,chow2018path})}\label{app:indifference_all}
We have seen in \mysec{value_form} and \myapp{indifference_all} that the path consistency conditions arise from the \textsc{kkt} conditions.   For \textsc{kl} divergence regularization, \citet{nachum2017bridging} observe the optimal policy $\piopt(a|s)$ and value $V_*(s)$ satisfy
\begin{align}
    r(a,s) + \gamma \transitionvstar - \frac{1}{\beta} \log \frac{\piopt(a|s)}{\pi_0(a|s)}  =  V_*(s) \, , \label{eq:path_kl}
\end{align}
where the Lagrange multiplier $\lambdaopt$ is not necessary since the $\piopt(a|s) > 0$ unless $\pi_0(a|s) = 0$ or the rewards or values are infinite.   This matches our condition in \cref{eq:kkt_path0}, where we can also recognize $\prpiopt(a,s) = \frac{1}{\beta} \log \frac{\piopt(a|s)}{\pi_0(a|s)} = r(a,s) + \gamma \transitionvstar -V_*(s) - \lambdaopt$ as the identity from \myprop{advantage}.
\citet{nachum2017bridging} use \cref{eq:path_kl}
to derive a learning objective, with the squared error $\mathbb{E}_{a, s, s\tick}\big[ \big( r(a,s) + \gamma \transitionv - \frac{1}{\beta} \log \frac{\pi(a|s)}{\pi_0(a|s)}  - V(s) \big)^2 \big] $ used as a loss for learning $\pi(a|s)$ and $V(s)$ (or simply $Q(a,s)$, \citet{nachum2017bridging} Sec. 5.1) using function approximation.


Similarly, \citet{chow2018path} consider a (scaled) Tsallis entropy regularizer, $\frac{1}{\beta}\Omega(\pi) = \frac{1}{\beta}\frac{1}{\alpha(\alpha-1)}(\sum_a \pi(a|s) - \sum_a \pi(a|s)^{\alpha})$.   For $\alpha = 2$, the optimal policy and value function satisfy
\begin{align}
r(a,s) + \gamma \transitionvstar  + \lambdaplus +  \frac{1}{\beta} \frac{1}{\alpha(\alpha-1)}  - \frac{1}{\beta} \frac{1}{\alpha-1} \pi(a|s)^{\alpha-1} &= V_*(s)  + \Lambda(s)  \label{eq:path1}
\end{align}
where $\Lambda(s)$ is a Lagrange multiplier whose value is learned in \citet{chow2018path}.
However, inspecting the proof of Theorem 3 in \citet{chow2018path}, we see that this multiplier is obtained via the identity $\Lambda(s) := \psiqopt - V_*(s) = \psipiopt$ (see 
\cref{eq:pr_normalizer}, \myapp{advantage_confirm}).   
Our notation in \cref{eq:path1} differs from \citet{chow2018path} in that we use $\frac{1}{\beta}$ as the regularization strength (compared with their $\alpha$).  We have also written \cref{eq:path1} to explicitly include the constant factors appearing in the $\alpha$-divergence.

In generalizing the path consistency equations,  we will consider the $\alpha$-divergence, with $\Omega(\pi) = \frac{1}{\alpha(\alpha-1)}((1-\alpha) \sum_a \pi_0(a|s) + \alpha \sum_a \pi(a|s) - \sum_a \pi_0(a|s)^{1-\alpha} \pi(a|s)^{\alpha})$.   Note that this includes an additional $\alpha$ factor which multiplies the $\sum_a \pi(a|s)$ term, compared to the Tsallis entropy considered in \citet{chow2018path}.  
In particular, this scaling will change the $\frac{1}{\beta} \frac{1}{\alpha(\alpha-1)}$ additive constant term in \cref{eq:path1}, to a term of $\frac{1}{\beta} \frac{1}{\alpha-1}$.  

Our expression for $\alpha$-divergence path consistency, derived using the identity $r(a,s) +  \gamma \transitionvstar  + \lambdaopt - \frac{1}{\beta} \log_{\alpha} \frac{\pi_*(a|s)}{\pi_0(a|s)} = V_*(s) + \psipiopt$ in \cref{eq:kkt_path0}, becomes
\begin{align}
 r(a,s) + \gamma \transitionvstar + \lambdaplus - \frac{1}{\beta} \log_{\alpha}\frac{\pi_*(a|s)}{\pi_0(a|s)}  &=  V_*(s) + \psipiopt \label{eq:path2}
\end{align}
where we have rearranged terms from \cref{eq:path} in the main text to compare with \cref{eq:path1}.   Note that we can recognize $\prpi(a,s) = \frac{1}{\beta} \log_{\alpha}\frac{\pi_*(a|s)}{\pi_0(a|s)} + \psipiopt$ and we substitute $\Lambda(s) := \psiqopt - V_*(s) = \psipiopt$ compared to \cref{eq:path1}.

\subsection{Indifference Condition}\label{app:indifference}
In the single step setting with \textsc{kl} divergence regularization, \citet{ortega2014adversarial} argue that the perturbed reward for the optimal policy is constant with respect to actions
\begin{align}
r(a) - \prpiopt(a) = c \, \, \, \forall a \in \mathcal{A} \, , \label{eq:single_step_indifference}
\end{align}
when $\prpiopt(a)$ are obtained using the optimal policy $\pi_*(a|s)$.
\citet{ortega2014adversarial} highlight that this is a well-known property of Nash equilibria in game theory where, for the optimal policy and worst-case adversarial perturbations, the agent obtains equal perturbed reward for each action and  thus is indifferent between them.

In the sequential setting, we can consider $Q(a,s) = r(a,s) + \gamma \transitionv$ as the analogue of the single-step reward or utility function.   Using our advantage function interpretation for the optimal policy in \myprop{advantage}, we directly obtain an indifference condition for the sequential setting
\begin{align}
Q_*(a,s) - \prpiopt(a,s) = V_*(s) \, \, \,
\end{align}
for actions with $\lambdaopt=0$ and nonzero probability under $\piopt(a|s)$.
Observe that $V_*(s)$ is a constant with respect to $a \in \mathcal{A}$ for a given state $s \in \mathcal{S}$.  
Recall that our notation for $V_*(s)$ omits its dependence on $\beta$ and $\alpha$.
This indifference condition indeed holds for arbitrary regularization strengths $\beta$ and choices of $\alpha$-divergence, since our proof of the advantage function interpretation in \myapp{advantage_pf} is general.
Finally, we emphasize that the indifference condition holds only for the optimal policy with a given reward $r(a,s)$ (see \myfig{perturb_opt}).

\headerv
\input{appendix/app_results/path_consistency}

\subsection{Confirming Optimality using Path Consistency and Indifference}\label{app:optimal_policy_confirm}
In \cref{fig:pyco}, we plotted the regularized policies and worst-case reward perturbations for various regularziation strength $\beta$ and \textsc{kl} divergence regularization.  In \myfig{pyco_indifference}, we now seek to certify the optimality of each policy using the path consistency or indifference conditions.  In particular, we confirm the following equality holds
\begin{align}
r(a,s) + \gamma \transitionv - \frac{1}{\beta} \log \frac{\pi(a|s)}{\pi_0(a|s)}  = V(s) \qquad \forall (a,s) \in \mathcal{A} \times \mathcal{S} \label{eq:confirming_path}
\end{align}
which aligns with the path consistency condition in \citep{nachum2017bridging}.  Compared with \cref{eq:path} or \cref{eq:path2}, \cref{eq:confirming_path} uses the fact that $\lambdaplus=0$ and $\psipr=0$ in the case of \textsc{kl} divergence regularization.
This equality also confirms the indifference condition since the right hand side $V_*(s)$ is a constant with respect to actions.  Finally, we can recognize our advantage function interpretation $Q_*(a,s) - \prpiopt(a,s) = V_*(s)$ in \cref{eq:confirming_path}, by substituting $Q_*(a,s) = r(a,s) + \gamma \transitionvstar$  and $\prpiopt(a,s) = \frac{1}{\beta} \log  \frac{\piopt(a|s)}{\pi_0(a|s)}$ for \textsc{kl} divergence regularization.

In \myfig{pyco_indifference}, we plot $r(a,s) + \gamma \transitionv - \frac{1}{\beta} \log \frac{\pi(a|s)}{\pi_0(a|s)} = Q(a,s) - \prpi(a,s)$ for each state-action pair to confirm that it yields a constant value and conclude that the policy and values are optimal.   Note that this constant is different across states based on the soft value function $V_*(s)$, which also depends on the regularization strength.    




%% file: appendix/new_new_app/07a_husain.tex
\subsection{Proof of \mythm{husain} (\citet{husain2021regularized})} \label{app:husain}
We rewrite the derivations of \citet{husain2021regularized} for our notation and setting, where $\Omega \big( \mu \big)$ represents a convex regularizer.   Starting from the regularized objective in \cref{eq:primal_reg},
 \ifx\omitindices\undefined
   \begin{align}
        \max \limits_{\mu(a,s) \in \mudomain} \obj_{\Omega, \beta}(\mu) &= \max \limits_{\mu(a,s) \in \mudomain} \,\,  \big\langle \mu(a,s), r(a,s)\rangle - \frac{1}{\beta} \Omega \big( \mu \big) ,
        \label{eq:reg_rl2} 
   \end{align}
   \else
      \begin{align}
        \max \limits_{\mu \in \mudomain} \obj_{\Omega, \beta}(\mu) &= \max \limits_{\mu \in \mudomain} \,\,  \big\langle \mu, r \rangle - \frac{1}{\beta} \Omega \big( \mu \big) ,
        \label{eq:reg_rl2} 
   \end{align}
   \fi
 note that the objective is \textit{concave}, as the sum of a linear term and the concave $-\Omega$.  Since the conjugate is an involution for convex functions, we can rewrite  $ \obj_{\Omega, \beta}(r) = -(-\obj_{\Omega, \beta}(r)) = - ((-\obj_{\Omega, \beta})^*)^* $, which yields
 \ifx\omitindices\undefined
 \begin{align}
     \obj_{\Omega, \beta}(r) &= \sup \limits_{\mu(a,s) \in \mudomain} - ((-\obj_{\Omega, \beta})^*)^* \nonumber \\
     &\overset{(1)}{=}\sup \limits_{\mu(a,s) \in \mudomain} -\bigg( \sup \limits_{r^\prime(a,s) \in \rprimedomain} \langle \mu(a,s), r^{\prime}(a,s) \rangle -  (-\obj_{\Omega, \beta})^*(r^{\prime}) \bigg) \nonumber \\[1.25ex]
     &= \sup \limits_{\mu(a,s) \in \mudomain}   \inf \limits_{r^\prime(a,s) \in \rprimedomain}  \langle \mu(a,s), -r^{\prime}(a,s) \rangle +  (-\obj_{\Omega, \beta})^*(r^{\prime})  \nonumber \\
      &\overset{(2)}{=}  \sup \limits_{\mu(a,s) \in \mudomain} \inf \limits_{r^\prime(a,s) \in \rprimedomain}  \langle \mu(a,s), r^{\prime}(a,s) \rangle +  (-\obj_{\Omega, \beta})^*(-r^{\prime})  \nonumber\\
     &\overset{(3)}{=} \sup \limits_{\mu(a,s) \in \mudomain}  \inf \limits_{r^\prime(a,s) \in \rprimedomain}  \langle \mu(a,s), r^{\prime}(a,s) \rangle +  \bigg( \sup \limits_{\mu^{\prime}} \, \langle \mu^{\prime}(a,s), -r^{\prime}(a,s) \rangle + \obj_{\Omega, \beta}(\mu^{\prime}) \bigg) \nonumber \\
     &\overset{(4)}{=}  \sup \limits_{\mu(a,s) \in \mudomain}  \inf \limits_{r^\prime(a,s) \in \rprimedomain}  \langle \mu(a,s), r^{\prime}(a,s) \rangle +  \bigg( \sup \limits_{\mu^{\prime}} \, \langle \mu^{\prime}(a,s), -r^{\prime}(a,s) \rangle + \langle \mu^{\prime}(a,s), r(a,s)\rangle - \frac{1}{\beta} \Omega(\mu^{\prime})   \bigg) \nonumber \\
     &=  \sup \limits_{\mu(a,s) \in \mudomain}  \inf \limits_{r^\prime(a,s) \in \rprimedomain}  \langle \mu(a,s), r^{\prime}(a,s) \rangle + \frac{1}{\beta}  \bigg( \sup \limits_{\mu^{\prime}}  \, \langle \mu^{\prime}(a,s), \beta \cdot \big( r(a,s) - r^{\prime}(a,s) \big)  \rangle -  \Omega(\mu^{\prime})   \bigg) \nonumber \\[1.25ex]
      &\overset{(5)}{=} \sup \limits_{\mu(a,s) \in \mudomain}  \inf \limits_{r^\prime(a,s) \in \rprimedomain}  \langle \mu(a,s), r^{\prime}(a,s) \rangle +  \frac{1}{\beta} \Omega^{*} \big( \beta \cdot ( r - r^{\prime} ) \big)   \nonumber \\
      &\overset{(6)}{=}  \inf \limits_{r^\prime(a,s) \in \rprimedomain}  \sup \limits_{\mu(a,s) \in \mudomain}  \langle \mu(a,s), r^{\prime}(a,s) \rangle +  \frac{1}{\beta}  \Omega^{*} \big( \beta \cdot ( r - r^{\prime} ) \big)  \label{eq:end}
 \end{align}
    \else
     \begin{align}
     \obj_{\Omega, \beta}(r) &= \sup \limits_{\mu \in \mudomain} - ((-\obj_{\Omega, \beta})^*)^* \nonumber \\
     &\overset{(1)}{=}\sup \limits_{\mu \in \mudomain} -\bigg( \sup \limits_{r^\prime \in \rprimedomain} \langle \mu, r^{\prime} \rangle -  (-\obj_{\Omega, \beta})^*(r^{\prime}) \bigg) \nonumber \\[1.25ex]
     &= \sup \limits_{\mu \in \mudomain}  \,\, \inf \limits_{r^\prime \in \rprimedomain}  \langle \mu, -r^{\prime} \rangle +  (-\obj_{\Omega, \beta})^*(r^{\prime})  \nonumber \\
      &\overset{(2)}{=}  \sup \limits_{\mu \in \mudomain}  \,\, \inf \limits_{r^\prime \in \rprimedomain}  \langle \mu, r^{\prime} \rangle +  (-\obj_{\Omega, \beta})^*(-r^{\prime})  \nonumber\\
     &\overset{(3)}{=} \sup \limits_{\mu \in \mudomain}  \,\,  \inf \limits_{r^\prime \in \rprimedomain}  \langle \mu, r^{\prime} \rangle +  \bigg( \sup \limits_{\mu^{\prime}} \, \langle \mu^{\prime}, -r^{\prime} \rangle + \obj_{\Omega, \beta}(\mu^{\prime}) \bigg) \nonumber \\
     &\overset{(4)}{=}  \sup \limits_{\mu \in \mudomain}  \,\, \inf \limits_{r^\prime \in \rprimedomain}  \langle \mu, r^{\prime} \rangle +  \bigg( \sup \limits_{\mu^{\prime}} \, \langle \mu^{\prime}, -r^{\prime} \rangle + \langle \mu^{\prime}, r\rangle - \frac{1}{\beta} \Omega(\mu^{\prime})   \bigg) \nonumber \\
     &=  \sup \limits_{\mu \in \mudomain}  \,\, \inf \limits_{r^\prime \in \rprimedomain}  \langle \mu, r^{\prime} \rangle + \frac{1}{\beta}  \bigg( \sup \limits_{\mu^{\prime}}  \, \langle \mu^{\prime}, \beta \cdot \big( r - r^{\prime} \big)  \rangle -  \Omega(\mu^{\prime})   \bigg) \nonumber \\[1.25ex]
      &\overset{(5)}{=} \sup \limits_{\mu \in \mudomain}  \,\, \inf \limits_{r^\prime \in \rprimedomain}  \langle \mu, r^{\prime} \rangle +  \frac{1}{\beta} \Omega^{*} \big( \beta \cdot ( r - r^{\prime} ) \big)   \nonumber \\
      &\overset{(6)}{=}  \inf \limits_{r^\prime \in \rprimedomain}  \,\, \sup \limits_{\mu \in \mudomain}  \langle \mu, r^{\prime} \rangle +  \frac{1}{\beta}  \Omega^{*} \big( \beta \cdot ( r - r^{\prime} ) \big)  \label{eq:end}
 \end{align}
    \fi
 where $(1)$ applies the definition of the conjugate of $(-\obj_{\Omega, \beta})^*$, $(2)$ reparameterizes the optimization in terms of $r^{\prime} \rightarrow -r^{\prime}$, $(3)$ is the conjugate for $(-\obj_{\Omega, \beta})$, and $(4)$ uses the definition of the regularized RL objective for occupancy measure $\mu^{\prime}(a,s)$ with the reward $r(a,s)$.  Finally, $(5)$ recognizes the inner maximization as the conjugate function for a modified reward and $(6)$ swaps the order of $\inf$ and $\sup$ assuming the problem is feasible.   
 
Note that \cref{eq:end} is a standard unregularized \textsc{rl} problem with modified reward $\mr(a,s)$.  As in \mysec{prelim}, introducing Lagrange multipliers $V(s)$ to enforce the flow constraints and $\lambdaplus$ for the nonnegativity constraint,
 \ifx\omitindices\undefined
\small
\begin{align}
 \obj_{\Omega, \beta}(r)&=   
\inf \limits_{r^\prime(a,s)}   \inf \limits_{V(s), \lambdaplus} \sup \limits_{\mu(a,s)}
 \blangle \mu(a,s), r^{\prime}(a,s) + \gamma \transitionv - V(s) + \lambdaplus \brangle  \label{eq:husain_lagr} \\
 &\phantom{\inf \limits_{r^\prime(a,s)}  \sup \limits_{\mu(a,s)} \inf \limits_{V(s), \lambdaplus} \blangle } +  \frac{1}{\beta}  \Omega^{*} \big( \beta \cdot ( r - r^{\prime} ) \big) + (1-\gamma) \blangle \nu_0(s), V(s) \brangle \nonumber
\end{align}
\normalsize
\else
\begin{align}
 \obj_{\Omega, \beta}(r)&=   
\inf \limits_{r^\prime}  \,  \inf \limits_{V, \lambda} \, \sup \limits_{\mu} \,\,
 \blangle \mu, r^{\prime} + \gamma \transitionvinds - V + \lambda \brangle  +  \frac{1}{\beta}  \Omega^{*} \big( \beta \cdot ( r - r^{\prime} ) \big) + (1-\gamma) \blangle \nu_0, V\brangle  \label{eq:husain_lagr}
\end{align}
\normalsize
\fi
Now, eliminating $\mu(a,s)$ yields the condition 
\begin{align}
     r^{\prime}(a,s) + \gamma \transitionv - V(s) + \lambdaplus = 0 \quad \implies \quad 
     V(s) = r^{\prime}(a,s) +  \gamma \transitionv + \lambdaplus \, .
     \label{eq:husain_result}
\end{align}
Letting $\prv(a,s) = r(a,s) - r^{\prime}(a,s)$, we can consider \cref{eq:husain_result} as a constraint and rewrite \cref{eq:husain_lagr} as
 \ifx\omitindices\undefined
\begin{align}
   \obj_{\Omega, \beta}(r)=   
\inf \limits_{\prv(a,s)} \inf \limits_{V(s), \lambdaplus} &(1-\gamma) \blangle \nu_0(s), V(s) \brangle + \frac{1}{\beta}  \Omega^{*} \big( \beta \cdot \prv(a,s)  \big) 
  \label{eq:husain_constrained} \\
 \text{subj. to } V(s) &= r(a,s) +  \gamma \transitionv - \prv(a,s) + \lambdaplus \nonumber
\end{align}
\else
\begin{align}
   \obj_{\Omega, \beta}(r)=   
\inf \limits_{\prv} \inf \limits_{V, \lambda} &\, \, (1-\gamma) \blangle \nu_0, V \brangle + \frac{1}{\beta}  \Omega^{*} \big( \beta \cdot \prv  \big) 
  \label{eq:husain_constrained} \\
 \text{subj. to } V(s) &= r(a,s) +  \gamma \transitionv - \prv(a,s) + \lambdaplus \nonumber
\end{align}
\fi
which matches \cref{thm:husain}.   See \citet{husain2021regularized} for additional detail.




%% file: appendix/app_results/path_consistency.tex
\begin{figure*}[t]
\vspace*{-.3cm}
\begin{subfigure}{.2\textwidth}
\includegraphics[width=\textwidth]{figs/pyco_perturb/pyco_world.pdf}\vspace*{.2cm}
\caption{Environment \\
($r(a,s) = -1$ for water, \\
$r(a,s) = 5$ for goal)}\label{fig:pyco_env2}
\end{subfigure}
\begin{subfigure}{.26\textwidth}
\includegraphics[width=\textwidth]{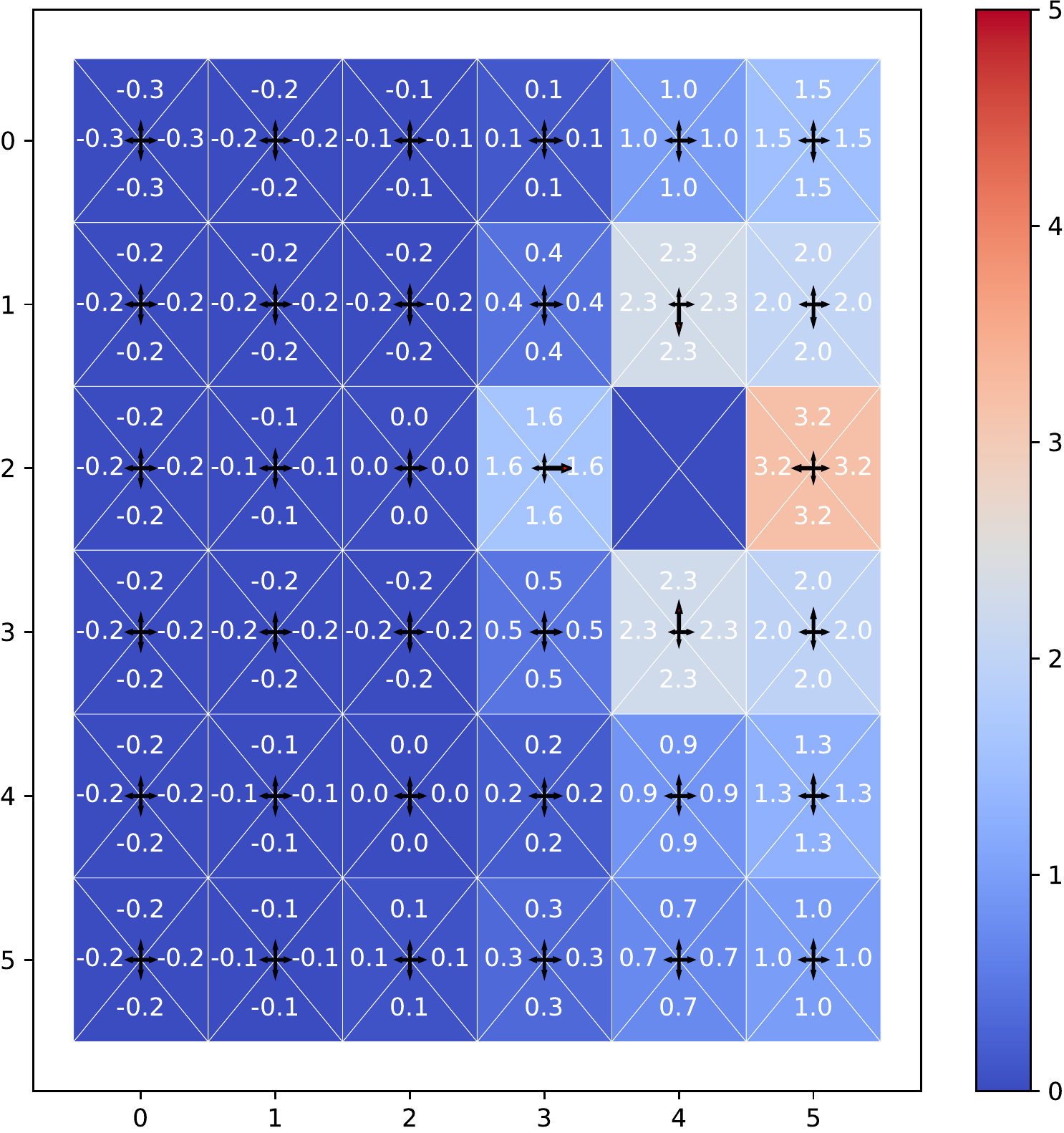}
\caption{$\beta = 0.2$ (High Reg.)}\label{fig:pyco_beta02}
\end{subfigure}
\begin{subfigure}{.26\textwidth}
\includegraphics[width=\textwidth]{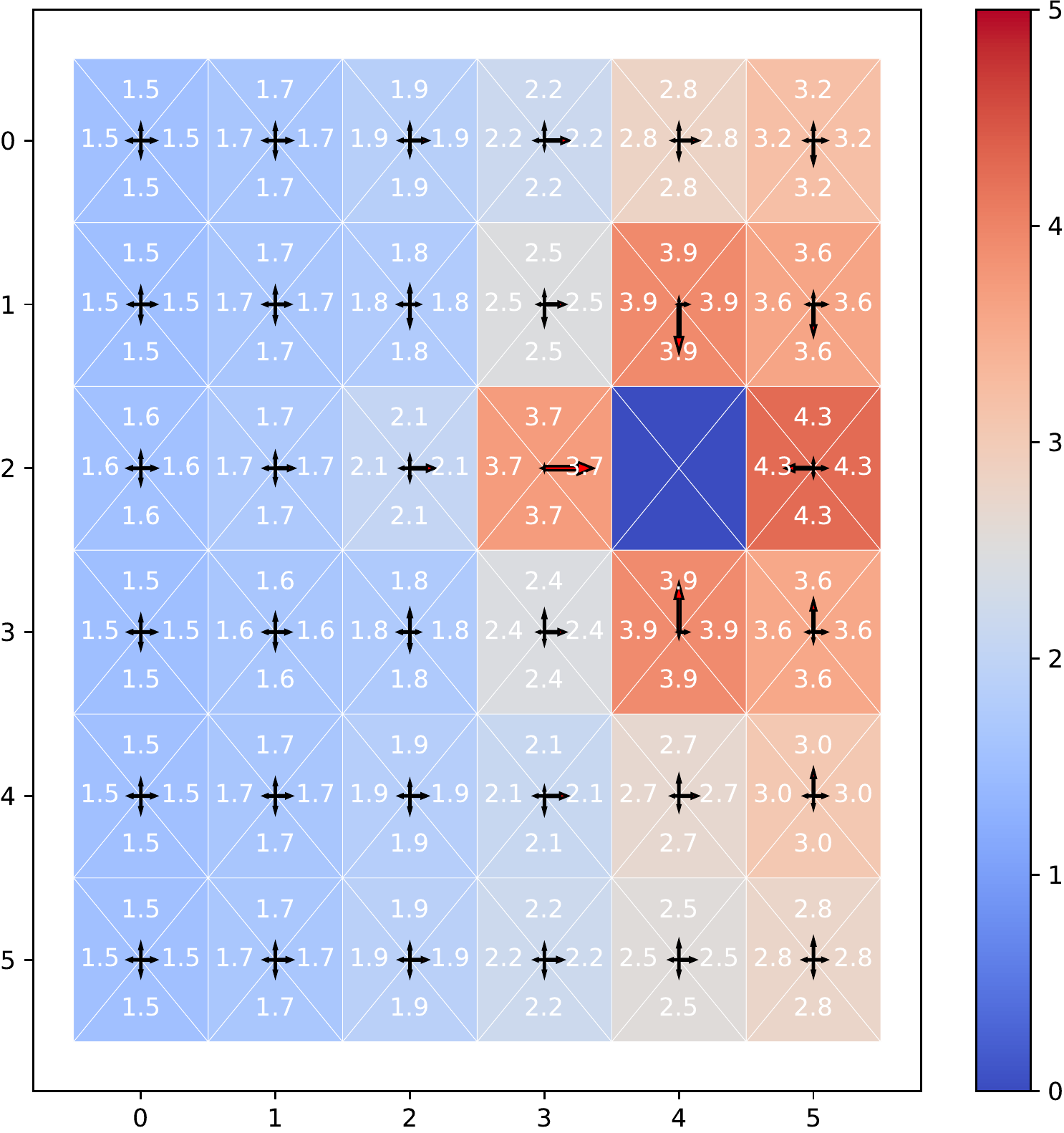}
\caption{$\beta = 1$}\label{fig:pyco_beta1}
\end{subfigure}
\begin{subfigure}{.26\textwidth}
\includegraphics[width=\textwidth]{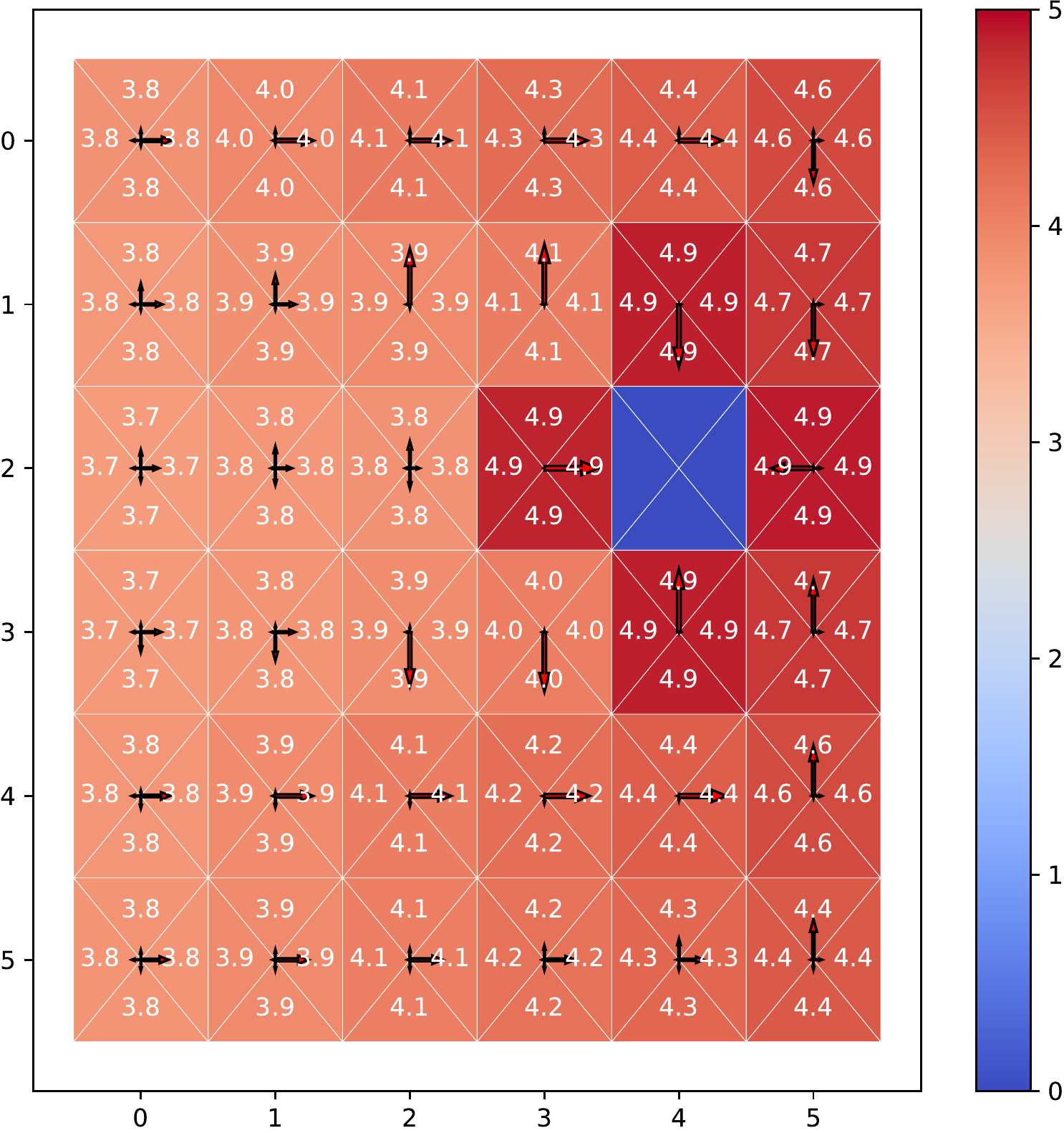}
\caption{$\beta = 10$ (Low Reg.)}\label{fig:pyco_beta02}
\end{subfigure}
\caption{\label{fig:pyco_indifference} \textbf{Confirming Optimality of the Policy.}  We show the perturbed rewards $\mr(a,s) = Q(a,s) - \pr(a,s)$ for policies trained with \textsc{kl} divergence regularization.  The indifference condition holds in all cases, with $\mr(a,s) = c(s)$ for each state-action pair, with $c(s) = V(s)$ for \textsc{kl} regularization.  This confirms that the policy is optimal \citep{ortega2014adversarial, nachum2017bridging}.}
\end{figure*}

%% file: appendix/new_new_app/08_entropy_reg.tex
\headerv
\section{Comparing Entropy and Divergence Regularization}\label{app:entropyvsdivergence}
\headerv
In this section, we provide proofs and discussion to support our observations in \cref{sec:entropy_vs_div} on the benefits of divergence regularization over entropy regularization.

\headerv
\subsection{Tsallis Entropy and $\alpha$-Divergence}\label{app:tsallis}
\headerv
To show a relationship between the Tsallis entropy and the $\alpha$-divergence, we first recall the definition of the $q$-exponential function $\log_q$ \citep{tsallis2009introduction}.   We also define $\log_{\alpha}(u)$, with $\alpha = 2-q$, so that our use of $\log_{\alpha}(u)$ matches \citet{lee2019tsallis} Eq. (5)
\begin{align}
  \log_{q}(u) = \frac{1}{1-q} \bigg( u^{1-q} - 1\bigg) \qquad   \log_{\alpha}(u) \coloneqq  \log_{2-q}(u) = \frac{1}{\alpha -1} \bigg( u^{\alpha-1} - 1\bigg)
\end{align}
The Tsallis entropy of order $q$ \citep{tsallis2009introduction, naudts2011generalised} can be expressed using either $\log_q$ or $\log_{\alpha}$
\begin{align}
    H^T_{q}[\pi(a)] = \frac{1}{q-1} \bigg( 1 -  \sum \limits_{a \in \mathcal{A}}  \pi(a)^{q} \bigg) =&  \sum \limits_{a \in \mathcal{A}} \pi(a) \log_q\big( \frac{1}{\pi(a)} \big)\label{eq:tsallis1} \\
    =   - &\sum \limits_{a \in \mathcal{A}} \pi(a) \cdot \log_{2-q} \big(\pi(a) \big) \label{eq:tsallis3}
\end{align}
\cref{eq:tsallis1} and \cref{eq:tsallis3} mirror the two equivalent ways of writing the Shannon entropy for $q=1$.  In particular, we have $q=2-q$ and $H_1[\pi(a)] = \sum \pi(a) \log \frac{1}{\pi(a)} =-\sum \pi(a) \log \pi(a)$.   See \citet{naudts2011generalised} Ch. 7 for discussion of these two forms of the deformed logarithm.

To connect the Tsallis entropy and the $\alpha$-divergence in 
\cref{eq:alpha_def}, we can consider a uniform reference measure $\pi_0(a) = 1 \,\, \forall a$.   For normalized $\sum_a \pi(a) = 1$,
\small
\begin{align}
    D_{\alpha}[\pi_0(a):\pi(a)] &= \frac{1}{\alpha(1-\alpha)} \left((1-\alpha) \sum \limits_{a \in \mathcal{A}} \pi_0(a) + \alpha \sum \limits_{a \in \mathcal{A}} \pi(a) - \sum \limits_{a \in \mathcal{A}} \pi_0(a)^{1-\alpha} \pi(a|s)^{\alpha} \right) \\
    &= \frac{1}{\alpha(1-\alpha)} \left( \alpha \textcolor{blue}{+(1-\alpha)} -\sum \limits_{a \in \mathcal{A}} \pi(a|s)^{\alpha} \right)  + \frac{1}{\alpha(1-\alpha)} \left( (1-\alpha) \sum \limits_{a \in \mathcal{A}} \pi_0(a) \textcolor{blue}{-(1-\alpha)} \right) \\
    &= - \frac{1}{\alpha} H^T_{\alpha}[\pi(a)] + c
\end{align}
\normalsize
which recovers the negative Tsallis entropy of order $\alpha$, up to an multiplicative factor $\frac{1}{\alpha}$ and additive constant.   
Note that including this constant factor via $\alpha$-divergence regularization allows us to avoid an inconvenient $\frac{1}{\alpha}$  factor in optimal policy solutions (\cref{eq:optimal_policy}) compared with Eq. 8 and 10 of \citet{lee2019tsallis}.

\headerv
\subsection{Non-Positivity of $\prpi(a,s)$ for Entropy Regularization}\label{app:entropy_nonnegative}
\headerv
We first derive the worst-case reward perturbations for entropy regularization, before analyzing the sign of these reward perturbations for various values of $\alpha$ in \myprop{nonpositive}.  
In particular, we have $\propt(a,s) \leq 0$ for entropy regularization with $0 < \alpha \leq 1$, which includes Shannon entropy regularization at $\alpha = 1$.  This implies that the modified reward $\mropt(a,s) \geq r(a,s)$ for all $(a,s)$.


\begin{restatable}{lemma}{increasing}
\label{lemma:ent_reg}
The worst-case reward perturbations for Tsallis entropy regularization correspond to 
\begin{align}
    \propt(a,s) = \frac{1}{\beta} \log_{\alpha} \pi(a|s) + \frac{1}{\beta} \frac{1}{\alpha} \big(1 - \sum \limits_{a \in \mathcal{A}} \pi(a|s)^{\alpha} \big) \,. \label{eq:ptwise_perturb_ent_app}
\end{align}
with limiting behavior of $\propt(a,s) = \frac{1}{\beta} \log \pi(a|s)$ for Shannon entropy regularization as $\alpha \rightarrow 1$.
\end{restatable}
\begin{proof}
We can write the Tsallis entropy using an additional constant $k$, with $k=\frac{1}{\alpha}$ mirroring the $\alpha$-divergence
\begin{align}
     H^T_{\alpha}[\pi] = \frac{k}{\alpha-1} \bigg( \sum \limits_{a \in \mathcal{A}} \pi(a|s) - \sum \limits_{a \in \mathcal{A}} \pi(a|s)^{\alpha} \bigg) \, .  \label{eq:tsallis2} 
\end{align}
Note that we use \textit{negative} Tsallis entropy for regularization since the entropy is concave.   Thus, the worst case reward perturbations correspond to the condition $\propt(a,s)= \nabla \frac{1}{\beta}\Omega^{(H_{\alpha})}_{\pi_0}(\mu) = - \nabla \mathbb{E}_{\mu(s)}\big[ H^T_{\alpha}[\pi]\big]$. Differentiating with respect to $\mu(a,s)$ using similar steps as in \myapp{conj3} \cref{eq:pr_equals_grad}-\eqref{eq:worst_case_reward_perturb}, we obtain
\begin{align}
\propt(a,s) =  k \cdot \frac{1}{\beta}  \alpha\,  \log_{\alpha} \pi(a|s) +    k \cdot \frac{1}{\beta} (\alpha-1) H^T_{\alpha}[\pi] 
\end{align}
For $k = \frac{1}{\alpha}$ and $(\alpha-1)H^T_{\alpha}[\pi] = ( \sum_a \pi(a|s) - \sum_a \pi(a|s)^{\alpha} )$, we obtain \cref{eq:ptwise_perturb_ent_app}.
\end{proof}

\begin{restatable}{proposition}{nonpositive}
\label{prop:nonpositive}
For $0 < \alpha \leq 1$ and $\beta > 0$,
the worst-case reward perturbations with Tsallis entropy regularization from \mylem{ent_reg} are non-positive, with $\propt(a,s) \leq 0$.  This implies that the entropy-regularized policy is robust to only pointwise reward increases for these values of $\alpha$.
\end{restatable}
\headerv
\begin{proof}
We first show $\log_{\alpha} \pi(a|s) \leq 0 $ for $0 < \pi(a|s) \leq 1$ and any $\alpha$.  Note that we may write 
\begin{align}
    \log_{\alpha} \pi(a|s) = \int_1^{\pi(a|s)} u^{\alpha-2} du = \frac{1}{\alpha-1}u^{\alpha-1}\bigg|^{\pi(a|s)}_1 = \frac{1}{\alpha-1}(\pi(a|s)^{\alpha-1}-1) \, .
\end{align}
Since $u^{\alpha-2}$ is a non-negative function for $0 \leq \pi(a|s) \leq 1$, then $\int_1^{\pi(a|s)}  u^{\alpha-2} du \leq 0$.   

To analyze the second term, consider $0 < \alpha \leq 1$.  We know that $\pi(a|s)^{\alpha} \geq \pi(a|s)$ for $0 < \pi(a|s)\leq 1$, so that $\sum_a \pi(a|s)^{\alpha} \geq \sum_a \pi(a|s) = 1$.  Thus, we have $\sum_a \pi(a|s) - \sum_a \pi(a|s)^{\alpha} \leq 0$ and $\alpha >0$ implies $ \frac{1}{\alpha}( \sum_a \pi(a|s) - \sum_a \pi(a|s)^{\alpha}) \leq 0$.  Since both terms are non-positive, we have $\propt(a,s) \leq 0$ for $0 < \alpha \leq 1$ as desired.

However, for  $\alpha > 1$ or $\alpha < 0$, we cannot guarantee the reward perturbations are non-positive.   Writing the second term in \cref{eq:ptwise_perturb_ent_app} as $\frac{\alpha-1}{\alpha} H^T_{\alpha}[\pi]$, we first observe that that $H^T_{\alpha}[\pi] \geq 0$.   Now, $\alpha > 1$ or $\alpha < 0$ implies that  $\frac{\alpha-1}{\alpha} > 0$, so that the second term is non-negative,   compared to the first term, which is non-positive.   
\end{proof}

\newcommand{\dualvarq}{Q}
\headerv
\subsection{Bounding the Conjugate Function}\label{app:qmax}
\textbf{Conjugate for a Fixed $\alpha,\beta$:}
We follow similar derivations as \citet{lee2019tsallis} to bound the value function for general $\alpha$-divergence regularization instead of entropy regularization.    We are interested in providing a bound for $\alphaconjn(\dualvarq)$ with fixed $\alpha,\beta$.
To upper bound the conjugate, consider the optimum over each term separately
\ifx\omitindices\undefined
\small
\begin{align*}
    \alphaconjn(\dualvarq) &= \max \limits_{\pi} \blangle \pi(a|s), \dualvarq(a,s) \brangle - \frac{1}{\beta} D_{\alpha}[\pi_0(a|s) || \pi(a|s)] \\
    &\leq   \max \limits_{\pi} \blangle \pi(a|s), \dualvarq(a,s) \brangle - \min \limits_{\pi} \frac{1}{\beta} D_{\alpha}[\pi_0(a|s) || \pi(a|s)] \\
    &=  \max \limits_{\pi} \blangle \pi(a|s), \dualvarq(a,s) \brangle  - 0 \\ 
    &= \dualvarq(a_{\max},s) \,.
\end{align*}
\normalsize
\else
\small
\begin{align*}
    \alphaconjn(\dualvarq) &= \max \limits_{\pi \in \Delta^{|\mathcal{A}|}} \blangle \pi, \dualvarq \brangle - \frac{1}{\beta} D_{\alpha}[\pi_0 : \pi] \\
    &\leq   \max \limits_{\pi \in \Delta^{|\mathcal{A}|}} \blangle \pi, \dualvarq \brangle - \min \limits_{\pi} \frac{1}{\beta} D_{\alpha}[\pi_0 : \pi] \\
    &=  \max \limits_{\pi \in \Delta^{|\mathcal{A}|}} \blangle \pi, \dualvarq \brangle  - 0 \\ 
    &= \dualvarq(a_{\max},s) \,.
\end{align*}
\normalsize
\fi
where we let $a_{\max}:= \argmax_{a} \dualvarq(a,s)$.  

We can also lower bound the conjugate function in terms of $\max_a \dualvarq(a,s)$. Noting that any policy $\pi(a|s)$ provides a lower bound on the value of the maximization objective, we consider $\pi_{\max}(a|s) = \delta( a = a_{\max})$.
For evaluating $\pi_{\max}(a|s)^{\alpha}$, we assume $0^{\alpha}=0$ for $\alpha > 0$ and undefined otherwise.   Thus, we restrict our attention to $\alpha > 0$ to derive the lower bound
\ifx\omitindices\undefined
\small
\begin{align*}
    \alphaconjn(\dualvarq) &= \max \limits_{\pi} \blangle \pi(a|s), \dualvarq(a,s) \brangle - \frac{1}{\beta} D_{\alpha}[\pi_0(a|s) || \pi(a|s)] \\
    &\geq    \blangle \pi_{\max}(a|s), \dualvarq(a,s) \brangle - \frac{1}{\beta} D_{\alpha}[\pi_0(a|s) || \pi_{\max}(a|s)] \\
    &\overset{(1)}{=}  \dualvarq(a_{\max},s) - \frac{1}{\beta} \Big(\frac{1}{\alpha(1-\alpha)} -  \frac{1}{\alpha(1-\alpha)} \pi_0(a_{\max}|s)^{1-\alpha} 1^{\alpha}\Big)\\
    &= \dualvarq(a_{\max},s) + \frac{1}{\beta} \frac{1}{\alpha} \frac{1}{1-\alpha}\left( \pi_0(a_{\max}|s)^{1-\alpha} - 1 \right)  \\
    &= \dualvarq(a_{\max},s) + \frac{1}{\beta} \frac{1}{\alpha} \log_{2-\alpha} \pi_0(a_{\max}|s).
\end{align*}
\normalsize
\else
\small
\begin{align*}
    \alphaconjn(\dualvarq) &= \max \limits_{\pi \in \Delta^{|\mathcal{A}|}} \blangle \pi, \dualvarq \brangle - \frac{1}{\beta} D_{\alpha}[\pi_0 : \pi] \\
    &\geq    \blangle \pi_{\max}, \dualvarq \brangle - \frac{1}{\beta} D_{\alpha}[\pi_0 : \pi_{\max}] \\
    &\overset{(1)}{=}  \dualvarq(a_{\max},s) - \frac{1}{\beta} \Big(\frac{1}{\alpha(1-\alpha)} -  \frac{1}{\alpha(1-\alpha)} \pi_0(a_{\max}|s)^{1-\alpha} 1^{\alpha}\Big)\\
    &= \dualvarq(a_{\max},s) + \frac{1}{\beta} \frac{1}{\alpha} \frac{1}{1-\alpha}\left( \pi_0(a_{\max}|s)^{1-\alpha} - 1 \right)  \\
    &= \dualvarq(a_{\max},s) + \frac{1}{\beta} \frac{1}{\alpha} \log_{2-\alpha} \pi_0(a_{\max}|s).
\end{align*}
\normalsize
\fi
where $(1)$ uses $\pi_{\max}(a|s) = \delta( a = a_{\max})$ and simplifies terms in the $\alpha$-divergence. 
One can confirm that $\frac{1}{\alpha} \log_{2-\alpha} \pi_0(a_{\max}|s) = \frac{1}{\alpha}\frac{1}{1-\alpha}(\pi_0(a_{\max}|s)^{1-\alpha}-1) \leq 0$ for $\alpha > 0$.
Combining these bounds, we can write
\begin{align}
  \text{ For $\alpha > 0$, } \qquad \dualvarq(a_{\max},s) + \frac{1}{\beta}\frac{1}{\alpha} \log_{2-\alpha} \pi_0(a_{\max}|s) \, \, \leq   \, \, \alphaconjn(\dualvarq) \, \, \leq   \, \,  \dualvarq(a_{\max},s). 
\end{align}
\textbf{Conjugate as a Function of $\beta$:}
Finally, we can analyze the conjugate as a function of $\beta$.   As $\beta \rightarrow 0$ and $1/\beta \rightarrow \infty$, the divergence regularization forces $\pi(a|s) = \pi_0(a|s)$ for any $\alpha$ and the conjugate $\alphaconjn(\dualvarq) = \langle \pi(a|s), \dualvarq(a,s) \rangle - \frac{1}{\beta}\Omega(\pi) \rightarrow \langle \pi(a|s), \dualvarq(a,s) \rangle$.   
As $\beta \rightarrow \infty$ and $1/\beta \rightarrow 0$, the unregularized objective yields a deterministic optimal policy with $\pi(a|s) = \max_a \dualvarq(a,s)$.  In this case, the conjugate $\alphaconjn(\dualvarq)\rightarrow \max_a \dualvarq(a,s)$.  Thus, treating the conjugate as a function of $\beta$, we obtain
\begin{align}
\mathbb{E}_{\pi_0(a|s)}\left[ \dualvarq(a,s) \right] \leq \alphaconjn(\dualvarq) \leq \max_a \dualvarq(a,s)    
\end{align}

For negative values of $\beta$, the optimization becomes a minimization, with the optimal solution approaching a deterministic policy with $\pi(a|s) = \min_a \dualvarq(a,s)$ as $\beta \rightarrow -\infty$,  $1/\beta \rightarrow 0$.

\headerv
\subsection{Feasible Set for Entropy Regularization}\label{app:eysenbach}
\headerv
In this section, we compare feasible sets derived from entropy regularization (\cref{fig:ela}) with those derived from the \textsc{kl} divergence (\cref{fig:feasible_set_main}, \cref{fig:elb},\cref{fig:elc}), and argue that entropy regularization should be analyzed as a special case of the \textsc{kl} divergence to avoid misleading conclusions.

In their App. A8, \citet{eysenbach2021maximum} make the surprising conclusion that policies with \textit{higher} entropy regularization are \textit{less} robust, as they lead to smaller feasible or robust sets than for lower regularization.   
This can be seen in the robust set plot for entropy regularization in \cref{fig:ela}, where higher $\beta$ indicates lower regularization strength ($1/\beta$).
The left panels in \citet{eysenbach2021maximum} Fig. 2 or Fig. 9 match our \cref{fig:ela}.  See their App. A8 for additional discussion.

To translate from Shannon entropy to \textsc{kl} divergence regularization, we include an additional additive constant of $\frac{1}{\beta} \log |\mathcal{A}|$ corresponding to the (scaled) entropy of the uniform distribution, with $ - \frac{1}{\beta} D_{KL}[\pi:\pi_0] = \frac{1}{\beta} H(\pi) - \frac{1}{\beta} \log |\mathcal{A}|$.   We obtain \cref{fig:elb} by shifting each curve in \cref{fig:ela} by this scaled constant
.
This highlights the delicate dependence on the feasible set on the exact form of the objective function, as the constant shifts the robust reward set by $(r^{\prime}(a_1), r^{\prime}(a_2)) \leftarrow (r^{\prime}(a_1) - \frac{1}{\beta}\log 2, r^{\prime}(a_2)-\frac{1}{\beta}\log 2)$.   For \textsc{kl} divergence regularization, the feasible set now includes the original reward function.  As expected, we can see that policies with higher regularization strength are \textit{more} robust with \textit{larger} feasible sets.

An alternative approach to correct the constraint set is to include the uniform reference distribution as in \myprop{feasible} and \cref{eq:kl_feasible}, so that we calculate $\sum_a \pi_0(a|s) \expof{\beta \cdot \pr} \leq 1$.   Similarly, the constraint in \citet{eysenbach2021maximum} can be modified to have $\sum_a \expof{\beta \cdot \pr} \leq |\mathcal{A}|$ in the case of entropy regularization.

\begin{figure}[t]
\vspace*{-.35cm}
\centering
\small
\includegraphics[width=.4\textwidth]{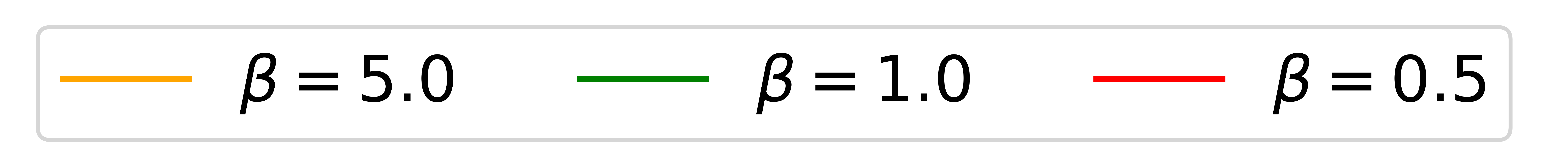} \\
\begin{subfigure}{0.32\textwidth}\includegraphics[width=\textwidth]{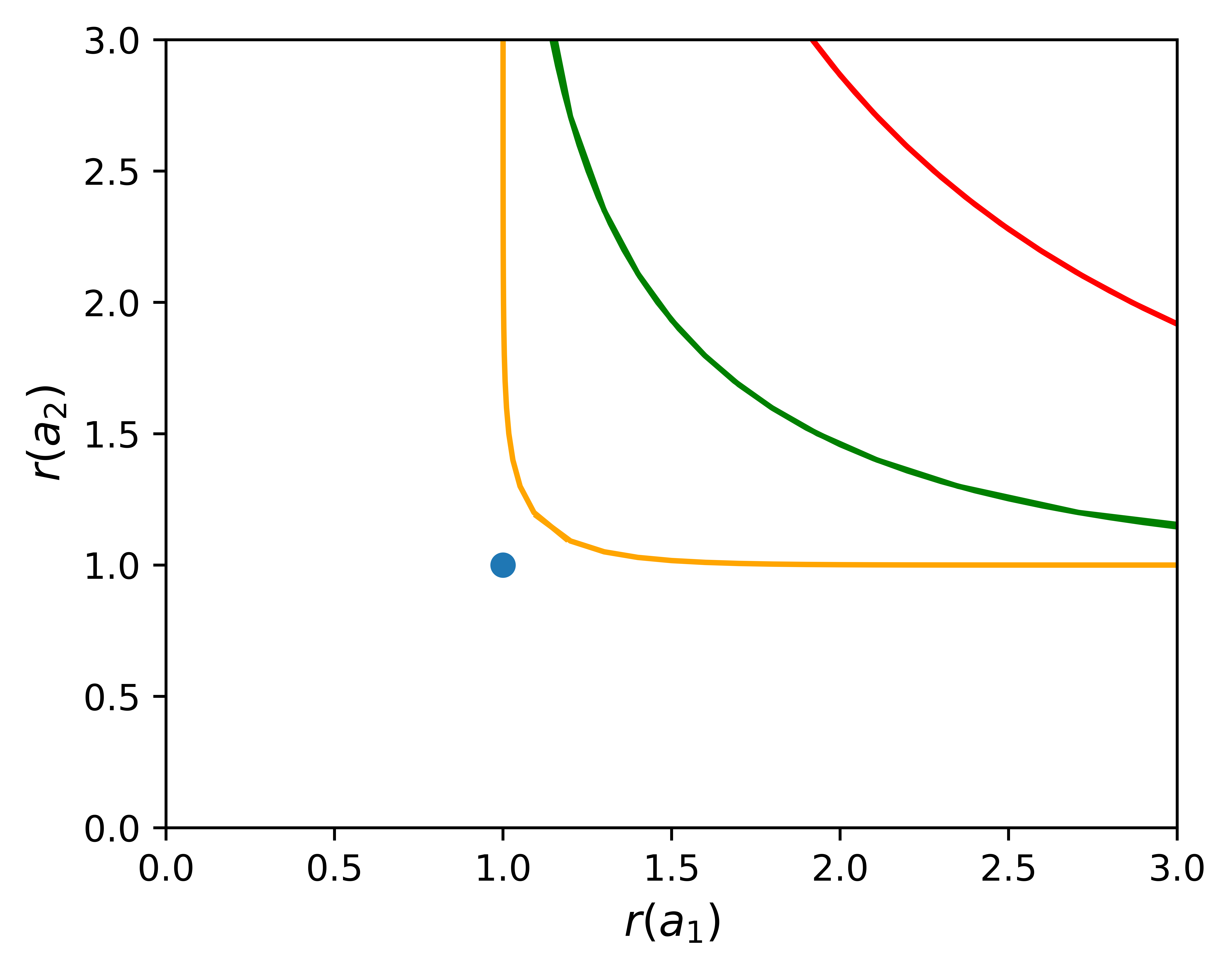}
\caption{\small Fig. 2 or Fig. 9 (left) of \\ \citet{eysenbach2021maximum} \\[1.25ex] Constraint: $\sum_a
\exp\{\beta  \Delta r(a)\} \leq 1$}\label{fig:ela}\end{subfigure} 
\begin{subfigure}{0.32\textwidth}\vspace*{-.3cm} \includegraphics[width=\textwidth]{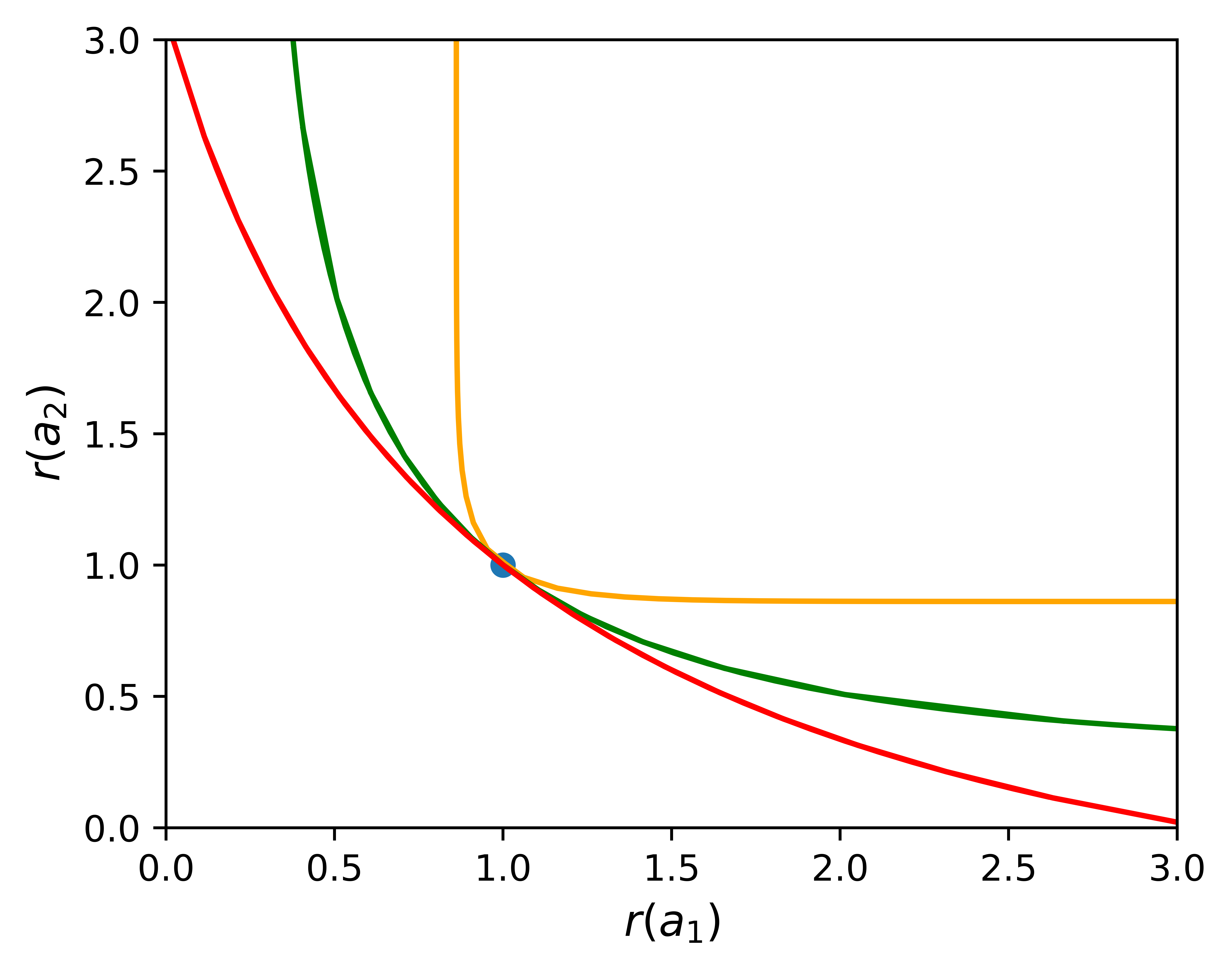}
\caption{\small Shifted $r^{\prime}= r - \Delta r \textcolor{blue}{- \ln(2)}$ \\[1.25ex] Constraint: $\sum_a
\exp\{\beta  \Delta r(a)\} \leq 1$}\label{fig:elb}\end{subfigure}
\begin{subfigure}{0.33\textwidth}\vspace*{-.2cm} \includegraphics[width=\textwidth]{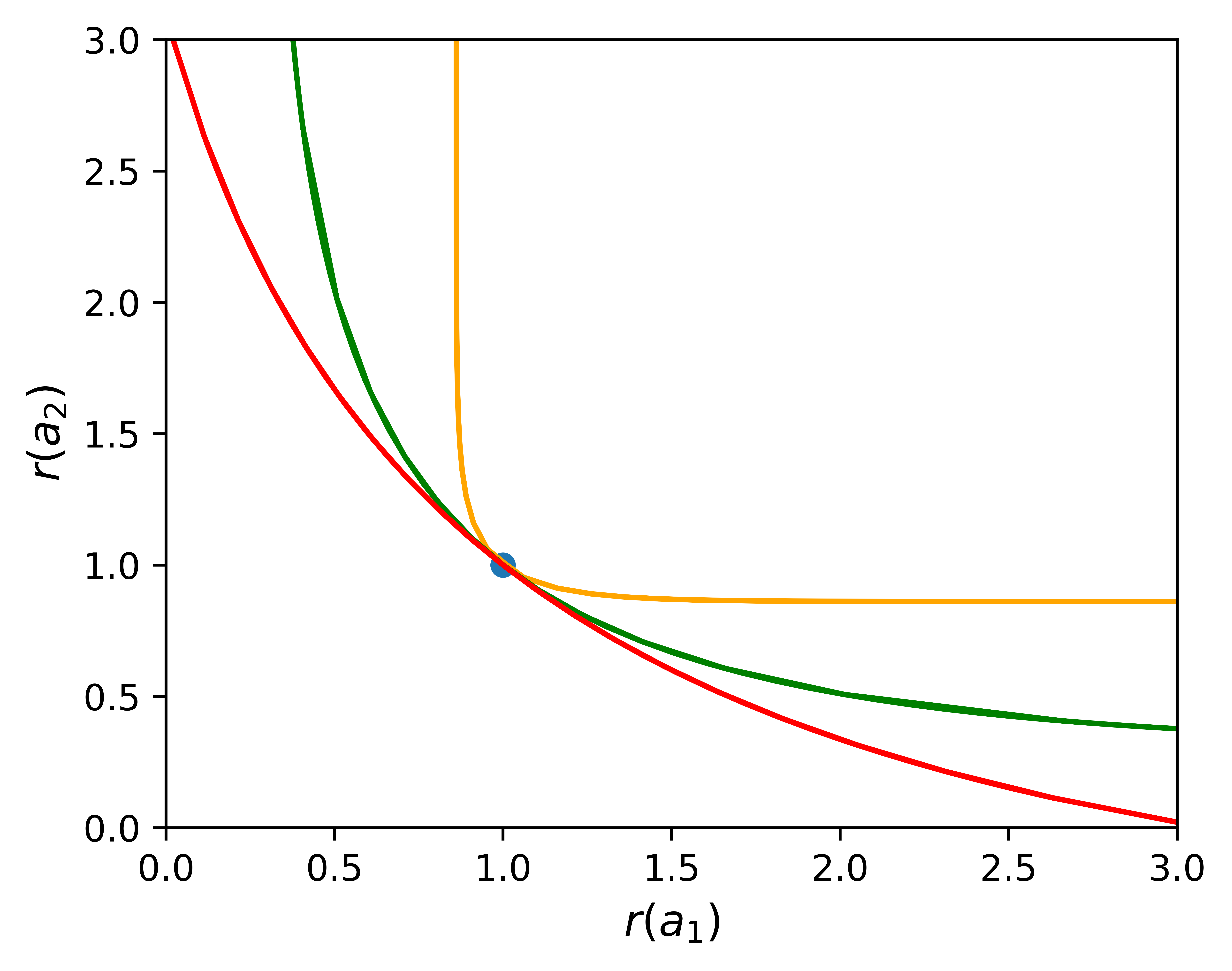}
\caption{ With prior $\pi_0(a|s) = u(a) = \frac{1}{2}$ \\[1.25ex] \text{\small Constraint}$\small \sum \limits_a \textcolor{blue}{\pi_0(a)} \exp\{\beta \Delta r(a)\} \leq 1$ }\label{fig:elc}\end{subfigure}
\normalsize
\caption{ Analyzing the feasible set for entropy regularization (a) versus divergence regularization (b,c).
}\label{fig:eysenbachlevine}
\end{figure}

Our modifications to the feasible set constraints clarify interpretations of how changing regularization strength affects robustness.
We do not give detailed consideration to the other question briefly posed in \citet{eysenbach2021maximum} App. A8, of `if a reward function $r^{\prime}$
is included in a robust set, what other reward functions are included in that robust set?', whose solution is visualized in \citet{eysenbach2021maximum} Fig. 9 (right panel) without detailed derivations.   However,  this plot matches our \cref{fig:elb} and \ref{fig:elc}.   This suggests that 
the constraint arising from \textsc{kl} divergence regularization with strength $\beta$,
$\sum \pi_0(a|s) \expof{\beta \cdot \pr(a,s)} = 1$,
is sufficient to define both the robust set for a given reward function and to characterize the other reward functions to which an optimal policy is robust.
The key observation is that the original reward function is included in the robust set when explicitly including the reference distribution $\pi_0(a|s)$ as in divergence regularization.



%% file: appendix/new_new_app/10_example.tex
\section{Worked Example for Deterministic Regularized Policy ($\alpha = 2$, $\beta = 10$)} \label{app:deterministic_alpha}
We consider the single-step example in \mysec{visualizations_feasible} \myfig{feasible_set_main} or \myapp{additional_results} \myfig{feasible_set_uniform_app}-\ref{fig:feasible_set_nonuniform_app}, with a two-dimensional action space, optimal state-action value estimates, $Q_*(a,s) = r(a,s) = \{ 1.1, 0.8\}$, and uniform prior $\pi_0(a|s)$.   

The case of policy regularization with $\alpha = 2$ and $\beta= 10$ is particularly interesting, since the optimal policy is deterministic with $\pi_*(a_1|s) = 1$. \footnote{We use $\alpha=2$ instead of $\alpha=3$ in \myfig{feasible_set_main} for simplicity of calculations.   See \myapp{additional_results} \myfig{feasible_set_uniform_app} for $\alpha=2$ robust set plots.}
First, we solve for the optimal policy for $Q_*(a,s) = r(a,s)$ as in \myapp{alpha_conj_norm},
\ifx\omitindices\undefined
\small
\begin{align*}
    \frac{1}{\beta} \Omega^{*}_{\pi_0,\beta}(Q_*) &= \max \limits_{\pi \in \Delta^{|\mathcal{A}|}} \blangle \pi(a|s), Q_*(a,s) \brangle - \frac{1}{\beta} \Omega^{(\alpha)}_{\pi_0}(\pi) - \psipiq \left(\sum \limits_{a} \pi(a|s) - 1 \right) + \sum \limits_{a} \lambdaplus \, \\
    \implies  \piopt(a|s) &= \pi_0(a|s) \Big[ 1 + \beta (\alpha-1) \big( Q_*(a,s)  + \lambdaopt \underbrace{ - V_*(s) - \psipropt }_{=\psiqopt \text{ (see \myapp{advantage_confirm} )} } \big) \Big]^{\frac{1}{\alpha-1}} .
\end{align*}
\normalsize
\else
\small
\begin{align*}
    \frac{1}{\beta} \Omega^{*}_{\pi_0,\beta}(Q_*) &= \max \limits_{\pi \in \Delta^{|\mathcal{A}|}} \blangle \pi, Q_* \brangle - \frac{1}{\beta} \Omega^{(\alpha)}_{\pi_0}(\pi) - \psipiq \left(\sum \limits_{a} \pi(a|s) - 1 \right) + \sum \limits_{a} \lambdaplus \, \\
    \implies  \piopt(a|s) &= \pi_0(a|s) \Big[ 1 + \beta (\alpha-1) \big( Q_*(a,s)  + \lambdaopt \underbrace{ - V_*(s) - \psipropt }_{=\psiqopt \text{ (see \myapp{advantage_confirm} )} } \big) \Big]^{\frac{1}{\alpha-1}} .
\end{align*}
\normalsize
\fi
where for $\alpha =2$, we obtain $\piopt(a|s) = \pi_0(a|s) \big[ 1 + \beta \big( Q_*(a,s)  + \lambdaopt - V_*(s) - \psipropt  \big) \big]$.

Using \textsc{cvx-py} \citep{diamond2016cvxpy} to solve this optimization for $\alpha = 2$, $\beta=10$, $\pi_0(a|s) =\frac{1}{2} \, \forall a$, and the given $Q_*(a,s)$, we obtain
\begin{alignat}{3}
    Q_*(a_1, s) &= 1.1 \qquad \qquad & Q_{*}(a_2, s) &= 0.8 \qquad \qquad & \lambda_*(a_1,s)&=0 \nonumber \\
     \piopt(a_1 |s)&= 1 \qquad  \qquad & \piopt(a_2 |s) &= 0 \qquad \qquad & \lambda_*(a_2,s)&= 0.1 \nonumber \\ 
     V_*(s) &= 1.05 \qquad \qquad & \psipiopt &= -0.05 \qquad \qquad & \psiqopt &= 1.0 \, . \label{eq:eg_vals}
\end{alignat}
Our first observation is that, although the policy is deterministic with $\piopt(a_1 |s)= 1$, the value function $V_*(s)=1.05$ is not equal to $\max_a Q_*(a, s)=1.1$ as it would be in the case of an unregularized policy.   Instead, we still need to subtract the $\alpha$-divergence regularization term, which is nonzero.   With $\alpha = 2$ and $1-\alpha = -1$, we have
\small
\begin{align*}
    V_*(s) = \langle \piopt(a|s), Q_*(a,s) \rangle - \underbrace{\frac{1}{\beta} \frac{1}{\alpha} \frac{1}{1-\alpha} \left( 1- \sum_a \pi_0(a|s)^{1-\alpha} \piopt(a|s)^{\alpha} \right)}_{\frac{1}{\beta} D_{\alpha}[\pi_0 : \piopt]} &= 1.1 - \frac{1}{10}\frac{1}{2}\frac{1}{-1}\left( 1 - .5^{-1}\cdot 1^{2} - .5^{-1}\cdot 0^{2} \right)\\
    &= 1.1 + .05\cdot(1-2) = 1.05
\end{align*}
\normalsize
Recall that for normalized $\pi_0$, $\pi$, we have $\psipiopt = -\frac{1}{\beta}(1-\alpha) D_{\alpha}[\pi_0(a|s) : \piopt(a|s)] = -0.05$, so that we can confirm \cref{eq:eg_vals} for $\alpha=2$.

Finally, we confirm the result of \myprop{advantage} by calculating the reward perturbations in two different ways.  For $a_1$, we have
\small
\begin{align*}
    \prpiopt(a_1,s) &= \frac{1}{\beta}\frac{1}{\alpha-1}\left( \frac{\piopt(a_1|s)}{\pi_0(a_1|s)}^{\alpha-1} -1 \right) + \psipiopt= \frac{1}{10}\frac{1}{1}\left(\frac{1}{.5}^{1} -1 \right)  - .05 = .05 \\
    &= Q_*(a_1,s) - V_*(s) + \lambda_*(a_1,s) = 1.1 - 1.05 + 0= .05 , 
\end{align*}
\normalsize
and for $a_2$,
\small
\begin{align*}
    \prpiopt(a_2,s) &= \frac{1}{\beta}\frac{1}{\alpha-1}\left( \frac{\piopt(a_2|s)}{\pi_0(a_2|s)}^{\alpha-1} -1 \right) + \psipiopt= \frac{1}{10}\frac{1}{1}\left(\frac{0}{.5}^{1} -1 \right)  - .05 = -.15 \\
    &= Q_*(a_2,s) - V_*(s) + \lambda_*(a_2,s) = 0.8 - 1.05 + 0.1 = -.15 
\end{align*}
\normalsize
so that we have $\prpiopt(a_1,s) = 0.05$ and $\prpiopt(a_2,s) = -0.15$.  

We can observe that the indifference condition \textit{does not} hold, since $Q_*(a_1,s) - \prpiopt(a_1,s) = 1.1 - 0.05 = 1.05$ does not match $Q_*(a_2,s) - \prpiopt(a_2,s) = 0.8 - (-0.15) = 0.95$.

However, adding the Lagrange multiplier $\lambda_*(a_2,s) = 0.1$ accounts for the difference in these values.   This allows us to confirm the path consistency condition (\cref{eq:path}), 
\small
\begin{align}
 \underbrace{r(a,s) + \gamma \transitionvstar \vphantom{\frac{1}{\beta}}}_{Q_*(a,s)} - \underbrace{\frac{1}{\beta} \log_{\alpha}\frac{\pi_*(a|s)}{\pi_0(a|s)} - \psipiopt}_{\prpiopt(a,s)} =  V_*(s) - \lambdaopt \qquad \forall (a,s) \in \mathcal{A} \times \mathcal{S}
\end{align}
\normalsize
with $Q_*(a_1,s) - \prpiopt(a_1,s) - V_*(s) + \lambda_*(a_1,s) = 1.1 - 0.05 - 1.05 + 0 = 0$ and 
$Q_*(a_2,s) - \prpiopt(a_2,s) - V_*(s) + \lambda_*(a_2,s) = 0.8 - (-0.15) - 1.05 + 0.1 = 0$.

%% file: appendix/new_new_app/11_addl_results.tex
\section{Additional Feasible Set Plots} \label{app:additional_results}\label{app:additional_feasible}

In \myfig{feasible_set_uniform_app} and \ref{fig:feasible_set_nonuniform_app}, we provide additional feasible set plots for the $\alpha$-divergence with $\alpha \in \{-1, 1, 2, 3\}$ and $\beta \in \{ 0.1, 1.0, 5, 10 \}$ with $r(a,s) = [1.1, 0.8]$.   As in \myfig{feasible_set_main}, we show the feasible set corresponding to the single-step optimal policy for $Q_*(a,s) = r(a,s)$ for various regularization schemes.   Since each policy is optimal, we can confirm the indifference condition for the \textsc{kl} divergence, with $Q_*(a,s) - \prpiopt(a,s) = V_*(s)$ constant across actions and equal to the soft value function, certainty equivalent, or conjugate function $V_*(s) = \alphaconjn(Q)$.   
When the indifference condition holds, we can obtain the ratio of action probabilities in the regularized policy by taking the slope of the tangent line to the feasible set boundary at $\mrpiopt(a,s)$, as in \mysec{visualizations_feasible}.

However, indifference does not hold in cases where the optimal policy sets $\piopt(a_2|s) =0$, which occurs for ($\alpha = 2, \beta=10$), $(\alpha = 3, \beta \in \{5,10\}$) for a uniform reference policy and additionally for ($\alpha =2, \beta = 5$) with the nonuniform reference in \myfig{feasible_set_nonuniform_app}.  In these cases, we cannot ignore the Lagrange multiplier in \cref{eq:path}, $\mr_{\piopt}(a,s) = Q_*(a,s) - \prpiopt(a,s) = V_*(s) - \lambda_*(a,s)$, and $\lambda_*(a_2,s) > 0$ results in a different perturbed reward $\mr_{\piopt}(a_1, s) \neq \mr_{\piopt}(a_2,s)$.

For $\alpha=-1$ and low regularization strength ($\beta=10$), we observe a wider feasible set boundary than for \textsc{kl} divergence regularization.   For $\alpha=2$ and $\alpha=3$, the boundary is more restricted and the worst-case reward perturbations become notably smaller when the policy is deterministic.  For example, we can compare $\beta=5$ versus $\beta=10$ for $\alpha = 2$.  However, as in \myfig{value_agg_main}, we do not observe notable differences in the robust sets at lower regularization strengths based on the choice of $\alpha$-divergence.

\input{appendix/app_results/addl_feasible_no_lambda}

%% file: appendix/app_results/addl_feasible_no_lambda.tex
\newcommand{\sidefeasible}{.08}
\newcommand{\mainfeasiblefour}{.187}
\begin{figure*}[!t]
\vspace*{-1.6cm}
\centering
\begin{center}
\[\arraycolsep=.01\textwidth
\begin{array}{ccccc} 
 \begin{subfigure}{\sidefeasible\textwidth}  \phantom{ \includegraphics[width=\textwidth]
 {figs/feasible/uniform_prior.png}} \end{subfigure} & \begin{subfigure}{\mainfeasiblefour\textwidth}\includegraphics[width=\textwidth, trim= 0 0 0 0, clip]
{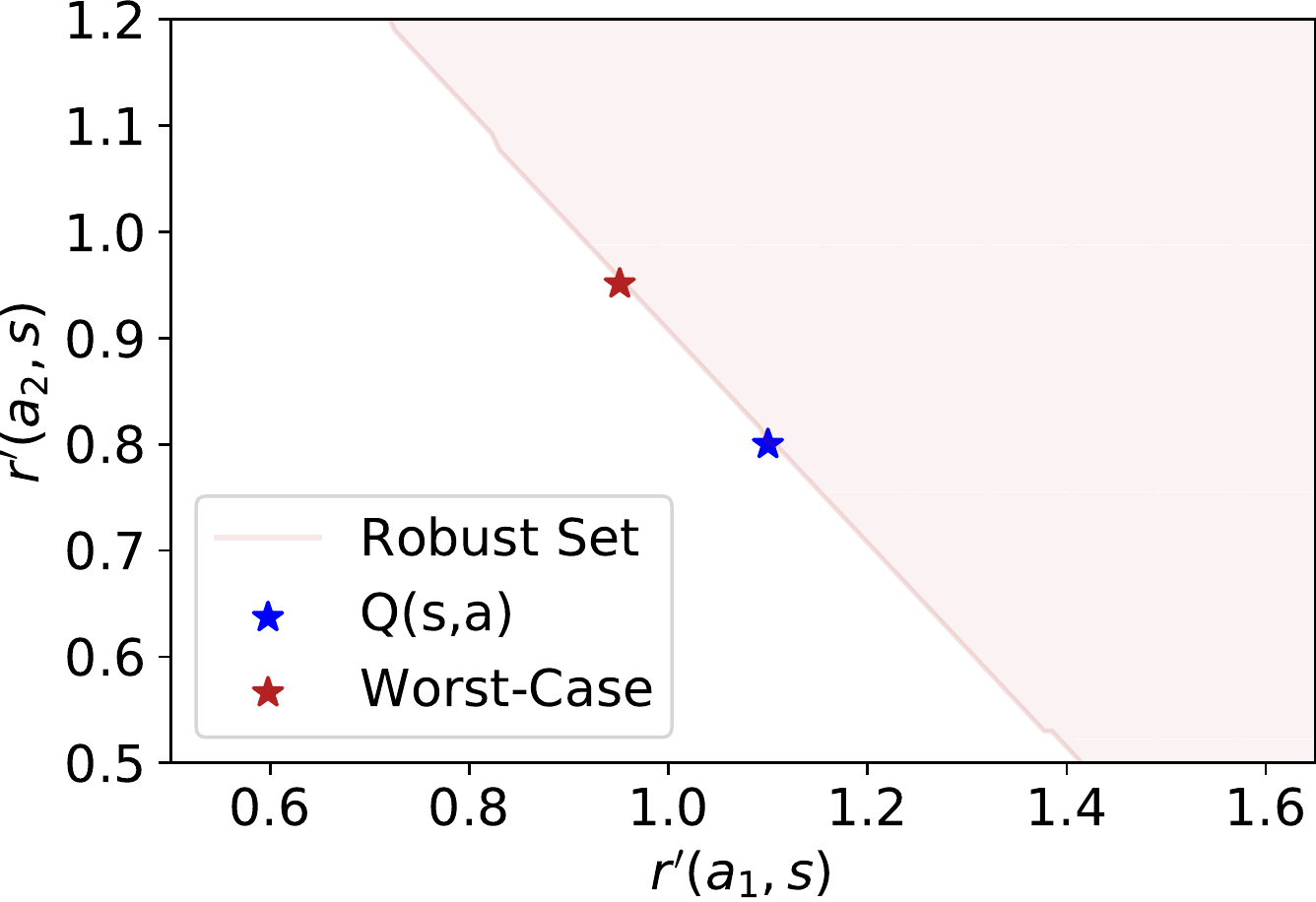}
\caption{ \centering \small $\alpha = -1, \beta = 0.1$}\end{subfigure}&
\begin{subfigure}{\mainfeasiblefour\textwidth}\includegraphics[width=\textwidth, trim= 0 0 0 0, clip]
{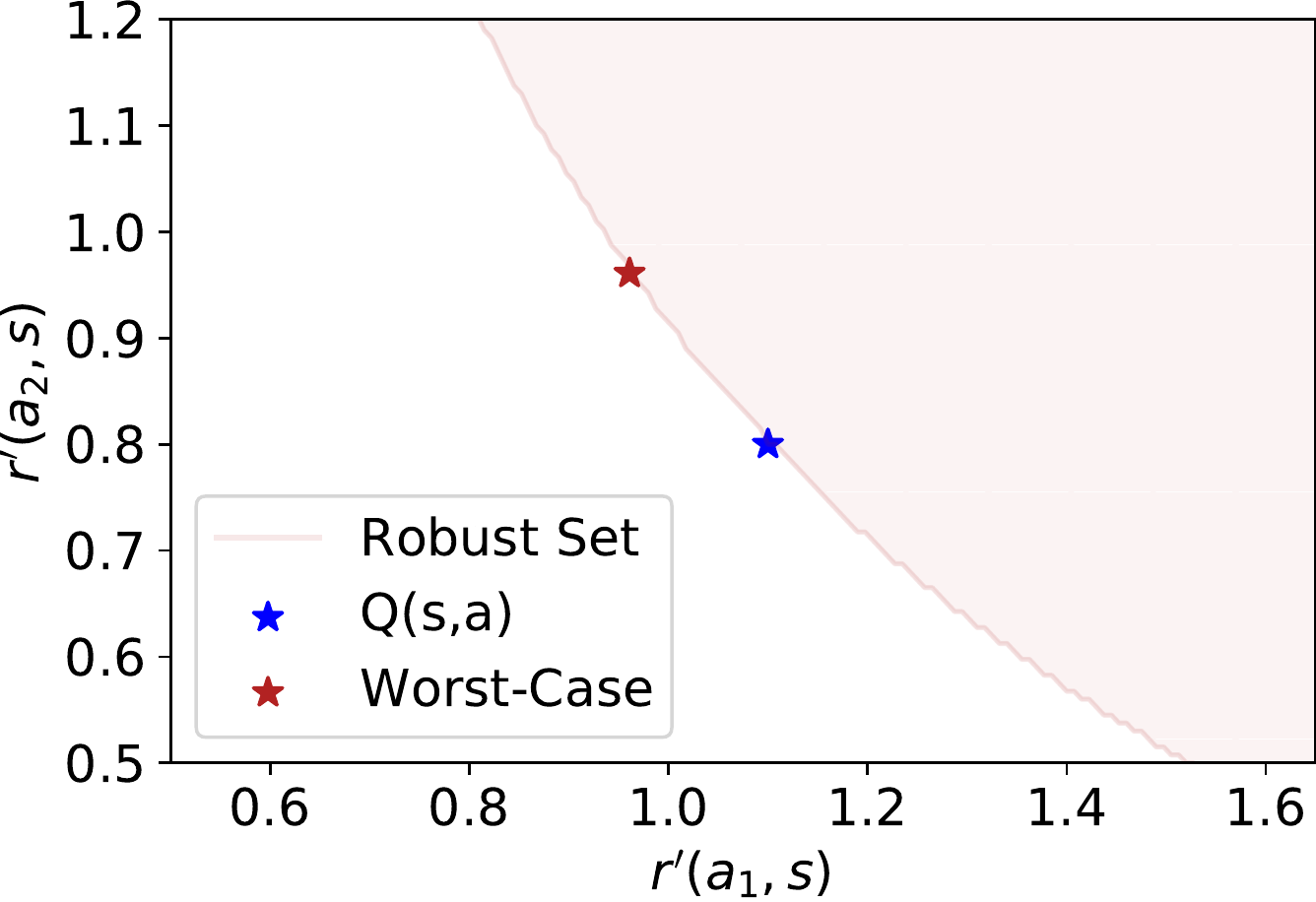}
\caption{ \centering \small $\alpha = -1, \beta = 1$}\end{subfigure} &
\begin{subfigure}{\mainfeasiblefour\textwidth}\includegraphics[width=\textwidth, trim= 0 0 0 0, clip]
{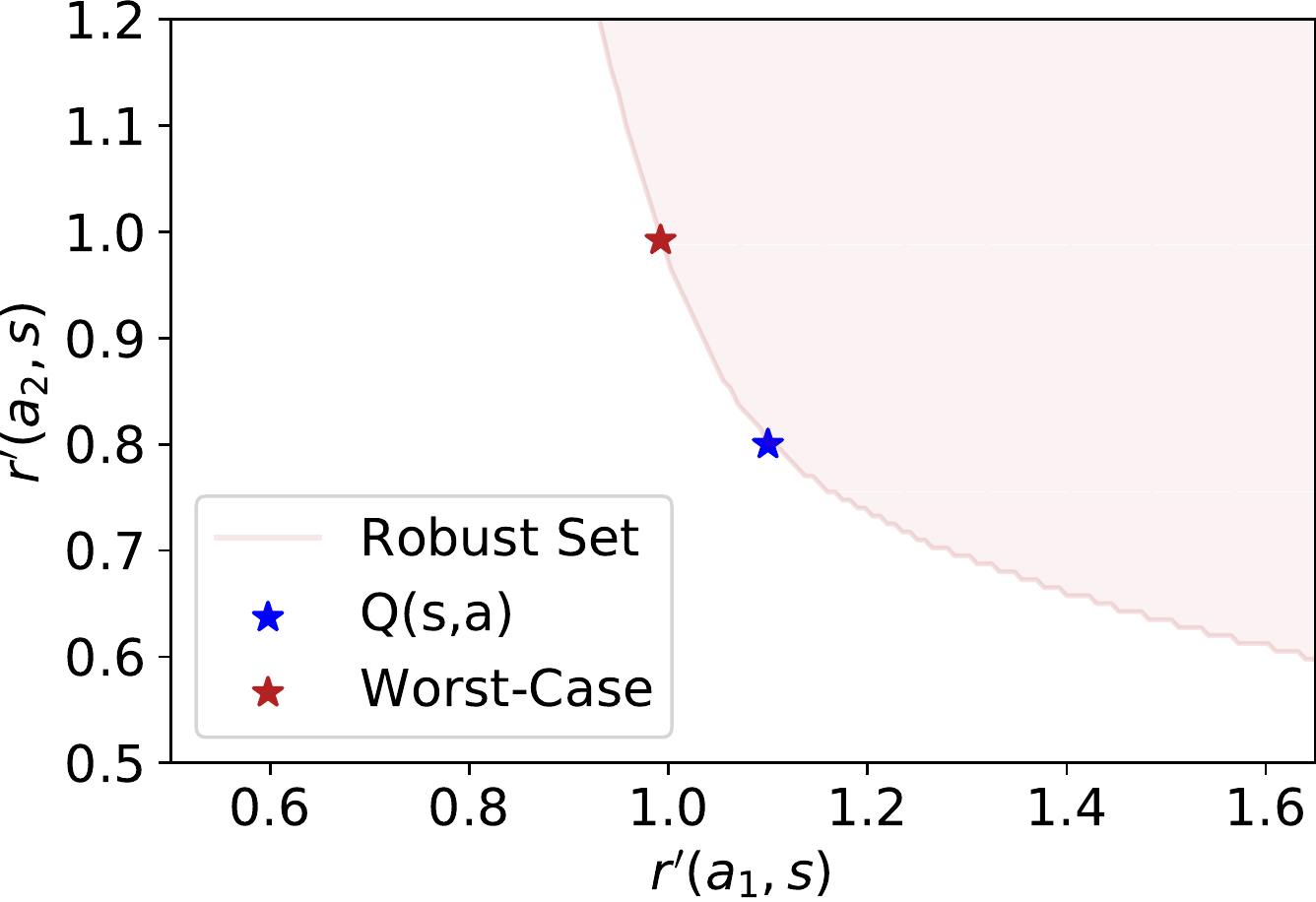}
\caption{ \centering \small $\alpha = -1, \beta = 5$}\end{subfigure} &
\begin{subfigure}{\mainfeasiblefour\textwidth}\includegraphics[width=\textwidth, trim= 0 0 0 0, clip]
{figs/feasible/alpha-1_no_lambda/alpha_neg1_beta10uniform.pdf}
\caption{ \centering \small $\alpha = -1, \beta = 10$}\end{subfigure} \\
 \begin{subfigure}{\sidefeasible\textwidth}  \includegraphics[width=0.9\textwidth]
 {figs/feasible/uniform_prior.png} \end{subfigure} & 
\begin{subfigure}{\mainfeasiblefour\textwidth}\includegraphics[width=\textwidth, trim= 0 0 0 0, clip]
{figs/feasible/kl/beta01alpha1uniform.pdf}
\caption{ \centering \small $D_{KL}, \beta = 0.1$} \end{subfigure} & 
\begin{subfigure}{\mainfeasiblefour\textwidth}\includegraphics[width=\textwidth, trim= 0 0 0 0, clip]{figs/feasible/kl/beta1alpha1uniform.pdf}
\caption{\centering \small $D_{KL}, \beta = 1$}\end{subfigure} &
\begin{subfigure}{\mainfeasiblefour\textwidth}\includegraphics[width=\textwidth, trim= 0 0 0 0, clip]
{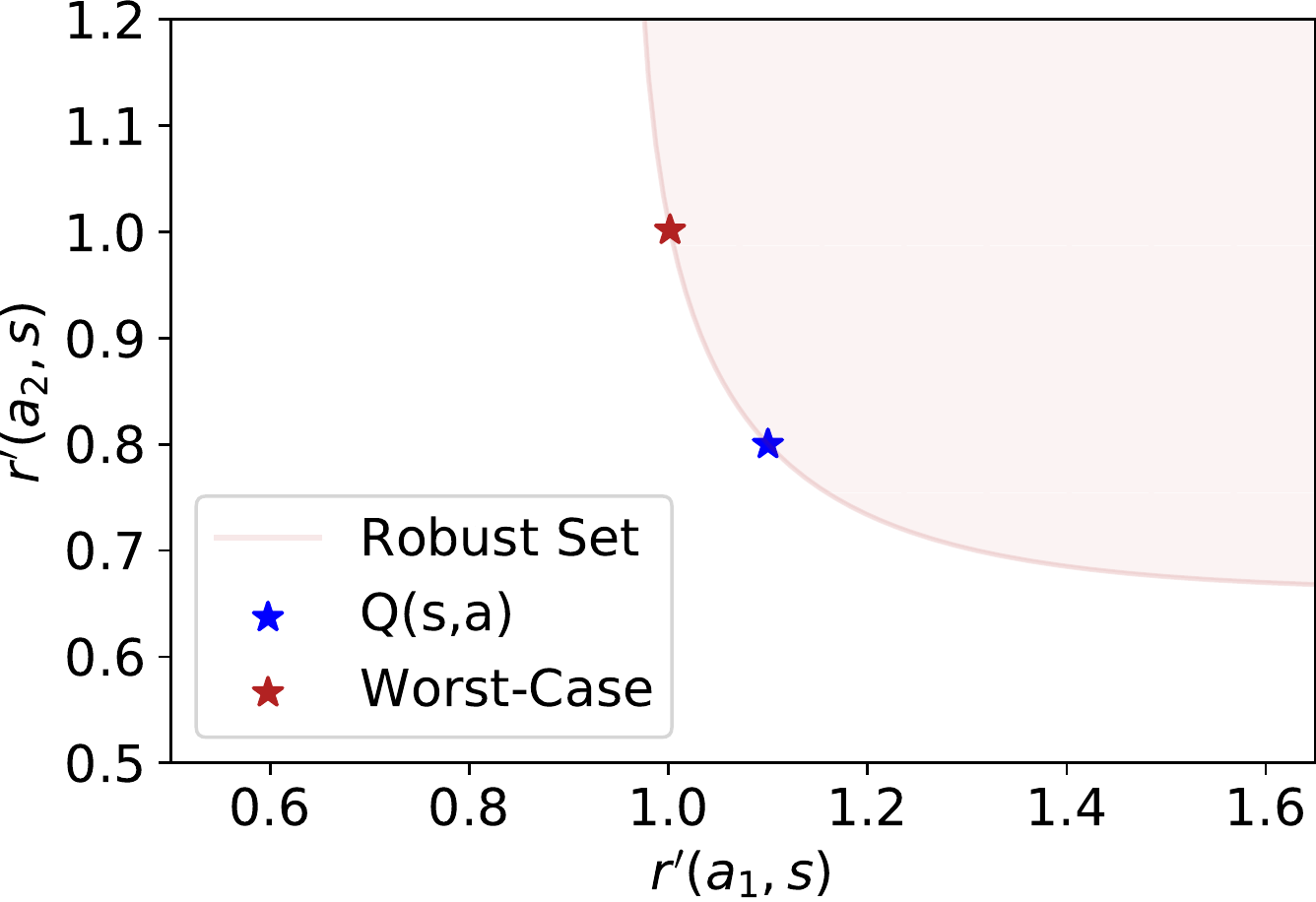}
\caption{ \centering \small $D_{KL}, \beta = 5$}\end{subfigure} &
\begin{subfigure}{\mainfeasiblefour\textwidth}\includegraphics[width=\textwidth, trim= 0 0 0 0, clip]
{figs/feasible/kl/beta10alpha1uniform.pdf}
\caption{ \centering \small $D_{KL}, \beta = 10$}\end{subfigure} \\
 \begin{subfigure}{\sidefeasible\textwidth} \phantom{ \includegraphics[width=\textwidth]
 {figs/feasible/uniform_prior.png}} \end{subfigure} & \begin{subfigure}{\mainfeasiblefour\textwidth}\includegraphics[width=\textwidth, trim= 0 0 0 0, clip]
{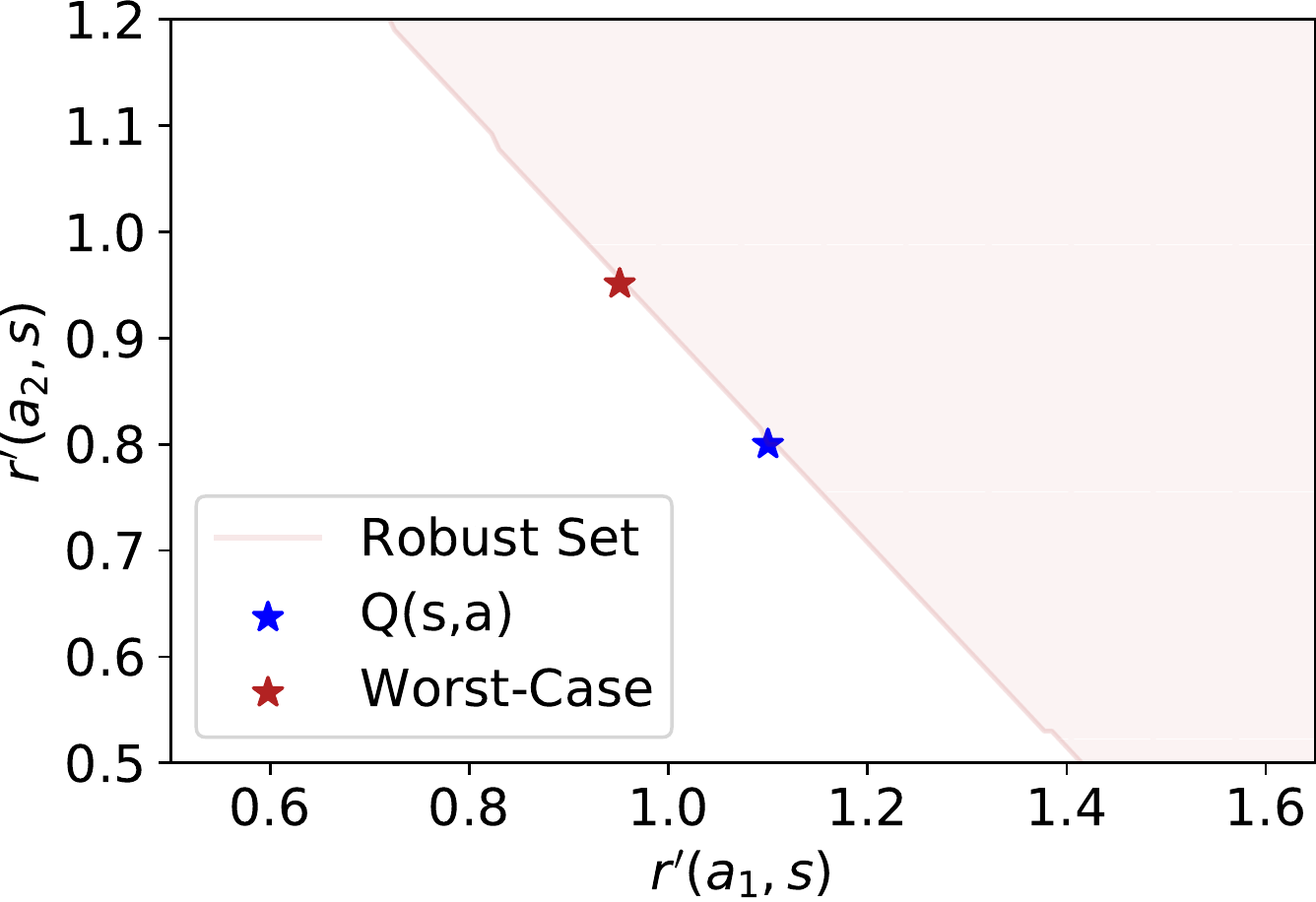}
\caption{ \centering \small $\alpha = 2, \beta = 0.1$}\end{subfigure}&
\begin{subfigure}{\mainfeasiblefour\textwidth}\includegraphics[width=\textwidth, trim= 0 0 0 0, clip]
{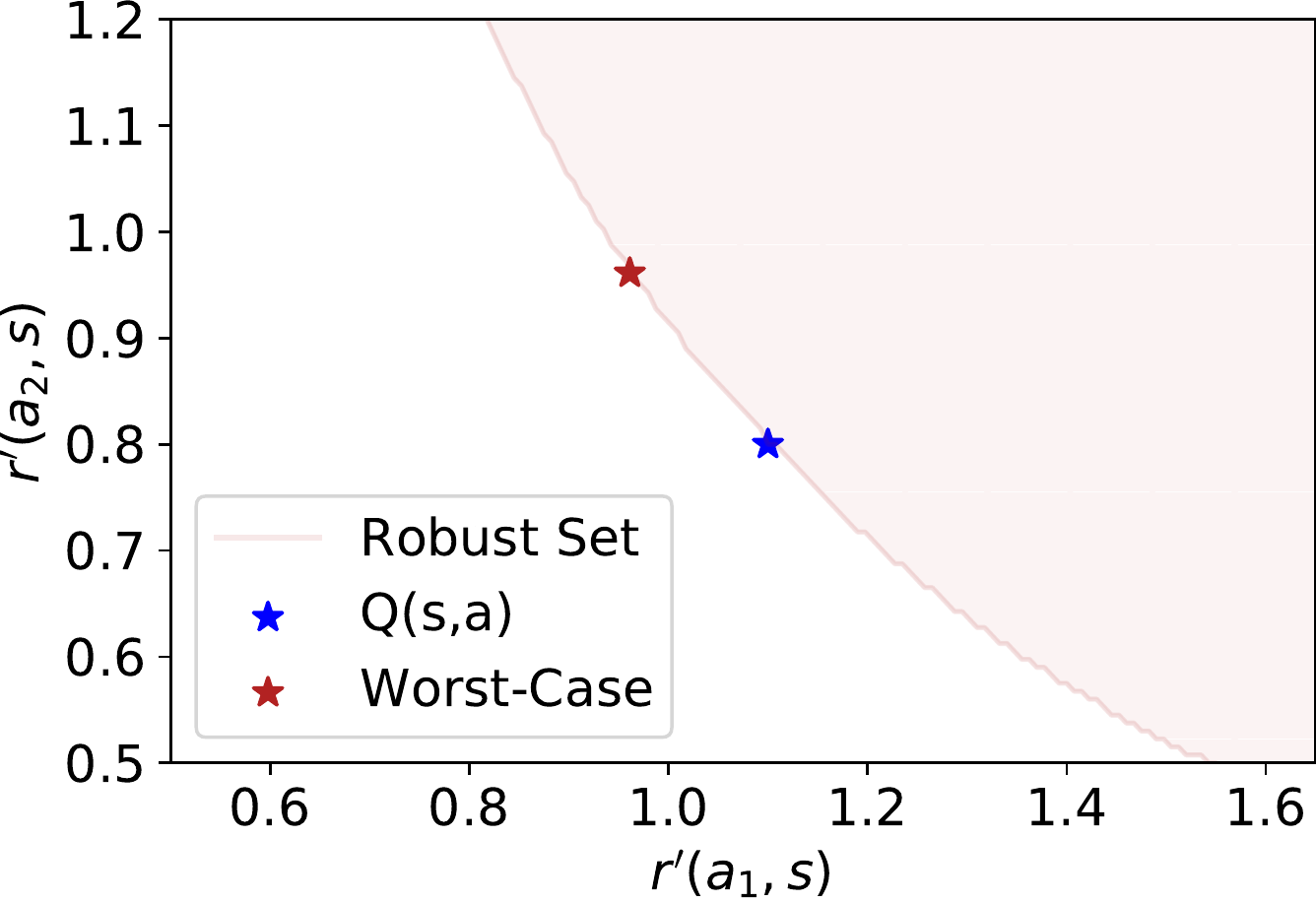}
\caption{ \centering \small $\alpha = 2, \beta = 1$}\end{subfigure}&
\begin{subfigure}{\mainfeasiblefour\textwidth}\includegraphics[width=\textwidth, trim= 0 0 0 0, clip]
{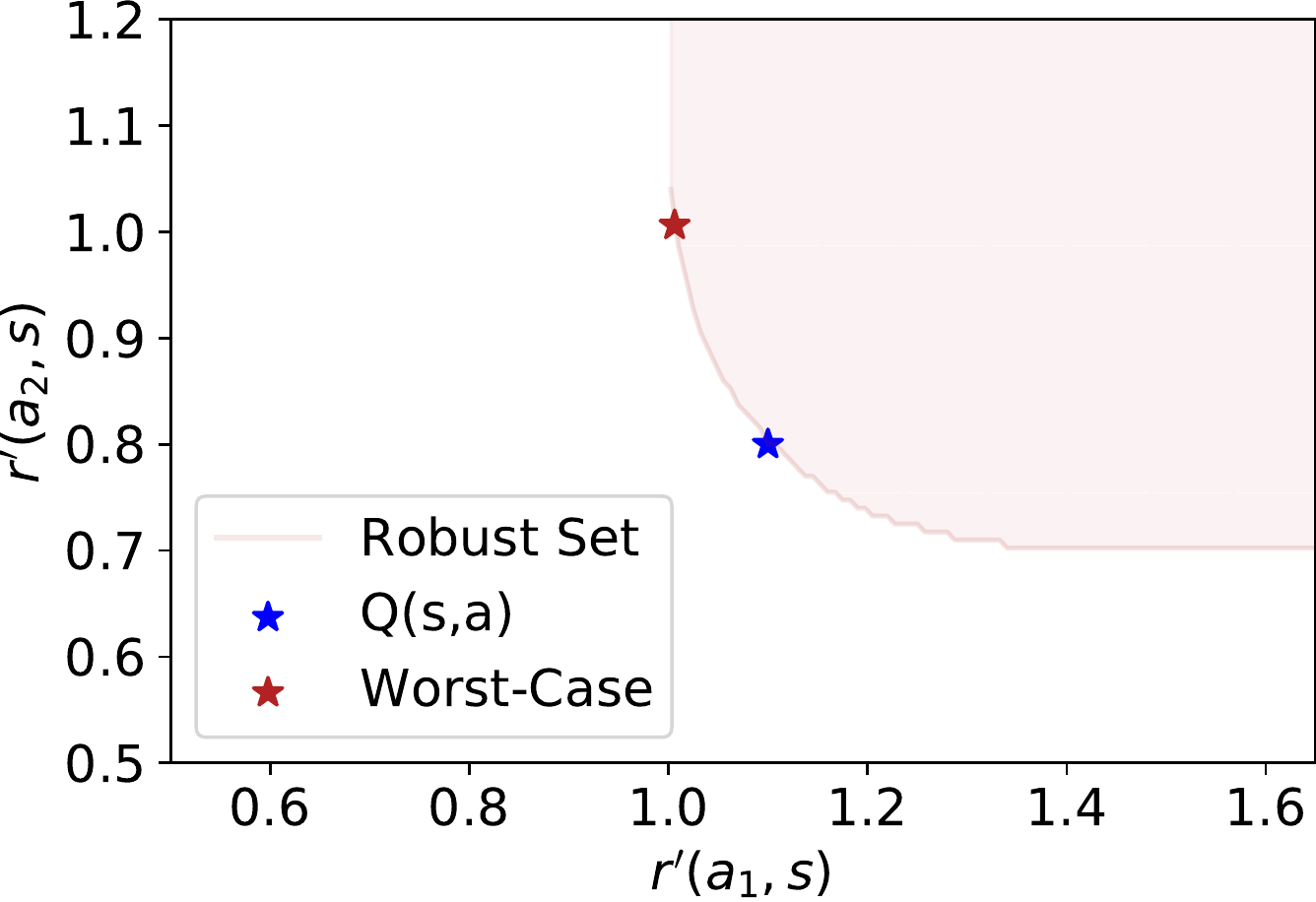}
\caption{ \centering \small $\alpha = 2, \beta = 5$}\end{subfigure} &
\begin{subfigure}{\mainfeasiblefour\textwidth}\includegraphics[width=\textwidth, trim= 0 0 0 0, clip]
{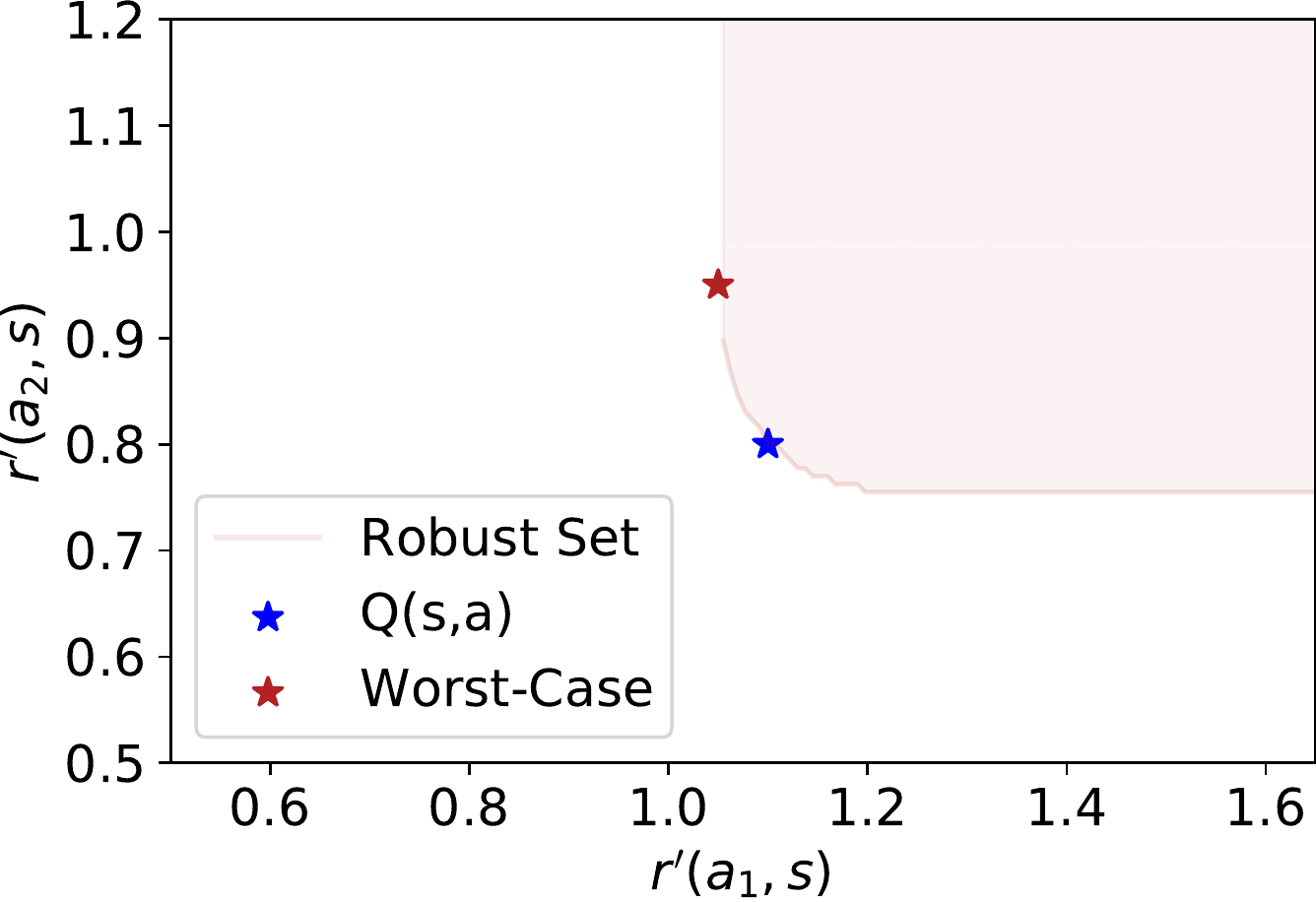}
\caption{ \centering \small $\alpha = 2, \beta = 10$}\end{subfigure} \\
\begin{subfigure}{\sidefeasible\textwidth} \phantom{ \includegraphics[width=\textwidth]
 {figs/feasible/uniform_prior.png}} \end{subfigure} & \begin{subfigure}{\mainfeasiblefour\textwidth}\includegraphics[width=\textwidth, trim= 0 0 0 0, clip]
{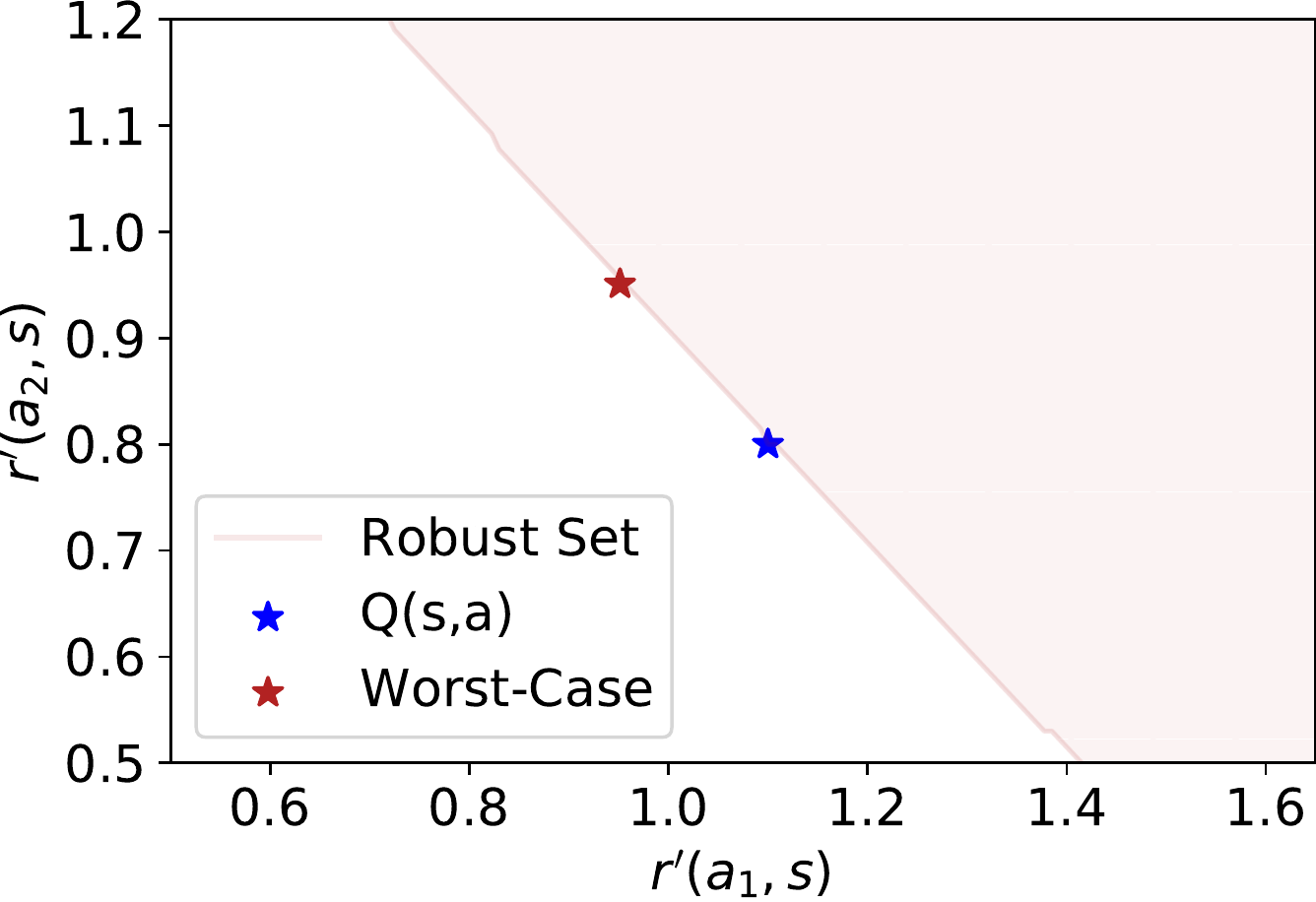}
\caption{\centering \small $\alpha = 3, \beta = 0.1$}\end{subfigure}&
\begin{subfigure}{\mainfeasiblefour\textwidth}\includegraphics[width=\textwidth, trim= 0 0 0 0, clip]
{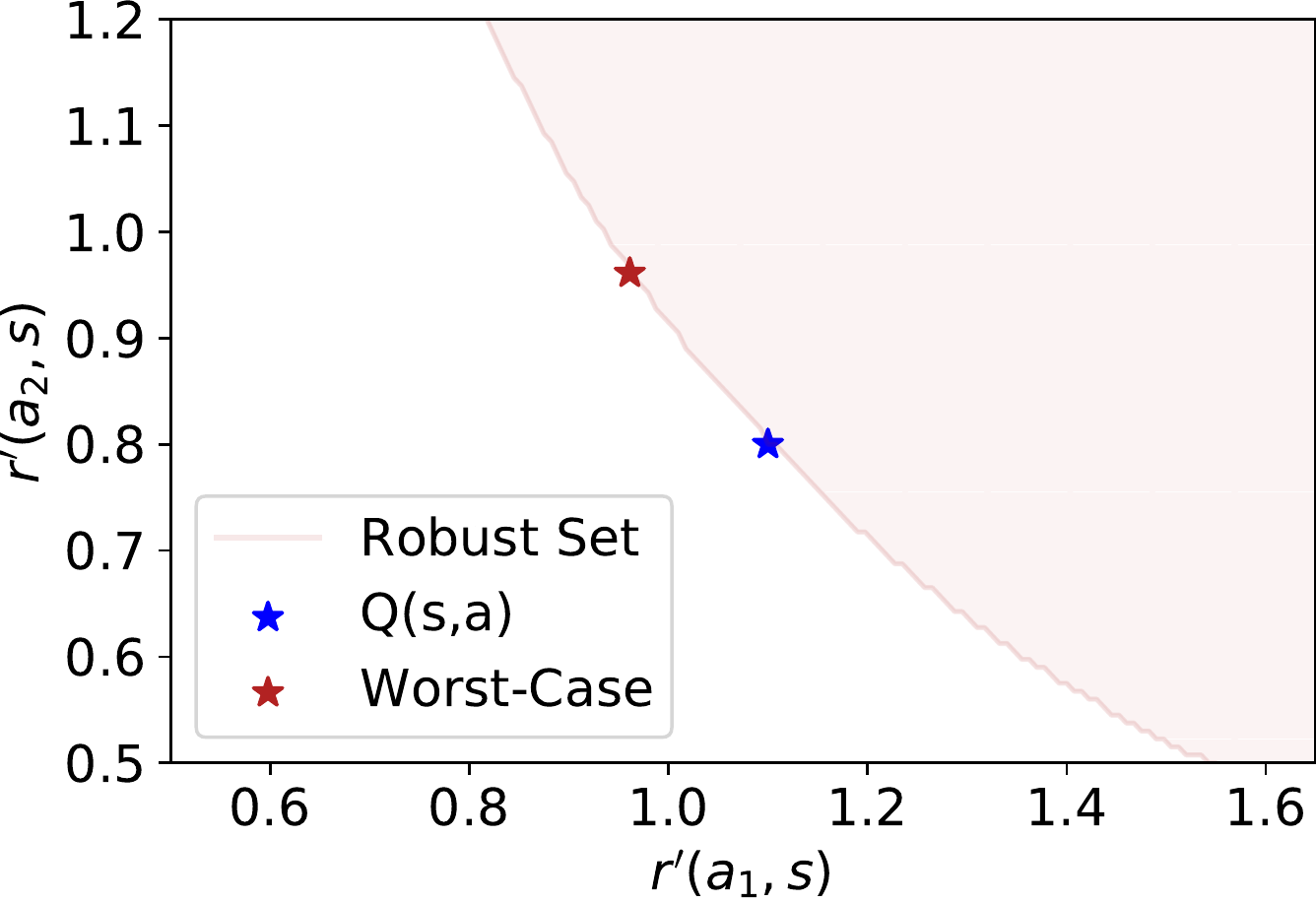}
\caption{ \centering \small $\alpha = 3, \beta = 1$}\end{subfigure}&
\begin{subfigure}{\mainfeasiblefour\textwidth}\includegraphics[width=\textwidth, trim= 0 0 0 0, clip]
{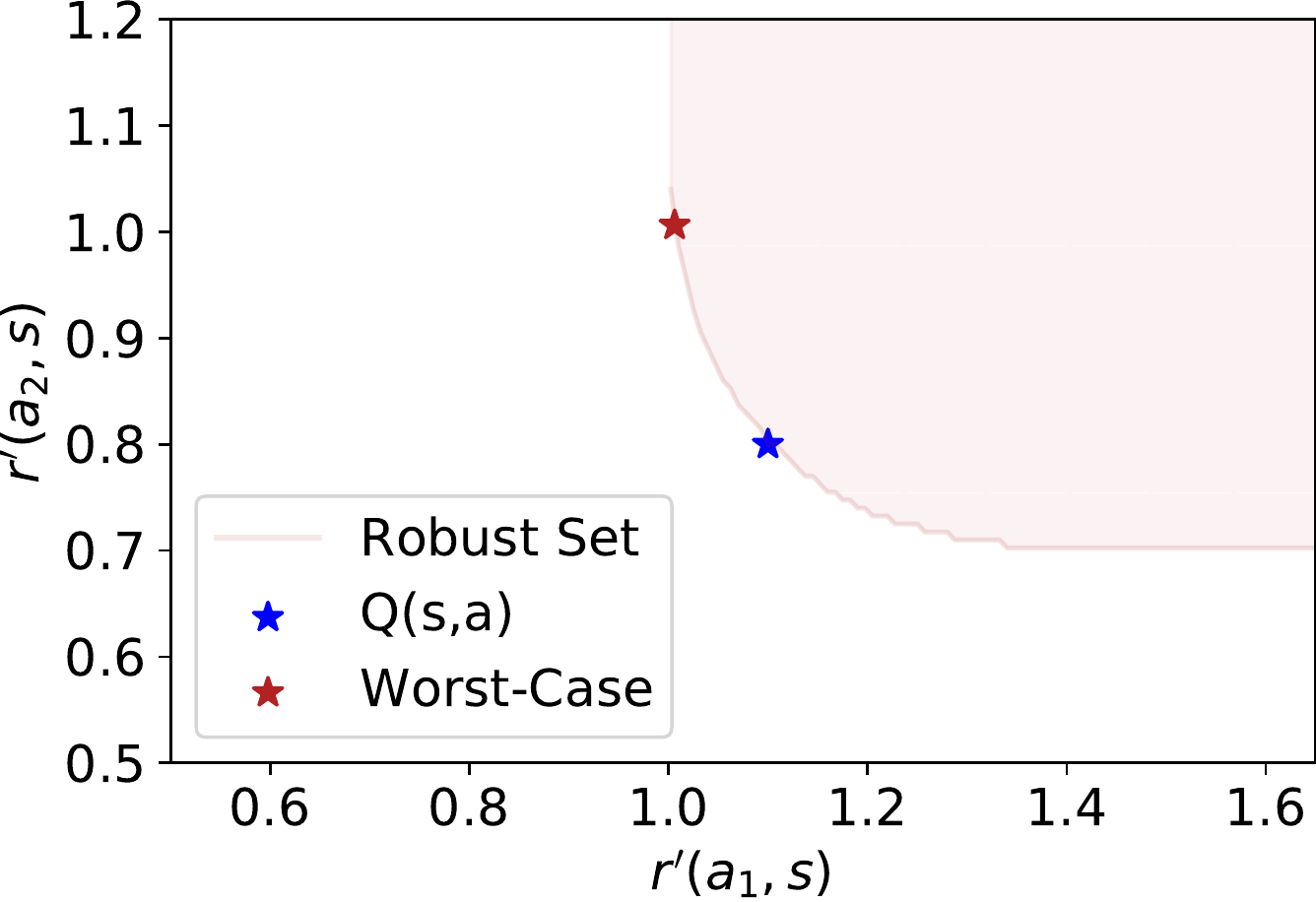}
\caption{ \centering \small $\alpha = 3, \beta = 5$}\end{subfigure}&
\begin{subfigure}{\mainfeasiblefour\textwidth}\includegraphics[width=\textwidth, trim= 0 0 0 0, clip]
{figs/feasible/alpha3_no_lambda/alpha3_beta10uniform.pdf}
\caption{ \centering \small $\alpha = 3, \beta = 10$}\end{subfigure}
\\ 
\end{array}
\] 
\end{center}
\vspace*{-.35cm}
\caption{Reference distribution $\pi_0=(\frac{1}{2},\frac{1}{2})$. See caption of \myfig{feasible_set_nonuniform_app}.
}\label{fig:feasible_set_uniform_app}
\vspace*{.15cm}
\hrule
\vspace*{-.3cm}
\centering
\begin{center}
\[\arraycolsep=.01\textwidth
\begin{array}{ccccc} 
 \begin{subfigure}{\sidefeasible\textwidth}  \phantom{ \includegraphics[width=\textwidth]
 {figs/feasible/nonuniform_prior.png}} \end{subfigure} & \begin{subfigure}{\mainfeasiblefour\textwidth}\includegraphics[width=\textwidth, trim= 0 0 0 0, clip]
{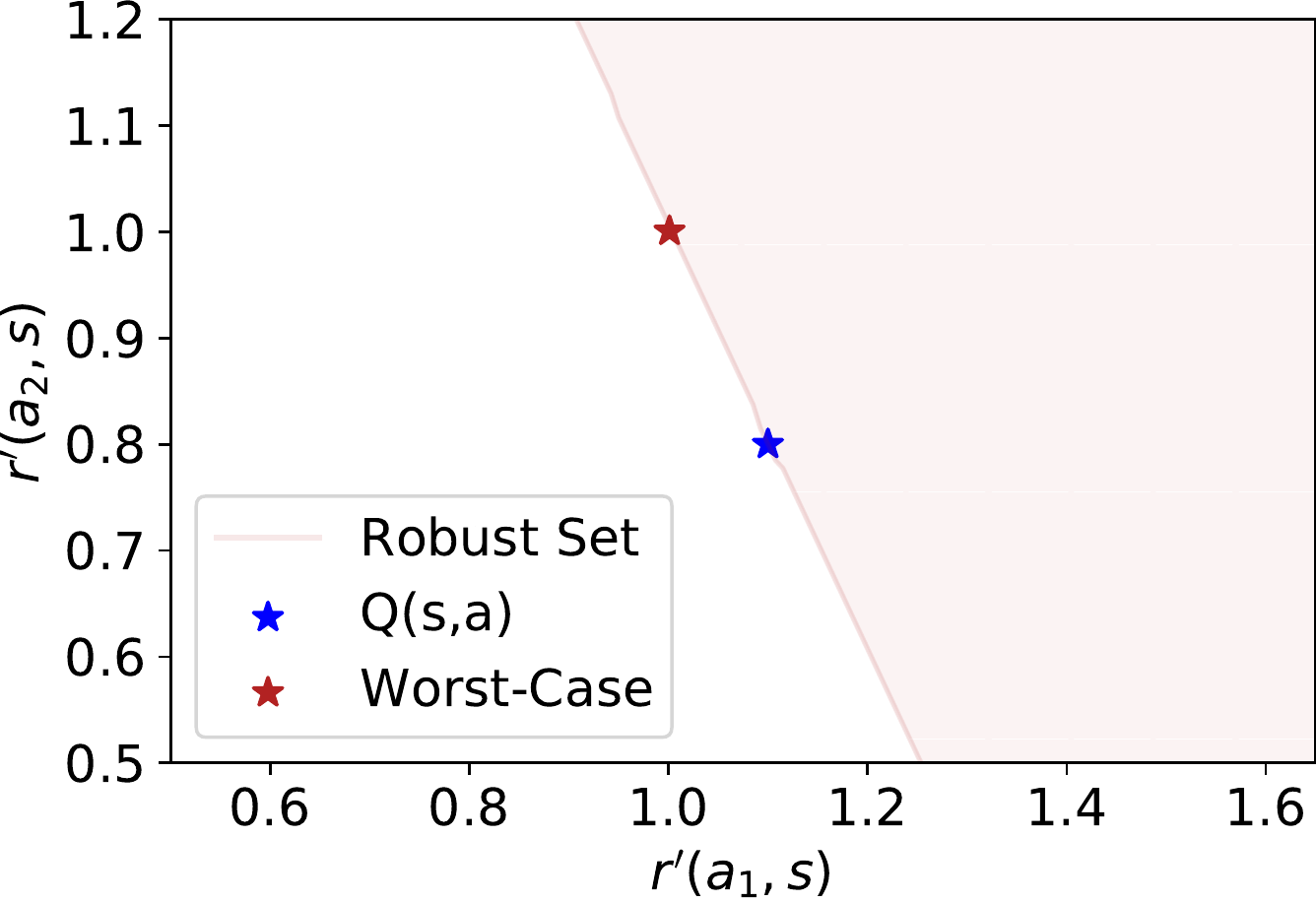}
\caption{ \centering \small $\alpha = -1, \beta = 0.1$}\end{subfigure}&
\begin{subfigure}{\mainfeasiblefour\textwidth}\includegraphics[width=\textwidth, trim= 0 0 0 0, clip]
{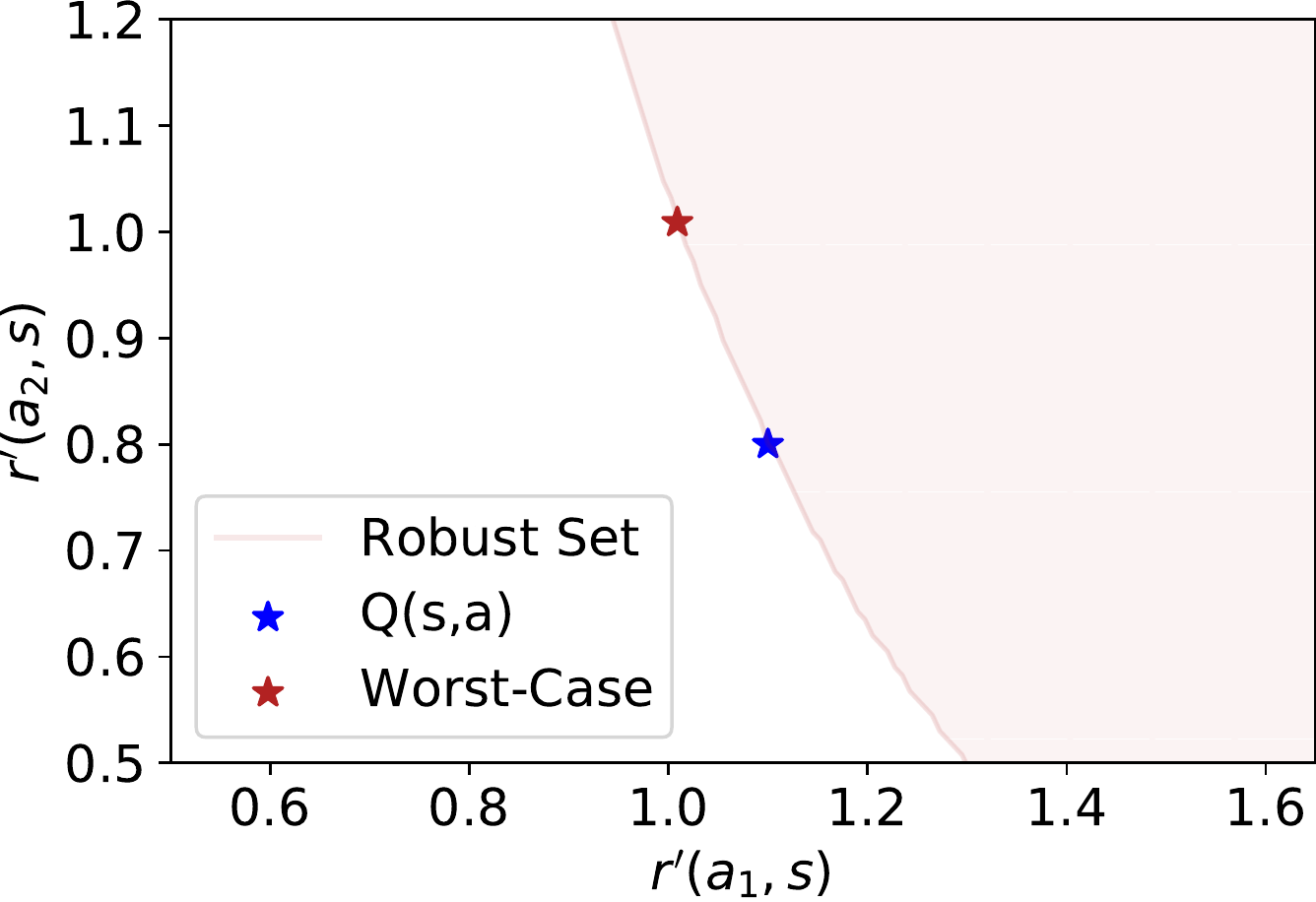}
\caption{ \centering \small $\alpha = -1, \beta = 1$}\end{subfigure}  &
\begin{subfigure}{\mainfeasiblefour\textwidth}\includegraphics[width=\textwidth, trim= 0 0 0 0, clip]
{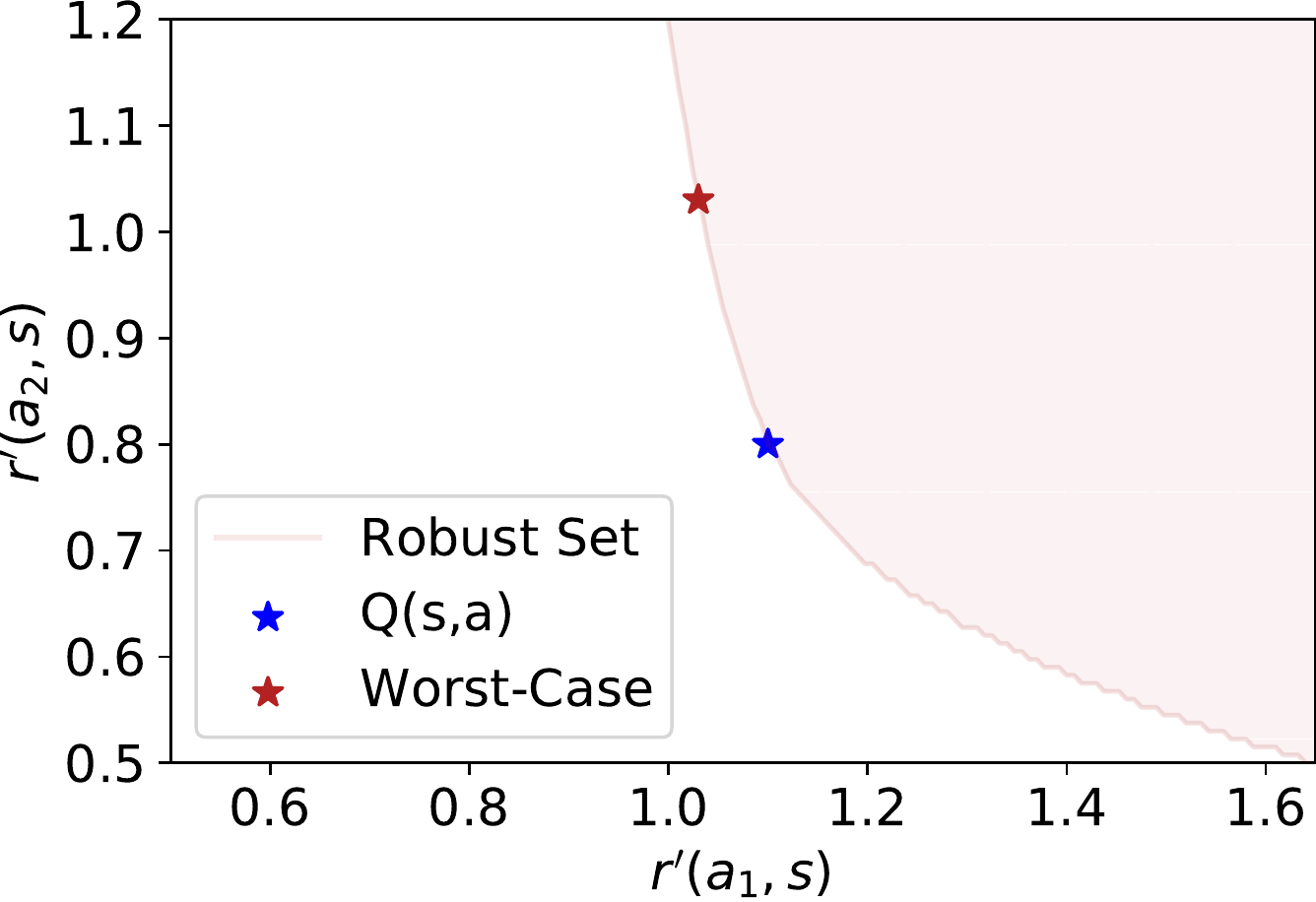}
\caption{ \centering \small $\alpha = -1, \beta = 5$}\end{subfigure} &
\begin{subfigure}{\mainfeasiblefour\textwidth}\includegraphics[width=\textwidth, trim= 0 0 0 0, clip]
{figs/feasible/alpha-1_no_lambda/alpha_neg1_beta10nonuniform.pdf}
\caption{ \centering \small $\alpha = -1, \beta = 10$}\end{subfigure} \\
 \begin{subfigure}{\sidefeasible\textwidth}  \includegraphics[width=0.9\textwidth]
 {figs/feasible/nonuniform_prior.png} \end{subfigure} & 
\begin{subfigure}{\mainfeasiblefour\textwidth}\includegraphics[width=\textwidth, trim= 0 0 0 0, clip]
{figs/feasible/kl/beta01alpha1nonuniform.pdf}
\caption{ \centering \small $D_{KL}, \beta = 0.1$} \end{subfigure} & 
\begin{subfigure}{\mainfeasiblefour\textwidth}\includegraphics[width=\textwidth, trim= 0 0 0 0, clip]{figs/feasible/kl/beta1alpha1nonuniform.pdf}
\caption{ \centering  $D_{KL}, \beta = 1$}\end{subfigure}& 
\begin{subfigure}{\mainfeasiblefour\textwidth}\includegraphics[width=\textwidth, trim= 0 0 0 0, clip]{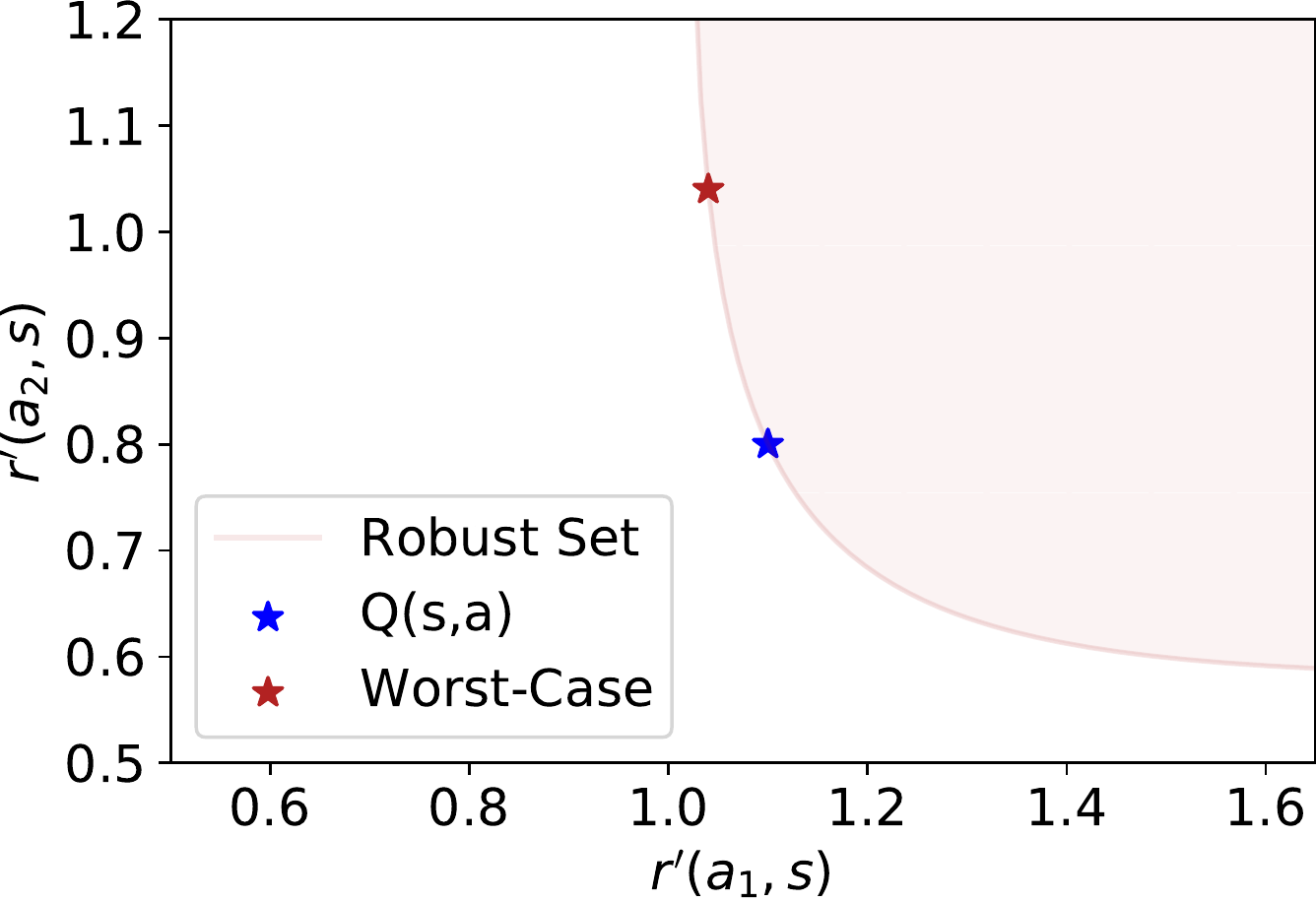}
\caption{ \centering  $D_{KL}, \beta = 5$}\end{subfigure} &
\begin{subfigure}{\mainfeasiblefour\textwidth}\includegraphics[width=\textwidth, trim= 5 0 5 0, clip]
{figs/feasible/kl/beta10alpha1nonuniform.pdf}
\caption{ \centering \small $D_{KL}, \beta = 10$}\end{subfigure} \\
 \begin{subfigure}{\sidefeasible\textwidth} \phantom{ \includegraphics[width=\textwidth]
 {figs/feasible/nonuniform_prior.png}} \end{subfigure} & \begin{subfigure}{\mainfeasiblefour\textwidth}\includegraphics[width=\textwidth, trim= 0 0 0 0, clip]
{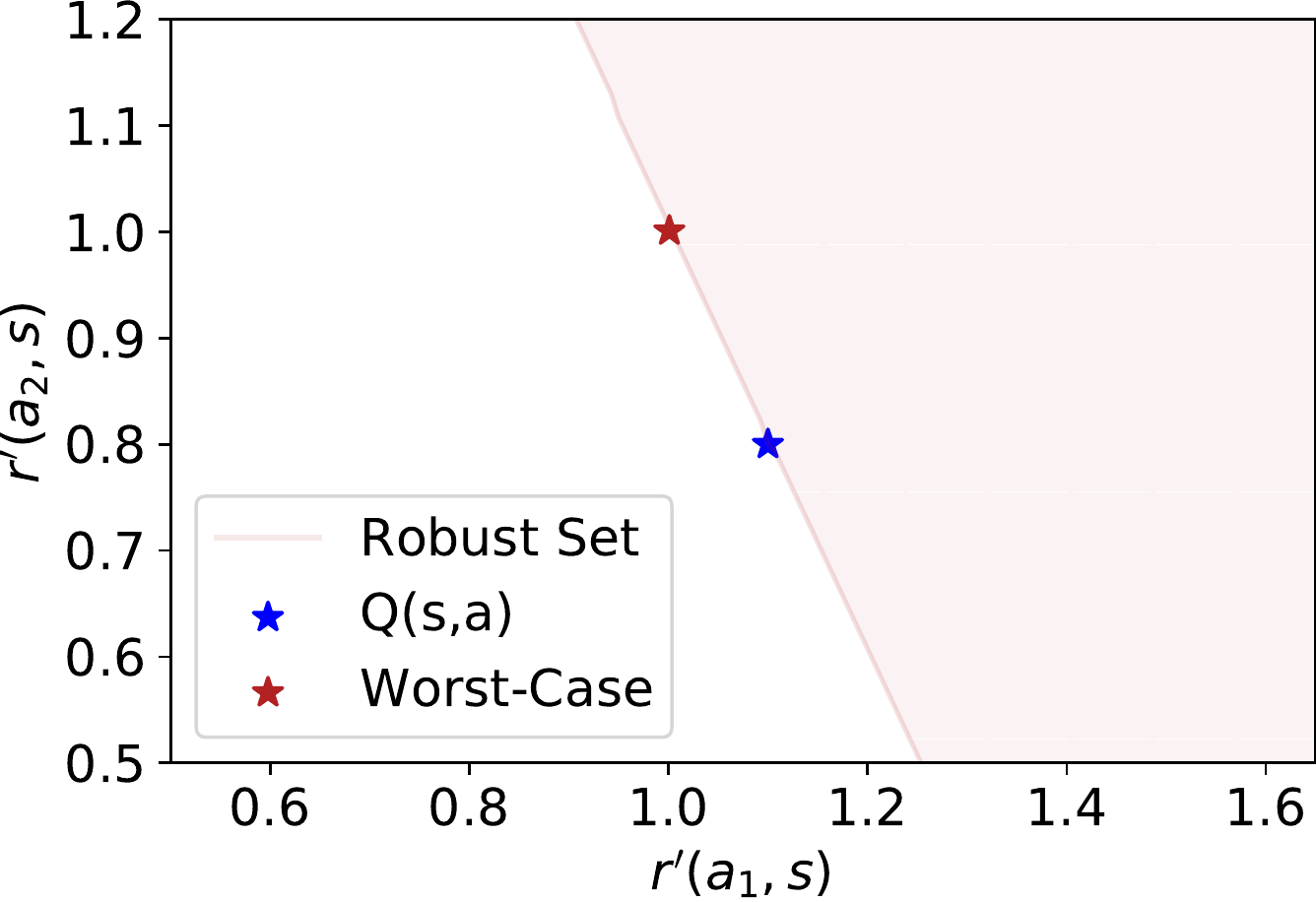}
\caption{ \centering \small $\alpha = 2, \beta = 0.1$}\end{subfigure}&
\begin{subfigure}{\mainfeasiblefour\textwidth}\includegraphics[width=\textwidth, trim= 0 0 0 0, clip]
{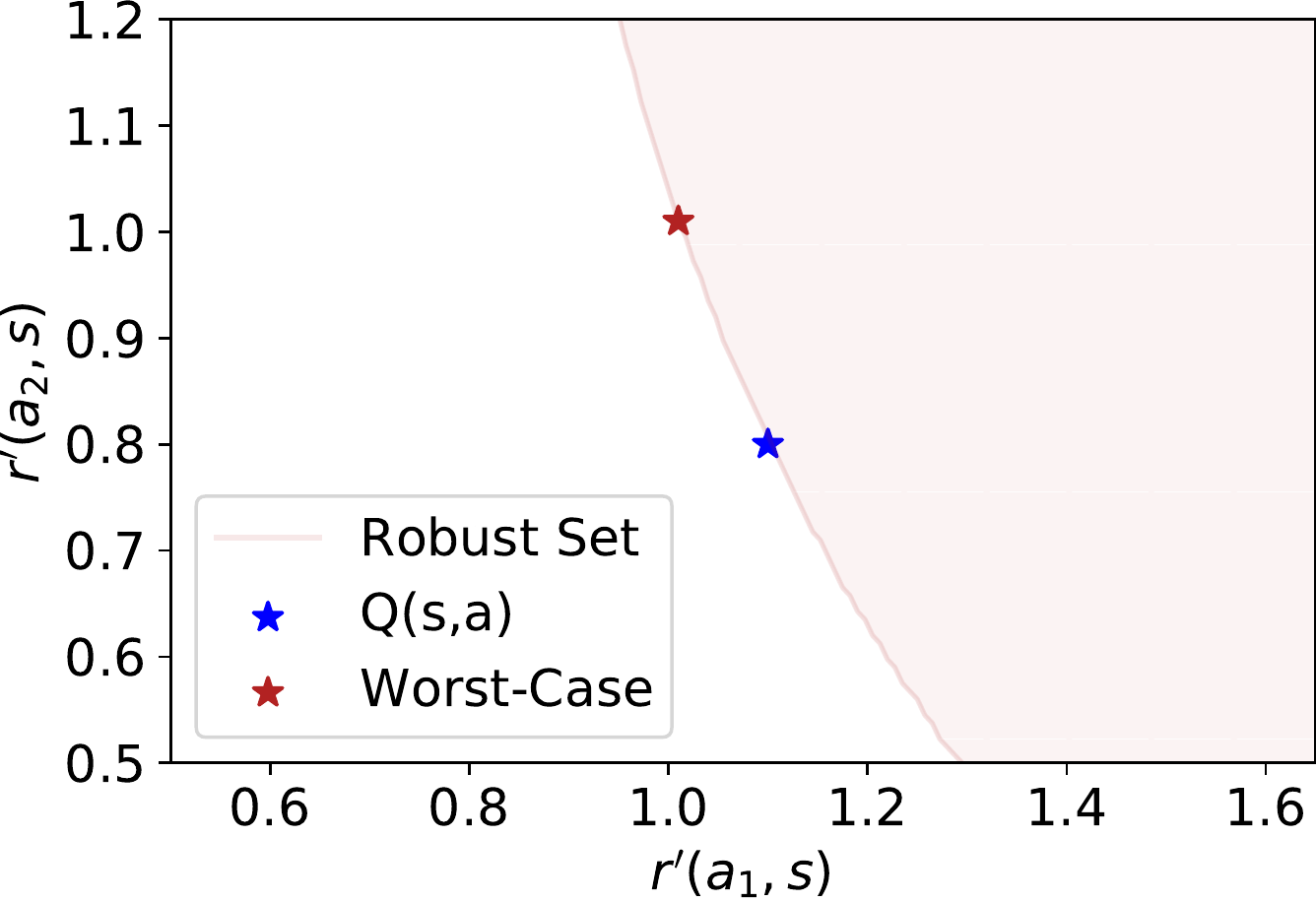}
\caption{ \centering \small $\alpha = 2, \beta = 1$}\end{subfigure} &
\begin{subfigure}{\mainfeasiblefour\textwidth}\includegraphics[width=\textwidth, trim= 0 0 0 0, clip]
{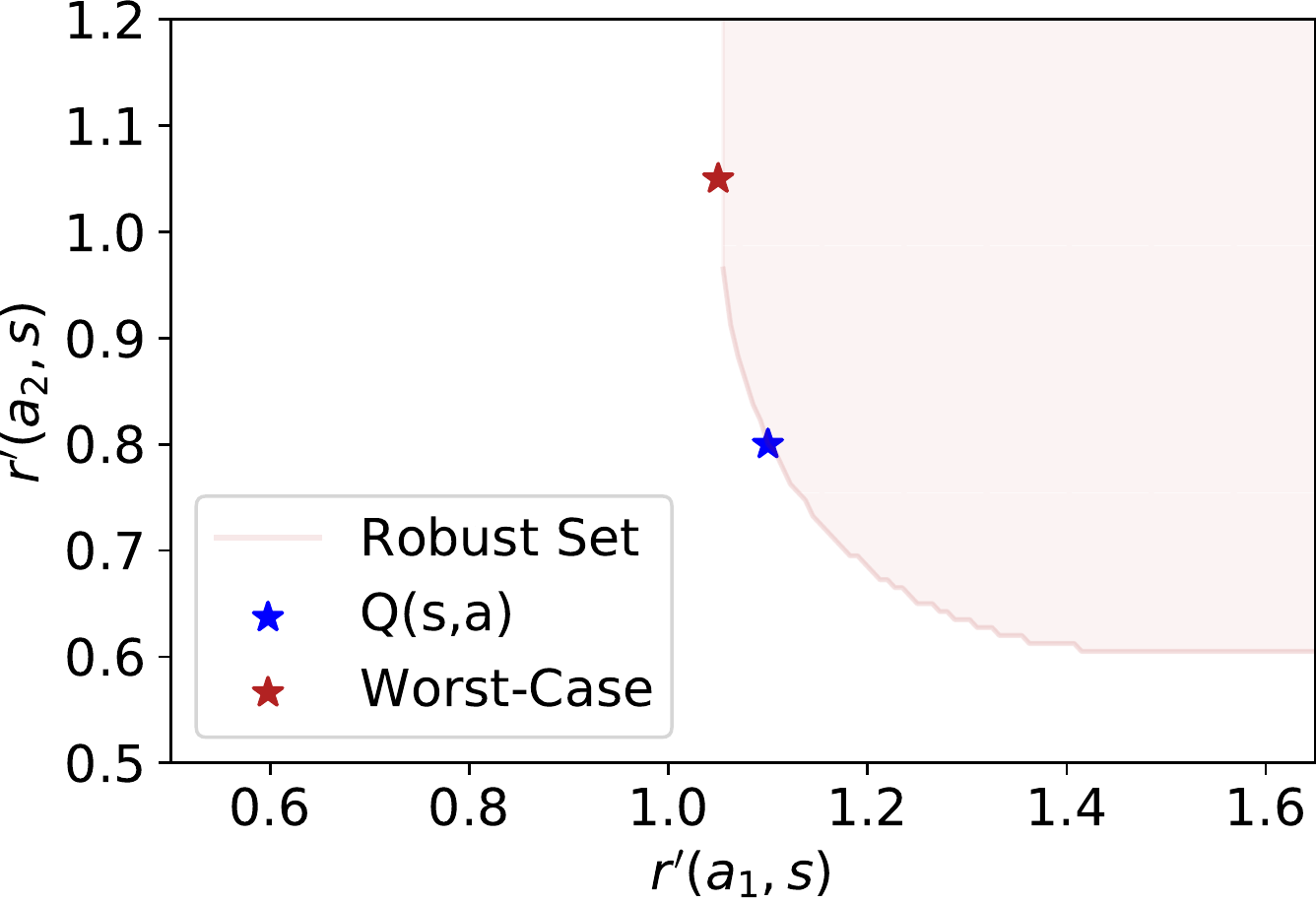}
\caption{ \centering \small $\alpha = 2, \beta = 5$}\end{subfigure} &
\begin{subfigure}{\mainfeasiblefour\textwidth}\includegraphics[width=\textwidth, trim= 0 0 0 0, clip]
{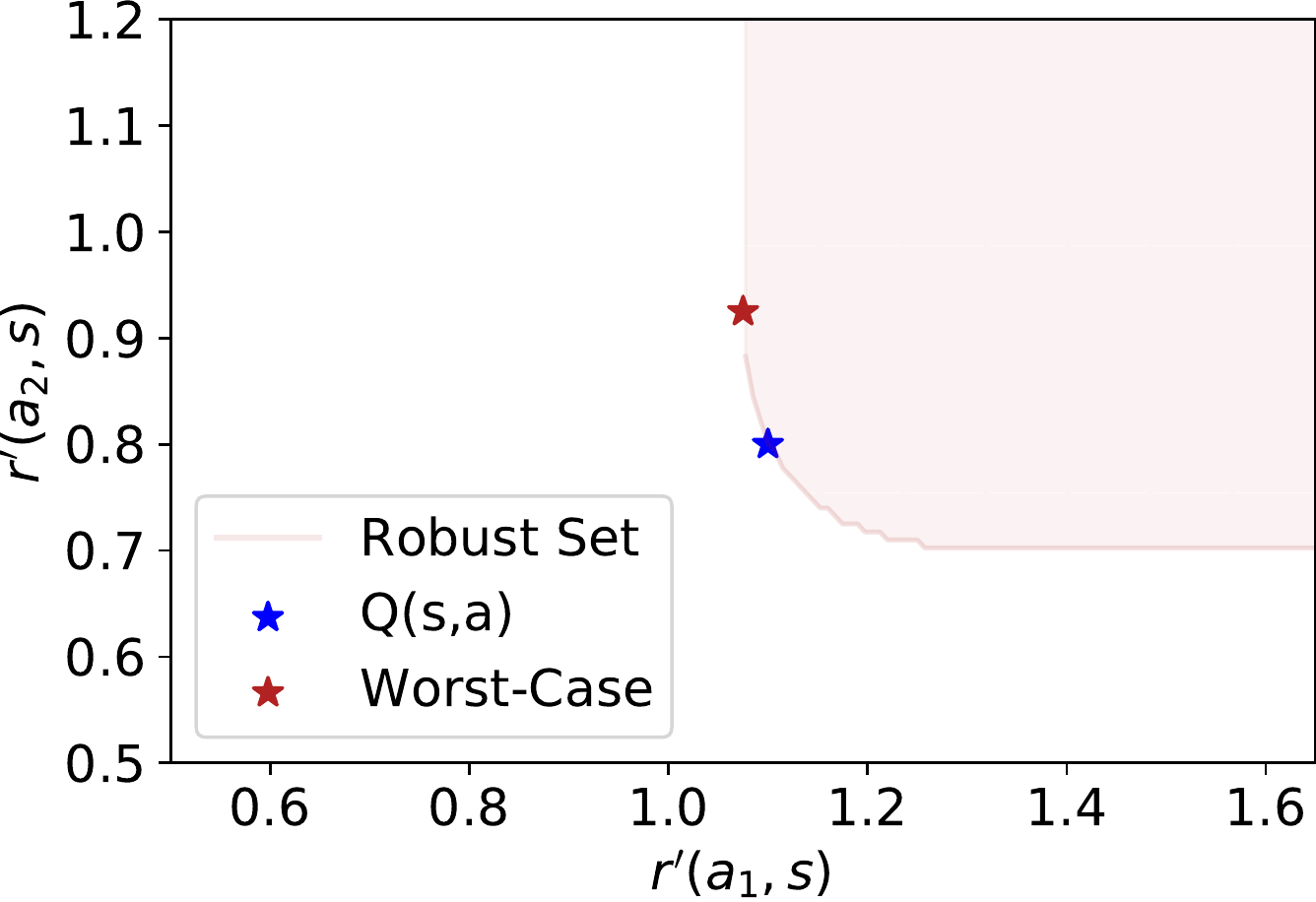}
\caption{ \centering \small $\alpha = 2, \beta = 10$}\end{subfigure} \\
\begin{subfigure}{\sidefeasible\textwidth} \phantom{ \includegraphics[width=\textwidth]
 {figs/feasible/nonuniform_prior.png}} \end{subfigure} & \begin{subfigure}{\mainfeasiblefour\textwidth}\includegraphics[width=\textwidth, trim= 0 0 0 0, clip]
{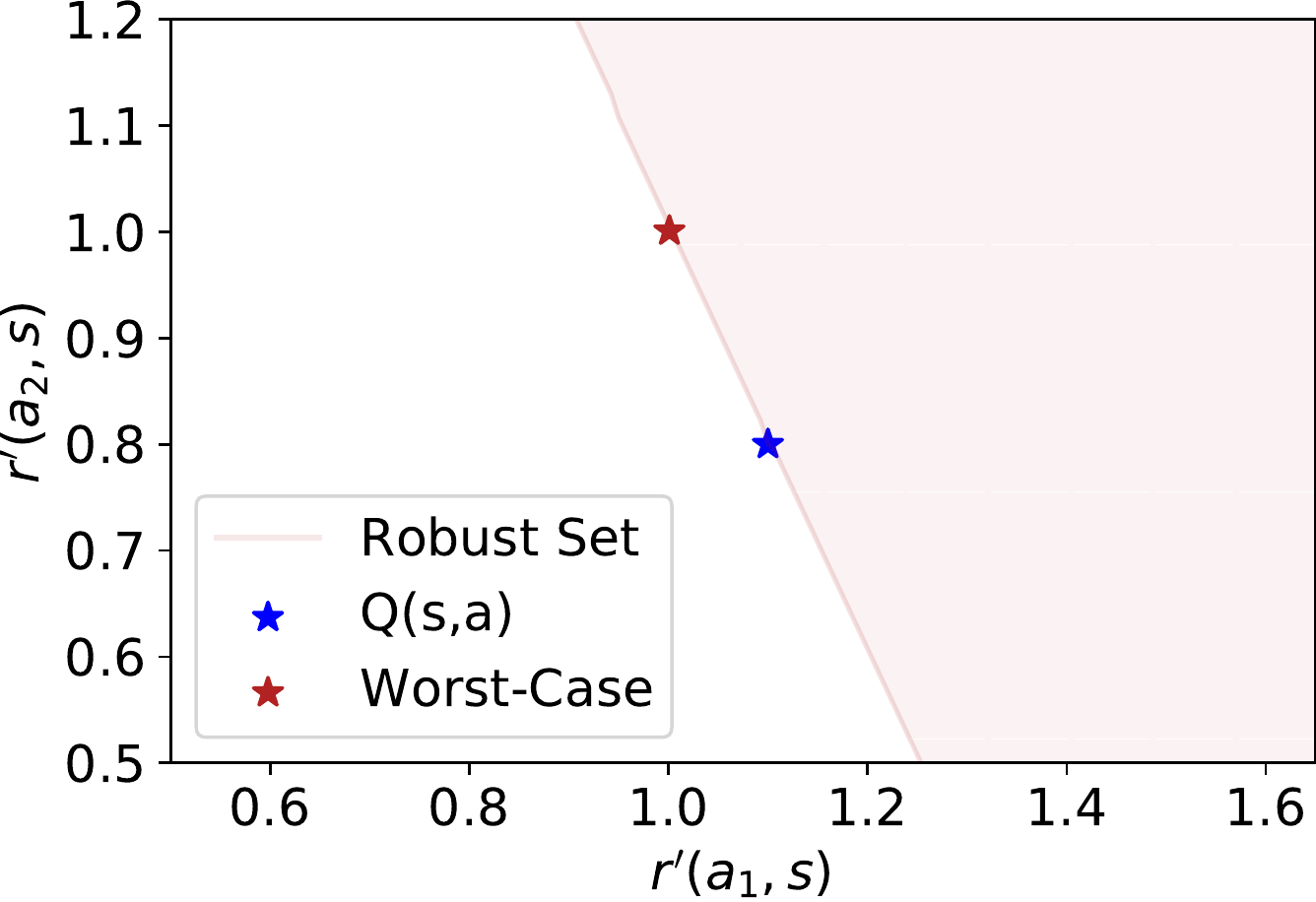}
\caption{ \centering \small $\alpha = 3, \beta = 0.1$}\end{subfigure}&
\begin{subfigure}{\mainfeasiblefour\textwidth}\includegraphics[width=\textwidth, trim= 0 0 0 0, clip]
{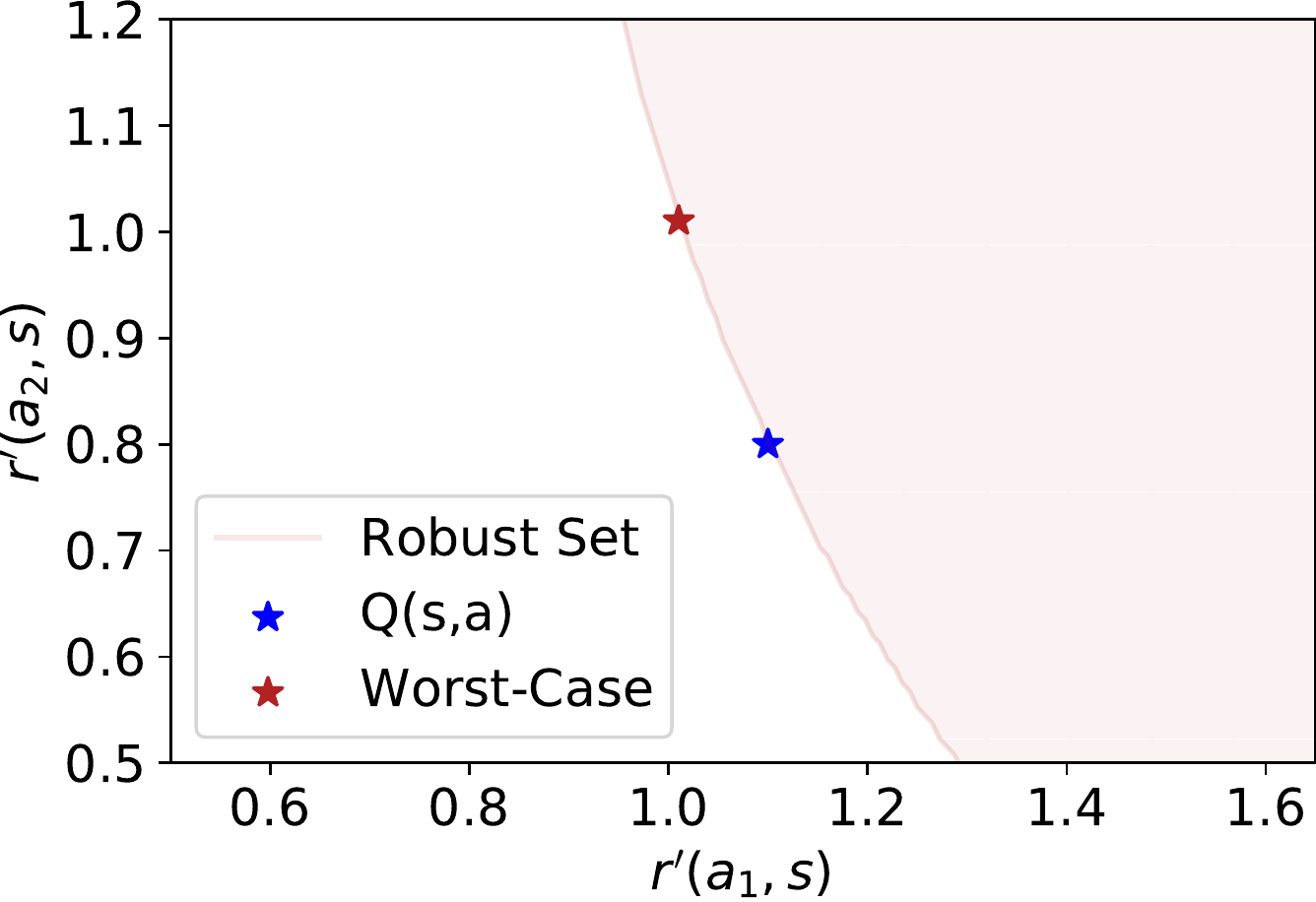}
\caption{ \centering \small $\alpha = 3, \beta = 1$}\end{subfigure}&
\begin{subfigure}{\mainfeasiblefour\textwidth}\includegraphics[width=\textwidth, trim= 0 0 0 0, clip]
{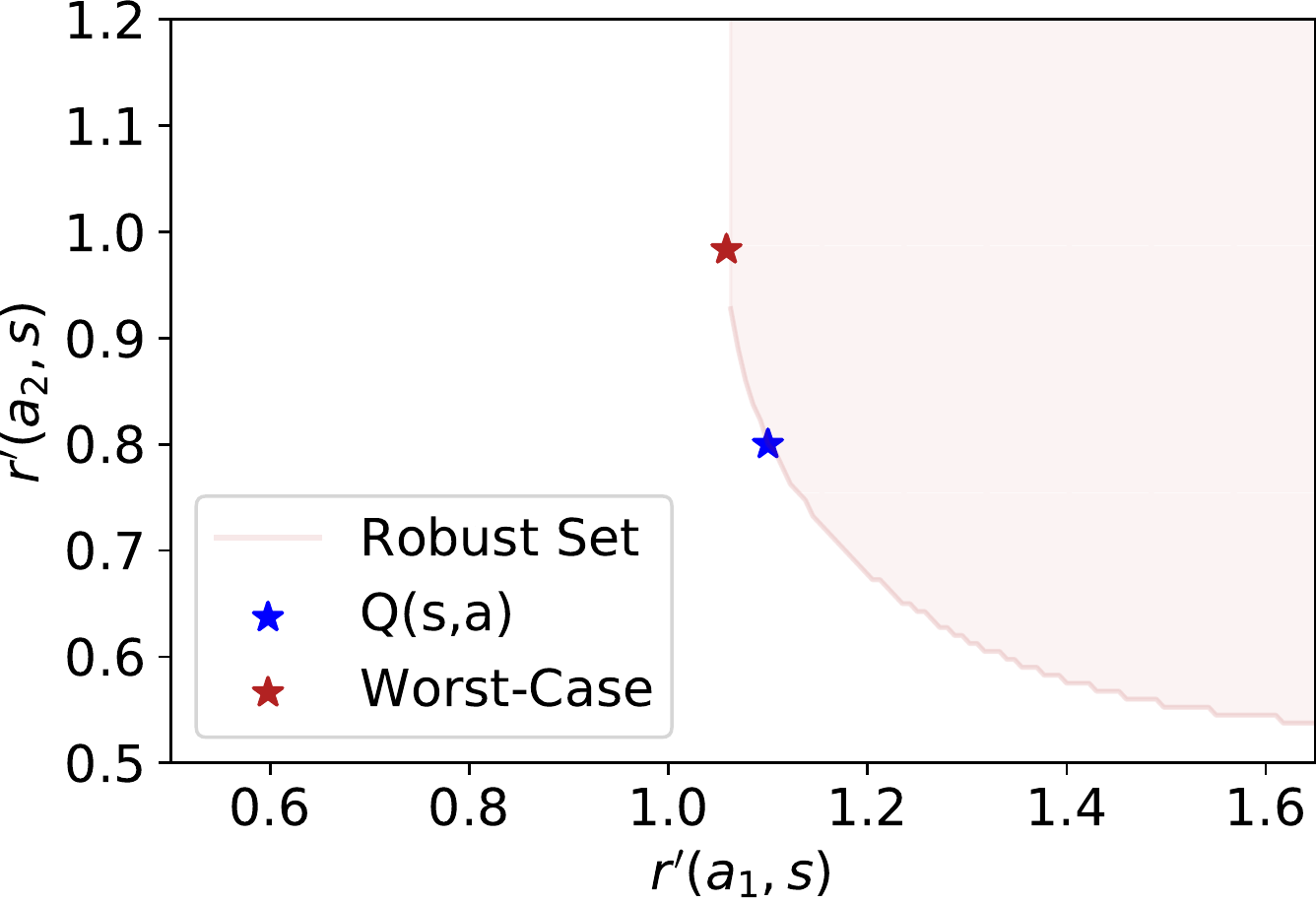}
\caption{ \centering \small $\alpha = 3, \beta = 5$}\end{subfigure}&
\begin{subfigure}{\mainfeasiblefour\textwidth}\includegraphics[width=\textwidth, trim= 0 0 0 0, clip]
{figs/feasible/alpha3_no_lambda/alpha3_beta10nonuniform.pdf}
\caption{\centering \small $\alpha = 3, \beta = 10$}\end{subfigure}
\\ 
\end{array}
\] 
\end{center}
\vspace*{-.45cm}
\caption{Reference distribution $\pi_0 = (\frac{2}{3}, \frac{1}{3})$.  Feasible Set (red region) of perturbed rewards available to the adversary, for \textsc{kl} ($\alpha=1$) and $\alpha$-divergence ($\alpha=\{-1, 2, 3\}$) regularization, various $\beta$, and fixed $Q_*(a,s) = r(a,s)$ values (blue star).  We consider the optimal $\pi_*(a|s)$ with regularization parameters $\alpha, \beta, \pi_0$ and the given $Q$-values.   Red star indicates worst-case perturbed reward $\mrpiopt = r - \prpiopt$ for optimal policy.
}\label{fig:feasible_set_nonuniform_app}
\end{figure*}